\documentclass[11pt]{academic_template}
\usepackage[toc,page,header]{appendix}

\usepackage{marvosym}

\usepackage{bbding}

\usepackage{amssymb}
\usepackage{amsmath}
\usepackage{booktabs}
\usepackage{multirow}
\usepackage{graphicx}
\usepackage{xcolor}    
\usepackage{colortbl}  
\usepackage{makecell}
\usepackage{algorithm}
\usepackage{algpseudocode}
\usepackage[utf8]{inputenc}
\usepackage[T1]{fontenc}
\usepackage{booktabs}      
\usepackage{array}         
\usepackage{colortbl}      
\usepackage{xcolor}        
\usepackage{caption}       
\usepackage{geometry}      
\usepackage{tabularx,booktabs,colortbl}
\usepackage{graphicx}

\usepackage{listings}
\usepackage{hyperref}
\usepackage{url}
\usepackage{booktabs}
\usepackage{wrapfig}

\usepackage{colortbl}
\usepackage{caption}
\usepackage{subcaption}

\usepackage{amsthm}
\newtheorem{theorem}{Theorem}

\newtheorem{remark}{Remark}
\newtheorem{definition}{Definition}

 \usepackage{amsmath,amssymb,amsthm,mathtools,bm}
 \usepackage{bbm}

 \newtheorem{lemma}{Lemma}
 \newtheorem{corollary}{Corollary}
 \theoremstyle{definition}
\newcommand{\E}{\mathbb{E}}
 \newcommand{\Prob}{\mathbb{P}}
 \newcommand{\R}{\mathbb{R}}
 \newcommand{\cF}{\mathcal{F}}
 \newcommand{\cH}{\mathcal{H}}
 \newcommand{\cW}{\mathcal{W}}
 \newcommand{\norm}[1]{\left\lVert #1 \right\rVert}
  \newcommand{\abs}[1]{\left\lvert #1 \right\rvert}
 \newcommand{\Var}{\operatorname{Var}}

\setlogoheight{8mm}
\setlogospacing{0mm}
\setsjtublue

\settitlerulethickness{3pt}

\abstractboxon
\setabstractframecolor{gray}
\setabstractbgcolor{gray!10}
\setlogotolineshift{0mm}

\renewcommand{\abstractinfont}{\fontsize{9}{10.8}\selectfont\linespread{1.0}\selectfont}

\settoprulethickness{2.5pt}
\setbottomrulethickness{1.5pt}

\usepackage{amsmath,amsfonts,bm}

\def\eqref#1{equation~\ref{#1}}

\def\1{\bm{1}}

\DeclareMathAlphabet{\mathsfit}{\encodingdefault}{\sfdefault}{m}{sl}
\SetMathAlphabet{\mathsfit}{bold}{\encodingdefault}{\sfdefault}{bx}{n}

\usepackage{dblfloatfix}
\usepackage{tabularx}

\usepackage{multicol}
\usepackage{multirow}
\usepackage[table]{xcolor}
\usepackage{array}

\newcolumntype{C}[1]{>{\centering\arraybackslash}p{#1}}
\newcolumntype{M}[1]{>{\centering\arraybackslash}m{#1}}

\usepackage{natbib}
\setcitestyle{numbers}
\setcitestyle{square}

\usepackage{graphicx}

\usepackage{threeparttable}

\title{Kairos: A Regret-Aware Native World-Action Model Stack for Physical AI}

\author[]{\parbox{\textwidth}{\centering\textbf{Learning, Maintaining, and Deploying Control-Sufficient World States}\\[0.5cm]Kairos Team}}

\abstract{\begin{spacing}{0.9}
World models are becoming a central substrate for Physical AI, yet their current development is differentiating across representational, generative, interactive, and unified understanding--generation--prediction world models. While these paradigms have advanced abstract world reasoning, high-fidelity visual generation, interactive simulation, and embodied world-action modeling, none alone provides the deployable, long-horizon, \emph{control-sufficient} state required by embodied agents. We introduce \textbf{Kairos}, a regret-aware native world-action model stack for Physical AI. Kairos is motivated by the view that a physical world model should not aim to fully simulate all future pixels, but should learn and maintain the information most relevant to embodiment control: object state, spatial relations, contact conditions, task progress, action consequences, failure boundaries, and deployment uncertainty.

Kairos establishes three model-side prerequisites toward this goal. First, it \textbf{learns} control-relevant information through a \textbf{Cross-Embodiment Data Curriculum}, which organizes open-world videos, human behavioral data, and robot interactions into an intervention-strength progression from passive physical observation to intentional behavior and embodied action grounding. Second, it \textbf{maintains} control-sufficient states through a unified \textbf{understanding, generation, and prediction architecture} equipped with \textbf{Hybrid Linear Temporal Attention}, where local, mid-range, and global temporal pathways support multi-timescale state maintenance under efficient inference. Third, it \textbf{deploys} these states through a \textbf{Deployment-Aware System Co-Design}, treating latency, memory footprint, and hardware compatibility as first-order constraints for future observation, action, and feedback loops.

Experiments on embodied world-model benchmarks, world-action benchmarks, long-horizon generation, and inference-efficiency evaluation show that Kairos achieves superior performance while offering a favorable efficiency to capability trade-off. These results provide \emph{proxy evidence} for regret-relevant capabilities, including physical plausibility, instruction grounding, joint world-action prediction, long-horizon consistency, and deployment readiness. Direct validation of real-robot closed-loop regret reduction, including rollout correlation, failure prediction, safety filtering, recovery learning, and measurable policy improvement from imagined experience, remains an important direction for future Kairos development.
\end{spacing}}

\date{\today}
\checkdata[Code]{\url{https://github.com/kairos-agi/kairos}}
\checkdata[Hugging Face]{\url{https://huggingface.co/kairos-agi}}
\checkdata[ModelScope]{\url{https://modelscope.cn/collections/kairos-team/kairos30}}

\begin{document}

\maketitle

\newpage

\begin{spacing}{0.9}
\tableofcontents
\end{spacing}

\newpage

\section{Introduction}

World models~\cite{WM_Survey} are rapidly emerging as a central substrate for Physical AI. They are no longer expected to be merely generative models that render visually plausible futures; they are increasingly expected to support physical understanding, temporal prediction, embodied interaction, action-conditioned reasoning, deployment-time decision support, and eventually continual adaptation in real environments. In this broader view, a useful world model is an internal predictive system that acquires, organizes, maintains, and operationalizes knowledge about how the world evolves under time, uncertainty, and action. Recent advances across video generation~\cite{brooks2024video,nvidia2025worldsimulationvideofoundation,kuaishou2024kling,luma2024dream,runway2024gen3,bytedance2025seedance,kondratyuk2023videopoet,teng2025magi,ho2020denoising,ho2022video,ho2022imagen,singer2022make,zhou2022magicvideo,wang2025lavie,kong2024hunyuanvideo,hacohen2024ltx,bruce2024genie,yang2023unisim}, latent dynamics modeling~\cite{ha2018world,hafner2019planet,hafner2019dream,hafner2020mastering,hafner2023mastering,assran2025vjepa2,murlabadia2026vjepa2_1,baldassarre2025featuresdinofoundationvideo}, model-based reinforcement learning~\cite{hafner2019dream,hansen2022temporal,hafner2025trainingagentsinsidescalable,kaplan2020scaling}, embodied AI, and interactive simulation~\cite{worldlabs_marble,chen2025teleworlddynamicmultimodalsynthesis,deepmind2025genie3,hyworld2025,lingbot-world,bi2025motus,cen2025rynnvla,cen2025worldvla,cheang2024gr,deng2026dexworldmodel,intelligence2026pi,shen2026videovla,yuan2026adaworldpolicy,gong2026acebrain0spatialintelligenceshared,kim2026cosmos,zheng2025flare,du2024video,bharadhwaj2024track2act,zhi20253dflowaction,ko2024learning} all point toward the same direction: world models are moving from offline demonstrations toward operational infrastructure for embodied and physical intelligence.

To understand this shift, recent advancements in world models can be broadly viewed through four industry-oriented directions. The first direction focuses on \textbf{generative world models}, or generative pixel-level rendering, which primarily encompasses video generation and the synthesis of high-fidelity visual futures. Models in this category aim to produce temporally coherent continuations of the world directly in pixel space. A prominent example is NVIDIA's Cosmos~\cite{nvidia2025worldsimulationvideofoundation}, which leverages generative video foundation models as digital twins and essential infrastructure for Physical AI.

The second direction shifts the focus toward \textbf{representational world models}, emphasizing predictive latent embedding and abstract world reasoning. Rather than rendering pixels, this approach explicitly frames world models as systems that learn physically meaningful predictive structures entirely within abstract representation spaces. Meta's JEPA family (e.g., V-JEPA 2~\cite{assran2025vjepa2}, V-JEPA 2.1~\cite{murlabadia2026vjepa2_1}, and DINO-world~\cite{baldassarre2025featuresdinofoundationvideo}) exemplifies this trajectory. By internally anticipating outcomes in a latent form, these models inherently support downstream tasks such as physical understanding, zero-shot planning, and robot control. The core premise here is that a world model's utility for decision-making relies on its capacity to anticipate the future abstractly, bypassing the immense computational overhead of pixel-level rendering. Recent evidence that semantic latent spaces can be more useful for robotic policy evaluation and planning than reconstruction-aligned latents further supports this emphasis on action-relevant structure over pixel fidelity~\cite{nilaksh2026reconstruction}.

The third direction advances \textbf{interactive world models}, emphasizing the creation of persistent simulations and interactive arenas. This encompasses both static spatial worlds and dynamic interactive environments. For instance, models focusing on static spatial intelligence, such as World Labs' Marble~\cite{worldlabs_marble} and TeleWorld~\cite{chen2025teleworlddynamicmultimodalsynthesis}, excel at building explorable 3D worlds that agents can perceive and navigate, emphasizing geometric consistency and "worldness." Extending into dynamic interaction, environment generators like DeepMind's Genie 3~\cite{deepmind2025genie3}, HY-World 1.5~\cite{hyworld2025}, and LingBot-World~\cite{lingbot-world} instantiate fully manipulable worlds from simple prompts. Furthermore, frameworks like Dreamer 4~\cite{hafner2025trainingagentsinsidescalable} utilize these models as internal simulators where agents can recursively optimize long-horizon behaviors through imagination. In this paradigm, models are judged by their capability to act as comprehensive engines for spatial exploration, interaction, and self-evolution.

The fourth direction is emerging around \textbf{unified understanding--generation--prediction world-action models}. This direction does not treat understanding, generation, prediction, and action as separate downstream modules. Instead, it aims to unify semantic and physical understanding, multimodal future imagination, future-state prediction, embodied action modeling, and deployment-time decision support within a shared world-action substrate. Recent systems such as Cosmos~3~\cite{aditi2026cosmos3omnimodalworld} and MotuBrain~\cite{team2026motubrain} reflect this broader movement by integrating multimodal understanding, generation, and world-action modeling within a unified Physical AI backbone. Taken together, these advances show that the field is no longer converging on a single definition of world models as ``video generators.'' For industrial deployment, this fourth direction is the most closely aligned with Physical AI, because its purpose is not merely to abstract the world, reproduce the world, or simulate an interactive scene, but to connect generation, physics, cognition, and action so that an embodied agent can preserve control-relevant information, evaluate action consequences, and run inside real observation--action--feedback loops. Kairos is our systematic, full-stack exploration along this fourth class of embodied world models.

Operationalizing this direction for Physical AI requires a first principle: a world model should not be understood as a full simulator of the world. The real world contains far more information than any robot can observe, compute, store, or act upon. A robot picking up a cup does not need to predict every future pixel of the table texture, the shape of clouds outside the window, or the motion of irrelevant background objects. What matters is whether the robot can preserve the information required for embodied control: the cup position, mass, friction, grasp affordance, contact condition, hand pose, task progress, failure risk, and the consequences of alternative actions. Therefore, the objective of an embodied world model is not to reproduce all future sensory information, but to learn a compact internal state that is sufficient for predicting task-relevant future variables. We refer to this internal representation as a \textbf{control-sufficient state}.

The challenge becomes sharper in robotics. Physical agents operate under partial observability, embodiment-specific constraints, contact-rich dynamics, discontinuous state transitions, limited real-world trial-and-error, and safety-sensitive deployment conditions. In manipulation, a small change in contact, friction, grasp pose, force distribution, or object geometry can determine whether a task succeeds or fails. In navigation, local observations must be integrated into persistent spatial and semantic memory. In long-horizon operation, delayed effects and accumulated errors can determine downstream outcomes. In real deployment, the model must produce useful predictions within the time budget of the control loop. Therefore, the central missing link is not another visual simulator, but a deployable, long-horizon, \textbf{control-sufficient state}.

Formally, given the observation--action history $H_t$, task goal $g$, and candidate future action sequence $a_{t:t+H-1}$, an embodied world model should maintain an internal state $Z_t$ that preserves the information needed to predict task-relevant future states, failure events, task progress, physical cost, and the discrepancy between imagined and real outcomes. State $Z_t$ is not merely a visual representation, but a control-sufficient state: a compressed representation that retains the variables necessary for action selection, risk assessment, failure anticipation, and recovery planning. A model that generates realistic videos but cannot predict failures, distinguish the consequences of different actions, or assess policy risk remains an incomplete embodied world model. Conversely, a model that does not synthesize every pixel but can reliably predict the consequences of different actions and failure boundaries is closer to the type of embodied world model required by Physical AI.

This perspective changes how embodied world models should be evaluated. In conventional generative modeling, success is often measured by visual realism, temporal smoothness and consistency, instruction following, or reconstruction quality. These criteria remain useful, but they are insufficient for Physical AI. A robot does not pay a cost when its predicted pixels are slightly blurry; it pays a cost when a cup slips, a collision occurs, a task fails, a human must intervene, or a safety boundary is crossed. For embodied intelligence, \textbf{regret} is not an abstract learning-theoretic quantity alone; it corresponds to real-world costs such as collision, damage, wasted time, recovery effort, unsafe contact, and failed task completion. In this report, we therefore use regret reduction as the guiding principle for embodied world-model design. Kairos is designed as a \emph{regret-aware} world-action model stack that establishes several model-side prerequisites for future regret-minimizing Physical AI: control-relevant information acquisition, control-sufficient state compression, joint world-action prediction, multi-timescale state maintenance, and deployment-ready inference.

To make this evaluation principle precise, we formalize the value of an internal world state by the excess physical cost it induces for future action selection. The key question is not whether the state reconstructs all observations, but whether it preserves the information needed for low-cost decisions under a task goal. Let $H_t$ denote the observation--action history up to time $t$, $g\in\mathcal{G}$ the task goal, and $Z_t = f(H_t) \in \mathcal{Z}$ a compact state representation derived from $H_t$ via a compression function $f: \mathcal{H} \to \mathcal{Z}$. For a future horizon $H$, let $\mathcal{A}$ denote the action space and let $\mathcal{A}^H$ denote the space of length-$H$ future action sequences. We use
\begin{equation*}
J_H(a_{t:t+H-1}\mid H_t,g)
\end{equation*}
to denote the conditional expected physical cost of executing a candidate future action sequence $a_{t:t+H-1}\in\mathcal{A}^H$ under the current history $H_t$ and goal $g$. This cost captures deployment-level physical and operational consequences, such as task failure, collision, unsafe contact, recovery effort, and downstream costs caused by discrepancies between imagined and real outcomes.

Since $Z_t$ is obtained by compressing $H_t$, it may discard information that is relevant for planning low-cost actions under $J_H$. To quantify this potential loss, we compare planners that act on $Z_t$ with planners that have access to the full history $H_t$. Let
\begin{equation*}
\pi_Z:\mathcal{Z}\times\mathcal{G}\to\mathcal{A}^H
\end{equation*}
denote a horizon-level planner that maps the compressed state and goal to a length-$H$ future action sequence. Similarly, let
\begin{equation*}
\pi_H:\mathcal{H}\times\mathcal{G}\to\mathcal{A}^H
\end{equation*}
denote a planner with access to the full history and goal. Thus, the two planners differ only in the information available for planning.

We define the representation-induced regret of $f$ as
\begin{equation*}
\operatorname{Reg}_H(f;g)
=
\inf_{\pi_Z}
\mathbb{E}\!\left[
J_H(\pi_Z(Z_t,g)\mid H_t,g)
\right]
-
\inf_{\pi_H}
\mathbb{E}\!\left[
J_H(\pi_H(H_t,g)\mid H_t,g)
\right],
\end{equation*}
where the outer expectation is taken over the distribution of histories $H_t$. The first term is the best achievable expected physical cost when planning from the compressed state $Z_t$, while the second term is the best achievable expected physical cost when planning from the full history $H_t$. Because $Z_t$ is derived from $H_t$, any planner based on $Z_t$ can be emulated by a planner based on $H_t$. Therefore, under the same admissible planner class, this regret is nonnegative and measures the excess physical cost induced by compressing $H_t$ into $Z_t$.

A regret-aware world-action model should therefore learn a compact state representation that makes $\operatorname{Reg}_H(f;g)$ small. Such a state should discard irrelevant visual detail while retaining information about action consequences, failure boundaries, safety risks, recovery costs, and imagined--real gaps.

Under the above formulation, an embodied world model for Physical AI must satisfy several requirements. The first is \emph{regret-aware information compression}. Its value lies not in maximizing the number of future pixels it can generate, but in preserving the information that reduces $\operatorname{Reg}_H(f;g)$ per unit of internal state, computation, latency, memory, and risk. A large model that generates visually impressive rollouts but cannot run before action execution may be less useful for robotics than a smaller model that maintains the right control variables with low latency. For a physical agent, information is valuable when it reduces uncertainty about action consequences, failure boundaries, contact dynamics, recovery strategies, and safety risks. This principle explains why world modeling, representation learning, data selection, temporal memory, and system efficiency should not be treated as separate engineering concerns: they are all part of the same question---what information should be preserved so that future costly mistakes can be reduced?

A second requirement is \emph{counterfactual closure}. A passive video model can answer the question: what is likely to happen next? A robot needs to answer a more difficult question: what will happen if I act now, and how would the future change if I chose a different action? In terms of $J_H$, counterfactual closure is needed because the model must compare the physical costs of alternative action sequences from the same state. A world model that can only continue the natural evolution of a training video remains a spectator of the world. An embodied world model must support multiple action branches from the same internal state. Given the same $Z_t$, it should represent how different candidate actions lead to different future states, different risks, and different task outcomes. In Physical AI, future actions should not be treated as an external policy output attached after world modeling; they should be part of the modeled physical evolution of the world under embodied intervention.

A third requirement is \emph{interventional generalization}. Classical supervised learning typically assumes that training and testing data are drawn from the same distribution. Robotic deployment violates this assumption. Once a robot acts, it changes the future data distribution. The agent does not merely observe the world; it intervenes in it. This creates an interventional distribution shift: training data may contain observation--action correlations, but deployment requires action--outcome causation. A model trained primarily on passive videos may learn broad visual and physical regularities but fail to predict how robot actions change the environment. A model trained only on narrow robot data may acquire action grounding but lack general physical and semantic priors. An embodied world model must therefore learn across a hierarchy of intervention strengths, moving from passive observation toward intentional behavior and finally embodied action grounding.

A fourth requirement is \emph{multi-timescale control-state maintenance}. Long-horizon world modeling should not be reduced to longer context modeling. For embodied agents, the key is not to store all historical tokens, but to maintain the state variables that remain relevant for future control. These variables naturally live on different timescales. Millisecond-to-second dynamics include contact, slip, collision, hand--eye correction, and short-term motion continuity. Second-to-minute dynamics include subtask status, object locations, tool use, and local interaction history. Minute-to-hour dynamics include task plans, environmental changes, user preferences, and accumulated scene context. Hour-to-day dynamics include repeated failure patterns, scene regularities, and individualized experience. A single memory mechanism is unlikely to serve all of these requirements. A deployable world model needs a structured temporal architecture that allocates fast feedback, mid-range event continuity, and long-range causal memory to different but interacting pathways.

A fifth requirement is \emph{control information density}. Data value in Physical AI should not be measured only by raw scale. A short segment containing a failure, recovery, contact transition, or boundary-case behavior may be more valuable than hours of ordinary successful video, because it reduces uncertainty about the variables that matter for control. Under the regret formulation, high-value data are those that reduce uncertainty about the variables that determine future physical cost, especially failures, contact transitions, recoveries, and boundary cases. Failure data reveals when and why the policy breaks. Contact-rich data reveals friction, force, deformation, slip, and grasp stability. Recovery data reveals how the system can return from an error state. Boundary-case data reveals the margin between success and failure. Ordinary successful trajectories remain useful, but they often contain less information about failure boundaries. Ordinary open-world videos provide broad physical priors, but usually lack action grounding. Thus, data engineering for Physical AI should not only increase the volume of training data; it should increase the density of control-relevant information.

These five requirements expose a set of \emph{coupled bottlenecks} in current world models:
\begin{enumerate}
    \item \emph{Fragmented interventional knowledge.} World knowledge is scattered across heterogeneous experience sources. Open-world videos provide broad physical priors, human behavior provides intentional task structure, and robot data provides action grounding, but these sources are often mixed without a principled progression.

    \item \emph{Lack of control-sufficient state preservation.} Existing representations often preserve visual detail while discarding contact, task progress, or failure risk---the very variables that matter for embodied control.

    \item \emph{Insufficient counterfactual action closure.} Many models predict plausible futures but cannot reliably compare alternative action outcomes from the same state.

    \item \emph{Fragile multi-timescale state maintenance.} Local visual continuation does not guarantee persistent state maintenance over long horizons.

    \item \emph{Offline capability without deployment-ready inference.} Models that achieve strong offline metrics often miss latency, memory, or hardware constraints required by real systems, leaving them unable to participate in observation--action--feedback loops.
\end{enumerate}
Addressing these bottlenecks separately risks producing systems that are strong along one dimension but fundamentally incomplete as substrates for physical intelligence. Figure~\ref{fig:motivation} visualizes this gap-to-design logic, showing how Kairos connects existing world-model capabilities to the requirements of Physical AI through a regret-aware world-action stack.

\begin{figure}[t]
    \centering
    \includegraphics[width=1\linewidth]{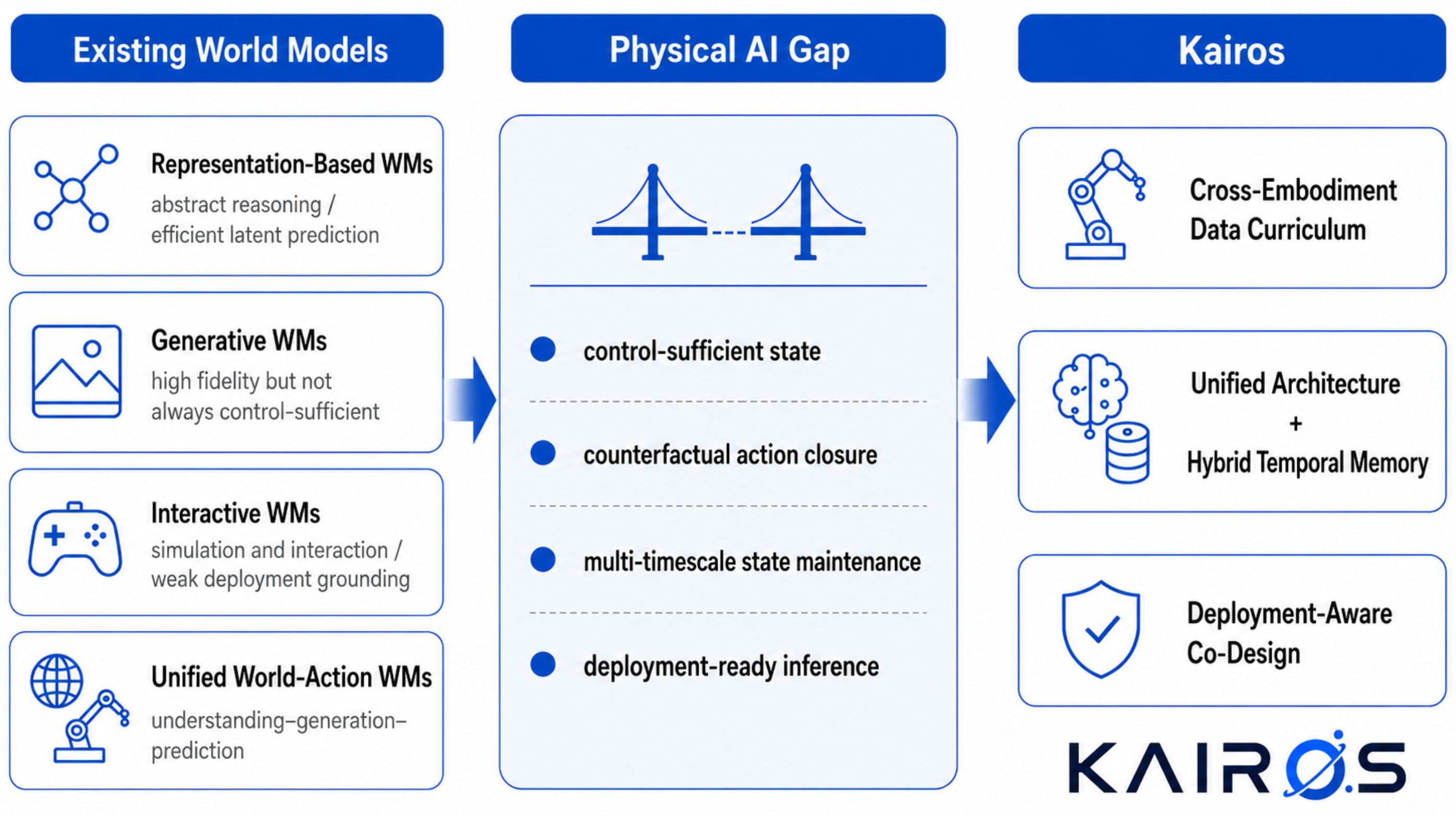}
    \caption{Motivation of Kairos. Existing world models have advanced along representational, generative, interactive, and unified world-action directions, providing useful capabilities such as abstract reasoning, high-fidelity future generation, simulation-based interaction, and embodied world-action modeling. However, Physical AI requires more than visually plausible future simulation: it needs control-sufficient states, counterfactual action closure, multi-timescale state maintenance, and deployment-ready inference. Kairos addresses these gaps through a cross-embodiment data curriculum, a unified world-action architecture with hybrid temporal memory, and deployment-aware co-design, forming a regret-aware world-action stack for learning, maintaining, and deploying control-sufficient states.}
    \label{fig:motivation}
\end{figure}

Kairos is designed around this bottleneck structure (Figure~\ref{fig:kairos_3_0_framework}). Rather than stopping at representational abstraction, pixel-level generation, or interactive simulation, Kairos systematically explores the fourth class of unified understanding--generation--prediction world-action models through a native world-action model stack for Physical AI. Its central objective is to learn, maintain, and deploy a control-sufficient state $Z_t$. This state is not intended to simulate the full world; it is intended to preserve the variables needed for joint world-action prediction, long-horizon consistency, failure anticipation, and future closed-loop evaluation. In this sense, Kairos should be understood as a regret-aware step toward control-sufficient Physical AI: it aims to learn, maintain, and deploy a compressed state $Z_t = f(H_t)$ that preserves information relevant to reducing $\operatorname{Reg}_H(f;g)$, while direct validation of reduced $\operatorname{Reg}_H$ in real closed-loop systems remains future work.

\begin{figure}[t]
    \centering
    \includegraphics[width=1\linewidth]{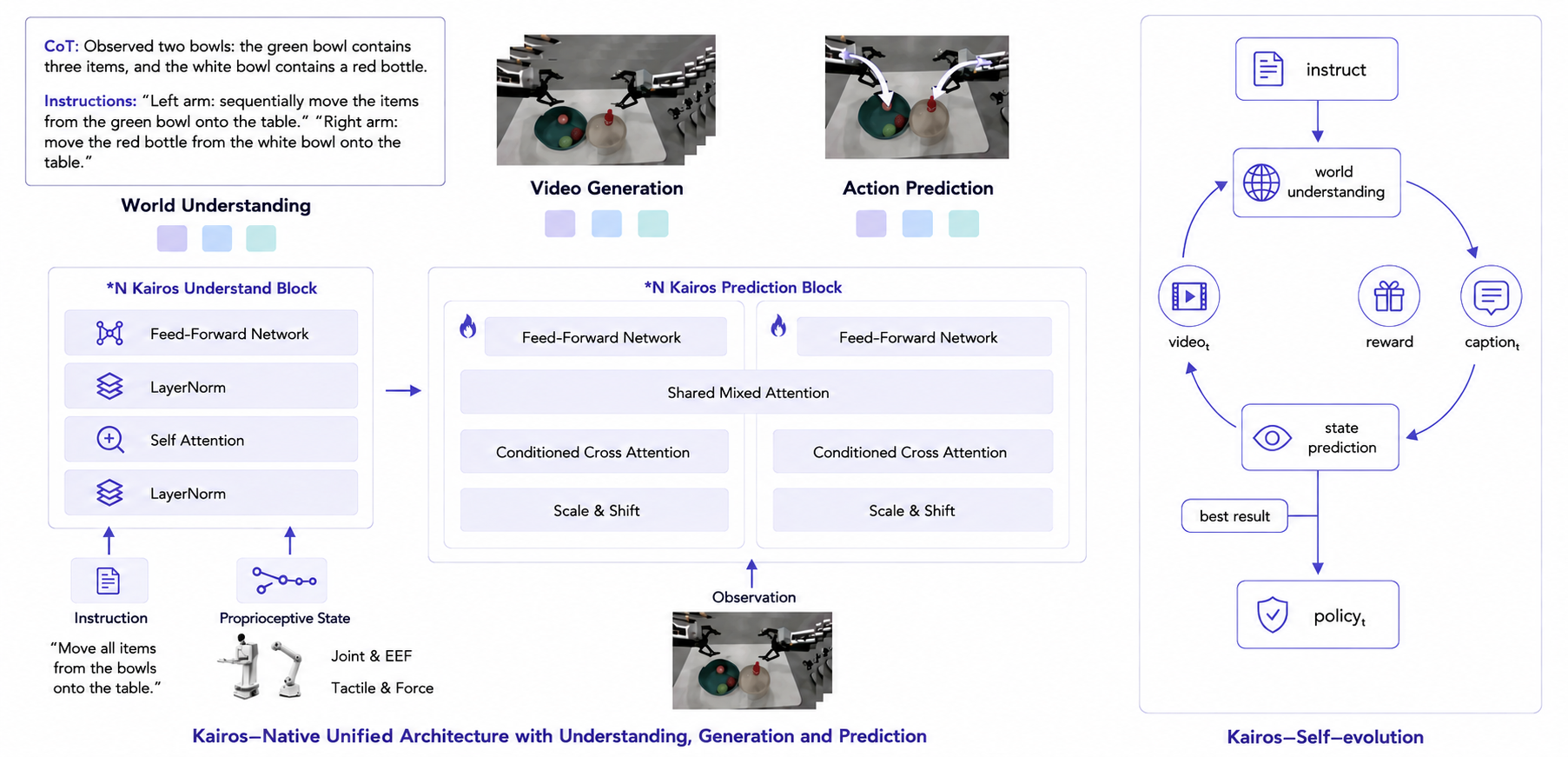}
    \caption{Framework of Kairos. World Understanding extracts a control-sufficient state $Z_t$; World Generation regularizes physical consistency of $Z_t$ through future imagination; World Prediction uses $Z_t$ for future state-action sequences; Deployment-Aware System Co-Design runs $Z_t$ under latency/memory constraints. The Proxy Rollout--Evaluation--Refinement loop establishes inference-side prerequisites for future closed-loop self-evolution.}
    \label{fig:kairos_3_0_framework}
\end{figure}

\begin{figure}[tp]
    \centering
    \begin{subfigure}{\linewidth}
        \centering
        \includegraphics[width=0.4\linewidth]{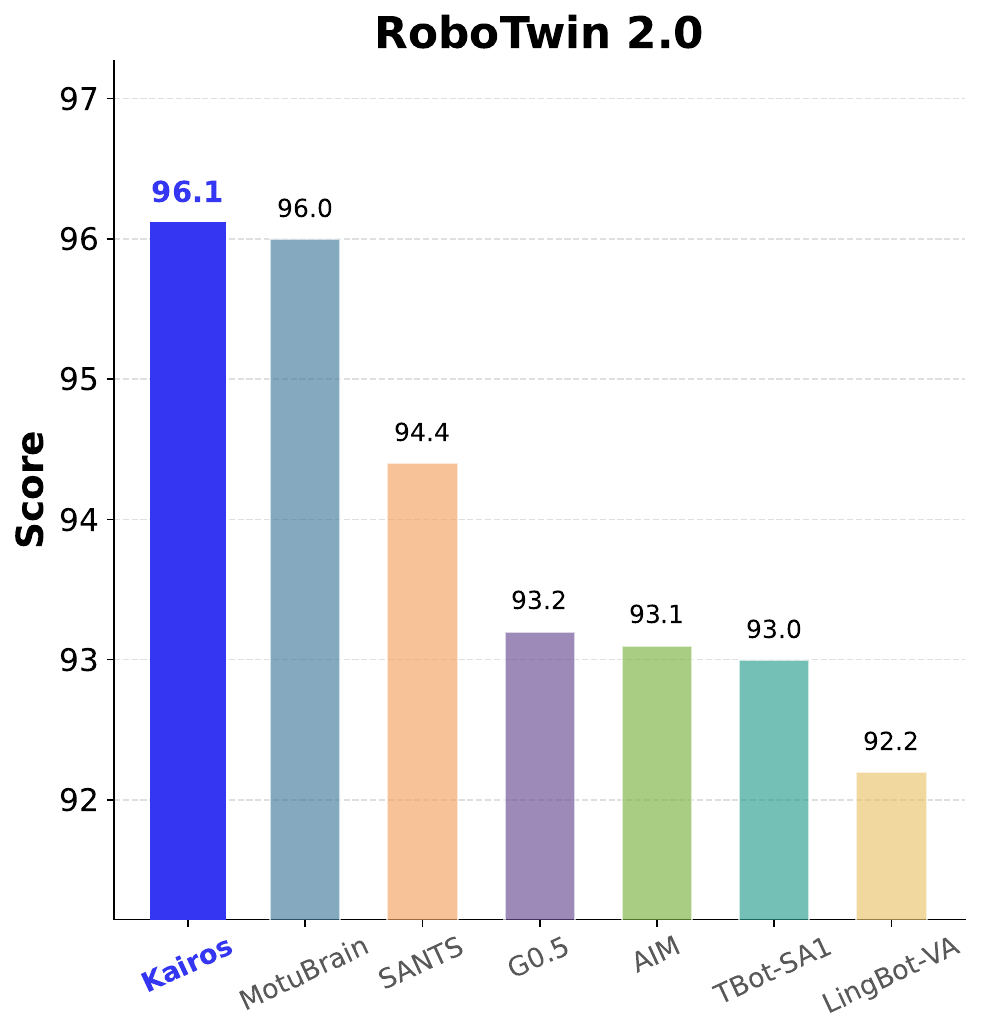}
        \hspace{1em}
        \includegraphics[width=0.4\linewidth]{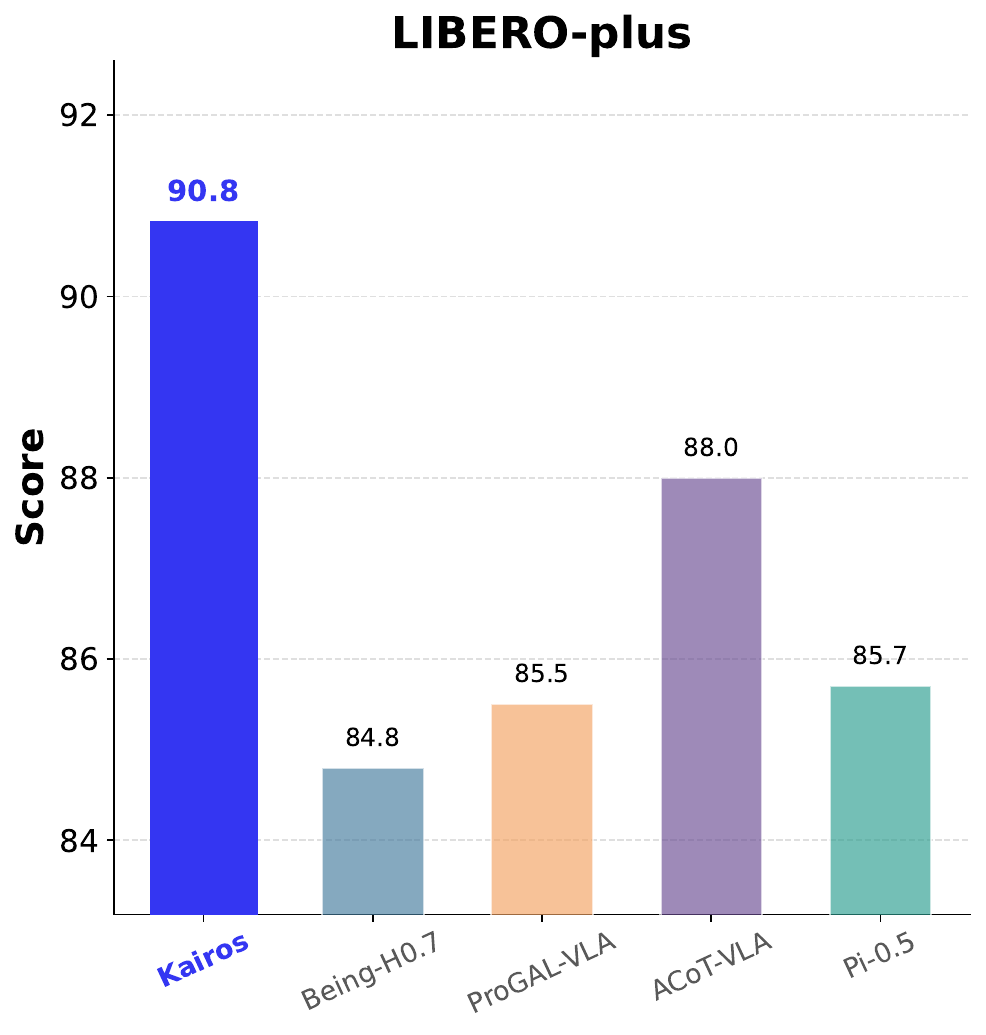}
        \caption{Performance comparison across world action model benchmarks.}
    \end{subfigure}

    \vspace{0.3em}

    \begin{subfigure}{\linewidth}
        \centering
        \includegraphics[height=0.4\linewidth]{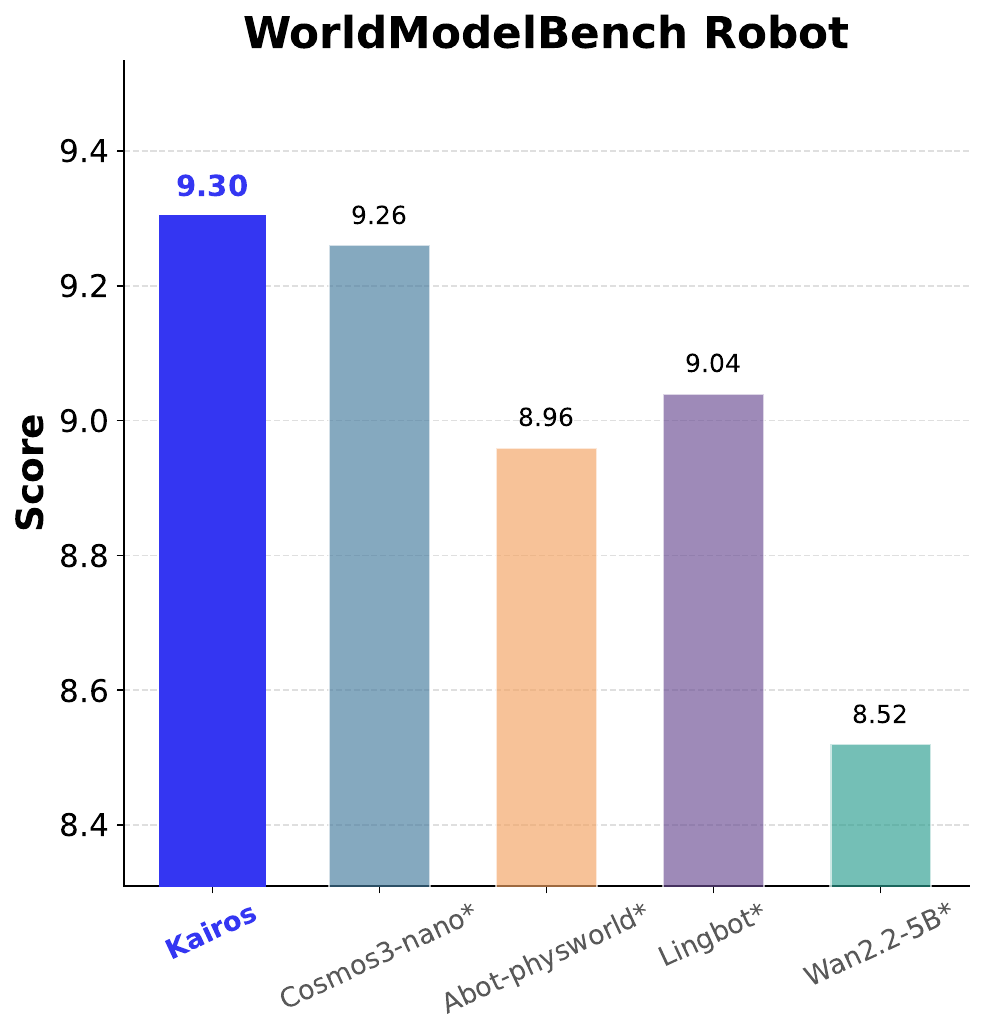}
        \hspace{1em}
        \includegraphics[height=0.4\linewidth]{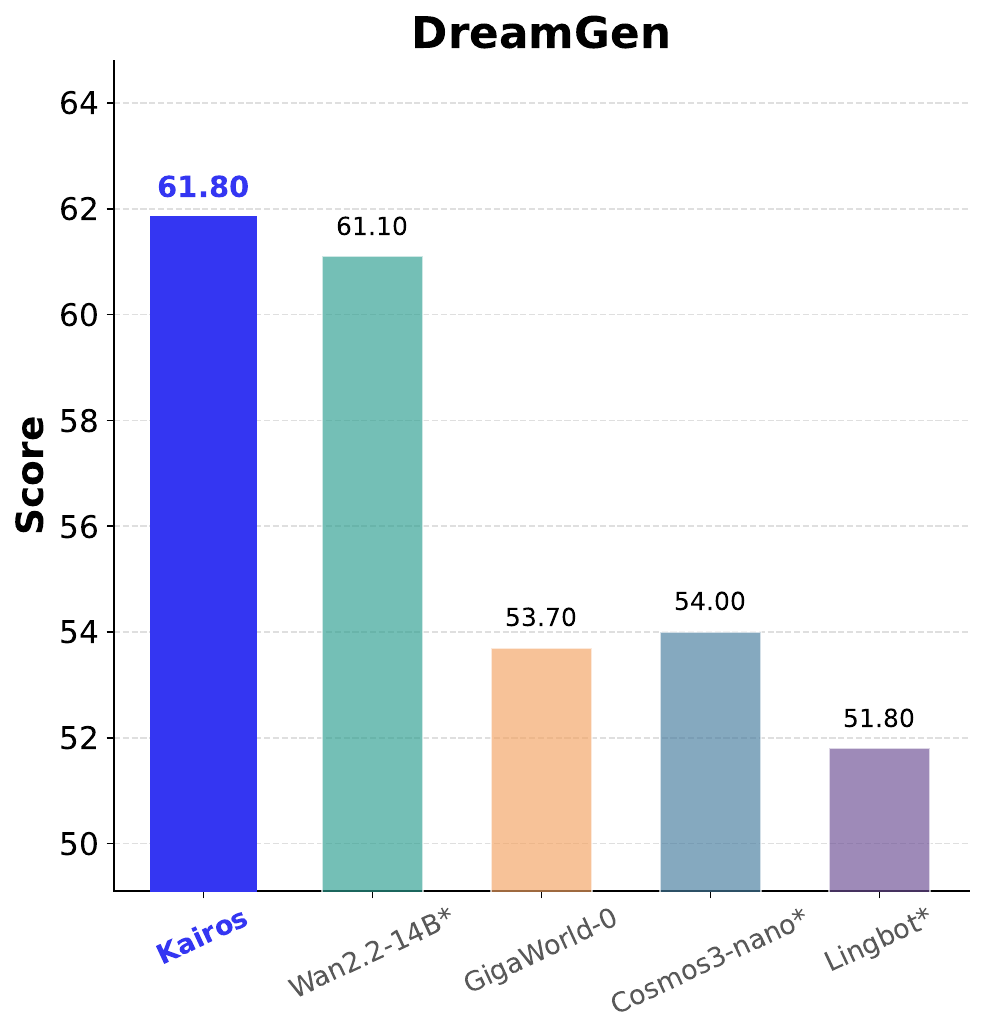}
        \caption{Performance comparison across embodied world model benchmarks.}
    \end{subfigure}

    \vspace{0.3em}

    \begin{subfigure}{\linewidth}
        \centering
        \includegraphics[width=0.49\linewidth]{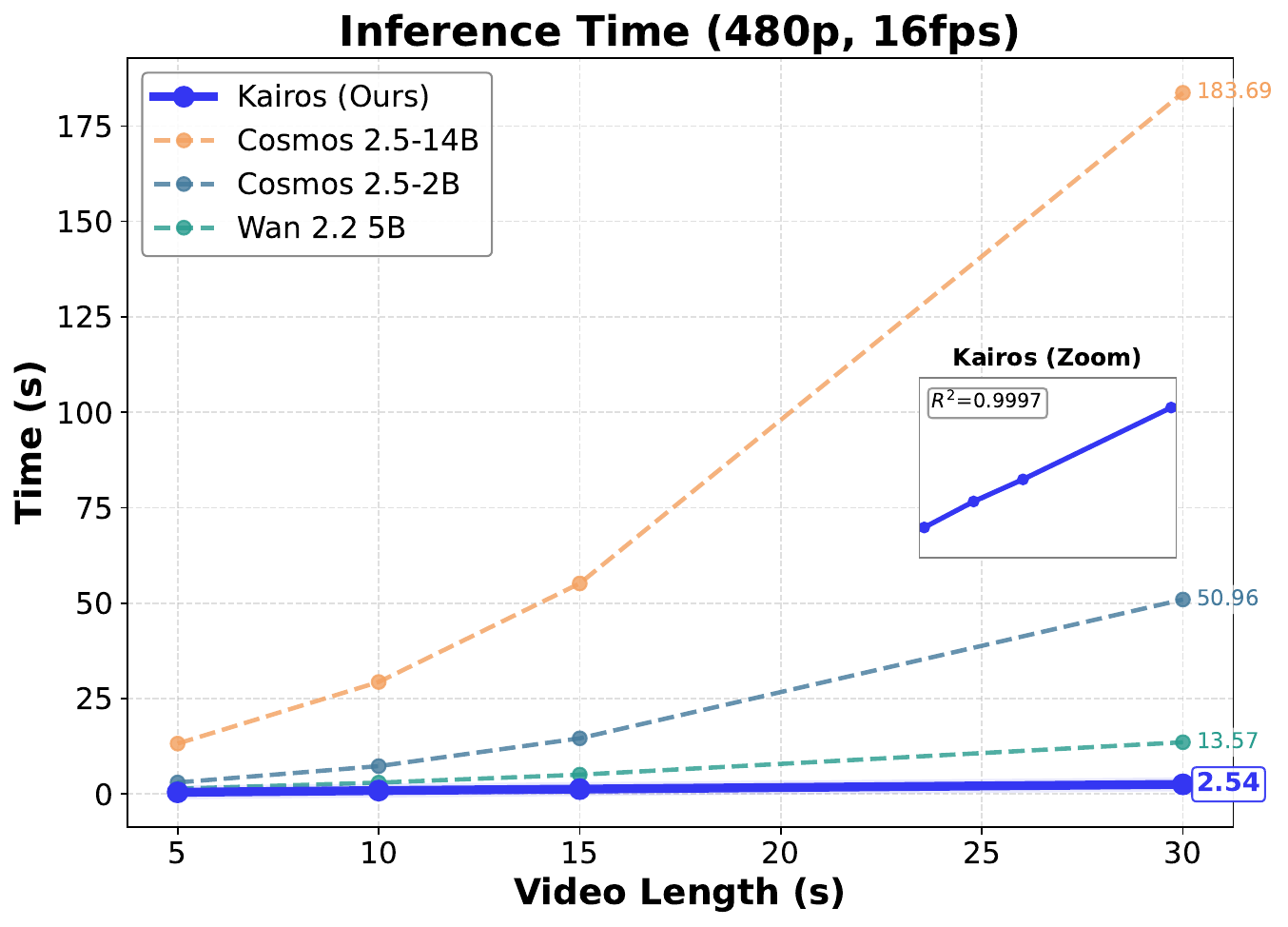}
        \includegraphics[width=0.49\linewidth]{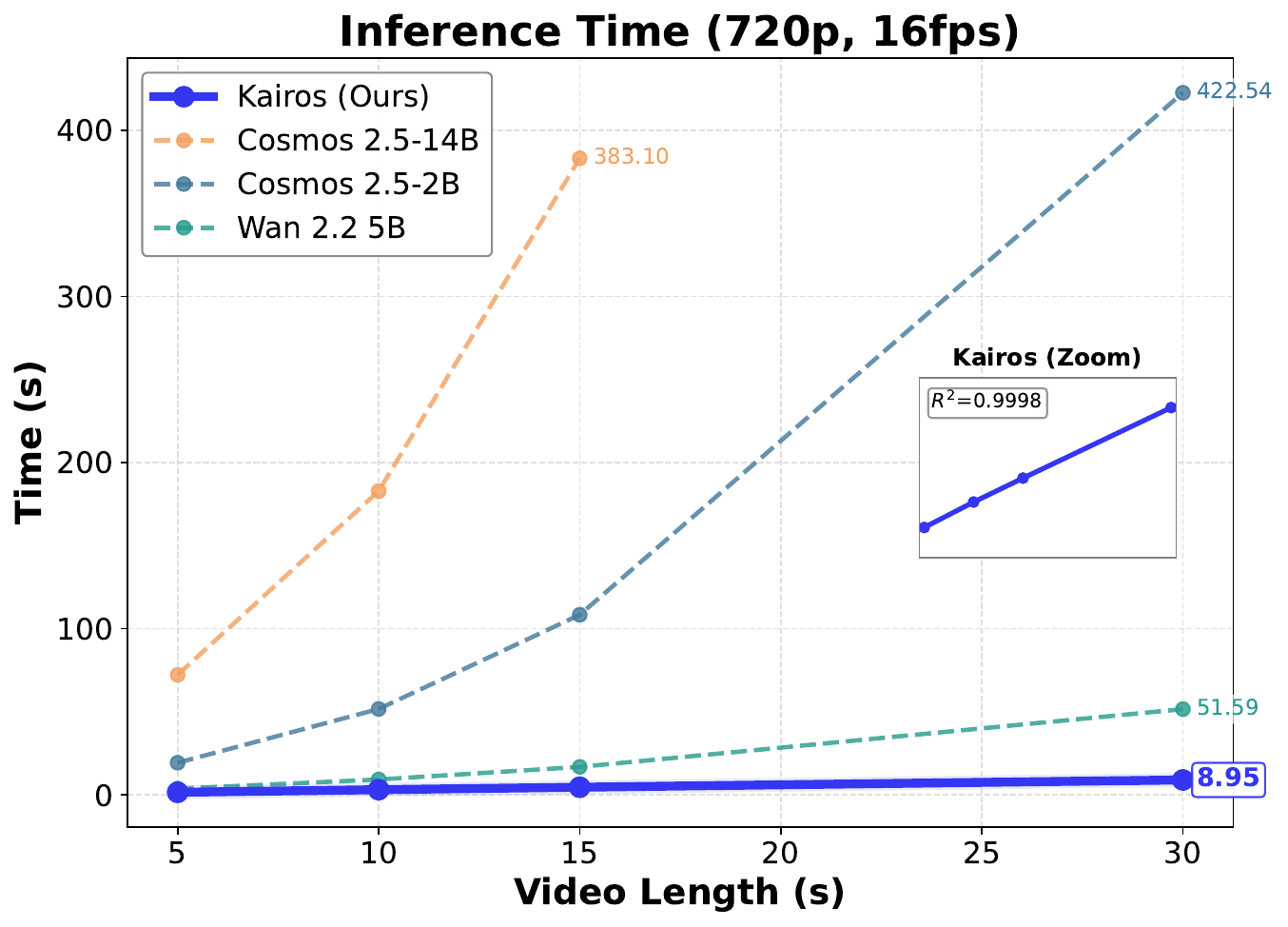}
        \caption{Inference time comparison per DiT step.}
    \end{subfigure}

    \caption{(a)(b) Kairos achieves competitive performance across embodied world-model and world-action benchmarks while delivering significant efficiency advantages over baselines. These results provide proxy evidence for regret-relevant capabilities. (c) Notably, Kairos scales linearly (see the zoom window for DiT inference time per step), ensuring consistent throughput for long-duration generation.}
    \label{fig:vs}
\end{figure}

The first pillar of Kairos is a \textbf{Cross-Embodiment Data Curriculum} for interventional control-information acquisition. Instead of treating open-world videos, human demonstrations, and robot data as a flat mixture, Kairos organizes them into a progression over intervention strength. Open-world videos provide passive physical observation: they expose broad environmental dynamics, object motion, scene evolution, and physical regularities without direct robot intervention. Human behavioral data provides intentional intervention: it reveals task organization, goal-directed interaction, object manipulation strategies, and structured behavior patterns. Robot interaction data provides embodied intervention: it grounds perception--action alignment, embodiment-specific constraints, actuation limits, execution errors, proprioception, and motor affordances. This curriculum is designed to move the model from observation--action correlation toward action--outcome causation.

The second pillar of Kairos is a \textbf{Native Unified Architecture} for control-sufficient state compression and counterfactual world-action prediction. Kairos does not simply connect separate modules for understanding, generation, and prediction. Instead, it aims to maintain semantic understanding, visual generation, physical prediction, and action prediction within a shared world-action state $Z_t$. World Understanding extracts control-relevant semantic and physical variables from heterogeneous observations. World Generation serves as an auxiliary interface for learning physically plausible future evolution and regularizing the shared latent state through visual consistency, object permanence, and physical coherence. World Prediction provides the world-action interface, jointly modeling environmental dynamics and future robot actions. By using Video DiT for future visual tokens, Action DiT for future action tokens, and mixed attention to couple video and action streams, Kairos treats future actions as part of the physical evolution of the world rather than as an external policy head.

The third pillar of Kairos is \textbf{Hybrid Linear Temporal Attention} for multi-timescale control-state maintenance. Long-horizon Physical AI requires more than extending context length. Kairos decomposes temporal modeling into complementary pathways. Sliding-Window Attention captures local dynamics, including short-term motion continuity, contact transitions, slip, collision, and fast hand--eye correction. Dilated Sliding-Window Attention captures mid-range dependencies, including subtask transitions, object--tool interaction histories, and delayed but still localized causal effects. Gated Linear Attention~\cite{yanggated} acts as persistent global causal memory, maintaining object permanence, task progress, long-range dependencies, delayed physical effects, and failure history under efficient inference. This temporal factorization should be interpreted not merely as an efficiency mechanism, but as a structured approach to maintaining control-relevant state across multiple timescales. The theoretical analysis in this report studies this design under stated assumptions and provides bounds that support its role in mitigating long-horizon error accumulation, rather than claiming universal real-world guarantees.

The fourth pillar of Kairos is a \textbf{Deployment-Aware System Co-Design} for closed-loop readiness. For Physical AI, efficiency is not a post-hoc acceleration problem; it is part of the modeling objective. The world model must run within the latency, memory, communication, and hardware constraints of real systems. Kairos therefore treats inference efficiency as a first-order design principle. Hardware-aware kernels, quantization, token streaming, timestep distillation, and runtime optimization are not only engineering improvements; they determine whether the model can participate in observation--action--feedback loops. Efficient deployment increases the amount of control-relevant information that can be processed before action execution, making future regret-aware evaluation and policy improvement more practical.

The fifth pillar is a \textbf{data-centric view of control information density}. Kairos already builds a large-scale data pipeline for data collection, filtering, tagging, captioning, and long-horizon task annotation. Under the control-sufficient perspective, the goal of this pipeline is not only to process more videos, but to identify and structure experience that most reduces uncertainty about action consequences, failure boundaries, contact dynamics, recovery strategies, and safety risks. This suggests a clear data priority for future Physical AI world modeling: near-boundary failure and recovery data, near-boundary successful data, contact-rich data, ordinary successful trajectories, and ordinary observation videos, in descending order of expected control information density. This principle prioritizes diagnostically useful boundary-region experience rather than treating all failures as equally valuable, and provides a more precise criterion for deciding what data should be scaled.

The core contributions of this report are organized around three tightly coupled pillars.

First, we introduce a \textbf{Native Pre-training Paradigm via Cross-Embodiment Data Curriculum}. Kairos learns control-relevant information through a progression from passive physical observation to intentional human behavior and embodied robot action grounding. This curriculum addresses the mismatch between broad but ungrounded open-world experience and scarce but actionable robot interaction. By organizing heterogeneous data by intervention strength, Kairos moves beyond flat data scaling and provides a principled pathway for learning world knowledge that can support embodied control.

Second, we introduce a \textbf{Native Understanding--Generation--Prediction Architecture with Hybrid Linear Temporal Memory}. Rather than framing long-horizon modeling as pure video continuation, Kairos treats it as the maintenance of a shared world-action state. Understanding extracts semantic and physical abstractions; generation regularizes physically plausible future evolution; prediction maps the shared state into future action and visual trajectories. Hybrid Linear Temporal Attention further maintains this state through local, mid-range, and global pathways, supporting efficient long-horizon modeling while preserving control-relevant variables.

Third, we introduce a \textbf{Deployment-Aware System Co-Design} that treats practical execution as a modeling requirement. Kairos integrates runtime optimization, memory efficiency, hardware-aware execution, quantization, and scalable inference into the world-action model stack. This design moves Kairos closer to practical observation--action--feedback loops, where future systems can evaluate imagined rollouts, anticipate failures, filter unsafe actions, and refine policies from real feedback.

These contributions form a single control-sufficient information chain. The Cross-Embodiment Data Curriculum determines where control-relevant information comes from. The unified architecture determines how visual, semantic, physical, and action-related information is compressed into a shared state. Hybrid Linear Temporal Attention determines how that state is maintained over time. World Prediction and Action DiT determine how the state becomes responsive to alternative actions. Deployment-Aware Co-Design determines whether the state can be used under real system constraints. Together, they move Kairos from a static generative model toward a deployable world-action substrate for Physical AI.

\textbf{Results.} Extensive evaluations in this report assess Kairos across embodied world-model benchmarks (WorldModelBench~\cite{li2025worldmodelbench}, DreamGen Bench~\cite{jang2025dreamgen}), world-action benchmarks (RoboTwin~2.0~\cite{chen2025robotwin}, LIBERO-Plus~\cite{fei25libero-plus}), general world-model benchmarks, long-horizon generation, and inference-efficiency settings (Figure~\ref{fig:vs}). Using the notation above, these evaluations should be interpreted as proxy evidence for the components that would enter $J_H$ or its model-predicted counterpart, rather than as direct estimates of $\operatorname{Reg}_H$ or direct proof of closed-loop regret minimization. Embodied world-model benchmarks assess physical plausibility, instruction grounding, and temporal consistency. World-action benchmarks assess whether jointly modeling world dynamics and action evolution improves action prediction and manipulation. Long-horizon generation evaluates whether the model can preserve state consistency over extended temporal windows. Inference-efficiency experiments evaluate whether the model moves closer to practical deployment constraints. Together, these results provide evidence that Kairos learns several capabilities relevant to regret-aware Physical AI: physical consistency, instruction grounding, joint world-action prediction, long-horizon state maintenance, and efficient deployment.

Kairos is therefore a step toward control-sufficient world modeling for Physical AI; demonstrating reduced $\operatorname{Reg}_H$ in real robot deployment remains future work. Direct validation of closed-loop regret reduction will require future real-robot experiments that measure the correlation between imagined and real rollouts, failure prediction before execution, safety filtering effectiveness, recovery learning, and measurable policy improvement from imagined experience. These future evaluations will be necessary to determine whether a world-action model can move from proxy capability to real-world regret reduction. The goal of Kairos is to establish the architectural, data, memory, and deployment foundation on which such future self-evolving physical agents can be built.

\section{Model}\label{sec:model}

Kairos is built around the notion of a \textbf{control-sufficient state}. This state is not intended to preserve all visual details of the future, but to retain the variables required for embodied decision-making: object state, spatial relation, contact condition, task progress, action consequence, failure risk, and deployment uncertainty. The unified architecture should therefore be interpreted as a mechanism for compressing heterogeneous observations and actions into a shared state that is sufficient for prediction, planning, and future closed-loop evaluation.

The regret formulation in the Introduction defines the target role of this state: $Z_t$ should preserve the information needed to support low-cost action choices under the horizon-level physical cost $J_H$, while omitting information that does not affect task-relevant outcomes. The model section below specifies how Kairos constructs, regularizes, predicts from, and maintains such a state. This connection should be understood as a model-side prerequisite for future regret-aware control, not as a claim that the current system directly proves closed-loop regret minimization.

\subsection{Native Architecture with Unified Understanding, Generation, and Prediction}\label{sec:native_arch}

The core architecture of Kairos is built around a single principle: a world model for Physical AI should not attempt to preserve or generate the entire world, but should learn and maintain a control-sufficient state. This state should retain the information needed for embodied prediction and decision-making, including object state, spatial relation, contact condition, task progress, action consequence, failure boundary, safety risk, and deployment uncertainty. We denote this internal state as $Z_t$, constructed in the current system from the observation--action history $H_t$ and the task goal or language instruction $g$. $Z_t$ is not a complete copy of the world; it is a compact internal state that should be sufficient for downstream generation, prediction, planning, risk assessment, and future closed-loop evaluation.

Under this view, the native architecture of Kairos should not be interpreted as a loose connection of three independent modules. World Understanding, World Generation, and World Prediction are three interfaces to the same underlying world-action state. World Understanding constructs $Z_t$ from heterogeneous experience. World Generation turns $Z_t$ into physically plausible imagined futures, thereby regularizing and probing the state. World Prediction maps $Z_t$ into joint future state-action trajectories and executable action tokens, thereby turning the model from a passive observer into an embodied world-action system. These components are coupled through a shared architecture and a hybrid temporal memory mechanism designed to maintain control-relevant information over long horizons (Figure~\ref{fig:kairos_archi}).

\begin{figure}[t]
\centering
\includegraphics[width=1.0\linewidth]{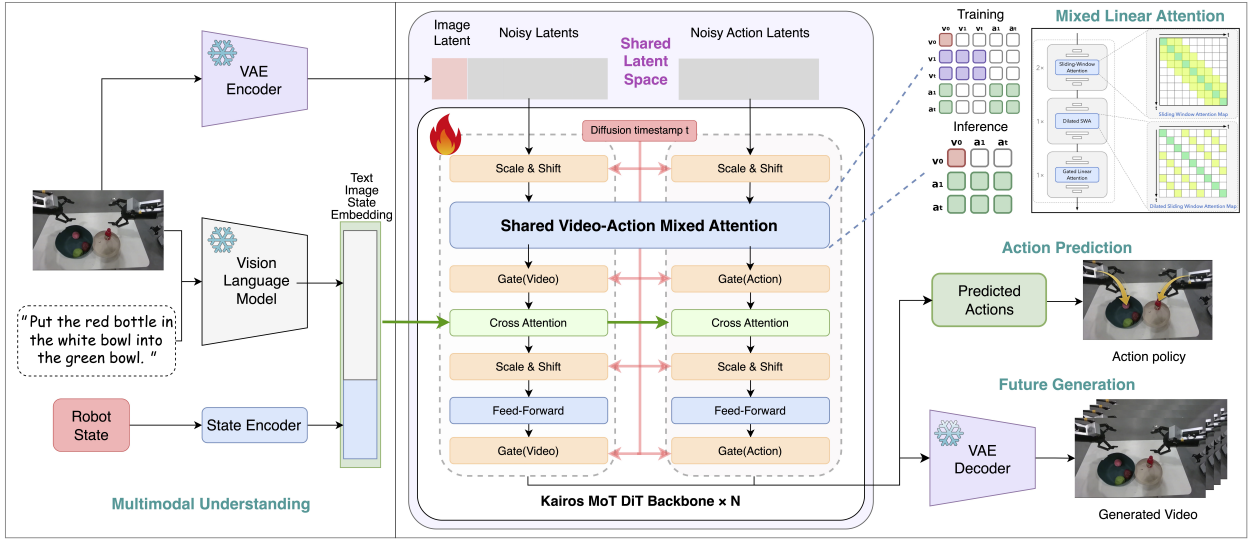}
\caption{Model architecture of Kairos. Understanding, Generation, and Prediction operate as three interfaces to the shared control-sufficient state $Z_t$.}
\label{fig:kairos_archi}
\end{figure}

This design departs from the common modular pipeline in which a perception model first produces semantic features, a video model then generates future frames, and a policy module finally predicts actions. Such modular pipelines can be useful, but they often produce mismatched internal states: the perception module may preserve semantic information but discard physical variables; the generation module may synthesize visually plausible futures without preserving task progress; the action module may learn observation--action correlations without understanding how actions change future states. Kairos instead aims to maintain semantic, visual, physical, and action-related information in a shared world-action state.

The purpose of this unified design is to instantiate the model-side component of the regret objective introduced above: Kairos learns a compressed state $Z_t$ that is intended to preserve the variables needed for low-cost action selection. Whether this state yields lower closed-loop physical cost than baseline representations must be validated in future real-robot or high-fidelity simulated experiments. It gives the model a place to compress control-relevant information, a mechanism to maintain this information over time, an interface to imagine possible futures, and an action branch that can use world dynamics without always materializing future video during deployment.
\subsubsection{World Understanding: Control-Sufficient State Construction}

World Understanding is the entry point through which Kairos converts heterogeneous sensory, linguistic, physical, and embodied experience into a shared internal state for Physical AI. In conventional multimodal systems, understanding is often treated as semantic recognition: recognizing objects, parsing instructions, describing scenes, or aligning images and language. These capabilities remain necessary, but they are not sufficient for an embodied world model. For a physical agent, the purpose of understanding is not to describe the world as completely as possible, but to preserve the subset of information that remains relevant for future control.

Kairos therefore defines World Understanding as control-sufficient state construction. Given an observation--action history $H_t$ and a task instruction or goal $g$, the Understanding module constructs the shared internal state $Z_t$. In the current architecture, $Z_t$ is realized through structured multimodal representations: visual observations are represented by visual latent tokens, and task instructions provide language-conditioned features. This state is intended to preserve task-relevant semantic variables, physical cues, task-progress information, and uncertainty variables needed by downstream generation, prediction, planning, and future closed-loop evaluation. In this sense, World Understanding is the first compression stage of the Kairos world-action stack.

From a regret-aware information-compression perspective, the goal is not to maximize the amount of information stored in $Z_t$, but to maximize the amount of control-relevant information preserved per unit of representation size, computation, latency, and risk. In terms of the regret formulation, the Understanding module should preserve precisely those variables that affect the physical cost $c$, including contact conditions, task progress, failure boundaries, safety margins, recovery difficulty, and imagined--real uncertainty. A robot does not need every visible detail of a scene. It needs to know which objects matter, where they are, how they can be manipulated, what contact conditions are likely to occur, how the current state relates to the task goal, and where failure boundaries may lie. For example, the exact texture of a wall may be irrelevant for a table-top manipulation task, while a subtle cue about object slippage may determine whether the next grasp succeeds or fails.

This compression objective is inherently task-aware. The same scene may require different understanding states under different goals. If the instruction is to pick up a cup, the model should prioritize cup pose, grasp affordance, obstacle layout, and gripper state. If the instruction is to avoid spilling liquid, it should additionally preserve tilt, fill level, acceleration risk, and stability margin. If the instruction is to clean a table, it should preserve object categories, reachable regions, task ordering, and disposal locations. Thus, World Understanding should not simply ask what is in the scene, but which aspects of the scene matter for the current and future control problem.

Finally, World Understanding should also be history-aware. Many control-relevant cues are
not inferable from a single frame: contact stability depends on recent motion,
task progress depends on earlier subtasks, object permanence depends on memory
through occlusion, and failure risk may depend on delayed effects. In Kairos,
the state constructed by World Understanding is maintained and updated through
the temporal modeling mechanisms described in Section~\ref{sec:attention},
rather than by a static frame-level encoder alone. This allows the shared state
to carry information about recent motion, task progress, object persistence,
and delayed physical effects for downstream generation and prediction.

In the current Kairos implementation, the VLM-based Understanding module~\cite{qwen2.5-VL,qwen3.5} provides a practical foundation for instruction grounding and multimodal semantic alignment. It does not yet establish complete physical understanding; rather, it provides an operational interface for extracting and conditioning on control-relevant semantic variables and physical cues from the available multimodal context.

In summary, World Understanding in Kairos should be interpreted as the process of constructing a control-sufficient state from heterogeneous experience. It compresses multimodal history into variables that matter for embodied control, supports multi-timescale state maintenance, and prioritizes high-density control information. Its success should ultimately be judged not only by visual--language benchmark performance or instruction-following score, but by whether the resulting state improves future prediction, failure anticipation, safety assessment, recovery planning, and eventually closed-loop physical cost relative to baseline states.
\subsubsection{World Generation: Control-State Regularization}

World Generation is the component through which Kairos imagines possible future observations, models physically plausible scene evolution, and regularizes the shared world-action state. In many recent world-model systems, generation is treated as the central objective: the model is judged by whether it can produce high-fidelity, temporally coherent, and visually impressive future videos. This objective has driven substantial progress in observation-level generative world models, but it is not sufficient for Physical AI. A robot does not act in order to make future pixels look realistic. It acts to complete tasks, avoid unsafe states, recover from mistakes, and reduce costly real-world failures. Therefore, in Kairos, World Generation is not the final purpose of world modeling, but an auxiliary interface for probing and regularizing a control-sufficient internal state.

Given the shared state $Z_t$ produced by World Understanding, World Generation predicts future visual observations and tests whether the state preserves physically relevant information. The generated sequence is not merely a video continuation. It is an operational probe of the shared world-action state. When used for training or evaluation, future visual prediction can expose whether the internal state preserves object permanence, temporal continuity, spatial layout, contact-relevant cues, task progress, and long-horizon causal structure.

At the generation interface, Kairos uses a latent diffusion design for efficient and scalable future-observation modeling. Visual observations are represented in a compact latent space, and multimodal conditioning features provide task and context information. A temporally scalable Diffusion Transformer~\cite{peebles2023scalable} performs denoising in latent space under these conditions. This branch supports flexible generation settings such as text-to-video, image-to-video, and text-image-to-video, while keeping the role of generation focused on probing and regularizing the shared state rather than defining the state representation itself.

The first role of World Generation is physical and temporal regularization. When the generation branch is trained to predict future observations, it encourages the shared state $Z_t$ to preserve information about motion, object permanence, spatial layout, contact, support, temporal continuity, task progress, and delayed effects. If the state omits critical physical variables, the generated future may drift, violate object identity, break contact consistency, ignore task progress, or produce impossible transitions. In this sense, visual generation acts as a state regularizer rather than a final evaluation target.

The second role of World Generation is long-horizon state probing. Short-horizon generation can often succeed by exploiting local appearance smoothness. Long-horizon Physical AI requires more: persistent object identity, stable spatial relations, consistency through occlusion, delayed contact effects, and multi-stage task progress. For example, if an object is temporarily occluded, the generated future should preserve its identity and approximate location rather than treating it as a new object when it reappears. If a multi-step task requires moving several items in order, the generated sequence should preserve which items have already been moved and which remain. These are not merely aesthetic video-quality issues. They are proxies for whether the internal world state remains useful for control over time.

World Generation is also deployment-aware. Video synthesis can be computationally expensive and may be unsuitable for real-time decision loops if used naively. Kairos addresses this through latent diffusion, compact video representation, temporally scalable DiT design, hybrid attention, timestep distillation, and hardware-aware optimization. The deeper point is not simply speed: in Physical AI, inference cost determines whether the model can be used before action execution. A generated future that arrives too late to affect action is not useful for regret-aware control.

At deployment time, Kairos need not always generate full future videos. In many robotic settings, the prediction branch may use the shared world-action representation to produce future action tokens directly, while the future video branch is disabled to reduce cost. During training, visual generation helps the model learn spatiotemporal and physical priors. During deployment, action-only inference can retain the benefits of those priors without paying the full cost of video synthesis. Thus, World Generation mainly serves as a training-time regularizer and diagnostic interface, while the system can choose a lighter inference pathway when the control loop requires speed.

The relationship between World Generation, World Understanding, and World Prediction is complementary. Understanding constructs the shared state; Generation probes and regularizes whether that state supports physically plausible future evolution; Prediction uses the same state for future world-action modeling. A better video model does not automatically produce a better policy, but a generation branch that captures physical dynamics, object interaction, and long-horizon consistency can provide useful priors for joint world-action prediction. This is why Kairos treats World Generation as one interface to the shared world-action state rather than as an isolated video synthesis module.

In summary, World Generation in Kairos should be understood as control-state regularization rather than merely photorealistic video synthesis. It models future observations under multimodal conditioning, while regularizing the shared world-action state toward physical plausibility, object permanence, task consistency, and long-horizon temporal coherence. Its value lies in whether it helps Kairos preserve and operationalize the information needed for physical decision-making.
\subsubsection{World Prediction: Joint World-Action Modeling}\label{sec:world_prediction}

World Prediction is the component that extends Kairos from future observation modeling to joint world-action modeling. Conventional world models primarily ask: given the past, what future observations are likely to occur? This is useful for forecasting, simulation, and video generation, but it is insufficient for Physical AI. A robot must also model how future actions and future world states are coupled under the current task and observation history. A world-action model should therefore represent future actions not as an external output appended after world modeling, but as part of the future evolution that the model learns to predict.

In the notation of the regret objective, World Prediction provides a model-side interface for connecting the shared state $Z_t$ to future physical consequences. Its role is to make action prediction depend on the same state used for physical and temporal world modeling, so that downstream action selection can be informed by task progress, contact-relevant cues, failure risk, and predicted future dynamics rather than by visual plausibility alone.

Kairos formulates World Prediction as a unified World-Action Model (WAM). Rather than treating future actions as an independent policy head detached from world modeling, Kairos couples future visual dynamics and future action tokens within the same world-action prediction interface. This design reflects a central assumption of Physical AI: the future state of the world and the future action of the agent are not independent. The robot does not merely observe the world; it acts within it. Therefore, action prediction should benefit from the physical priors and temporal dynamics learned by world modeling.

Let $Z_t$ be the control-sufficient state constructed by World Understanding and regularized by World Generation. The World Prediction module estimates future action trajectories and, when required, future visual trajectories from this shared state. Future observations or latent visual states provide the world-dynamics side of prediction, while future robot action tokens provide the execution side of prediction. During training, Kairos can jointly optimize future visual evolution and action sequence generation. During deployment, the model can disable future video generation and generate only future action tokens, thereby reducing computational cost while retaining the world dynamics learned from joint training.

Architecturally, Kairos implements World Prediction using a Mixture-of-Transformer design with two tightly coupled components: a Video DiT and an Action DiT. The Video DiT models future visual tokens and is initialized from the pretrained Kairos World Generation model, allowing it to inherit spatiotemporal priors, physical regularities, and long-horizon visual dynamics. The Action DiT predicts future action tokens and follows the same architectural logic while using a smaller scale for deployment efficiency. The purpose of this design is not simply to attach an action head to a video model, but to let action prediction benefit from the physical priors learned by world generation.

The input sequence is organized into three token groups:
\begin{itemize}
\item History Video Tokens, representing historical visual observations;
\item Future Video Tokens, representing future visual states;
\item Future Action Tokens, representing future robot actions.
\end{itemize}

This token organization allows Kairos to jointly model visual dynamics and action prediction while preserving clear causal constraints.

Inspired by the mixed-attention strategy in~\cite{yuan2026fastwam}, Kairos adopts a unified attention masking strategy for joint video--action modeling. History video tokens attend only to historical tokens, preventing information leakage from future states and maintaining a stable representation of the observed past. Future video tokens and future action tokens are both conditioned on the historical visual context, ensuring that imagined futures and predicted actions are grounded in the same observed state. Future video tokens use sparse spatiotemporal attention to efficiently capture local visual dynamics, while future action tokens use broader attention across the action sequence to support coherent long-horizon action prediction. Importantly, the action branch does not depend on future video tokens. This creates an asymmetric interaction scheme: video and action are jointly trained, but future action inference does not require explicit future video synthesis.

This asymmetry is central to the Kairos design. During training, the video objective encourages the model to learn environmental dynamics, physical consistency, object permanence, and future state evolution. The action objective encourages the model to learn executable control trajectories grounded in the same observed context. Joint optimization aligns environmental transitions with action prediction. During deployment, however, future video synthesis can be skipped when action-only inference is sufficient. Since action tokens are much fewer than video tokens, this substantially reduces diffusion and attention computation while preserving the benefits of jointly learned world dynamics.

From the perspective of world-action prediction, the Action DiT is the action-side interface of the world model. It uses the shared state to generate future action tokens while benefiting from the physical and temporal structure learned by the visual branch. The key point is that actions are not merely labels attached to observations; they are part of the agent's interaction with the world. By jointly modeling future visual and action tokens, Kairos moves beyond pure observation prediction toward a representation that is more suitable for embodied control.

World Prediction also connects to regret-aware modeling. A policy that acts without modeling consequences can only react to observations. A world-action model can provide a substrate for evaluating future consequences before execution. In future closed-loop systems, this interface can support policy evaluation, risk filtering, recovery planning, and imagined rollout ranking. The value of World Prediction therefore lies not only in action accuracy, but also in its potential to support lower-cost decisions under physical constraints.

The relationship between World Prediction, World Understanding, and World Generation is complementary. Understanding constructs the shared state $Z_t$ that contains task-relevant information. Generation regularizes this state by modeling physically plausible future observations. Prediction uses the same state for future action modeling. A better video model does not automatically produce a better policy, but a generation branch that captures physical dynamics, object interaction, and long-horizon consistency can provide useful priors for joint world-action prediction. This is why Kairos should be described as a native world-action stack rather than a VLM plus video model plus policy head.

World Prediction also reinforces the deployment-aware philosophy of Kairos. In real Physical AI systems, a world model must produce useful action predictions within the time budget of the robot's control loop. A high-fidelity future video that arrives too late to affect action cannot reduce real-world mistakes. The action-only inference mode directly addresses this constraint. It preserves the training-time benefits of world simulation while allowing deployment-time inference to avoid expensive future video materialization.

In summary, World Prediction in Kairos is a world-action prediction interface. It jointly trains video and action branches to align physical dynamics with executable control, while supporting efficient action-only inference during deployment. Its role is to establish a model-side prerequisite for future regret-aware Physical AI. Direct validation of whether predicted futures and action rollouts improve real closed-loop behavior remains an important direction for future work.

\subsection{Efficient Diffusion Transformer with Hybrid Linear Attention}\label{sec:attention}

\begin{figure}[h]
\centering
\includegraphics[width=1.0\textwidth]{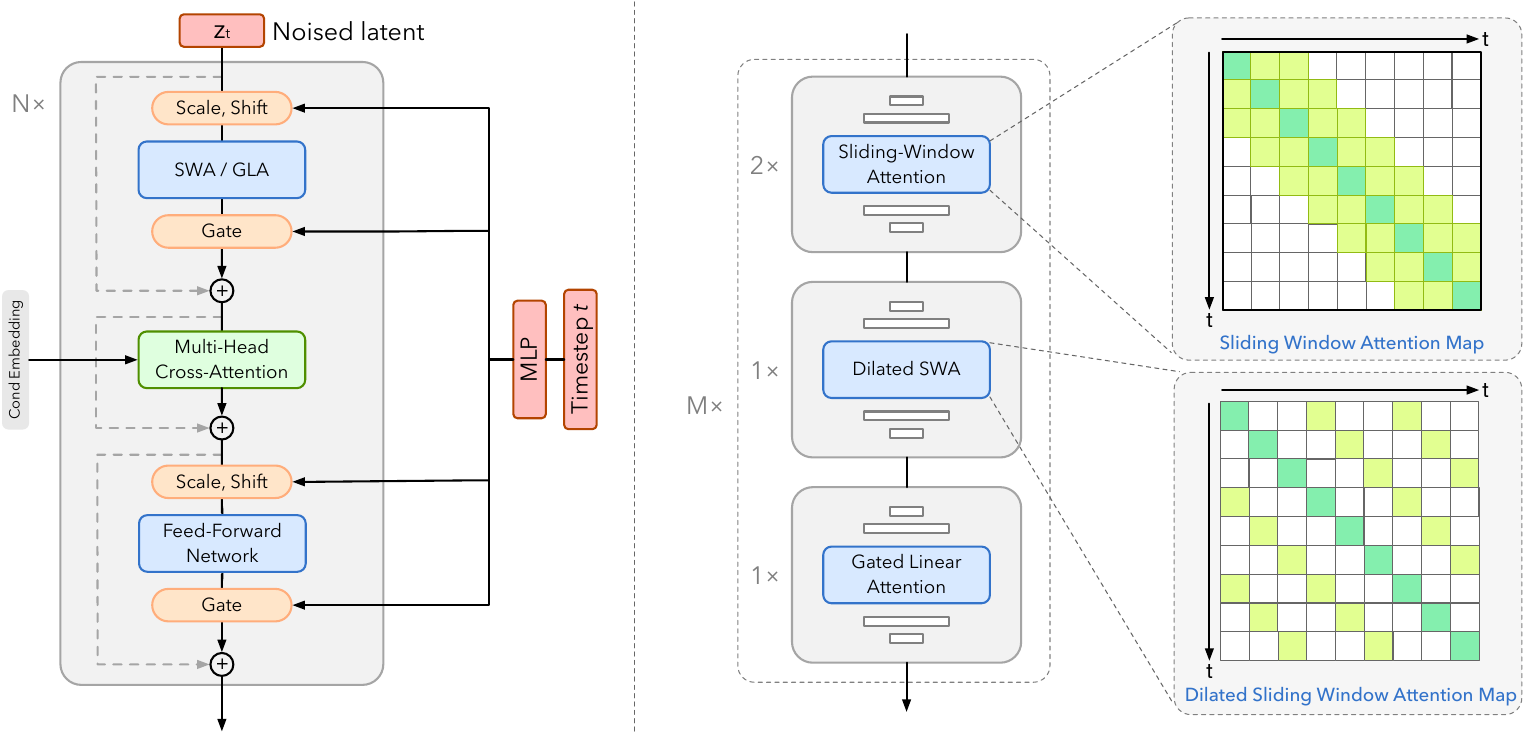}
\caption{DiT block architecture of the proposed Hybrid Linear Temporal Attention.}
\label{fig:kairos_dit}
\end{figure}

Diffusion Transformers have become powerful backbones for image generation, video generation, and action prediction. However, standard DiT architectures rely on full Softmax self-attention, whose quadratic complexity in sequence length becomes prohibitive for long videos and high-resolution embodied observations. In Physical AI, this problem is not only computational. It is also representational. Long-horizon world modeling does not simply require longer context; it requires maintaining the right control-relevant variables across different temporal scales.

Kairos therefore introduces a \emph{LinearDiT} backbone with \emph{Hybrid Linear Temporal Attention} (Fig.~\ref{fig:kairos_dit}). The goal is not merely to make long-video generation cheaper. The goal is to maintain a control-sufficient state across local, mid-range, and global timescales. Fast local dynamics include motion continuity, contact transitions, slip, collision, and hand--eye correction. Mid-range dynamics include object interaction history, tool use, and subtask transitions. Global dynamics include object permanence, task progress, delayed physical effects, scene memory, and failure history. These temporal responsibilities are different and should not be forced into a single attention mechanism.

Kairos factorizes temporal modeling into three complementary pathways: Sliding-Window Attention (SWA) for local dynamics, Dilated Sliding-Window Attention (DSWA) for mid-range dependencies, and Gated Linear Attention (GLA) for global causal memory. The resulting backbone is organized into repeated groups of hybrid blocks that interleave SWA, DSWA, GLA, conditioning layers, and feed-forward transformations. This design supports efficient local motion modeling, mid-range interaction aggregation, and persistent global causal memory within a unified DiT architecture.

This factorization should be interpreted as a multi-timescale control-state maintenance mechanism. SWA handles the local physics that must be updated rapidly. DSWA expands the temporal receptive field without quadratic cost, allowing intermediate dependencies to be captured. GLA maintains a persistent global state with linear complexity, allowing supra-window information to influence future predictions even when it is no longer visible within the local context. Together, these components allow Kairos to preserve control-relevant information without relying on full dense attention over all tokens.

\subsubsection{Global Pathway: Gated Linear Attention as Persistent Causal Memory}

To enable global temporal reasoning with linear complexity, Kairos employs \emph{Gated Linear Attention} (GLA) as the primary mechanism for long-range information propagation. Concretely, GLA is implemented using GatedDeltaNet~\cite{yanggated}, a gated linear attention variant closely related to structured state-space models (SSMs) and to a broader line of efficient and linear-attention architectures~\cite{beck2024xlstm,child2019generating,beltagy2020longformer,zaheer2020big,dao2024transformers,katsch2023gateloop,yang2023gated,qin2023hierarchically,qin2024hgrn2,huang2024localmamba,pei2025efficientvmamba,lahoti2026mamba}. Unlike Softmax attention, whose complexity scales quadratically with sequence length, GLA scales linearly and thus remains efficient even for long video sequences.

\begin{figure}[h]
\centering
\includegraphics[width=0.5\textwidth]{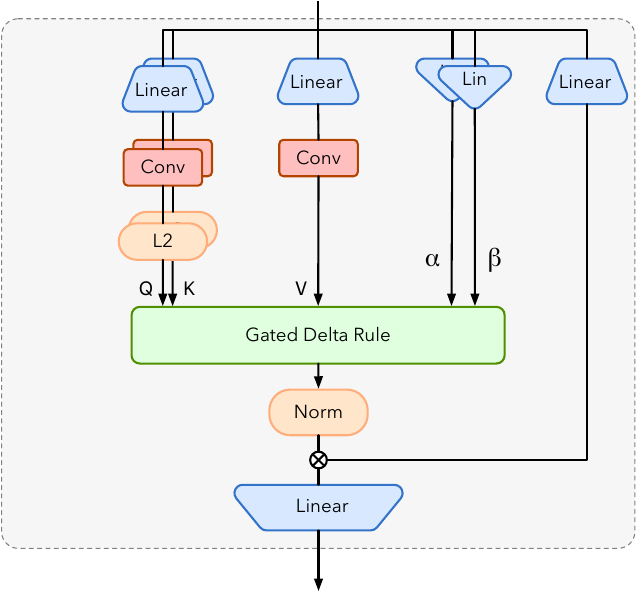}
\caption{Architecture of the gated linear attention module (GDN).}
\label{fig:gdn_arch}
\end{figure}

As illustrated in Fig.~\ref{fig:gdn_arch}, the core of the GDN lies in the \textbf{Delta Update Rule}, which addresses the ``key collision'' problem found in vanilla linear transformers. Instead of purely additive updates, GDN learns to remove outdated or less important key--value associations to make room for new information.

The computation at each time step $t$ is defined as follows:
\begin{enumerate}
\item \textbf{Feature Extraction.} The query $\mathbf{q}_t$, key $\mathbf{k}_t$, and value $\mathbf{v}_t$ are projected from the input $\mathbf{x}_t$. Simultaneously, a soft ``writing strength'' $\beta_t$ is computed via a sigmoid gate:
\begin{equation}
\mathbf{q}_t = \mathbf{W}_Q \mathbf{x}_t, \quad \mathbf{k}_t = \mathbf{W}_K \mathbf{x}_t, \quad \mathbf{v}_t = \mathbf{W}_V \mathbf{x}_t, \quad \beta_t = \sigma(\mathbf{W}_\beta \mathbf{x}_t).
\end{equation}

\item \textbf{Memory Retrieval and Interpolation.} The model retrieves the old value $\mathbf{v}_t^{\text{old}}$ using the current key and interpolates it with the current value to generate $\mathbf{v}_t^{\text{new}}$:
\begin{equation}
    \mathbf{v}_t^{\text{old}} = \mathbf{S}_{t-1} \mathbf{k}_t, \quad \mathbf{v}_t^{\text{new}} = \beta_t \mathbf{v}_t + (1 - \beta_t) \mathbf{v}_t^{\text{old}},
\end{equation}
where $\mathbf{S}_t \in \mathbb{R}^{d_v \times d_k}$ denotes a learnable associative memory that stores key--value correlations over time.

\item \textbf{Delta State Update.} The state matrix $\mathbf{S}_t$ is updated by removing the old association and writing the new one, a process equivalent to a single step of SGD on an online regression loss:
\begin{equation}
    \mathbf{S}_t = \mathbf{S}_{t-1} - \underbrace{\mathbf{v}_t^{\text{old}} \mathbf{k}_t^\top}_{\text{remove}} + \underbrace{\mathbf{v}_t^{\text{new}} \mathbf{k}_t^\top}_{\text{write}}.
\end{equation}
This update can be interpreted as an online delta-rule update that approximates one step of gradient descent on $\|\mathbf{v}_t - \mathbf{S}\mathbf{k}_t\|^2$.

\paragraph{Gated Delta Update.}
While the above delta rule corrects key--value associations locally, it does not explicitly control global forgetting of past information. To improve memory management, a gating mechanism is introduced to adaptively modulate the contribution of the previous state. Specifically, a decay gate $\alpha_t \in (0,1)$ is computed as:
\begin{equation}
\alpha_t = \sigma(\mathbf{W}_\alpha \mathbf{x}_t).
\end{equation}
The state update is then modified as:
\begin{equation}
\mathbf{S}_t = \alpha_t \mathbf{S}_{t-1} - \mathbf{v}_t^{\text{old}} \mathbf{k}_t^\top + \mathbf{v}_t^{\text{new}} \mathbf{k}_t^\top.
\end{equation}
Equivalently, this can be written as:
\begin{equation}
\mathbf{S}_t = \alpha_t \mathbf{S}_{t-1} + \beta_t (\mathbf{v}_t - \mathbf{v}_t^{\text{old}}) \mathbf{k}_t^\top.
\end{equation}
Here, $\alpha_t$ acts as a forget gate that globally scales the previous memory state, enabling the model to discard outdated information more efficiently. Combined with the local delta correction term, this gated update provides both precise associative correction and adaptive long-term memory control.

\item \textbf{Output Generation.} The final output is retrieved from the updated memory: $\mathbf{o}_t = \mathbf{S}_t \mathbf{q}_t$.
\end{enumerate}

In Kairos, GLA plays the role of persistent global causal memory. It propagates information related to object permanence, task progress, delayed physical effects, scene-level context, and long-range dependencies across the temporal extent of the video or world-action trajectory. This global pathway is especially important when a future prediction depends on an event outside the recent local window: an object temporarily hidden behind another may become relevant later; a prior failed grasp may change the next action; a delayed instability may only become visible after several seconds. These are precisely the cases where purely local attention is insufficient. GLA should therefore not be described merely as a computational shortcut; it is a \emph{control-state memory mechanism}.

Importantly, GLA serves as the \emph{only} global attention mechanism in the backbone. All other self-attention layers are restricted to local temporal neighborhoods. This architectural choice enforces a clear separation of responsibilities: local attention handles fine-grained motion and interactions, while GLA is responsible for global temporal consistency and causal structure.

\subsubsection{Local and Mid-Range Pathways: Sliding-Window and Dilated Sliding-Window Attention}

\textbf{Sliding Window Attention (SWA).} For a sequence of hidden states $\mathbf{x} \in \mathbb{R}^{B \times (F\cdot L) \times D}$, where $F$ is the number of frames and $L$ the number of tokens per frame, SWA restricts attention for query index $i$ to keys/values within a local window $w$:
\begin{equation}
\mathrm{SWA}(\mathbf{Q}, \mathbf{K}, \mathbf{V})_i = \sum_{j \in [i-\tfrac{w}{2},\, i+\tfrac{w}{2}]} \mathrm{Softmax}\!\left(\tfrac{\mathbf{Q}_i \mathbf{K}_j^{\top}}{\sqrt{d}}\right)\mathbf{V}_j.
\end{equation}
The window size is chosen as $w = L \times \text{window\_size}$ to cover a small number of adjacent frames and their spatial tokens. This local structure is appropriate for fast physical dynamics: motion continuity, short-range hand--object interaction, immediate contact changes, and local geometric consistency.

\textbf{Dilated Sliding Window Attention (DSWA).} To expand the temporal receptive field without quadratic cost, Kairos incorporates DSWA, which applies a dilation factor $d$ along the temporal dimension:
\begin{equation}
\mathrm{DSWA}(\mathbf{Q}, \mathbf{K}, \mathbf{V}) = \mathrm{SWA}\!\left(\mathrm{rearrange}(\mathbf{Q}), \mathrm{rearrange}(\mathbf{K}), \mathrm{rearrange}(\mathbf{V})\right),
\end{equation}
where the input is reshaped from $(B, F\cdot L, D)$ to $(B\cdot d, \tfrac{F}{d}\cdot L, D)$. By interleaving SWA $(d=1)$ and DSWA $(d \in \{6,12\})$ blocks, the backbone progressively aggregates information across local and mid-range timescales: SWA captures frame-adjacent interaction; DSWA captures event-level continuity and delayed local dependencies; GLA preserves global context. All sliding-window attention blocks use RoPE-based relative positional encoding to preserve local temporal and spatial geometry, while global positional and causal information is delegated to the GLA pathway.

\textbf{Expandability through Modular Hybrid Attention.} The architectural decoupling in Kairos provides a robust foundation for multi-dimensional expansion:

\begin{itemize}

\item \textbf{Interactive World Modeling.} The Gated Linear Attention (GLA) state matrix $\mathbf{S}_t$ acts as a compressed latent memory. Action tokens can be injected into the GLA gating mechanism or the latent state update, enabling low-overhead, latency-aware state updates for future closed-loop control.

\item \textbf{Long-Horizon Generation.} Unlike Softmax-based models constrained by a fixed context window, the SWA and DSWA components maintain a constant memory footprint per step, while the GLA compresses historical context. This enables generation through recurrent state passing, mitigating the memory bottlenecks typical of long-video diffusion.

\item \textbf{Cross-Modal Integration.} The modular block design supports cross-modal integration through additional streams in the hybrid blocks, where local spatial-temporal features are fused via SWA/DSWA, and global cross-modal alignment is maintained by the GLA pathway.

\end{itemize}

\subsection{Theoretical Scope and Analysis of Hybrid Multi-Scale Temporal Memory}\label{sec:theory_intro}

The purpose of the theoretical analysis is to clarify why persistent temporal memory is necessary for long-horizon world modeling and why a hybrid decomposition into local, mid-range, and global pathways is a reasonable architectural choice. The analysis should not be read as a universal guarantee of real-world robotic performance, closed-loop regret reduction, or perfect physical correctness. Instead, it provides theoretical support for a narrower but important claim: \emph{when future targets depend on information outside a bounded recent window, purely local temporal mechanisms are insufficient; under stated assumptions, a hybrid multi-scale memory can bound long-horizon prediction error in terms of branch-wise approximation error and global-memory perturbation.} In this section we state the core theorems that motivate the design; full proofs appear in Appendix~\ref{sec:theory}.

\textbf{Problem setup.} We model the available interaction stream as a discrete-time partially observed controlled process
\begin{equation}
    \{(O_t,A_t)\}_{t\ge 1},
\end{equation}
where $O_t\in\mathcal O$ denotes the observable input available to the model, and $A_t\in\mathcal A$ denotes the action signal that can influence future evolution. Let
\begin{equation}
    \mathcal{H}_t
    =
    \sigma\big(O_{1:t},\, A_{1:t-1}\big)
\end{equation}
denote the complete information available up to time $t$, and for $1\le w<t$, let
\begin{equation}
    \mathcal{W}_t^{(w)}
    =
    \sigma\big(O_{t-w+1:t},\, A_{t-w:t-1}\big)
\end{equation}
denote the information contained in the recent $w$-step window. Here, $\sigma(\cdot)$ denotes the sigma-field generated by the enclosed random variables, i.e., the collection of all events or information measurable from the corresponding observation--action history. Thus, $\mathcal{H}_t$ and $\mathcal{W}_t^{(w)}$ are information sets rather than raw tuples of observations and actions. Let $Y$ be a square-integrable future target. In Physical AI, $Y$ may represent a future latent-frame coordinate, an object-permanence indicator, a delayed physical-effect event, a task-progress variable, a failure event, or any other long-horizon functional of the future world state. For any predictor $Z$, define the squared prediction risk
\begin{equation}
    \mathcal{R}_t(Z) = \mathbb{E}\!\left[(Y - Z)^2\right].
\end{equation}
Let $R_{\text{full}}^{\star} = \inf_{Z \in L^2(\mathcal{H}_t)} \mathcal{R}_t(Z)$ be the optimal risk among predictors that can access the full history, and let $R_{w}^{\star} = \inf_{Z \in L^2(\mathcal{W}_t^{(w)})} \mathcal{R}_t(Z)$ be the optimal risk among predictors restricted to the recent $w$-step window. This setup formalizes the difference between a model with persistent history and a model restricted to a bounded local context.

\subsubsection{Necessity of Persistent Internal States}\label{sec:theory_necessity}

The first theoretical observation is that purely local temporal models are fundamentally limited when future targets depend on information outside the recent window.

\begin{theorem}[Supra-window dependence implies the necessity of persistent state]
\label{thm:necessity}
Let $m_t = \mathbb{E}[Y \mid \mathcal{H}_t]$ and $m_t^{(w)} = \mathbb{E}[Y \mid \mathcal{W}_t^{(w)}]$. The excess risk incurred by restricting prediction to the recent window satisfies the exact identity
\begin{equation}
    R_{w}^{\star} - R_{\text{full}}^{\star} \;=\; \mathbb{E}\!\left[(m_t - m_t^{(w)})^2\right] \;=\; \mathbb{E}\!\left[\mathrm{Var}\!\left(m_t \mid \mathcal{W}_t^{(w)}\right)\right].
\end{equation}
Consequently,
\begin{equation}
    R_{w}^{\star} > R_{\text{full}}^{\star} \quad\Longleftrightarrow\quad m_t \text{ is not } \mathcal{W}_t^{(w)}\text{-measurable}.
    \label{eq:iff_nonmeasurable}
\end{equation}
That is, the excess risk is strictly positive if and only if the optimal full-history predictor $m_t$ is not perfectly recoverable from the recent window $\mathcal{W}_t^{(w)}$.
\end{theorem}

\begin{corollary}[Explicit lower bound under recent-window mismatch]
\label{cor:lower_bound}
Let $E$ denote a recent-window observation event with $\mathbb{P}(E) > 0$. If an influential past event remains unobservable within $E$, producing two distinct future-target expectations $\mu_1$ and $\mu_2$ with conditional probabilities $\alpha$ and $1-\alpha$, then
\begin{equation}
    R_{w}^{\star} - R_{\text{full}}^{\star} \;\geq\; \mathbb{P}(E)\,\alpha\,(1-\alpha)\,(\mu_1 - \mu_2)^2.
\end{equation}
\end{corollary}

\begin{remark}[Local smoothness is insufficient for long-horizon consistency]
\normalfont
Theorem~\ref{thm:necessity} explains why local temporal smoothness may remain visually plausible over short horizons but fail over longer horizons. Once an influential event falls outside the current context window, the model can no longer condition on the true historical cause; it must average over multiple plausible hidden histories. The identity of an occluded object, the result of a previous subtask, or the state of a prior contact interaction may no longer be visible, but may still determine the future. The key point is that this lower bound is \emph{information-theoretic}, not an optimization failure: it does not arise because the model is too small or insufficiently trained, but because the relevant information is absent from the accessible window. Simply scaling parameters or training compute cannot eliminate the gap when the architecture lacks a mechanism to preserve supra-window information.
\end{remark}

\subsubsection{Approximate Sufficiency of Hybrid Multi-Scale Temporal Memory}\label{sec:theory_sufficiency}

The second theoretical observation concerns the sufficiency of the hybrid design under a structured assumption. Suppose the Bayes-optimal predictor $\mu_t^{\star}$ admits a four-component decomposition
\begin{equation}
    \mu_t^{\star} \;=\; \Psi\!\left(U_t^{\star},\, C_t^{\star},\, D_t^{\star},\, G_t^{\star}\right),
\end{equation}
where $U_t^{\star}$ is a shared predictive representation, $C_t^{\star}$ is a short-range local state corresponding to SWA, $D_t^{\star}$ is a mid-range dilated state corresponding to DSWA, and $G_t^{\star}$ is a global recurrent causal memory corresponding to GLA. Let the learned hybrid predictor $\hat{\mu}_t$ approximate these components with branch-wise approximation error bounded by $\varepsilon$. Suppose further that the global-memory update is contractive with factor $\rho \in (0,1)$, and let $\bar{\xi}$ denote the maximum one-step perturbation error in the global-memory update.

\begin{theorem}[Approximate sufficiency of a hybrid multi-scale temporal memory]
\label{thm:sufficiency}
Under the assumptions above, the learned hybrid predictor satisfies the asymptotic long-horizon excess-risk bound
\begin{equation}
    \mathcal{R}_t(\hat{\mu}_t) - \mathcal{R}_t^{\star} \;\leq\; \left(\, L\,\varepsilon \;+\; \frac{L_G\,\bar{\xi}}{1-\rho}\,\right)^{2}
    \quad \text{as } t \to \infty,
\end{equation}
where $L$ and $L_G$ are Lipschitz constants associated with the decoder and the global-memory pathway.
\end{theorem}

\begin{corollary}[Exact sufficiency in the realizable case]
\label{cor:exact_sufficiency}
If the learned hybrid state exactly recovers the Bayes decomposition at every time step ($\varepsilon = 0$ and $\bar{\xi} = 0$), then $\hat{\mu}_t = \mu_t^{\star}$ and $\mathcal{R}_t(\hat{\mu}_t) = \mathcal{R}_t^{\star}$.
\end{corollary}

\begin{remark}[Why hybrid multi-scale memory is sufficient]
\normalfont
The first term $L\,\varepsilon$ reflects the approximation quality of the shared, local, and mid-range branches. The second term $L_G\,\bar{\xi}/(1-\rho)$ reflects the asymptotic contribution of global-memory perturbation. Crucially, because the gated delta update is contractive, the global-memory error does not accumulate arbitrarily but satisfies
\begin{equation}
    e_t \;\leq\; \rho^{t}\,e_0 \;+\; \frac{1 - \rho^{t}}{1 - \rho}\,\sup_{1 \leq i \leq t}\xi_i,
\end{equation}
yielding $e_t \leq \bar{\xi}/(1-\rho)$ as $t \to \infty$. This geometric damping ensures that one-step perturbations are strictly bounded rather than amplified over time. Architecturally, the local and mid-range branches efficiently capture local appearance changes and intermediate temporal structures, while GLA selectively propagates the persistent causal state under bounded drift. GLA therefore functions as a stable information bottleneck: it preserves essential long-range context without accumulating compounding errors, enabling consistent generation over extended horizons.
\end{remark}

\subsubsection{Interpretation for Kairos}\label{sec:theory_interpretation}

The theoretical analysis supports three architectural conclusions.

\textbf{First, persistent memory is unavoidable} for long-horizon world modeling whenever relevant information can fall outside a recent context window. This justifies the need for a recurrent or persistent state pathway in Kairos. Without such a pathway, long-horizon consistency can fail even if short-horizon visual quality remains strong.

\textbf{Second, separating temporal responsibilities} into short-range (SWA), mid-range (DSWA), and global (GLA) branches is theoretically motivated. Local attention is appropriate for fast motion and contact dynamics; dilated attention is appropriate for mid-range event dependencies; gated global memory is appropriate for supra-window causal context. This maps naturally onto the Physical AI requirement of multi-timescale control-state maintenance.

\textbf{Third, the theory clarifies the scope.} Under the stated assumptions, the proposed factorization provides a way to bound long-horizon prediction error in terms of approximation error and controlled memory perturbation. This is a meaningful theoretical justification for the architecture; complementary real-robot validation is positioned as future work. The full proofs appear in Appendix~\ref{sec:theory}.

\subsubsection{From Theoretical State Maintenance to Regret-Aware Physical AI}\label{sec:theory_regret_link}

The connection between this theoretical analysis and the broader goal of Kairos is \emph{control-sufficient state maintenance}. A Physical AI system needs to preserve information that affects future action outcomes. Some of this information is local and fast-changing, such as contact and motion. Some is mid-range, such as object interaction history and subtask progress. Some is global and persistent, such as object permanence, delayed effects, accumulated failures, and task-level context. The hybrid memory structure provides an architectural mechanism for maintaining these different kinds of information.

From a regret-aware perspective, the purpose of persistent memory is not to remember the past for its own sake. It is to preserve information that can reduce future costly mistakes. A previous failure, a hidden object state, a delayed contact effect, or an unfinished subtask may determine whether the next action succeeds or fails. If the model forgets this information, its predictions may remain visually plausible but become control-irrelevant. Long-horizon state maintenance is therefore a necessary model-side prerequisite for future closed-loop regret reduction.

In summary, Section~\ref{sec:model} establishes the model-side foundation of Kairos. World Understanding constructs a control-sufficient state. World Generation regularizes and probes that state through physically plausible future imagination. World Prediction turns the state into a world-action prediction interface. Hybrid Linear Temporal Attention maintains the state across local, mid-range, and global temporal scales. Theoretical analysis explains why persistent memory is necessary and why hybrid factorization is well-motivated under stated assumptions. Together, these components define Kairos as a regret-aware world-action model stack for Physical AI, while leaving direct real-world closed-loop regret minimization as a future validation objective.

\section{Native Pretraining Paradigm for Physical AI}\label{sec:pretraining}

The purpose of native pretraining in Kairos is to provide the information required for constructing a control-sufficient state $Z_t$. Section~\ref{sec:model} described how Kairos compresses heterogeneous observations into $Z_t$, regularizes this state through future generation, turns it into a world-action prediction interface, and maintains it through hybrid temporal memory. This section addresses the complementary question: \emph{where does the control-relevant information in $Z_t$ come from?}

For Physical AI, pretraining cannot be reduced to scaling generic videos or fine-tuning a pretrained video generator on a small robot dataset. A world-action model needs to acquire broad physical priors, intentional task structure, and robot-specific action grounding. These three forms of knowledge are not equally available in the same data source. Open-world videos are abundant and diverse, but they mostly capture passive observation. Human-centric behavioral data contains rich intentional interaction, but it does not directly match robot embodiment. Robot interaction data provides action--outcome grounding, but it is expensive, narrow, and difficult to scale. A flat mixture of these data sources is therefore insufficient. It treats fundamentally different intervention regimes as if they were interchangeable samples.

Kairos adopts a different view. Native pretraining is organized as a curriculum over intervention strength. The model first learns passive physical regularities from open-world video, then learns intentional task organization from human-centric behavior, and finally grounds these priors in robot-specific action trajectories. This progression defines the \textbf{Cross-Embodiment Data Curriculum (CEDC)}. Its goal is not merely to expose the model to more data, but to move the model from observation--action correlation toward action--outcome causation.

This design is motivated by a central property of Physical AI: \emph{deployment is interventional}. In ordinary supervised learning, training and test samples are often assumed to come from similar distributions. In robotics, this assumption breaks down because the agent's own actions change the future data distribution. A robot does not simply observe the world; it intervenes in it. Therefore, the statistical challenge is not only i.i.d.\ generalization, but \emph{interventional generalization}. A model trained only on passive videos may learn what usually happens next, but not what happens if the robot acts. A model trained only on narrow robot data may learn action execution in a limited setting, but not general physical and semantic regularities. CEDC provides a developmental pathway between these extremes.

We define the three data regimes as
\begin{equation}
    \mathcal{D}_{\text{CEDC}} = \mathcal{D}_{\text{obs}} \,\cup\, \mathcal{D}_{\text{human}} \,\cup\, \mathcal{D}_{\text{robot}},
\end{equation}
where $\mathcal{D}_{\text{obs}}$ denotes open-world observational videos, $\mathcal{D}_{\text{human}}$ denotes human-centric behavioral and ego-centric interaction data, and $\mathcal{D}_{\text{robot}}$ denotes robot interaction data with action, proprioception, tactile, force, or other embodiment-specific signals when available. These regimes can be ordered by intervention strength:
\begin{equation}
    \tau(\mathcal{D}_{\text{obs}}) \;<\; \tau(\mathcal{D}_{\text{human}}) \;<\; \tau(\mathcal{D}_{\text{robot}}),
\end{equation}
where $\tau(\cdot)$ denotes how directly the data expresses an agent's intervention on the world. Open-world videos mainly express passive dynamics. Human-centric data expresses intentional intervention. Robot data expresses embodied intervention with concrete action spaces.

Under this formulation, native pretraining is not simply a data scaling recipe. It is a \textbf{control-information acquisition strategy}. The model first builds broad physical priors, then learns task-level intentional structure, and finally aligns these priors with robot actions. This sequence is designed to produce a world-action state $Z_t$ that is not merely visually rich, but increasingly control-sufficient.

A second principle is \emph{control information density}. The value of data for Physical AI should not be measured only by raw scale. Data is valuable when it reduces uncertainty about action consequences, failure boundaries, contact dynamics, recovery strategies, and safety risks. Therefore, a small amount of near-boundary failure, recovery, marginal success, or contact-rich data can be more valuable for control than a large amount of ordinary successful video. CEDC uses scale where scale is useful, but it should ultimately be guided by control information density rather than volume alone. In this report, the current curriculum establishes the main developmental pathway; future iterations can further prioritize high-density control data within each stage.

The native pretraining pipeline is implemented in three progressive stages:
\begin{itemize}
    \item \textbf{Stage~I: Physical Pretraining.} Trains the Video DiT backbone on large-scale open-world image and video data to acquire broad spatial--temporal and physical priors.
    \item \textbf{Stage~II: Embodied Pretraining with Human-centric Data.} Adapts the Video DiT with human-centric and robot-centric visual data, strengthening task structure, instruction following, multi-view embodiment awareness, and action-relevant visual dynamics without yet requiring full robot action grounding.
    \item \textbf{Stage~III: Regret-Aware World-Action Training.} Introduces robot action trajectories and execution preference pairs with high control information density, then jointly trains the Action DiT together with the pretrained Video DiT so that visual world dynamics, executable action prediction, and regret alignment training are connected in a native world-action model.
\end{itemize}
The pipeline therefore moves from physical priors, to task-structured embodied visual dynamics, to robot-grounded world-action alignment through regret alignment training.

\subsection{Native Pretraining with Cross-Embodiment Data Curriculum}\label{sec:cedc}

The Cross-Embodiment Data Curriculum is the structural backbone of Kairos pretraining (Fig.~\ref{fig:data_pyramid}). Its central idea is that different data sources contain different forms of control-relevant information. \emph{Open-world videos} (Phase~I, on the order of millions of hours of internet-scale clips) provide broad physical regularities, including object motion, gravity, collision, deformation, fluid motion, human--object interaction, and scene evolution. \emph{Human-centric data} (Phase~II, on the order of $10^5$ hours of human behavioral and ego-centric data) provides intentional behavior, including task organization, tool use, object manipulation, recovery behavior, and multi-step procedural structure. \emph{Robot data} (Phase~III, including embodiment-specific interaction datasets such as AgiBotWorld-Beta~\cite{AgiBotWorldContributors2025AgiBotWC} and Droid~\cite{khazatsky2024droid}) provides embodiment-specific grounding, including action tokens, proprioception, gripper state, tactile or force cues, actuation constraints, execution errors, and sensorimotor feedback.

\begin{figure}[t]
    \centering
    \includegraphics[width=0.9\linewidth]{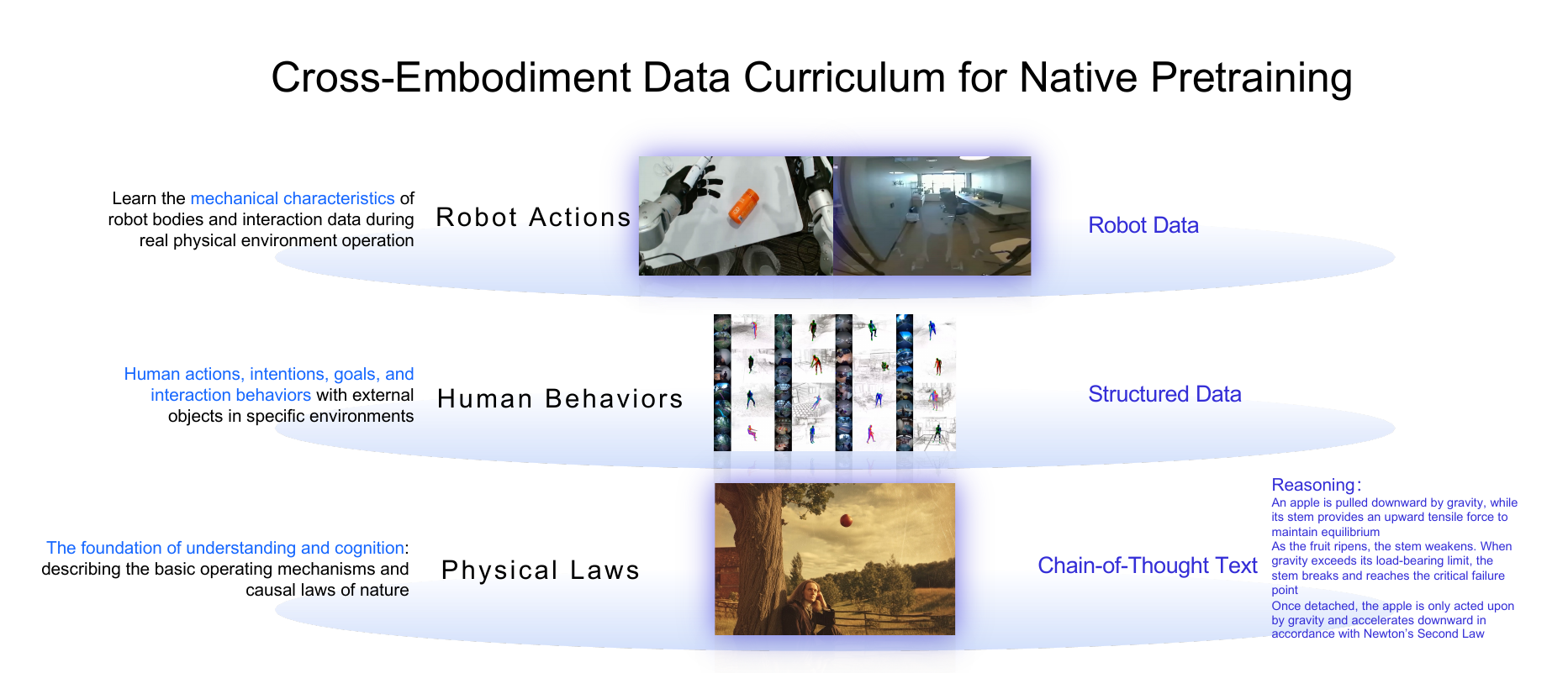}
    \caption{Cross-Embodiment Data Curriculum. A flat mixture obscures the differences between passive observation, intentional behavior, and embodied action; CEDC treats these regimes as a developmental pathway over intervention strength.}
    \label{fig:data_pyramid}
\end{figure}

A flat data mixture would obscure these differences. It may improve visual diversity but fail to create a coherent transition from observation to intervention. CEDC instead treats the data sources as a developmental pathway $\mathcal{D}_{\text{obs}} \to \mathcal{D}_{\text{human}} \to \mathcal{D}_{\text{robot}}$. This progression has three functions. First, it allows the model to acquire broad physical priors before being exposed to narrow robot-specific distributions. Second, it allows the model to learn intentional task structure from human behavior before mapping behavior to robot embodiment. Third, it allows robot data to ground the previously learned priors in concrete action spaces, reducing the gap between visual forecasting and physical execution.

The curriculum can also be understood as a staged approximation to control-sufficient state learning. In Stage~I, the model learns a state useful for predicting passive physical evolution. In Stage~II, the state becomes more task-, instruction-, and embodiment-sensitive at the visual level. In Stage~III, the state becomes robot-action-aware and is further shaped by regret-relevant execution preferences. The final objective is not to make the model remember all details of every training video, but to make $Z_t$ preserve variables relevant to control: object state, contact condition, task progress, action consequence, and failure risk.

This formulation also explains why native pretraining is preferable to simple downstream fine-tuning. If a generic video generator is trained only to synthesize visually plausible futures, then later robot fine-tuning must retrofit embodiment grounding onto a representation that may not preserve the right control variables. In contrast, Kairos exposes the model to physical regularities, intentional behavior, and robot actions within a unified developmental process; the pretraining process is designed to establish the model-side prerequisites for future regret-aware Physical AI.

\textbf{Multi-Stage Native Pre-training Pipeline.} To operationalize CEDC within the Mixture-of-Transformers (MoT) architecture (Section~\ref{sec:world_prediction}), the native pretraining pipeline is systematically structured as three progressive stages, each dominated by its corresponding data layer. This multi-stage optimization decouples broad physical prior learning from later robot action grounding, improving both scaling efficiency and control fidelity. Specifically, Stage~I (Physical Pretraining) and Stage~II (Embodied Pretraining with Human-centric Data) focus on optimizing the Video DiT component. In these phases, the model internalizes open-world physical dynamics, task structure, instruction-conditioned behavior, and embodiment-aware visual dynamics through dense video forecasting tokens, optimizing the unified spatial--temporal representation without low-level action-space interference. Transitioning to Stage~III (Regret-Aware World-Action Training), the training emphasis shifts toward joint optimization of the Action DiT alongside the pretrained Video DiT, together with cost-sensitive execution preference supervision from robot experience with high control information density. By injecting low-level robot trajectories, robot-state signals, and regret-relevant execution comparisons, Stage~III aligns visual forecasting with executable action prediction within the same world-action stack.

\textbf{Training Objective.} Across stages, Kairos adopts Flow Matching~\cite{lipman2023flowmatching,rectifiedflow} as the primary generative training objective. Flow Matching provides a continuous-time formulation for learning a conditional velocity field that transports samples from a simple prior distribution to the data distribution in latent space. This objective is suitable for large-scale image and video generation, while also being extensible to robot action trajectories in Stage~III. Let $\mathbf{z}_0 = \mathcal{E}(\mathbf{x})$ denote a clean latent video sample produced by the video VAE encoder, and let $c$ denote conditioning inputs (text, image, instruction, camera control, robot state, or other multimodal conditions). Sample Gaussian noise $\boldsymbol{\epsilon} \sim \mathcal{N}(0, \mathbf{I})$ with the same shape as $\mathbf{z}_0$, and define a continuous interpolation $t \in [0,1]$. Under the rectified-flow parameterization, the noisy latent is
\begin{equation}\label{eq:fm_interpolation}
    \mathbf{z}_t = (1-t)\,\mathbf{z}_0 + t\,\boldsymbol{\epsilon},
\end{equation}
so that $\mathbf{z}_t$ continuously interpolates from a clean sample at $t=0$ to pure noise at $t=1$. The ground-truth velocity along this path is
\begin{equation}
    \mathbf{u}_t = \boldsymbol{\epsilon} - \mathbf{z}_0.
\end{equation}
Kairos trains a conditional velocity predictor $v_\theta$ by minimizing
\begin{equation}
    \mathcal{L}_{\text{FM}} = \mathbb{E}_{t, \mathbf{z}_0, \boldsymbol{\epsilon}, c}\!\left[\big\lVert v_\theta(\mathbf{z}_t, t, c) - \mathbf{u}_t \big\rVert_2^2\right].
\end{equation}
This objective is shared across conditioning modes. In Stage~I, $c$ may include text, image, or unconditional generation conditions. In Stage~II, $c$ increasingly includes task instructions, human-centric behavior context, robot-centric visual context, and multi-view conditions. In Stage~III, $c$ includes robot states and action-related signals, and the same Flow-Matching principle is extended to action trajectories.

From the control-sufficient perspective, the Flow-Matching objective should be interpreted carefully. It does not directly optimize regret. It provides a scalable learning objective through which the model can acquire physical, semantic, temporal, and action-related priors. The regret-aware aspect comes from how the data, conditioning, temporal structure, and action grounding are organized around control-relevant variables.

\begin{figure}[t]
    \centering
    \captionsetup{justification=raggedright,singlelinecheck=false}
    \includegraphics[width=0.6\linewidth]{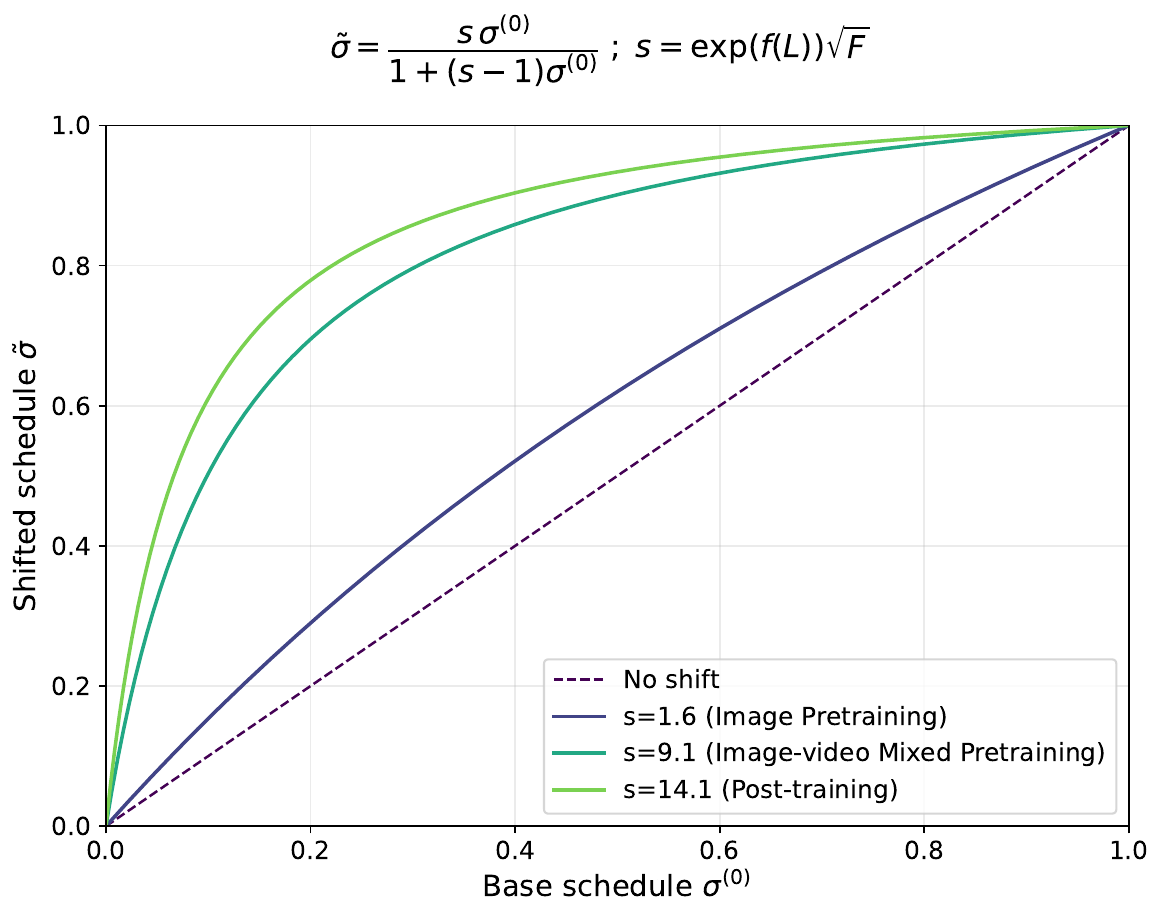}
    \caption{\textbf{Shape-aware exponential timestep shifting curves in Kairos across training stages.} This figure shows the $\sigma$-domain remapping of the Kairos scheduler for three representative training stages, which will be introduced in the following sections. As the resolution becomes higher and the video becomes longer, the effective shift strength $s$ increases, resulting in a stronger upward remapping of the timestep schedule.}
    \label{fig:kairos-shift-curves-stages}
\end{figure}

\subsection{Stage I: Physical Pretraining}\label{sec:stage_i}

Stage I establishes the passive physical foundation of Kairos. Its purpose is to train the Video DiT backbone to model broad spatial--temporal regularities from large-scale open-world video and image data. At this stage, the model is not yet expected to understand robot embodiment or produce executable actions. It functions primarily as a \emph{physical observer}: it learns how scenes evolve, how objects move, how interactions unfold, and how visual dynamics remain coherent across time.

The key role of Stage I is to build broad physical priors: motion continuity, object permanence, gravity-consistent movement, collision patterns, support relations, deformation, fluid-like motion, camera motion, and temporal coherence. These priors are not yet sufficient for action grounding, but they provide the world-dynamics substrate that later stages can reuse. Without such priors, robot data alone would be too narrow to support generalization across scenes, objects, and tasks.

Stage I relies on two core strategies:
\begin{itemize}
    \item \textbf{Web-Scale Physical Prior Injection.} Open-world images and videos expose the model to a wide variety of physical phenomena. This data helps the model learn the statistical structure of the visual world and the broad regularities of natural temporal evolution. The model learns broad physical priors useful for future world generation and joint world-action prediction.

    \item \textbf{Progressive Training Strategy.} Training begins with spatial--semantic pretraining on images and then gradually moves toward video sequences with increasing resolution and temporal length. This progression reduces optimization difficulty and allows the model to establish strong spatial representations before learning long-horizon temporal dynamics. In practical terms, Stage~I proceeds through image pretraining, image--video mixed pretraining, continual video pretraining, domain-specific supervised fine-tuning or model merging, and preference-based refinement where applicable.
\end{itemize}

A representative training progression is summarized in Table~\ref{tab:stage_i_progression}.

\begin{table}[h]
\centering
\small
\begin{tabular}{lllc}
\toprule
Task Type & Stage & Resolution & Maximum Frames \\
\midrule
T2I / N2I & Image Pretraining                 & 256P & 1 \\
T2V / TI2V / I2V / N2V & Video Pretraining     & 256P & 81 \\
T2V / TI2V / I2V / N2V & Video Pretraining     & 480P & 81 \\
T2V / TI2V / I2V / N2V & Video Pretraining     & 720P & 81 \\
T2V / TI2V / I2V / N2V & Continual Training    & 720P & 241 \\
T2V / TI2V / I2V / N2V & Domain SFT / Merging  & 720P & 241 \\
T2V / TI2V / I2V / N2V & Preference / RL Refinement & 720P & 241 \\
\bottomrule
\end{tabular}
\caption{Representative Stage~I training progression. T2I: text-to-image; N2I: unconditional image; T2V, TI2V, I2V, N2V: video counterparts. Image tasks are special cases of video tasks with one frame.}
\label{tab:stage_i_progression}
\end{table}

This workflow is structured into three key phases of progressive physical pretraining:
\begin{itemize}
    \item \textbf{Image Pretraining.} Initial T2I pretraining establishes robust spatial--semantic priors, allowing the model to focus exclusively on temporal consistency and motion dynamics in subsequent stages. To optimize computational efficiency, this foundational phase is conducted at a reduced resolution of 256P, facilitating rapid convergence of core visual features.

    \item \textbf{Image--Video Mixed Pretraining.} We initiate video training using a progressive resolution strategy, scaling from 256P to 720P to optimize computational efficiency. To maintain spatial fidelity, we incorporate a small fraction of image data (at a 10\% sampling ratio) into the pipeline. By jointly training on both static images and video sequences, the model effectively bridges the gap between spatial structures and motion dynamics, fostering temporal coherence while preventing the catastrophic forgetting of high-quality spatial details.

    \item \textbf{Continual Pretraining.} The model undergoes continued pretraining on a meticulously curated, high-quality dataset to refine its aesthetic appeal and capture more nuanced, long-range temporal dependencies for superior video generation. A key advancement in this phase is extending the maximum temporal length of training sequences from 81 to 241 frames, enabling the model to directly synthesize high-fidelity videos of up to 15 seconds.
\end{itemize}

This progression has a control-sufficient interpretation. Image pretraining gives the model spatial--semantic priors. Image--video mixed pretraining connects spatial structure with motion dynamics. Continual long-video pretraining improves temporal state maintenance. Domain-specific fine-tuning introduces higher-density physical or embodied scenarios. Preference or DPO-like refinement can be used to improve human preference, physical plausibility, and instruction following. The overall purpose is to make the Video DiT backbone a stronger state regularizer for $Z_t$.

Stage~I also benefits from structured textual supervision. Ordinary captions describe visible content, but physical world modeling requires more. Physics-centric captions and chain-of-thought annotations can expose motion trajectories, causal transitions, collisions, force-like effects, gravity-related behavior, object permanence, and failure-like events. They do not make the model a perfect physics engine, but they provide weak supervision for organizing physical variables in latent space.

\textbf{Timestep Scheduler Distribution Shift.} The effective timestep distribution induced by a fixed scheduler changes with latent spatiotemporal shape: longer videos and higher resolutions create different denoising difficulties than short clips or low-resolution videos. Kairos therefore applies a shape-aware timestep distribution shift~\cite{sd3,hunyuanvideo,opensora2,wan2025} to the scheduler-defined $\sigma$ sequence, as visualized in Figure~\ref{fig:kairos-shift-curves-stages}. Formally, let $\{\sigma_i^{(0)}\}_{i=1}^{N}$ denote a base schedule before shifting, where $\sigma_i^{(0)} \in (0,1)$ is monotonically ordered. The default exponential shift is
\begin{equation}
\tilde{\sigma}_i = \frac{s\,\sigma_i^{(0)}}{1 + (s-1)\sigma_i^{(0)}},
\label{eq:kairos_exp_shift_rational}
\end{equation}
where $s$ controls the shift strength. An equivalent form in logit space is
\begin{equation}
\tilde{\sigma}_i = \mathrm{sigmoid}\!\left(\mathrm{logit}\!\left(\sigma_i^{(0)}\right) + \log s\right),
\label{eq:kairos_exp_shift_logit}
\end{equation}
which makes the monotonicity of the transformation explicit and preserves the boundary behavior of the schedule. For adaptive support of varying video lengths and resolutions in a unified model, Kairos uses a shape-dependent rule for $s$. With $F$ the number of latent frames and $L = H \times W$ the number of latent tokens per frame, we set
\begin{equation}
    s = \exp(f(L))\,\sqrt{F},
\end{equation}
where $f(L) = mL + b$ linearly maps $L \in [L_{\min}, L_{\max}]$ to a predefined range $[r_{\min}, r_{\max}]$, with
\begin{equation}
m = \frac{r_{\max}-r_{\min}}{L_{\max}-L_{\min}}, \qquad b = r_{\min} - mL_{\min}.
\label{eq:kairos_mu_coeffs}
\end{equation}
This design increases the effective shift for larger latent spatial token counts and longer videos, thereby reallocating scheduler steps toward trajectory regions that are more sensitive to prediction error.

Throughout the pretraining phase, we consistently employ the AdamW optimizer~\cite{loshchilov2019decoupledweightdecayregularization} with phase-specific hyperparameter configurations. Specifically, the learning rate is set to $5\!\times\!10^{-5}$ for image pretraining and is then progressively decayed to $4\!\times\!10^{-5}$, $3\!\times\!10^{-5}$, $2\!\times\!10^{-5}$, and $1\!\times\!10^{-5}$ across the successive image--video mixed pretraining (256P, 480P, 720P) and continual pretraining stages. Weight decay is configured at $10^{-3}$ during the image and 256P mixed pretraining stages, and is deactivated (set to $0$) for all subsequent high-resolution and continual phases.

In addition, we implement delicate finetuning techniques to further boost the performance of physical alignment, which consist of three phases: domain-specific supervised fine-tuning (SFT), model merging, and reinforcement learning.
\begin{itemize}
\item \textbf{Domain-specific SFT and Model Merging.} We partition high-quality datasets into domains and train domain-specific models independently. To fully leverage the strengths of individual models, we employ a model merging strategy to integrate their features~\cite{tang2025fusionbench,yang2026ModelMergingSurvey}, exploring multiple merging approaches including Model Soup~\cite{wortsman2022modelsoupsaveragingweights}, CART~\cite{choi2024revisiting}, TIES~\cite{yadav2023tiesmergingresolvinginterferencemerging}, DARE~\cite{yu2024languagemodelssupermario}, and WUDI-Merging~\cite{cheng2025startedinterferenceendit}. By analyzing the performance of merged models on carefully constructed evaluation subsets, we select the configuration that achieves the most balanced performance across all domains.

\item \textbf{Reinforcement Learning.} Finally, we apply Direct Preference Optimization (DPO)~\cite{wallace2023diffusion,rafailov2024directpreferenceoptimizationlanguage,liu2025videodpo} to align model outputs with human preferences and physical plausibility. The DPO preference pairs are constructed by generating video candidates from multiple generative models conditioned on identical prompts; the highest- and lowest-scoring samples form (chosen, reject) pairs, which are then used to fine-tune the model following the DPO objective.
\end{itemize}

For this stage, the AdamW optimizer is retained, with learning rates tailored to $1\!\times\!10^{-5}$ for SFT and $1\!\times\!10^{-6}$ for RL. To maintain training stability, weight decay is uniformly set to $0$ throughout the entire process.

From the CEDC perspective, Stage~I mainly learns passive world evolution and broad physical common sense. Direct knowledge of how a robot's own actions change the world is built on top in Stages~II and~III, with Stage~I providing the broad Video DiT foundation.

\subsection{Stage II: Embodied Pretraining with Human-centric Data}\label{sec:stage_ii}

Stage II moves the curriculum from passive physical observation to intentional behavior. Open-world videos expose how the world changes, but they often do not reveal task goals, action intent, or structured manipulation. Human-centric data provides this missing layer: goal-directed actions, tool use, object manipulation, task ordering, recovery behavior, and long-horizon procedural structure. These signals are crucial for Physical AI because many robotic tasks are not just physical transitions; they are intentional sequences organized around goals.

The objective of Stage II is to make the Video DiT backbone \emph{task-sensitive and instruction-aware} before full robot action grounding. This phase bridges the gap between passive observation and active robotic execution through two primary paradigms:
\begin{itemize}
    \item \textbf{Task-Structured Semantic Injection.} By leveraging large-scale, human-centric behavioral datasets (e.g., intentional actions, tool manipulation, and everyday chores), the model transitions from unconditioned video generation to task-structured video forecasting, internalizing high-level behavioral semantics and task taxonomies.

    \item \textbf{Video DiT Optimization for Behavioral Causality.} The training execution remains focused exclusively on the Video DiT component. By predicting intentional human movements and scene state transitions, Video DiT learns to construct a robust causal representation of goal-directed actions and their environmental consequences, without yet binding to a specific robotic action space.
\end{itemize}
This stage does not yet solve robot control. Instead, it provides transferable action-relevant priors that can later support Action DiT training.

Stage II uses human-centric data from first-person and third-person perspectives. First-person data is especially valuable because it captures hand--object interaction from an embodied viewpoint, making it closer to robot observation. Third-person data provides richer context about whole-body actions, object relations, and task-level structure. Robot-centric visual data can also be included in this stage to expose the model to diverse embodiments and camera viewpoints, even if low-level action grounding is deferred to Stage~III.

A central feature of Stage II is the evolution of textual supervision. Early training benefits from detailed captions that describe objects, motions, contact, temporal events, and scene changes. As training progresses, the supervision shifts toward instruction-style captions that express task intent more abstractly. The model therefore moves from ``what is happening in the video?'' toward ``what task is being carried out, and what future state should follow?''.

The Stage II training protocol is divided into three progressive sub-stages, each with distinct focus:
\begin{itemize}
    \item \textbf{Human-centric Pretraining.} Mixed training on human-centric and robot-centric visual data. Human data provides task-structured semantic priors and intentional behavior patterns; robot visual data exposes the model to embodiment-specific appearances and physical configurations. Variable-length clips (3--15 s) expose both short interactions and longer task segments.

    \item \textbf{Robot-Centric Training.} Increases the proportion of high-quality robot-centric video clips. Caption sampling gradually shifts from fully detail-oriented captions toward a mixture of detail-oriented and instruction-style captions, improving instruction following while preserving physical grounding.

    \item \textbf{Target-Embodiment Fine-Tuning.} For downstream platforms or specific robot embodiments, the model can be fine-tuned on target-embodiment data. For multi-camera embodiments, multi-view video generation jointly models synchronized observations across viewpoints.
\end{itemize}

Stage II has an important control-sufficient interpretation: it increases the amount of task-relevant information contained in $Z_t$. After Stage I, the model may know \emph{how} objects move. After Stage II, it should better know \emph{why} objects are moved, what task structure constrains the sequence, and which future states are consistent with an instruction. Human-centric data also provides a scalable source of intentional behavior that can complement scarce robot demonstrations: it cannot replace robot data, because human and robot morphologies differ, but it can provide transferable priors about tool use, object manipulation, task decomposition, and recovery behavior that are especially valuable when later grounded through robot trajectories. Figure~\ref{fig:samples_wall} shows representative samples generated by Kairos after Stage~II, illustrating cross-embodiment generalization across single-arm, dual-arm, dexterous, and humanoid platforms.

Throughout Stage~II we uniformly use the AdamW optimizer with weight decay set to $0$ to preserve the general capabilities acquired in Stage~I. The learning rate is dynamically decayed across sub-stages: $1\!\times\!10^{-5}$ for Human-centric Pretraining, $5\!\times\!10^{-6}$ for Robot-Centric Training, and $1\!\times\!10^{-6}$ for Target-Embodiment Fine-Tuning.

\begin{figure}[!t]
    \centering
    \includegraphics[width=\linewidth]{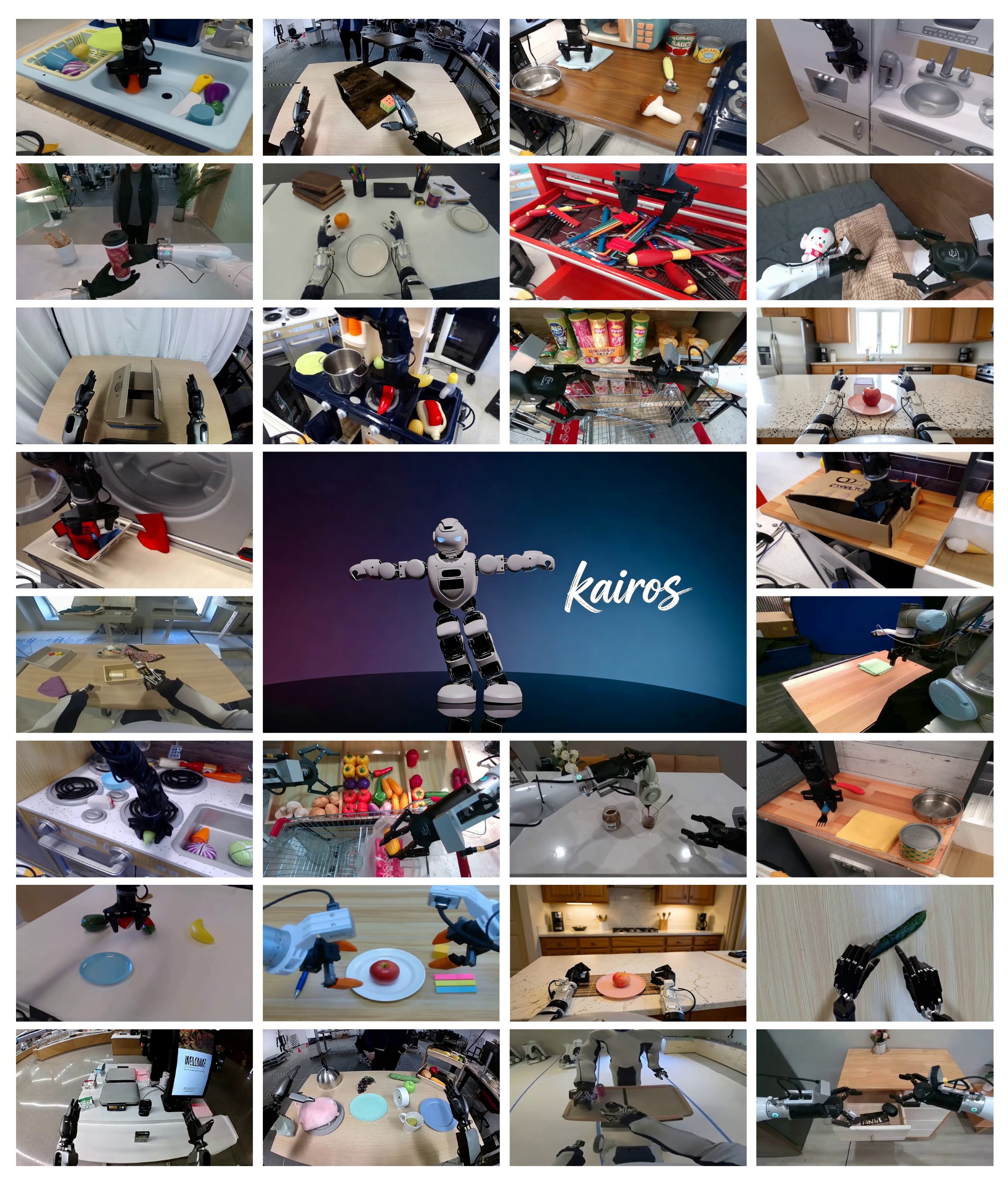}
    \caption{Samples generated by Kairos after Stage~II. Kairos exhibits cross-embodiment generalization across single-arm, dual-arm, dexterous hands, and humanoids---a behavior consistent with the goal of building a unified ``brain for multi-embodiment and multi-tasking'' that allows world knowledge to transfer across physical configurations.}
    \label{fig:samples_wall}
\end{figure}

The limitation of Stage II should also be made clear. Human behavior does not directly specify robot actuation: a human hand grasping a tool does not map one-to-one onto a robot gripper, dexterous hand, or humanoid actuator. Stage II should therefore not be claimed as final action grounding. Its role is to prepare the Video DiT backbone with intentional, instruction-conditioned, task-structured world dynamics, so that Stage~III can align these priors with actual robot actions.

\subsection{Stage III: Regret-Aware World-Action Training}\label{sec:stage_iii}

Stages I and II build the foundation of Kairos by equipping the Video DiT backbone with broad physical priors and instruction-conditioned task semantics. As the final stage of the Cross-Embodiment Data Curriculum, Stage III introduces robot interaction data to connect these priors and task semantics with executable robot actions, execution outcomes, safety risks, and recovery costs. This stage performs regret-aware alignment between world prediction and embodied action: the goal is not only to predict plausible futures, but to use regret-relevant preference supervision from execution outcomes to shape the shared world-action representation toward futures that are safer, more stable, more recoverable, and more consistent with logged robot execution.

This regret-aware alignment is implemented through two coupled components. The first component, regret alignment training, converts robot experience with high control information density into execution preference supervision. Failures, recoveries, unsafe contacts, unstable executions, and prediction--observation mismatches are used as contrastive signals for identifying which futures are physically more costly and which variables should be preserved by the shared state. The second component, joint world-action training, trains the Action DiT alongside the pretrained Video DiT so that visual forecasting and executable action prediction are aligned in a native World-Action Model. This coupling ensures that regret-relevant preference supervision acts on the same representation used for future action prediction, rather than remaining only a visual or language-level adjustment.

\subsubsection{Regret Alignment Training}

Regret Alignment Training operationalizes regret-aware alignment by turning robot experience with high control information density into preference supervision. In the regret-aware formulation, the physical cost $J_H$ provides a lens for identifying execution segments that are especially informative for embodied execution, including task failure, collision, unsafe contact, slippage, recovery effort, hardware damage, and human intervention. These segments reveal where physical cost can rise, where the margin between success and failure is thin, and where the shared world-action state should preserve more control-relevant structure.

Kairos therefore mines failures, recoveries, contact-rich events, near-unsafe interactions, boundary cases, and mismatches between model-predicted rollouts and logged robot executions from $\mathcal{D}_{\text{robot}}$. Rather than treating these segments as ordinary samples, Kairos uses them to construct execution preference pairs within matched task contexts.

The preference labels are derived from criteria intrinsic to embodied execution, including task completion, physical plausibility, stable contact, safety margin, recovery quality, and agreement with logged robot outcomes. Given the same task context, low-cost execution outcomes with safer interaction, more stable contact, better recoverability, higher execution fidelity, or smaller prediction--observation discrepancy are preferred over high-cost outcomes that exhibit failure, unsafe contact, unstable execution, or large mismatch. In this way, failures and boundary cases are not simply filtered out as bad data; they become contrastive evidence about which futures are more physically costly.

These preferences turn robot experience with high control information density into contrastive supervision for the shared world-action representation. They emphasize variables that are important for future physical cost, such as contact stability, safety margin, task progress, recoverability, and agreement between imagined and logged outcomes.

Given these preference pairs, Kairos uses a DPO-style pairwise objective to align the world-action model with regret-relevant execution preferences. The preferred sample is not defined by generic visual quality or aesthetic realism, but by embodied criteria such as task completion, physical plausibility, stable contact, safety, recoverability, and consistency between imagined and real rollouts. This turns the constructed preference pairs into a training signal for the shared world-action representation, preparing it for subsequent joint world-action training.

\subsubsection{Joint World-Action Training}

The second component grounds the preference-aligned representation in executable robot action. The pretrained Video DiT retains macro-level physical priors and coarse task semantics from previous stages, but it still needs robot-specific grounding for contact-rich, high-frequency control phenomena such as small slip boundaries, end-effector corrections, and force--torque-sensitive interactions. Stage III therefore trains the multi-stage Mixture-of-Transformers (MoT) stack~\cite{yuan2026fastwam} by introducing temporally aligned robot action trajectories through the Action DiT while adapting the Video DiT for active world-action prediction. The robot's action is no longer treated merely as an external label or a downstream policy output; it becomes part of the modeled future trajectory.

During training, video frame sequences and action chunk sequences are strictly aligned in the temporal dimension so that both branches observe corresponding segments of the same trajectory. Both the Video DiT and the Action DiT are trained using flow matching objectives. The joint training loss is
\begin{equation}
    \mathcal{L}_{\text{joint}} = \mathcal{L}_{\text{video}} + \lambda\,\mathcal{L}_{\text{action}},
\end{equation}
where $\mathcal{L}_{\text{video}}$ regularizes future visual evolution and physical consistency, $\mathcal{L}_{\text{action}}$ trains future action prediction, and $\lambda$ balances the two.

To improve training efficiency, the Action DiT is initialized by interpolating the pretrained Video DiT weights and uses a fixed timestep shift during training, unlike the dynamic exponential shift used by the Video DiT that scales with latent token count. The Video DiT models high-dimensional spatiotemporal latents with varying sequence lengths, whereas the Action DiT operates on a low-dimensional and nearly fixed-length action space; a fixed shift is sufficient for stable optimization in the latter.

Joint training also supports the counterfactual interface introduced in Section~\ref{sec:world_prediction}: given the same $Z_t$, different future action candidates should correspond to different predicted outcomes. Stage III establishes the architectural and data-alignment basis for such counterfactual world-action modeling by coupling visual evolution and action prediction within the same trajectory state. By connecting execution preference pairs with the same shared state used for action prediction, joint world-action training reduces the risk that preference alignment only adjusts visual plausibility or language-level task consistency.

\subsection{Training Infrastructure for Control-Sufficient Pretraining}\label{sec:training_infra}

The training infrastructure of Kairos is part of the native pretraining paradigm, not merely an implementation detail. A control-sufficient world-action model must be trained on long videos, high spatial resolutions, heterogeneous data sources, and multi-branch architectures involving Video DiT, Action DiT, hybrid temporal attention, multimodal conditioning, and potentially multi-view or robot-state inputs. Without a scalable training system, the model cannot acquire the breadth of physical priors or the depth of action grounding required by Physical AI.

Standard video generation models with full spatial--temporal attention---such as Wan2.2~\cite{wan2025}, Hunyuan1.5~\cite{hyworld2025}, and Cosmos~2.5~\cite{nvidia2025worldsimulationvideofoundation}---often rely on parallel partitioning strategies such as Ulysses~\cite{jacobs2023deepspeedulyssesoptimizationsenabling} or RingAttention~\cite{liu2023ringattentionblockwisetransformers} to reduce training memory and communication cost. The Kairos architecture, however, exhibits more complex computational dependency characteristics:
\begin{itemize}
    \item Linear Attention mechanisms impose stringent sequential dependencies on computation order;
    \item Dilated local attention does not depend on the complete set of global tokens.
\end{itemize}
Direct application of standard parallelization strategies such as Context Parallel, Ulysses, or RingAttention to this architecture introduces significant performance degradation due to unnecessary token broadcasting, redundant computation, and substantial communication overhead.

Kairos therefore adopts an \textbf{operator-level parallel training strategy}:
\begin{itemize}
    \item \textbf{Operator-level customization.} Sliding-window attention benefits from local sequence partitioning. Dilated attention requires careful rearrangement to preserve mid-range receptive fields. Gated Linear Attention requires sequential or recurrent state handling. Feed-forward and projection layers benefit from tensor parallelism.
    \item \textbf{Operator fusion and communication optimization.} Hybrid temporal attention can create communication bottlenecks if intermediate activations are repeatedly moved across devices. Kairos reduces overhead by fusing compatible operators, reordering execution where possible, and minimizing unnecessary token broadcasting.
    \item \textbf{Adaptive parallelism.} The optimal partitioning strategy changes across stages. Stage~I may emphasize large-scale video sequences and high-resolution latent tokens. Stage~II may include multi-view or instruction-conditioned embodied videos. Stage~III must handle temporally aligned video and action tokens. The training infrastructure adapts parallel strategies to the current model configuration, data type, sequence length, and attention pattern.
\end{itemize}

This infrastructure supports stable and efficient training on high-resolution, long-duration sequences, including 720P and 15-second video settings. From the perspective of Physical AI, this is important because long-horizon temporal context is not optional: many control-relevant variables---object permanence, task progress, delayed effects, and failure history---require extended temporal modeling. The infrastructure determines not only how large the model can be, but how effectively the model can absorb the most valuable control information. It also prepares the model for deployment-aware inference: hybrid attention, efficient temporal factorization, token streaming compatibility, and action-only inference support make deployment more practical. Kairos is not trained as a generic offline video generator and then retrofitted for robotics; it is trained from the beginning as a world-action model whose representation, memory, action interface, and runtime constraints are aligned.

\subsection{Summary: From Flat Data Scaling to a Staged Cross-Embodiment Curriculum}\label{sec:pretraining_summary}

The Native Pretraining Paradigm of Kairos reframes pretraining for Physical AI. Instead of treating heterogeneous data as a flat mixture, Kairos organizes data by intervention strength. Instead of treating video generation as the final objective, it uses generation to build and regularize a control-sufficient state. Instead of treating robot data as a small downstream fine-tuning set, it uses robot trajectories to ground previously learned physical and intentional priors in action--outcome dynamics. Instead of treating training infrastructure as implementation detail, it treats scalable long-horizon training as a prerequisite for learning the state variables required by embodied control.

Stage~I provides physical pretraining from passive observation. Stage~II provides embodied pretraining from human-centric and robot-centric visual experience. Stage~III provides regret-aware world-action training from robot trajectories and execution preference pairs. Together, they form a progression from physical priors, to task-structured embodied visual dynamics, to robot-grounded world-action alignment:
\begin{equation}
    \underbrace{\mathcal{D}_{\text{obs}}}_{\text{Stage~I}} \longrightarrow \underbrace{\mathcal{D}_{\text{human}}}_{\text{Stage~II}} \longrightarrow \underbrace{\mathcal{D}_{\text{robot}}}_{\text{Stage~III}}.
\end{equation}
The native pretraining paradigm is designed to create the control-relevant world-action state needed for future regret-aware Physical AI. In terms of the regret formulation, the curriculum supplies the experience needed to learn variables that affect physical cost: physical priors from passive observation, task-structured embodied visual dynamics from human-centric and robot-centric visual experience, and action consequences plus execution preferences from robot data. CEDC's role is not simply to scale data, but to turn heterogeneous experience into a developmental pathway for control-sufficient world modeling.

\section{Data}\label{sec:data}

Data is the substrate through which Kairos acquires the information required for constructing and maintaining the control-sufficient state $Z_t$. Section~\ref{sec:pretraining} described the Cross-Embodiment Data Curriculum as a progression from passive physical observation, to intentional human intervention, to embodied robot intervention. This section details how the corresponding data are collected, curated, tagged, captioned, and processed at scale. The central principle is that data for Physical AI should not be evaluated only by raw scale, visual diversity, or aesthetic quality. It should also be evaluated by \textbf{control information density}.

For a physical world model, data is valuable when it reduces uncertainty about the variables that matter for control: action consequences, contact dynamics, failure boundaries, recovery strategies, task progress, safety risks, and the gap between imagined and real outcomes. A short clip containing a near-boundary failure, a recovery maneuver, a marginal success, a slip, or a collision may be more valuable than hours of visually clean but ordinary successful video. This is because near-boundary events reveal the limits of controllability and the conditions under which actions remain recoverable, while ordinary successes may only reinforce common patterns.

We define control information density (CID) conceptually as the information gain that a data sample provides about control-relevant variables per unit cost:
\begin{equation}
    \mathrm{CID}(d) \;=\; \frac{H(\Theta \mid \mathcal{D}) \,-\, H\!\left(\Theta \,\big|\, \mathcal{D} \cup \{d\}\right)}{\mathrm{Cost}(d)},
\end{equation}
where $d$ denotes a data segment, $\mathcal{D}$ denotes the existing dataset, and $\Theta$ denotes the control-relevant variables that matter for embodied decision-making, such as action consequences, contact dynamics, failure boundaries, recovery strategies, safety risks, task progress, and the gap between imagined and real outcomes. $\mathrm{Cost}(d)$ may include acquisition cost, annotation cost, compute cost, safety risk, or deployment cost. This definition is not introduced as a fully implemented optimization objective in the current system. Rather, it clarifies the data-side proxy for the regret objective: data should be selected, structured, and scaled according to how much it helps the model preserve variables that enter the physical cost $c$ and improve the control-sufficient state $Z_t$.

Under this view, different data types have different expected control value. At a coarse level, Kairos uses the following priority order:
\begin{equation}
    \begin{aligned}
    &\text{near-boundary failure and recovery data}
    >
    \text{near-boundary successful data} \\
    &>
    \text{contact-rich data}
    >
    \text{ordinary successful trajectories} \\
    &>
    \text{ordinary observation videos},
    \end{aligned}
    \label{eq:cid_priority}
\end{equation}

This ordering should be interpreted by diagnostic value rather than by event label alone. \emph{Near-boundary failure and recovery data} reveals where a task breaks within the normal operating regime and how the agent can return from an error state to a feasible state. \emph{Near-boundary successful data} reveals the conditions under which the system remains successful close to failure or safety margins. \emph{Contact-rich data} reveals friction, force, deformation, support, slip, collision, grasp stability, and tool--object interaction. \emph{Ordinary successful trajectories} provide task execution patterns. \emph{Ordinary observation videos} provide broad physical priors. All are useful, but they are not equally informative for control. Extreme, non-diagnostic, or out-of-distribution failures may still be useful for safety filtering and anomaly detection, but they should not be assumed to have the highest CID for learning action consequences, failure boundaries, or recovery strategies.

This principle does not replace data scale. Scale remains necessary for coverage, robustness, and general physical priors. However, scale alone is not sufficient. A Physical AI model trained on massive but low-density observation data may become a strong visual predictor while remaining weak at failure anticipation, action--outcome reasoning, or safety filtering. The Kairos data pipeline therefore combines large-scale collection with filtering, tagging, captioning, and data-engineering infrastructure designed to make high-value control information retrievable and learnable.

The current data system should be understood as the first stage of this direction. It builds the large-scale foundation required for world modeling: diverse open-source datasets, in-house Internet-scale data, first-person manipulation data, standardized shot segmentation, hierarchical filtering, structured tags, high-quality captions, physics-centric annotations, long-horizon task decomposition, and high-throughput data processing. Future versions can further strengthen explicit measurement of control information density through real robot rollouts, simulation alignment, failure mining, contact event detection, recovery annotation, and safety-risk calibration.

\subsection{Data Collection: Multi-Source Experience Acquisition}\label{sec:data_collection}

Kairos adopts a hybrid data collection strategy that combines open-source public datasets with in-house proprietary data (Figure~\ref{fig:data_distribution}). The purpose is not merely to maximize the number of videos, but to cover the different experience regimes required by the Cross-Embodiment Data Curriculum.

\begin{figure}[h]
    \centering
    \includegraphics[width=0.85\linewidth]{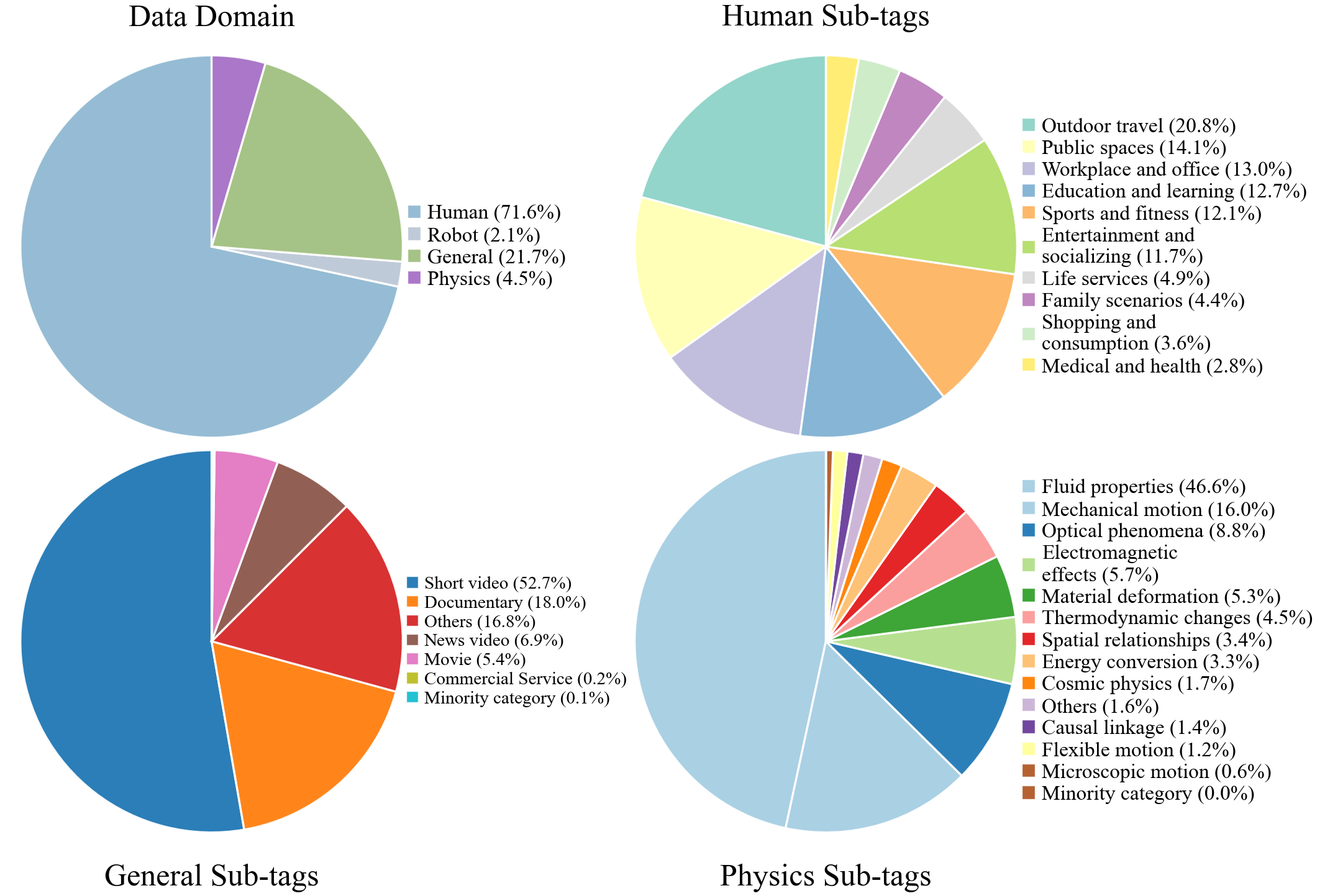}
    \caption{Distribution of Kairos data sources across passive observation, human intervention, and robot intervention regimes.}
    \label{fig:data_distribution}
\end{figure}

The collection pipeline includes three main sources: public open-source datasets, compliant acquisition of publicly available Internet data, and first-person human manipulation data. Public datasets such as Koala-36M~\cite{wang2025koala36m}, OpenHumanVid~\cite{li2025openhumanvid}, and VidGen~\cite{tan2024vidgen} contribute diverse visual and temporal patterns; robotics-oriented datasets such as AgiBotWorld-Beta~\cite{AgiBotWorldContributors2025AgiBotWC} and DROID~\cite{khazatsky2024droid} contribute robot-relevant scenes, manipulation examples, and embodied interaction patterns. To overcome the limitations of public datasets, Kairos further constructs a large-scale in-house proprietary data system consisting of Internet-crawled data and real-world collected data. The Internet-scale component is organized through a hierarchical taxonomy with tens of millions of leaf nodes, covering four core domains: \emph{human}, \emph{robot}, \emph{general scenes}, and \emph{physical phenomena}. Each domain is subdivided into hundreds of secondary and thousands of tertiary categories, with further fine-grained subdivisions per category.

This taxonomy has a control-sufficient interpretation. The human domain provides intentional behavior and task structure. The robot domain provides robot--environment interaction and action-relevant scenes. The physics domain provides phenomena such as motion, collision, gravity, fluid-like behavior, deformation, support, and contact. The general-scene domain provides broad environmental context and semantic diversity.

During raw collection, videos are obtained from publicly available Internet sources and preprocessed according to platform-specific characteristics. The cleaning phase removes corrupted videos, duplicates, and invalid clips shorter than five seconds. After this process, Kairos accumulates several millions of hours of valid raw video data. This scale is a foundation rather than the final measure of data value.

A major limitation of existing open-source and Internet-scale data is the shortage of fine-grained manipulation and embodied interaction. To address this gap, Kairos additionally collects high-precision first-person human manipulation data. Ego-centric data captures hand--object interaction from a viewpoint closer to embodied operation, providing information about reachability, manipulation, tool use, occlusion, contact transitions, task progress, and recovery behavior.

From the perspective of interventional generalization, collection sources are organized as
\begin{equation}
    \mathcal{D}_{\text{obs}} \longrightarrow \mathcal{D}_{\text{human}} \longrightarrow \mathcal{D}_{\text{robot}}.
\end{equation}
This ordering is not only a training curriculum; it is also a data collection principle. The system should collect data that fills gaps along the path from passive physical regularities to action--outcome grounding.

\textbf{Shot segmentation and clip standardization.} Kairos performs unified shot segmentation on all raw videos using PySceneDetect with multiple scene detectors, achieving over 95\% segmentation precision and approximately 80\% recall. The pipeline applies the following rules:
\begin{itemize}
    \item keep segments between 5 and 40 seconds;
    \item further split long shots longer than 40 seconds into 20-second clips;
    \item discard segments shorter than 5 seconds to avoid fragmented information.
\end{itemize}
Through this pipeline, Kairos obtains hundreds of millions of standardized video clips, providing a searchable reservoir from which control-relevant, high-density experiences can be filtered, tagged, captioned, and sampled.

\subsection{Data Curation: From Quality Filtering to Control-Relevant Filtering}\label{sec:curation}

Data curation transforms the raw clip pool into a training-ready corpus. The current Kairos pipeline uses hierarchical filtering (Figure~\ref{fig:filter}) to improve visual quality, temporal coherence, semantic diversity, safety, and redundancy control. These filters are necessary because noisy, corrupted, trivial, or unsafe videos can degrade training. However, for Physical AI, quality filtering is only the first layer. A visually clean clip is not necessarily control-informative. Therefore, Kairos interprets data curation as a two-level process: basic quality filtering followed by control-relevant event filtering.

\begin{figure}[h]
    \centering
    \includegraphics[width=0.9\linewidth]{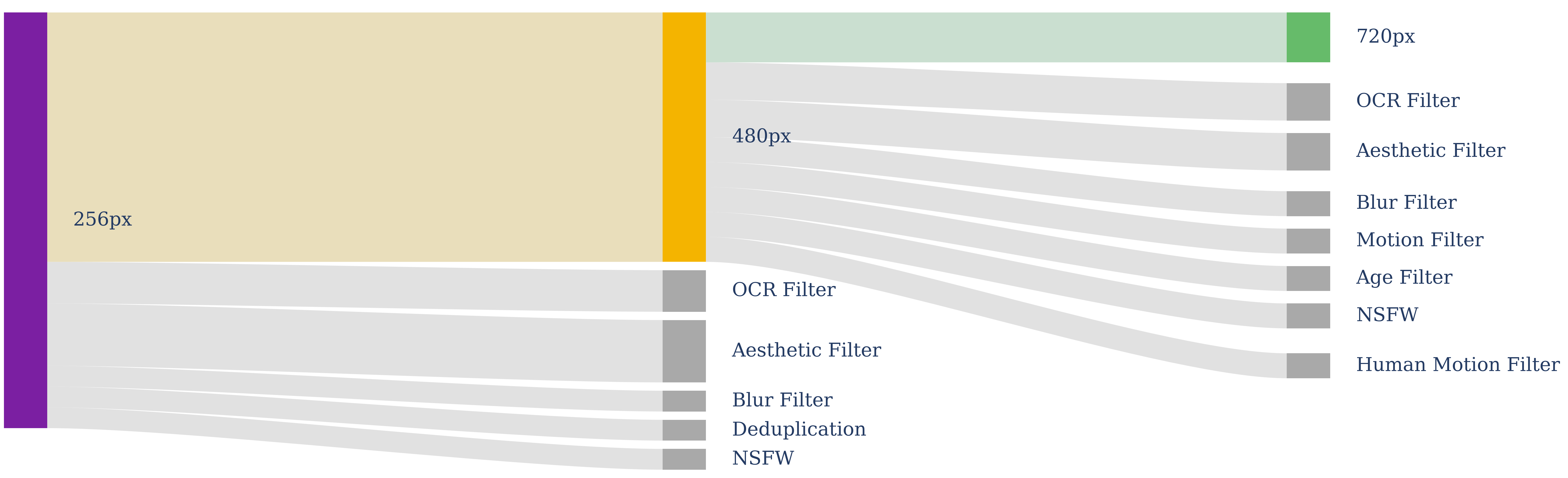}
    \caption{Hierarchical curation pipeline: basic quality filtering removes noisy clips, while control-relevant event filtering prioritizes clips with high control information density.}
    \label{fig:filter}
\end{figure}

\subsubsection*{Basic Quality Filtering}

The first level removes clips that are unsuitable for stable world-model training:
\begin{itemize}
    \item \textbf{Aesthetic Score.} Aesthetic quality indicates visual richness, composition, and content diversity. Low-aesthetic videos may contain simplistic scenes, poor color distribution, or insufficient visual structure, limiting their usefulness for learning broad visual priors. Kairos uses an aesthetic predictor built on a CLIP backbone with an MLP head and filters out samples below a predefined threshold.
    \item \textbf{Motion Score.} Temporal dynamics are essential for world modeling. Static videos contain insufficient temporal information, while videos with extreme flickering or unstable motion can harm temporal consistency. Kairos uses RAFT~\cite{teed2020raft} to compute optical flow between consecutive frames; the magnitude is aggregated into a global motion score, and clips with excessively low or high motion are filtered.
    \item \textbf{AIGC Score.} AI-generated videos are increasingly common on the Internet. Some synthetic videos may be useful in later stages if validated, but low-quality synthetic content can introduce artifacts, unrealistic dynamics, or misleading physical patterns. Kairos trains a ViT-Large-based discriminator on a proprietary synthetic-video dataset and excludes clips above a predefined AIGC threshold in relevant training stages.
    \item \textbf{NSFW Score.} Pornographic, violent, or otherwise unsafe videos can mislead model training and are inappropriate for general-purpose Physical AI. Kairos uses the open-source Falconsai model~\cite{Falconsai2023nsfw} to filter NSFW content from crawled Internet data.
    \item \textbf{Blurriness Score.} Sharpness affects the model's ability to learn object boundaries, contact points, motion cues, and spatial structure. Kairos uses the Laplacian operator to assess image sharpness; higher Laplacian scores indicate richer edges and clearer visual details. Clips below the sharpness threshold are removed or downweighted.
    \item \textbf{Human Motion Score.} Human behavior is a major source of intentional intervention data. However, a dataset dominated by static human clips may not teach useful body, hand, or manipulation dynamics. Kairos uses YOLOX~\cite{ge2021yolox} for human detection and ByteTrack~\cite{zhang2021bytetrack} for trajectory tracking, then computes normalized pixel velocity for each detected person to support selection of clips with meaningful human motion.
    \item \textbf{OCR Score.} Text-heavy videos may introduce unwanted bias, especially in early video-generation training where text rendering can interfere with convergence. Kairos uses DBNet~\cite{liao2019real} for text-region detection and computes an OCR score based on the proportion of text regions in the frame; clips with excessive text are excluded in early phases, while text-rich clips may be selectively used in later stages if text generation or GUI-like tasks are relevant.
    \item \textbf{Data Deduplication.} Internet videos contain large amounts of near-duplicate content. Redundant clips increase storage and compute cost without adding new information. Kairos extracts video embeddings using CLIP and maintains an embedding pool for large-scale deduplication; for each new clip, pairwise similarity with historical clips is computed and only the higher-resolution or higher-quality version is retained when similarity exceeds a threshold.
\end{itemize}
From the control-sufficient perspective, these filters also remove noise that can obscure physical and action-relevant variables: blurry videos make contact points harder to infer; unstable flickering weakens temporal state learning; corrupted synthetic videos may teach incorrect dynamics.

\subsubsection*{Control-Relevant Event Filtering}

The second level would prioritize clips according to control information density. The current pipeline does not yet compute CID directly; future iterations should explicitly target:
\begin{itemize}
    \item \textbf{Near-boundary failures:} marginal grasp failure, near slip that becomes a drop, near collision that becomes contact, unstable stack collapse, partial task failure;
    \item \textbf{Recovery events:} regrasping, repositioning, replanning, human correction, retry behavior, post-failure continuation;
    \item \textbf{Near-boundary successes:} marginal grasps that remain stable, near slips that are corrected, near collisions that are avoided, partial successes completed through correction;
    \item \textbf{Contact transitions:} first contact, loss of contact, grasp closure, slip, collision, support change;
    \item \textbf{Safety and anomaly events:} human proximity, sharp-object contact, excessive force, irreversible state change, extreme or non-diagnostic failures;
    \item \textbf{Long-horizon dependencies:} delayed failure, multi-step dependency, hidden object state, task-progress change.
\end{itemize}
These events are rare compared with ordinary successful clips, but they are highly valuable when they are close to the task manifold and diagnostically interpretable. A model that never sees near-boundary failures may generate clean success futures but remain unable to predict where deployment will break. A model that never sees recovery may fail to plan after mistakes. A model that never sees near-boundary successes may not learn which small corrections or safety margins preserve successful execution.

In the current Kairos report, these control-relevant filters should be presented as a guiding extension built on top of the existing curation pipeline. The current implementation provides strong quality, safety, motion, and redundancy filtering. Future iterations can incorporate explicit detectors, human-in-the-loop annotation, simulation labels, robot rollout logs, and tactile--force signals to compute more precise control information density scores.

\subsection{Tagging: Structured Indexing for Control-Relevant Sampling}\label{sec:tagging}

Tagging converts raw videos into structured, searchable, and samplable training assets. For Internet-scale data, manual inspection is impossible. Tags allow the system to organize heterogeneous clips, balance data proportions, retrieve specific domains, reduce sampling bias, and support downstream captioning. In Kairos, tagging should be understood not only as semantic indexing, but as a mechanism for exposing control-relevant structure.

The current tagging system contains two primary categories: \emph{video attribute tags} and \emph{video domain tags}. \emph{Attribute tags} describe global properties of the clip that are useful for filtering and sampling, such as camera motion, static content, blur, motion type, or other intrinsic video characteristics. \emph{Domain tags} describe the semantic or functional category of the clip; the main domain tags are summarized in Table~\ref{tab:tagging}.

\begin{table}[t]
\centering
\caption{Video Domain Tags. Each video is assigned to exactly one domain to ensure unambiguous data partitioning.}
\label{tab:tagging}
\footnotesize
\renewcommand{\arraystretch}{1.3}
\begin{tabular}{p{1.8cm} p{4.5cm} p{6cm}}
\toprule
\textbf{Tag} & \textbf{Sub-tags} & \textbf{Brief Description} \\
\midrule
Human   & scenes/actions/occupation/\allowbreak gender/age/context/\allowbreak face blur/body motion & Human behavior videos for learning human behavior patterns \\
\rowcolor{gray!10}
Robot   & scenes/actions & Robot interaction/task execution videos for learning robot--environment interaction mechanisms \\
Physics & principles & Videos of physical laws/natural phenomena for physical-rule modeling \\
\rowcolor{gray!10}
General & content type/scene/animal & General-scene videos ensuring complete coverage of the tagging system \\
\bottomrule
\end{tabular}
\end{table}

This tag system supports the Cross-Embodiment Data Curriculum. Human tags help sample intentional behavior. Robot tags help sample action-grounded interaction. Physics tags help sample passive physical phenomena. General content tags ensure broad environmental coverage. Together, they provide the indexing structure required for stage-specific data mixtures.

For control-sufficient world modeling, the tagging system can be extended with control-oriented tags. These tags need not all be fully implemented in the current version, but they clarify the desired direction of the Kairos data engine:
\begin{itemize}
    \item \textbf{Intervention-Level Tags:} passive observation, human intervention, robot intervention, or mixed human--robot interaction. Supports intervention-strength curricula rather than flat data mixtures.
    \item \textbf{Physical-Event Tags:} contact, collision, sliding, rolling, deformation, support, fluid motion, gravity-driven motion, tool use, occlusion, object transfer.
    \item \textbf{Object--Contact--Action Tags:} which object, what action, what contact state, what physical transition.
    \item \textbf{Task-Progress Tags:} subtask completion, remaining steps, task ordering, long-horizon dependencies.
    \item \textbf{Failure--Recovery Tags:} failure type, cause, recovery action, recovery success, return to feasible state.
    \item \textbf{Safety-Risk Tags:} unsafe contact, human proximity, collision risk, excessive force, unstable objects, irreversible state changes, high-uncertainty situations.
    \item \textbf{Embodiment Tags:} robot type, gripper type, camera viewpoint, end-effector state, proprioceptive context, tactile/force availability, control modality.
\end{itemize}

The purpose of these tags is not to create a rigid symbolic simulator. Rather, tags provide structured weak supervision and retrieval handles. They help the model and the data pipeline focus on the variables that matter for control. They also enable stage-specific sampling: Stage~I may prioritize physics and general-scene tags; Stage~II may prioritize human-intervention and task-progress tags; Stage~III may prioritize robot, contact, action, failure, and recovery tags.

Tagging also improves caption quality. When tag information is passed into the captioning pipeline, the generated descriptions become more complete, less ambiguous, and more aligned with the actual content of the video. This is especially important for control-relevant captions, where the model must describe not only what appears in the video, but also what action occurs, what state changes, what risk emerges, and what causal relation connects events.

\textbf{Tag-annotation pipeline.} For large-scale annotation, Kairos adopts an end-to-end automatic pipeline based on Qwen3-VL-8B~\cite{qwen3technicalreport}. Original videos are uniformly sampled at fixed time steps to balance representativeness and efficiency; structured tag rules, annotation paradigms, and semantic constraints are integrated into the prompt; and inference results are stored in a structured JSON format. The unified structured representation supports downstream data filtering, proportioning, and sampling without requiring manual annotation at scale.

\subsection{Captioning: From Visual Description to Control-State Supervision}\label{sec:captioning}

\begin{figure}[h]
    \centering
    \includegraphics[width=0.85\linewidth]{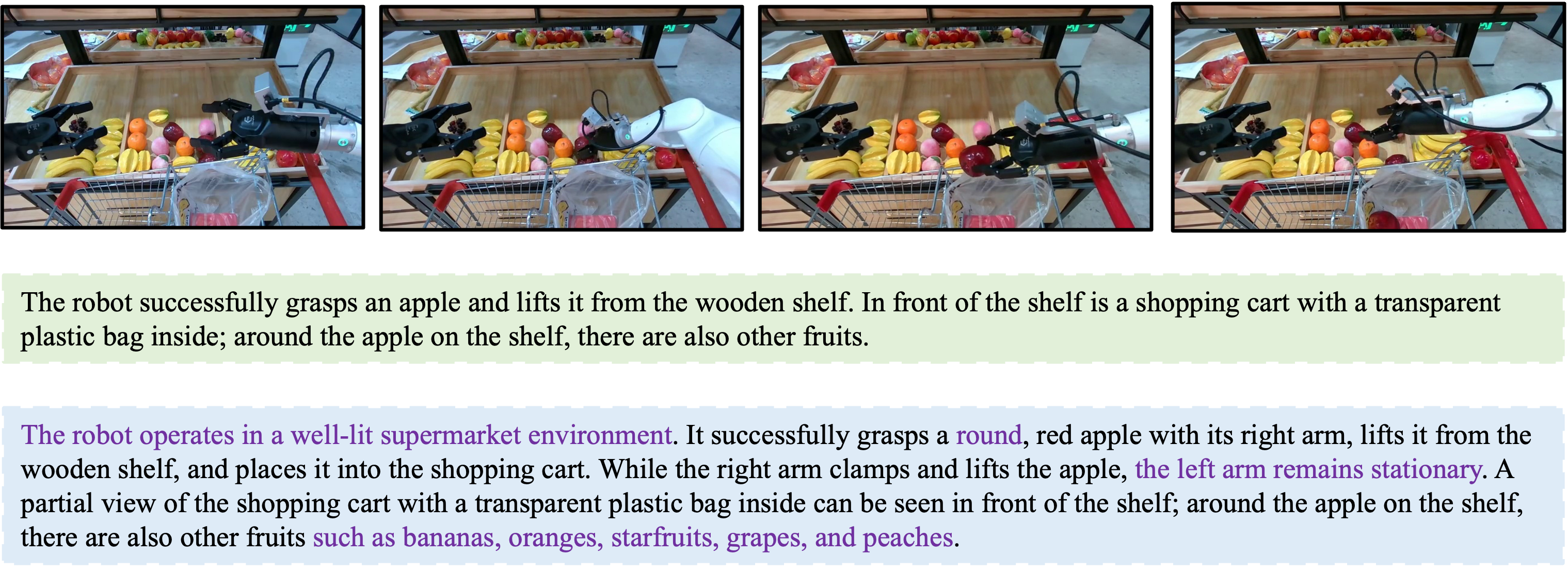}
    \caption{Divide-and-conquer captioning: tag information provides semantic constraints; the VLM supplies detailed natural-language interpretation; a multi-model ensemble fuses initial outputs into a final caption.}
    \label{fig:dc_caption}
\end{figure}

Captioning provides language supervision for world-model training. In ordinary video generation, captions mainly describe the visible content of a clip. For Kairos, captions should play a broader role: they should help shape the control-sufficient state $Z_t$. A useful caption should describe not only objects and scenes, but also actions, physical changes, contact events, task progress, uncertainty, and possible failure or recovery logic.

The current Kairos captioning pipeline uses Qwen3-VL-8B~\cite{qwen3technicalreport} as the core model for video description generation. The baseline caption is required to cover five dimensions: (i)~core subjects, including humans, animals, and objects; (ii)~subject actions; (iii)~surrounding environment and background details; (iv)~lighting and atmosphere; and (v)~camera motion.

These dimensions provide a strong foundation for visual--semantic alignment. However, a generic caption can still be incomplete or hallucinated. To improve detail richness and reduce ambiguity, Kairos adopts a \textbf{Divide-and-Conquer (DC)} captioning strategy (Fig.~\ref{fig:dc_caption}). Instead of generating captions directly from the raw video alone, the pipeline combines structured tag information with visual content. Tags provide semantic constraints; the VLM supplies detailed natural-language interpretation. This combination improves caption completeness and reduces hallucination.

Kairos further uses a multi-model ensemble annotation pipeline. Specifically, multiple medium-parameter multimodal models---including Qwen3-VL-8B~\cite{qwen3technicalreport}, InternVL3.5-8B~\cite{wang2025internvl35}, Mimo-7B~\cite{xiaomi2025mimo}, and MiniCPM~V4.5~\cite{yu2025minicpmv45}---independently generate initial descriptions. A stronger multimodal model (Qwen3-VL-8B) then acts as a large-parameter fuser, referencing consensus information from the initial outputs and supplementing finer details. Because this pipeline is computationally expensive, it is applied mainly to high-quality continuous training data, balancing annotation quality and processing efficiency.

From the control-sufficient perspective, captioning should be evaluated by how well it supervises the variables in $Z_t$. A caption such as ``a robot arm moves a red block'' is useful but incomplete. A more control-relevant caption would include: the robot arm approaches the red block from the left, closes the gripper, establishes contact, lifts the block slightly, avoids collision with the bowl, and places the object on the table. If a slip occurs, the caption should state that the grasp becomes unstable, the object rotates, the gripper loses contact, and the robot attempts to recover. Such descriptions expose action consequences and failure boundaries.

We therefore organize captioning objectives into several layers:
\begin{itemize}
    \item \textbf{Semantic Captioning.} This describes objects, agents, scenes, attributes, and high-level actions.
    \item \textbf{Physical Captioning.} This describes motion, contact, gravity, collision, support, deformation, friction, and object permanence.
    \item \textbf{Embodied Captioning.} This describes gripper state, hand--object relation, viewpoint, robot motion, end-effector interaction, and action sequence where available.
    \item \textbf{Task-Progress Captioning.} This describes the goal, current subtask, completed steps, remaining steps, and dependencies between steps.
    \item \textbf{Failure--Recovery Captioning.} This describes failure events, failure causes, recovery attempts, and post-recovery state.
    \item \textbf{Safety-Risk Captioning.} This describes unsafe contact, collision risk, human proximity, instability, irreversible action, or uncertainty.
\end{itemize}
The current implementation already supports semantic, physical, and long-horizon task-oriented captioning. Future versions can extend this framework with explicit failure--recovery and safety-risk captions, particularly for real robot and simulation-aligned data. This extension would directly support regret-aware world modeling by teaching the model which states are costly and how errors can be corrected.

Captioning should therefore be interpreted as control-state supervision. It converts raw video into language signals that guide the model toward learning the variables that matter for prediction and action. It is not a replacement for real action labels, tactile data, or robot rollouts, but it provides scalable weak supervision for organizing large-scale visual experience.

\subsection{Enhanced Text with Control-Oriented Chain of Thought}\label{sec:cot}

\begin{figure}[h]
    \centering
    \includegraphics[width=0.85\linewidth]{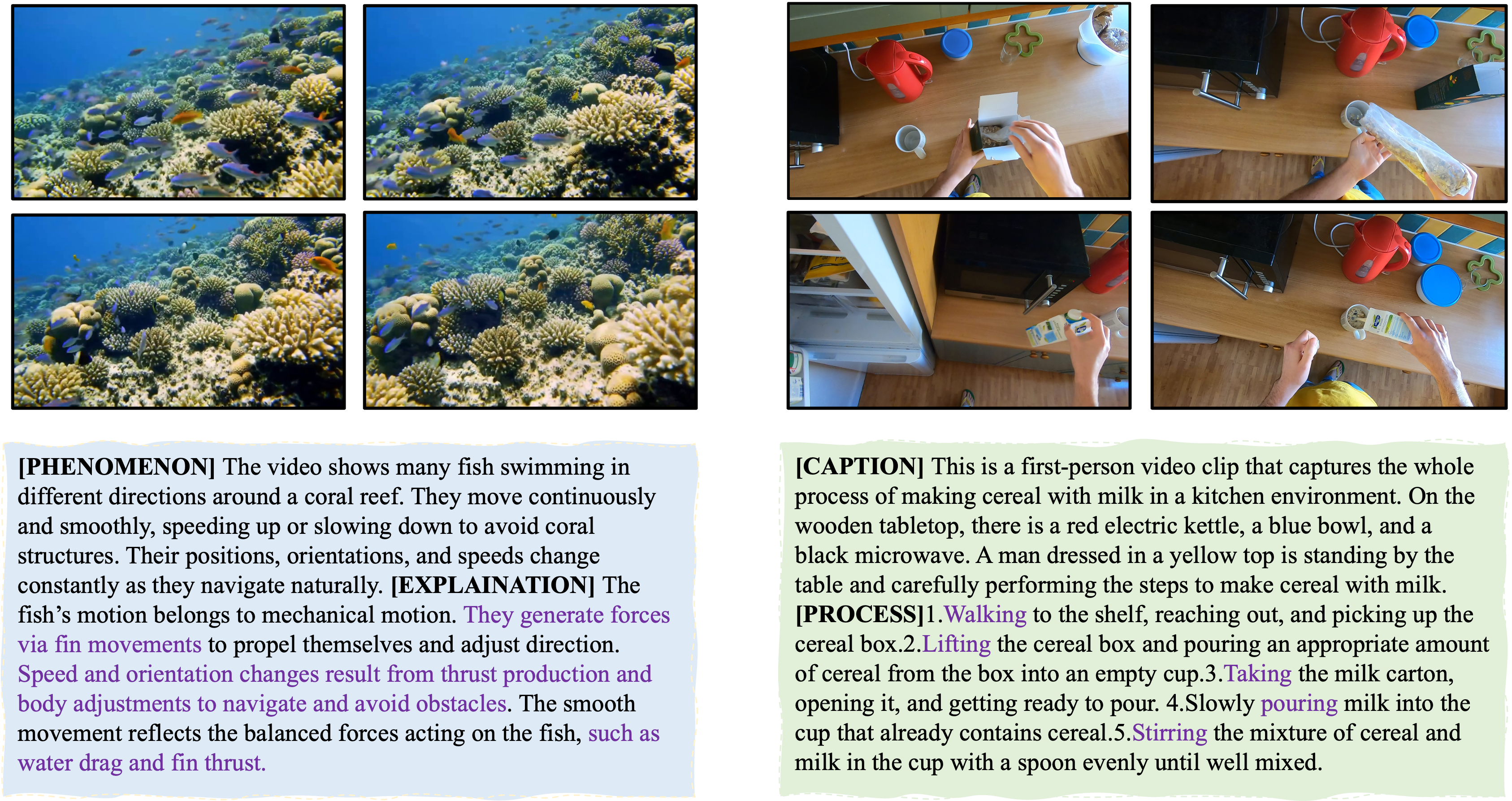}
    \caption{Enhanced text with control-oriented chain-of-thought annotations: physics-centric captions explain underlying physical principles; long-horizon task captions decompose tasks into sub-steps with explicit dependencies.}
    \label{fig:cot}
\end{figure}

Kairos further enriches selected data with enhanced text and chain-of-thought-style annotations (Fig.~\ref{fig:cot}).

\textbf{Physics-centric.} The core of world models lies in learning and reproducing the operational logic of the real world, where physical laws serve as the key element to achieve this goal and represent the fundamental characteristic that distinguishes world models from ordinary generative models---namely, their ``world cognition capability.'' To fully exploit the core value of physical principles and ensure the physical realism of world models, we construct specialized text datasets that explicitly reinforce physical laws from the data source. Specifically, for data tagged as ``physics,'' we explicitly strengthen the expression of physical laws within their captions. Our physics-centric caption data not only describes the surface phenomena observed in videos but also explains the underlying physical principles embedded in these phenomena, including dynamic behavioral constraints such as object motion trajectories, collision interactions, gravitational effects, and frictional forces. Through this approach, we ensure that the content generated by the model conforms to real-world physical constraints, thereby enhancing physical realism. The left part of Figure~\ref{fig:cot} presents a typical example of physics-centric captions.

\textbf{Long-horizon Tasks.} Causality and long-horizon sequence consistency constitute another core element that distinguishes world models from ordinary generative models. To ensure the logical consistency and rationality of long-horizon task execution, we focus on strengthening chain-of-thought construction and task decomposition capabilities, creating datasets that target long-horizon sequence scenarios. For selected human manipulation and robot interaction data, we construct a long-horizon task-oriented annotation framework by clearly decomposing task steps and delineating the causal logical chains between events. Specifically, we decompose complex task captions into multiple executable sub-steps, explicitly defining the dependencies and causal relationships between each step, enabling the model to understand and reproduce task execution logic over extended temporal sequences. This design effectively ensures the logical coherence of the model when handling long-horizon sequence scenarios. The right part of Figure~\ref{fig:cot} illustrates a typical example of task decomposition in long-horizon tasks data.

\textbf{Failure, recovery, and risk annotations.} A natural extension of enhanced text is failure--recovery--risk annotation, especially important for regret-aware Physical AI. Failure annotations describe the failure type (slip, collision, unstable grasp, occlusion error, wrong-object selection, task interruption). Recovery annotations describe regrasping, repositioning, replanning, human correction, or retry behavior. Risk annotations describe unsafe contact, high uncertainty, marginal stability, or near-boundary states. The existing physics-centric and long-horizon annotation pipeline provides a strong foundation; adding failure, recovery, and risk annotations would make the data engine more directly relevant to deployment-time decision support.

Enhanced text functions as \emph{weak supervision for control-sufficient state learning}. It does not make the model a symbolic reasoner; rather, it encourages the model to preserve the right variables---physical relations, causal transitions, task dependencies, failure boundaries, and recovery logic. Its outputs should ideally support three downstream functions: (1)~train the model to associate visual dynamics with physical and causal explanations; (2)~support retrieval and sampling of high-control-information clips; (3)~provide interpretable supervision for failure analysis, long-horizon consistency, and risk-aware evaluation.

\subsection{Data Engineering Infrastructure for Scalable Control-Information Processing}\label{sec:data_infra}

The Kairos data pipeline requires large-scale data processing infrastructure. In the data pre-processing phase, we built an efficient infrastructure that achieves significant performance improvements for three core operators: Shot Detection, Frame Filtering, and Captioning. The optimization strategy focuses on three dimensions: \emph{computational parallelization}, \emph{I/O optimization}, and \emph{task scheduling}, as summarized in Table~\ref{tab:data_infra}.

\begin{table}[t]
\centering
\caption{Throughput optimization across the three core data engineering operators (8$\times$4090 GPUs, 180 vCPUs).}
\label{tab:data_infra}
\footnotesize
\renewcommand{\arraystretch}{1.3}
\setlength{\tabcolsep}{4pt}
\begin{tabular}{p{2.3cm} r r r p{4.3cm}}
\toprule
\textbf{Operator} & \textbf{Baseline} & \textbf{Optimized} & \textbf{Speedup} & \textbf{Key techniques} \\
                  & \textbf{(hrs/day)} & \textbf{(hrs/day)} & & \\
\midrule
Shot Detection  & 1{,}169.6 & 8{,}640.0 & 7.4$\times$ & Distributed scheduling, frame skipping, load balancing \\
\rowcolor{gray!10}
Frame Filtering & 612.0     & 18{,}332.7 & 29.9$\times$ & FP16 inference, CPU decoding, concurrent scheduling \\
Captioning      & 137.0     & 4{,}665.9 & 34.0$\times$ & CPU decoding, pipeline overlap, two-level batching \\
\bottomrule
\end{tabular}
\end{table}

\textbf{Shot Detection Operator Optimization.} The shot detection operator segments long videos into semantically coherent clips. Kairos employs a combination of multiple detectors on a single machine with 8$\times$4090 GPUs and 180 vCPUs. Key techniques include: distributed scheduling to replace the original pipeline; resolution downscaling and frame skipping to reduce redundant computation; dynamic worker allocation to increase parallelism; and duration-based task partitioning for load balancing. These optimizations improve throughput from 1{,}169.6 to 8{,}640.0 hours/day ($7.4\times$), while maintaining a shot detection recall rate of 77.44\%. Good segmentation preserves event integrity: contact onset, collision, object transfer, task-step completion, or recovery sequence.

\textbf{Frame Filtering Operator Optimization.} The frame filtering operator conducts multi-dimensional quality assessment, including luminance, blur, aesthetics, and pose detection. Throughput is improved from 612.0 to 18{,}332.7 hours/day ($29.9\times$) through: pre/post-processing refactoring to reduce redundancy; adaptive frame sampling; CPU-concurrent decoding to eliminate serial I/O bottlenecks; FP16 mixed-precision inference; and automatic resource scheduling with zero-copy video access to achieve computation--I/O overlap. From the control-sufficient perspective, filtering determines whether the model receives enough clear evidence about physical state: blurry frames may hide contact; poor lighting may obscure object boundaries; unstable motion may corrupt temporal continuity.

\textbf{Caption Operator Optimization.} The captioning operator generates textual descriptions for video clips. Throughput is improved from 137.0 to 4{,}665.9 hours/day ($34.0\times$) by reusing CPU-concurrent decoding, adaptively sampling key segments, and decoupling video loading from model inference into a pipeline architecture with a two-level batching mechanism that balances throughput and memory utilization. For Physical AI, the caption operator should eventually prioritize clips with high control information density.

\textbf{End-to-end gain.} Overall, the Kairos data engineering infrastructure achieves over 30$\times$ end-to-end throughput improvement on a single machine with 8$\times$4090 GPUs, with key contributions from: (1)~\textbf{Computational Parallelization} via distributed scheduling and load-balanced task partitioning; (2)~\textbf{I/O Optimization} via CPU-concurrent decoding, zero-copy streaming, and pipeline overlap; and (3)~\textbf{Task Scheduling} via two-level batching and dynamic resource allocation across sub-modules. The deeper significance is not just speed---it makes \emph{scalable control-information processing} possible. Without high-throughput infrastructure, high-CID data remains too expensive to mine and structure.

\subsection{Limitations and Future Data Directions}\label{sec:data_limitations}

The current Kairos data pipeline provides a strong foundation for large-scale world modeling. The following items are scoped as future work.

\begin{enumerate}
    \item \emph{Visual and semantic quality metrics do not directly measure control value.} A clip may be visually clean but uninformative for action--outcome learning; a near-boundary failure or recovery clip may be visually imperfect but highly valuable for control. Future data selection should explicitly measure contact, near-boundary failure, recovery, marginal success, boundary, and safety information.
    \item \emph{Heterogeneous data sources are not fully aligned.} Internet-scale videos provide broad physical priors but limited action grounding; human-centric data provides intentional behavior but does not directly match robot embodiment; robot data provides action grounding but remains expensive and narrow. Future data construction should strengthen alignment among ego-centric data, robot trajectories, and simulation rollouts through shared event labels and temporally aligned state--action annotations.
    \item \emph{Captions and CoT annotations may contain hallucinations or incomplete causal explanations.} Future systems should calibrate captions against physical signals, robot logs, simulator states, and human verification where necessary.
    \item \emph{Synthetic or generated data should not be treated as automatically valid training experience.} It can enrich rare events, stress-test policies, and generate candidate rollouts, but it must be calibrated against real or high-fidelity simulated outcomes.
    \item \emph{Direct closed-loop regret reduction requires further data infrastructure.} Future datasets should explicitly measure whether imagined rollouts correlate with real rollouts, whether failure predictions anticipate real failures, whether safety filtering reduces unsafe events, and whether recovery data improves real policy behavior.
\end{enumerate}

These limitations define the next stage of Kairos data development: moving from large-scale data engineering toward closed-loop data infrastructure with high control information density. Future data collection should also support direct estimation of cost-relevant outcomes, including failure events, recovery effort, unsafe contact, human intervention, and imagined--real rollout discrepancy. These outcomes are needed to move from data-side proxies toward empirical validation of regret-aware decision support.

\subsection{Summary: A Data Engine for Control-Sufficient World Modeling}\label{sec:data_summary}

This section reframes the Kairos data system as a \emph{control-sufficient data engine}. The goal is not merely to collect more videos or produce cleaner captions; the goal is to acquire, filter, annotate, retrieve, and process the experience most useful for constructing $Z_t$.

The current pipeline provides five key capabilities: (i)~large-scale multi-source data collection spanning open-world videos, public datasets, in-house Internet data, first-person manipulation data, and robot-relevant corpora; (ii)~standardization of raw videos into training clips through shot segmentation; (iii)~hierarchical curation to remove noise, low-quality clips, unsafe content, AIGC contamination, redundant videos, and visually unsuitable samples; (iv)~structuring of data through tags and captions, including physics-centric and long-horizon task annotations; (v)~high-throughput engineering infrastructure for scalable shot detection, frame filtering, and captioning.

The broader significance is that Kairos data infrastructure can move beyond raw scale toward control information density. Control information density should therefore be interpreted as a data-side route toward reducing $\operatorname{Reg}_H(f;g)$, not as a separate objective disconnected from the rest of the model stack. The most valuable data will be the data that reduces uncertainty about action consequences, contact dynamics, failure boundaries, recovery strategies, and safety risks. This connects data to the rest of the Kairos stack: data provides control-relevant information; Section~\ref{sec:cedc} organizes it through a cross-embodiment curriculum; Section~\ref{sec:model} compresses and maintains it as a control-sufficient state; deployment-aware inference will eventually test whether this state reduces real-world mistakes. The Kairos data engine should be understood as a model-side prerequisite for future regret-aware Physical AI.

\section{Inference}\label{sec:inference}

Inference is the stage where the control-sufficient state $Z_t$ becomes operational. Sections~\ref{sec:model}--\ref{sec:data} described how Kairos constructs $Z_t$, learns the information required for $Z_t$, and builds a data engine for experience with high control information density. This section addresses the next question: \emph{how can such a state be used under practical deployment constraints?}

For Physical AI, inference is not a passive decoding step. It is the mechanism through which a world model enters an observation--action--feedback loop. A model may generate visually impressive futures offline, but if its inference is too slow, too memory-intensive, or too dependent on unrealistic hardware assumptions, it cannot support action selection, risk assessment, failure anticipation, recovery planning, or future closed-loop improvement. Latency, memory footprint, communication cost, and hardware compatibility are not secondary engineering details---they determine whether the world model can be used \emph{before} an action is executed.

Kairos therefore treats inference as \textbf{regret-relevant information throughput}. The practical question is not only how good a generated rollout looks, but how much control-relevant information the model can produce per unit of time, memory, compute, and communication. A slow model may preserve rich visual detail but fail to affect the robot's decision before the action deadline. A more efficient model may preserve fewer pixels but provide timely information about task progress, action consequence, or failure risk---for Physical AI, the latter may be more valuable.

\subsection{Toward Self-Evolution: Proxy Rollout--Evaluation--Refinement}\label{sec:proxy_loop}

The vision of the Kairos framework extends beyond serving as a static inference engine: it is designed as deployable infrastructure for future self-evolutionary learning. From a \emph{regret-aware} perspective, self-evolution is the mechanism through which a world-action model could progressively reduce the predicted cost $\widehat{J}_H$, and eventually the real physical cost $J_H$, by generating candidate futures, comparing their predicted consequences, learning from discrepancy, and refining the next decision. This capability is rooted in Kairos's unified understanding--generation--prediction architecture. By tightly coupling these three pathways, the model provides a substrate for future observation--action--feedback loops and continuous refinement of its internal representations and decision-making strategies.

In practice, the current system instantiates this idea through the rollout--evaluation--refinement cycle shown in Figure~\ref{fig:self-evolution}. Upon receiving an instruction, the generation and prediction modules simulate multiple physically plausible future rollout paths and action trajectories. Leveraging the Chain-of-Thought (CoT) framework, Kairos's understanding module acts as a proxy evaluator, analyzing, scoring, and ranking these diverse paths according to physical plausibility, task progress, and risk-relevant variables. This reflective process enables the system to identify preferable strategies, correct generation-level errors, and improve prompt-level precision over iterative refinement.

As a concrete self-evolutionary strategy, this mechanism is validated by instantiating prompt rewriting agents. Building upon the prompt alignment strategy, the rewriter (i.e., the prompt template) serves as an active evolutionary proxy. The understanding module scores the outputs generated by different prompts, dynamically evaluates and rewrites user instructions, and constructs a localized self-improving loop. We empirically observe that this automated process directly enhances the model's precision and alignment performance in deployed generation scenarios.

While prompt optimization and trajectory selection represent successful initial validations, this self-evolutionary paradigm is equally applicable to the policy configurations of the World-Action Model (WAM). By enabling the model to autonomously simulate, evaluate, and refine its policy parameters within the closed-loop understanding framework, Kairos can continuously improve its decision-making and execution strategies for physical interactions. In the regret-aware view, this extension is the natural path toward measurable regret reduction: the evaluator can score predicted task success, collision risk, contact stability, recovery cost, and safety margin; refinement can target the policy, the action proposal distribution, the rollout evaluator, or the world-action model itself. We will thoroughly investigate this application in subsequent versions of Kairos.

The current self-evolution mechanism should nevertheless be interpreted as a \emph{controlled proxy} for future regret-aware Physical AI: it is demonstrated at the prompt and generation level rather than through end-to-end robot deployment, and it does not yet establish that imagined rollouts correlate with real robot rollouts, that failure predictions reduce unsafe events, or that imagined policy updates measurably improve real task success. It nonetheless provides the model-side and inference-side prerequisites on which future closed-loop, regret-minimizing validation can be built.

\begin{figure}[t]
    \centering
    \includegraphics[width=0.95\linewidth]{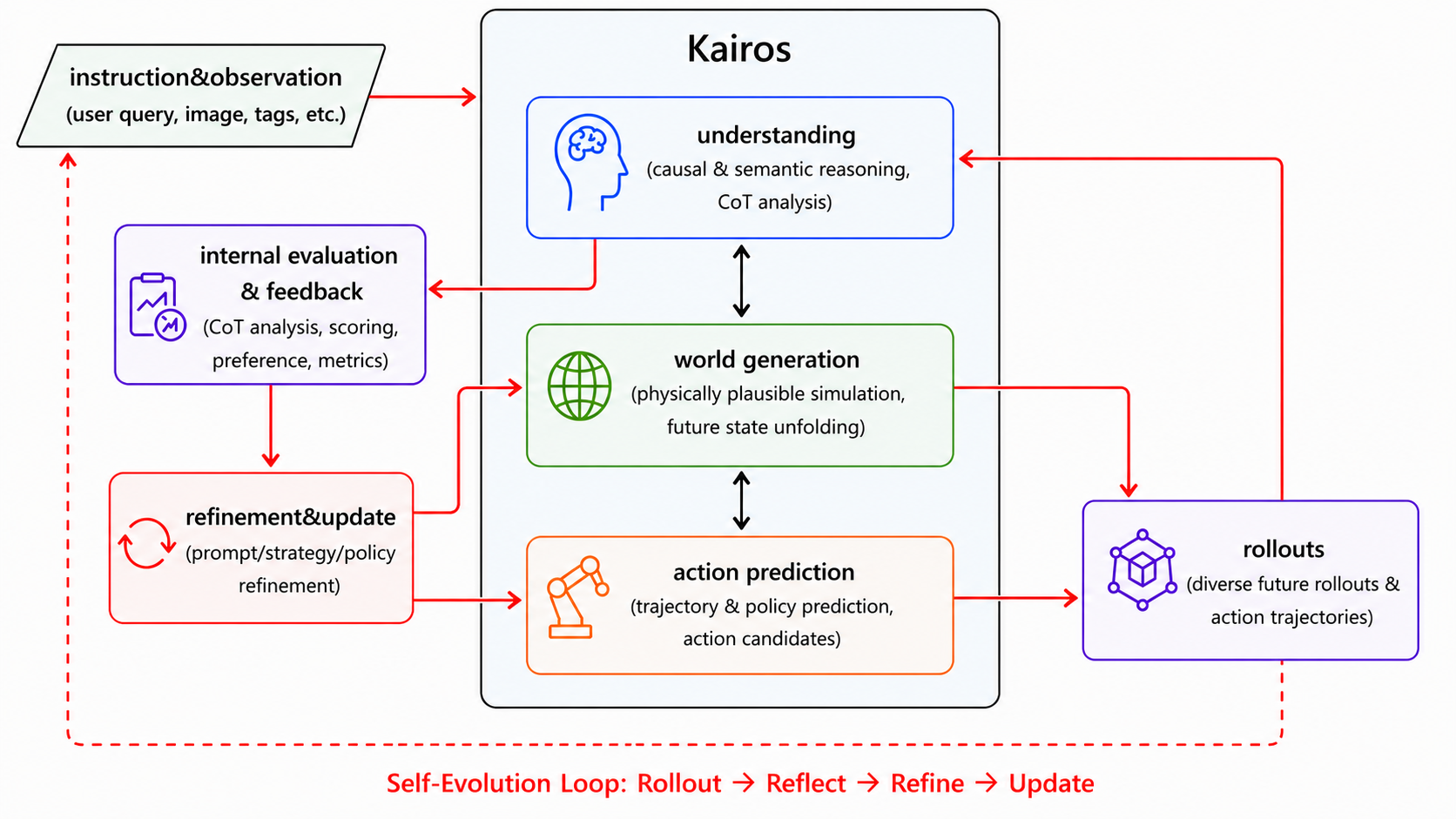}
    \caption{Proxy rollout--evaluation--refinement framework. The current implementation demonstrates prompt-level and generation-level refinement; real-robot policy self-evolution remains future work.}
    \label{fig:self-evolution}
\end{figure}

\subsection{Prompt Self-Alignment}\label{sec:prompt_self_align}

In real-world deployment scenarios, user queries and prompts exhibit significant variations in expression styles, level of detail, and linguistic conventions. Current mainstream prompt alignment solutions, adopted by leading text-to-video models such as Wan~\cite{wan2025} and Seedance~2.0~\cite{seedance2026seedance2}, universally follow a static template mapping approach: they convert user inputs into standardized dense caption formats through manual induction of model-preferred prompt paradigms and best practices. While this method is simple and effective, it fundamentally relies on the accumulation of human expertise, rendering it incapable of adaptively updating alongside iterative improvements in model capabilities and insufficient to cover the diversity of user inputs.

Built upon Kairos's unified self-evolutionary infrastructure, we propose a closed-loop Prompt Self-Alignment strategy following the generation--evaluation--iterative refinement paradigm, which reduces reliance on static templates. This strategy instantiates Kairos's general rollout--evaluation--refinement paradigm at the prompt and generation level, and is co-driven by two core components: a prompt rewriting agent with self-reflective and iterative capabilities, and a specially designed multi-dimensional video quality reward function. This reward function comprehensively quantifies the physical plausibility, task completion rate, and overall visual quality of generated videos, providing objective and consistent evaluation criteria for the self-evolutionary process. From a regret-aware perspective, these dimensions are not arbitrary aesthetic criteria: physical plausibility, task completion, and instruction alignment are direct indicators of whether the imagined future preserves the control-relevant variables on which downstream action quality depends.

The implementation workflow is as follows. Upon receiving a user instruction, the prompt rewriting agent first generates an initial batch of diversified rewritten candidates. All candidate prompts are fed into the video generation model to obtain corresponding outputs, which are then subjected to fully automated comprehensive evaluation by the multi-dimensional reward function. The system feeds back the evaluation results along with their corresponding candidate prompts to the rewriting agent, enabling it to perform self-reflection, defect attribution, and evolutionary refinement based on actual generation outcomes, thereby producing a new round of higher-quality prompt candidates. After multiple rounds of closed-loop iteration, the system gradually converges to a higher-scoring prompt candidate, ultimately achieving significant improvement in the comprehensive quality of generated videos. The process can operate without human intervention after initialization in the tested prompt-refinement setting, while broader deployment adaptation to model-capability evolution and user-input drift remains to be validated.

From the perspective of control-sufficient world modeling, Prompt Self-Alignment should not be treated merely as aesthetic prompt enhancement. The reward function should favor prompts that help the model preserve control-relevant variables: for embodied scenes, a better prompt may explicitly describe object identity, contact relation, task goal, action sequence, physical constraints, camera viewpoint, and expected state transition. In this sense, Prompt Self-Alignment is the prompt-level instantiation of regret-aware refinement: better-conditioned rollouts expose more of the variables that matter for control, and the refinement loop systematically steers the model toward prompts that reduce ambiguity, hallucination, and downstream failure. Prompt Self-Alignment nonetheless remains a proxy: it operationalizes the rollout--evaluation--refinement loop at the prompt and generation level, while validation against real robot policy improvement is pending.

\subsection{Inference Efficiency}\label{sec:efficient_inference}

World generation within embodied AI applications imposes two contradictory yet critical operational requirements. First, to power massive data simulation platforms on cloud infrastructure, the world model must achieve low-latency and high-throughput video synthesis to accelerate policy rollouts. Second, to democratize development and facilitate rapid prototyping, the inference stack must remain highly cost-effective, allowing individual researchers to execute the model on resource-constrained, consumer-grade computing hardware. To reconcile the tension between generation fidelity and operational cost across these distinct environments, the Kairos stack introduces a deployment-aware optimization framework. We systematically address these efficiency bottlenecks through two complementary vectors: \emph{Timestep Distillation}, which structurally compresses the diffusion sampling trajectories to reduce empirical sampling latency and lower the cost of imagined rollouts; and \emph{Hardware-Aware Inference Optimization}, which co-designs low-level computational kernels and memory footprints to maximize hardware utilization. The formal mechanisms of these optimizations are detailed below.

\subsubsection{Timestep Distillation}

\begin{figure}[!t] \centering \includegraphics[width=\linewidth]{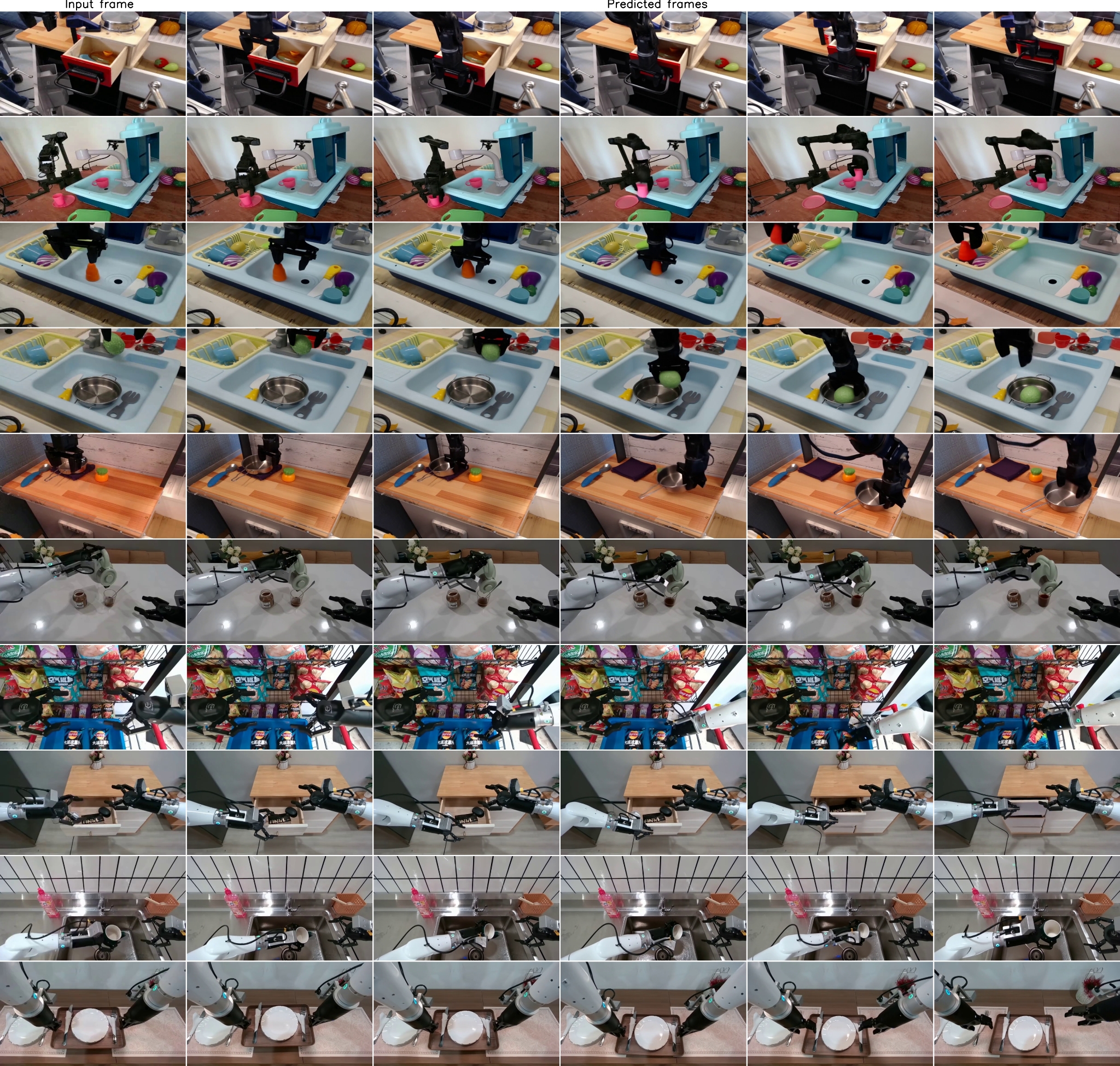} \caption{Performance of Distilled Kairos Robot Model on PAI-Bench Dataset} \label{fig:distill_paibench_robot} \end{figure}

High-resolution embodied world models can generate realistic environment dynamics and agent interactions, but their diffusion-based iterative sampling often requires dozens or even hundreds of denoising steps, creating a major computational bottleneck for deployment-oriented inference. To address this issue, we distill a pretrained 480P Embodied World Model into an efficient 4-step generator. Following Eq.~\ref{eq:fm_interpolation} with the timestep relabeled as $\sigma \in [0,1]$ to match standard distillation notation, noisy samples $\mathbf{z}_{\sigma} = \alpha_\sigma\mathbf{z}_0^{T} + \sigma\boldsymbol{\epsilon}$ (with $\alpha_\sigma \triangleq 1-\sigma$ and $\boldsymbol{\epsilon}\sim\mathcal{N}(0,\mathbf{I})$) are constructed by perturbing clean samples $\mathbf{z}_0^{T}$ from the pretrained teacher model. The student model is then trained to approximate the teacher distribution using significantly fewer sampling steps.

\textbf{Distribution Matching Distillation.}
We adopt Distribution Matching Distillation (DMD)~\cite{yin2024onestepdiffusiondistributionmatching,yin2024improveddistributionmatchingdistillation} to distill the teacher distribution into a compact generator. The student is parameterized as a feed-forward generator
\begin{equation}
\mathbf{z}_0^\theta = G_\theta(\boldsymbol{\xi}),
\end{equation}
which maps Gaussian noise $\boldsymbol{\xi}$ directly to the data space with only a few sampling steps.

Since the student defines an implicit distribution whose score function is unavailable in closed form, we introduce an auxiliary fake-score network $\phi$. Under the Rectified Flow formulation, $\phi$ learns the velocity field through
\begin{equation}
\mathcal{L}_{\text{fake}}
=
\mathbb{E}_{\mathbf{z}_0^\theta,\boldsymbol{\epsilon},\sigma}
\left[
\left\|
\mathcal{V}_\phi(\mathbf{z}_\sigma,\sigma)
-
(\boldsymbol{\epsilon}-\mathbf{z}_0^\theta)
\right\|_2^2
\right].
\end{equation}

The predicted velocity can be re-parameterized as a reconstruction of the clean sample
\begin{equation}
\hat{\mathbf{z}}_0^\theta
=
\frac{
\mathbf{z}_\sigma
-
\sigma
\mathcal{V}_\phi(\mathbf{z}_\sigma,\sigma)
}
{\alpha_\sigma},
\end{equation}
which yields the corresponding student score
\begin{equation}
\mathbf{s}_\phi(\mathbf{z}_\sigma,\sigma)
=
-\frac{\mathbf{z}_\sigma-\alpha_\sigma\hat{\mathbf{z}}_0^\theta}{\sigma^2}.
\end{equation}

This estimated score serves as a surrogate for the student's score field and enables distribution matching against the teacher during distillation. As the supervision signal, we employ the teacher's classifier-free guidance (CFG) score:
\begin{equation}
\tilde{\mathbf{s}}_T
=
(1+w)\mathbf{s}_T(\mathbf{z}_\sigma,\sigma,\mathbf{c}_{\mathrm{pos}})
-
w\,\mathbf{s}_T(\mathbf{z}_\sigma,\sigma,\mathbf{c}_{\mathrm{neg}}),
\end{equation}
where $\mathbf{s}_T(\cdot,\cdot,\mathbf{c})$ denotes the conditional teacher score, $\mathbf{c}_{\mathrm{pos}}$ and $\mathbf{c}_{\mathrm{neg}}$ are positive and negative prompt embeddings, and $w>0$ is the CFG guidance weight (we use $w$ matched to the teacher's deployment value). Distilling the CFG-enhanced teacher score enables the student to inherit both the teacher's generation quality and guided sampling behavior. Moreover, the discrepancy between positive and negative conditioning maintains a non-vanishing optimization signal even when the student approaches the unguided teacher distribution, mitigating optimization stagnation.

To align the student distribution $p_\theta$ with the teacher distribution $p_T$, we minimize the forward KL divergence $D_{\mathrm{KL}}(p_\theta \| p_T)$. Since the student distribution is implicit, directly optimizing the KL objective is intractable. Using score-function identities, its gradient can be expressed as
\begin{equation}
\nabla_\theta \mathcal{L}_{\mathrm{DMD}}
=
\mathbb{E}_{\mathbf{z}_0^\theta,\boldsymbol{\epsilon},\sigma}
\left[
\omega(\sigma)
\left(
\tilde{\mathbf{s}}_T
-
\mathbf{s}_\phi(\mathbf{z}_\sigma,\sigma)
\right)
\frac{\partial \mathbf{z}_\sigma}{\partial \mathbf{z}_0^\theta}
\frac{\partial \mathbf{z}_0^\theta}{\partial \theta}
\right],
\end{equation}
where $\omega(\sigma)$ is a noise-dependent weighting factor. Consequently, the DMD gradient estimator encourages the student score field to match the teacher score field, thereby aligning the student distribution with the teacher distribution.

Consistent with degradation phenomena reported in prior distillation literature, applying DMD to embodied world models can exhibit several failure modes during development: prolonged training tends toward \textbf{instability and mode collapse}, producing repetitive or distorted generations; the distilled generator can show a \textbf{motion diminution effect}, favoring conservative trajectory updates and underestimating agent dynamics; and \textbf{visual homogenization} can appear, where scene diversity gradually decreases and backgrounds converge toward simplified textures. These observations---qualitative in our development process rather than quantitatively benchmarked here---suggest that distribution matching alone is insufficient to fully preserve the temporal dynamics and geometric fidelity of the teacher's trajectories, motivating the additional regularization strategies introduced in the following sections.

\textbf{Consistency Distillation.}
To further stabilize training and preserve the teacher's generation trajectory, we introduce a continuous-time consistency objective based on consistency models (CM)~\cite{song2023consistencymodels,lu2025simplifyingstabilizingscalingcontinuoustime,zheng2026largescalediffusiondistillation}. While DMD aligns the student with the teacher at the distribution level, CM additionally enforces trajectory consistency by encouraging the student prediction at a noisy state to match the teacher prediction at a neighboring point along the teacher ODE trajectory. This regularization improves training stability and helps preserve fine-grained temporal and structural details.

Given a noisy sample $\mathbf{z}_{\sigma_{n+1}}$, we first construct its neighboring state on the teacher ODE trajectory using a single Euler step:
\begin{equation}
\hat{\mathbf{z}}_{\sigma_n}
=
\mathbf{z}_{\sigma_{n+1}}
+
(\sigma_n-\sigma_{n+1})
\mathcal{V}_{\text{tea}}
(\mathbf{z}_{\sigma_{n+1}},\sigma_{n+1}),
\end{equation}
where $\mathcal{V}_{\text{tea}}$ denotes the teacher velocity field.

The consistency objective is then defined as
\begin{equation}
\mathcal{L}_{\text{CM}}
=
\mathbb{E}_{\mathbf{z}_{\sigma_{n+1}},n}
\left[
\lambda(\sigma_n)
\left\|
\mathbf{f}_{\theta}
(\mathbf{z}_{\sigma_{n+1}},\sigma_{n+1})
-
\mathbf{f}_{\text{tea}}
(\hat{\mathbf{z}}_{\sigma_n},\sigma_n)
\right\|_2^2
\right],
\end{equation}
where $\mathbf{f}_{\theta}$ and $\mathbf{f}_{\text{tea}}$ denote the student and frozen teacher models, respectively. Optimizing $\mathcal{L}_{\text{CM}}$ enables the student to approximate the teacher's multi-step generation trajectory with significantly fewer sampling steps while maintaining high fidelity.

\textbf{Hybrid Objective.}
Our final training objective integrates the strengths of both distributional alignment and trajectory consistency. Specifically, we jointly optimize the consistency objective and the DMD score-matching objective:
\begin{equation}
\mathcal{L}
=
\mathcal{L}_{\text{CM}}
+
\lambda_{\text{score}}
\mathcal{L}_{\text{DMD}},
\end{equation}
where $\lambda_{\text{score}}$ balances the two terms. This hybrid formulation creates a synergistic effect: $\mathcal{L}_{\text{CM}}$ stabilizes the distillation process by anchoring the student to the teacher's trajectory and preserving structural integrity, while $\mathcal{L}_{\text{DMD}}$ leverages CFG-guided teacher supervision to improve generation fidelity and perceptual quality.

\textbf{Qualitative Results on PAI-Bench.}
Fig.~\ref{fig:distill_paibench_robot} presents qualitative comparisons on the PAI-Bench benchmark under the TI2V setting. Despite requiring only four inference steps, the distilled generator qualitatively reproduces much of the teacher's spatial structure, motion dynamics, and physical interactions. Fine-grained scene details and coherent object trajectories largely persist, and dynamic agents exhibit stable and realistic motions; we leave quantitative comparison against the teacher (e.g., FVD, success-rate parity on downstream embodied benchmarks) to future work.

Compared with the original teacher sampling process, the distilled model achieves comparable visual quality and temporal consistency with substantially reduced sampling cost. These results demonstrate that the proposed distillation framework effectively transfers both distributional knowledge and trajectory dynamics from the teacher model, enabling high-fidelity embodied video generation with efficient 4-step inference.

\subsubsection{Hardware-Aware Inference Optimization}\label{sec:hw_aware}

\paragraph{Low-Latency Generation on Cloud Service Platforms.}
This aims to achieve low-latency and rapid acquisition of generated video data on cloud service platforms to meet user demand for fast interactive experience in embodied scenarios.

\textbf{Mixed-Parallel Inference Optimization.} While the Kairos DiT model has a moderate parameter count, it exhibits extremely long input sequence lengths in each attention block. We adopt a parallel strategy centered on sequence parallelism and supplemented by tensor parallelism. We further incorporate design insights from Megatron-LM Sequence Parallelism and DeepSpeed-Ulysses Sequence Parallelism, and propose a customized hybrid parallelism scheme tailored to our model architecture:
\begin{itemize}
    \item \textbf{Sliding-Window Attention Block.} Using Ulysses sequence parallelism, each GPU maintains the full weights, and the input is split along the sequence dimension. Data synchronization between GPUs is achieved via All-To-All communication, with each GPU only responsible for attention computation for its assigned heads.

    \item \textbf{Cross-Attention Block.} Adopting the basic Sequence Parallel method. Since the context information of the attention KV is limited, the results can be precomputed and cached in full on each GPU. We only split the query sequence along the sequence dimension, enabling each GPU to compute attention using its local query fragment and the global full key--value (KV) pairs. The results are then aggregated via All-Gather.

    \item \textbf{Gated DeltaNet.} Adopting a modified Tensor Parallel method. The weights are split by head and distributed across different GPUs. Each GPU receives the full input sequence but processes attention computation in micro-batches to reduce memory usage.

    \item \textbf{VAE Decoder.} The VAE decoder divides the video into multiple segments along the timeline and executes these segments in parallel across different GPUs.
\end{itemize}

\textbf{DiT-Cache Optimization.} TeaCache acts as a dedicated optimization accelerator for the Kairos DiT. It reuses calculation results correlated with time steps, which significantly reduces inference latency and GPU memory footprint, thus improving overall inference efficiency without altering the model structure.

\textbf{Compilation and Kernel Fusion Optimization.} Torch.compile directly accelerates neural network training and inference by automatically applying graph optimizations, kernel fusion, and hardware-specific optimizations. In addition, we implement a set of dedicated fusion operators to improve performance.

\paragraph{Cost-Effective Computation on Consumer-Grade Devices.}
This aims to enable cost-effective inference computing on consumer-grade low-memory computing devices to meet the prototype development and usage needs of individual developers.

\textbf{Low-Precision Computing Optimization (FP8).} We apply quantization to the attention layers but not to linear layers. This mainly involves the following technical points:
\begin{itemize}
    \item \textbf{Q/K in INT8/INT4.} Keep query ($Q$) and key ($K$) in INT8 or INT4 for faster QK computation, while ensuring their outputs match the dynamic range of FP8 for the subsequent $PV$ MatMul.

    \item \textbf{PV MatMul ($P\cdot V$) in FP8.} Quantize the attention weight matrix $P$ (after softmax) and the value matrix $V$ to FP8, leveraging hardware-accelerated FP8 Tensor Core instructions.

    \item \textbf{Smooth $Q$} by computing $Q = Q - \mathrm{mean}(Q)$ (channel-wise mean subtraction) to narrow the value distribution, improving the precision of INT4/INT8 quantization.

    \item \textbf{Per-thread / per-warp quantization for $Q/K$.} Finer granularity ($16\times$ smaller than per-block) boosts quantization precision without extra overhead, critical for FP8 downstream computation.
\end{itemize}

\textbf{Hardware-Aware Memory Optimization.} Due to the stringent memory capacity constraints inherent to consumer-grade graphics hardware, Kairos implements specialized low-level architectural optimizations to drastically compress the runtime memory footprint:
\begin{itemize}
    \item \textbf{Tiled Gated DeltaNet with Streaming Access.} Beyond conventional intra-GPU tensor parallelization across attention heads, we introduce an intra-sequence batching and tiling mechanism tailored for long-horizon tokens. By deploying a \emph{Tile-Based Computation and Streaming Access} paradigm, the network partitions extensive sequence dimensions into discrete, highly localized blocks (tiles). Queries ($Q$), Keys ($K$), and Values ($V$) are subsequently processed in a synchronized streaming pipeline. This strategy maximally overlaps hardware tensor computations with asynchronous memory transactions, effectively concealing DRAM access latency and avoiding out-of-memory (OOM) triggers during long-horizon generation.

    \item \textbf{Weight-Only INT4 Text Encoder Quantization.} To mitigate the massive memory overhead imposed by text conditioning, the text encoder utilizes an aggressive INT4 weight-only quantization protocol. By preserving high-precision activations while compressing stationary model weights to 4-bit representations, this scheme maximizes computational throughput and drastically reduces structural memory allocation. Empirically, this design facilitates sub-millisecond keyword grounding with negligible semantic accuracy degradation, making high-fidelity world-action generation highly viable on standard edge-computing hardware.
\end{itemize}

\subsubsection{Efficiency Comparison}\label{sec:efficiency_comparison}

We tested the performance of the current model on chips with varying architectures and types, as shown in the following tables. Experimental results show that Kairos achieves strong generation performance on both professional server GPUs and consumer-grade GPUs.

\begin{table}[H]
\centering
\footnotesize
\caption{Latency Comparison on Various Hardware Platforms, tested on the Kairos-4B-robot 480P (5s) distillation model.}
\label{tab:different hardware}
\begin{tabular}{ccccc}
\toprule
\textbf{GPU} & \textbf{Resolution} & \textbf{Memory (GB)} & \textbf{1 GPU (s)} & \textbf{4 GPUs (s)} \\
\midrule
NVIDIA A800     & 480P & 23.5 & 11.7 & 3.0 \\
NVIDIA RTX5090  & 480P & 13.9 & 11.4 & 5.7 \\
\bottomrule
\end{tabular}
\vspace{0.3em}
\end{table}

Table~\ref{tab:different hardware} reports latency comparisons on various hardware platforms. These results confirm that our model maintains robust generation performance on both professional and consumer-grade GPUs, while its efficient memory utilization supports the generation of longer videos and higher-resolution visual content. Notably, 480P video generation on the NVIDIA A800 reaches real-time throughput under the 4-GPU setting.

\begin{table}[h]
\centering
\caption{Latency Comparison on Various Models. The evaluation was conducted in TI2V mode, with a video resolution of 720P and a duration of 5 seconds.}
\label{tab:cmp table}
\footnotesize
\setlength{\tabcolsep}{4pt}
\begin{tabular}{lcccc}
\toprule
\textbf{Model}  & \textbf{Memory (GB)} & \textbf{Complexity (PFlops)} & \textbf{1 GPU (s)} & \textbf{4 GPUs (s)} \\
\midrule
Lingbot-28B~\cite{lingbot-world}                                  & 46.1 & 347.4 & 5525 & 1436 \\
Cosmos-Predict2.5-14B~\cite{nvidia2025worldsimulationvideofoundation} & 70.2 & 156.5 & 2526 & 687  \\
Wan2.2-5B~\cite{wan2025}                                          & 23.4 & 16.6  & 201  & 85   \\
\textbf{Kairos-4B}                                                & 23.5 & 2.3   & 43   & 9    \\
\bottomrule
\end{tabular}
\end{table}

To further contextualize these findings, we conducted performance tests to evaluate performance discrepancies among different models under identical configurations, as presented in Table~\ref{tab:cmp table}. The table compares the memory usage, computational complexity, and inference latency of four models on an NVIDIA A800 server. Among them, Kairos-4B demonstrates outstanding efficiency: it consumes only 23.5~GB of memory (comparable to the lightweight Wan2.2-5B), boasts the lowest computational complexity (2.3 PFlops), and achieves the fastest inference speeds---43 seconds on 1 GPU and 9 seconds on 4 GPUs---far outperforming Lingbot-28B, Cosmos-Predict2.5-14B, and Wan2.2-5B across all evaluated metrics.

To eliminate confounding effects from other components in these models, we tested the performance differences of the single-step DiT model across multiple resolutions and various video generation durations, as shown in Figure~\ref{fig:vs}~(b). We can clearly observe the time cost (inference latency) of four models across different resolutions and video durations, which directly reflects the performance advantages of Kairos-4B:
\begin{itemize}
    \item \textbf{The Lowest Latency across All Test Scenarios.} Under all combinations of 480P/720P resolutions and 5s/10s/15s durations, Kairos-4B consistently achieves the lowest latency, outperforming all competing models.

    \item \textbf{Orders-of-Magnitude Speedup over Larger Models.} Compared to Cosmos-Predict2.5-14B, Kairos-4B achieves a $28\times$--$85\times$ latency reduction. Even against the smaller Cosmos-Predict2.5-2B model, it maintains a performance advantage of $6\times$ to $23\times$, validating the effectiveness of our optimization strategies.

    \item \textbf{Superior Efficiency over Similar-Parameter Competitors.} Relative to Wan2.2-5B, Kairos-4B delivers a $2.5\times$ to $3.7\times$ speedup, demonstrating higher computational efficiency despite its smaller parameter scale.

    \item \textbf{Stable Scalability under Increasing Workloads.} As resolution and duration rise (from 480P 5s to 720P 15s), other models show exponential latency growth, while Kairos-4B scales linearly with increased workload, making it highly suitable for long-duration, high-resolution video generation.
\end{itemize}

\subsection{Inference Modes for Kairos}\label{sec:inference_modes}

Kairos supports multiple inference modes because different Physical AI use cases require different trade-offs between visual fidelity, action relevance, latency, and compute cost.

\textbf{Full Visual Rollout Mode.} Kairos generates future video observations conditioned on instruction, image context, camera control, or other signals. Useful for simulation, data generation, qualitative inspection, failure analysis, prompt self-alignment, and long-horizon state probing. It is valuable when the objective is to inspect whether the model preserves physical plausibility, object permanence, task progress, and temporal coherence---making the model's imagined future visible and interpretable. It is, however, the most expensive mode because it requires future visual token generation and VAE decoding, and is less suitable for low-latency control loops where explicit video materialization is unnecessary.

\textbf{Latent Rollout and Evaluation Mode.} Kairos can evaluate imagined futures in latent or compressed form without always producing full-resolution video. Useful for candidate ranking, risk estimation, trajectory evaluation, and future policy filtering. It offers a compromise between interpretability and efficiency. This mode is especially relevant for future regret-aware deployment because many decisions do not require photorealistic rendering: a policy evaluator may only need to know whether an action is likely to succeed, collide, slip, or require recovery.

\textbf{Action-Only Prediction Mode.} Central to the World-Action Model design. During training, Kairos jointly learns future visual dynamics and future action trajectories; during deployment, the future video branch can be disabled and only future action tokens are generated. Since action tokens are much fewer than video tokens, this is expected to reduce both attention and diffusion cost while retaining the benefits of jointly learned world dynamics; quantitative latency savings for the action-only path are not yet reported in this work. This asymmetric design is, in our view, the clearest path toward deployment-ready world-action modeling.

\textbf{Proxy Self-Alignment Mode.} Uses the rollout--evaluation--refinement loop for prompt rewriting, generation refinement, and candidate selection. This is the current validated self-improvement mechanism in Kairos. It demonstrates that Kairos can generate candidates, evaluate outputs, and refine inputs without direct human intervention. It does not yet demonstrate autonomous robot policy self-improvement.

\textbf{Future Closed-Loop Robot Mode.} The future goal: Kairos would receive real observations, maintain $Z_t$, generate or evaluate candidate actions, predict failure or safety risk, execute selected actions, observe real outcomes, and update its model/evaluator/policy based on discrepancy between imagined and real rollout. Required validation includes:
\begin{itemize}
    \item correlation between imagined and real rollouts;
    \item failure prediction accuracy before execution;
    \item calibration of risk and uncertainty estimates;
    \item effectiveness of safety filtering;
    \item recovery success after predicted or actual failures;
    \item measurable improvement in real policy performance from imagined experience.
\end{itemize}
The current report establishes the prerequisites for this stage; end-to-end real-robot validation of the items above is positioned as future work.

\subsection{Summary: From Fast Generation to Deployment-Aware World-Action Operation}\label{sec:inference_summary}

This section reframes inference in Kairos as deployment-aware operation of control-sufficient states. Inference is not merely the final decoding step of a generative model; it is the mechanism through which $Z_t$ becomes useful for future Physical AI. Kairos currently provides four inference-side capabilities: (1)~a rollout--evaluation--refinement loop; (2)~Prompt Self-Alignment as a controlled proxy; (3)~timestep distillation for low-step embodied generation; (4)~hardware-aware optimization across server and consumer devices. These should be interpreted as model-side and inference-side prerequisites for future regret-aware Physical AI. The next stage is to connect these inference mechanisms to real robot rollouts, failure prediction, safety filtering, recovery learning, and measurable policy improvement.

Thus, deployment-aware inference is regret-relevant not because faster generation directly proves lower regret, but because low latency and efficient memory use determine whether cost-relevant predictions can enter the observation--action--feedback loop before execution.

\section{Evaluation Results}\label{sec:evaluation}

\subsection{Evaluation Scope and Proxy-Evidence Framing}\label{sec:eval_scope}

The goal of this evaluation is to assess whether Kairos learns several capabilities required for control-sufficient world-action modeling. As argued in earlier sections, a world model for Physical AI should not be evaluated only by visual fidelity or video-generation quality. The more important question is whether the model preserves information useful for action: physical plausibility, instruction grounding, task progress, joint world-action prediction, long-horizon state consistency, and deployment-ready inference.

However, the current evaluation should be interpreted with a precise scope. The evaluations in this section provide \emph{proxy evidence} for regret-relevant capabilities. They do not directly measure real-world closed-loop regret reduction. Using the notation introduced in the Introduction, the current evaluation does not directly estimate $\operatorname{Reg}_H(f;g)$. Instead, it evaluates observable proxies for quantities that would enter $J_H$ or $\widehat J_H$: physical plausibility, instruction grounding, action prediction, long-horizon state consistency, regret-relevant future prediction, failure-relevant reasoning, and deployment readiness. In particular, the current evaluation does not yet establish whether imagined rollouts are highly correlated with real robot rollouts, whether Kairos can predict failures before execution in real environments, whether safety filtering reduces unsafe events, or whether imagined experience improves a real robot policy. These remain central directions for future evaluation (see Section~\ref{sec:future_works}). Table~\ref{tab:eval_scope} summarizes the scope: each row pairs an evaluation target with the current proxy evaluation, what it supports, and what it does not yet prove.

\begin{table}[t]
\centering
\caption{Scope of the current Kairos evaluation. Each evaluation provides proxy evidence for a regret-relevant capability; closed-loop real-world regret reduction is left as future work.}
\label{tab:eval_scope}
\footnotesize
\renewcommand{\arraystretch}{1.4}
\setlength{\tabcolsep}{5pt}
\begin{tabular}{p{2.8cm} p{3.6cm} p{4cm} p{3.6cm}}
\toprule
\textbf{Evaluation target} & \textbf{Current proxy evaluations} & \textbf{What it supports} & \textbf{What it does not yet prove} \\
\midrule
Embodied physical plausibility \& instruction grounding & WorldModelBench-Robot, DreamGen, PAI-Bench-Robot, human eval & Kairos learns physically plausible and instruction-aligned embodied futures & Direct real-world task success or failure avoidance \\
\rowcolor{gray!10}
Action-outcome prediction \& embodied generalization & RoboTwin~2.0, LIBERO-Plus & Joint world-action modeling improves manipulation and robustness & Full counterfactual validation under matched real initial states \\
Interventional generalization & WorldModelBench-Robot, LIBERO-Plus (human-centric pretraining and joint training ablations) & CEDC and generation--prediction joint training improve action-relevant representations & Formal guarantees under interventional distribution shift \\
\rowcolor{gray!10}
Generic physical reasoning (beyond embodied) & VideoPhy, PAI-Bench (full), WorldModelBench (full) & Kairos retains broad physical and semantic priors across diverse domains & Direct test of action prediction or closed-loop policy \\
Long-horizon state maintenance & PAI-Bench-15s & Kairos better preserves scene consistency over extended generation horizons & Minute-, hour-, or day-scale real-world task memory \\
\rowcolor{gray!10}
Regret-relevant future prediction & Regret-Relevant Cases & Evidence that Kairos preserves goal-conditioned control information in embodied rollouts & Quantitative regret reduction or closed-loop action-cost improvement \\
Deployment-oriented efficiency & Latency comparison on A800/RTX5090 and against Lingbot-28B, Cosmos-Predict2.5-14B, Wan2.2-5B & Kairos offers a favorable efficiency--capability trade-off for future deployment & Real-time closed-loop robot control across all hardware and tasks \\
\bottomrule
\end{tabular}
\end{table}

This organization follows the principle that current benchmarks are necessary but not sufficient. A model that fails physical plausibility or instruction grounding is unlikely to support Physical AI; a model that performs well on these benchmarks may have learned useful physical and semantic priors. But the final test of a Physical AI world model is whether it reduces costly real-world mistakes. The results below should therefore be read as evidence that Kairos establishes several model-side prerequisites for regret-aware Physical AI.

\subsection{Embodied World Model Benchmarks}\label{sec:eval_embodied}

Embodied world-model benchmarks evaluate whether Kairos can generate future observations that are physically plausible, instruction-aligned, temporally coherent, and relevant to embodied scenarios. We evaluate Kairos-robot-4B on three benchmarks---WorldModelBench-Robot~\cite{li2025worldmodelbench}, DreamGen Bench~\cite{jang2025dreamgen}, and PAI-Bench-Robot~\cite{zhou2025pai}---and complement them with human evaluation. The results show that Kairos-robot-4B achieves strong performance across these embodied benchmarks despite its compact 4B parameter scale. This supports the claim that Kairos learns physical and instruction-grounded priors relevant to embodied world modeling.

\subsubsection{WorldModelBench-Robot}\label{sec:eval_wmb_robot}

WorldModelBench \cite{li2025worldmodelbench} is designed to assess the world modeling capability of video generation models, particularly their ability to follow instructions and adhere to real-world physics across diverse domains. It primarily evaluates models along two dimensions: (1) Instruction Following, which measures whether the generated videos accurately follow the given text prompts (and image), and (2) Future Frame Generation, which assesses whether the generated videos represent plausible future states of the world, including adherence to physical laws and common sense reasoning. As these capabilities are essential for embodied reasoning, we conduct our evaluation on the robotics subset of WorldModelBench.

As shown in Table~\ref{tab:worldmodelbench}, Kairos-4B achieves the highest total score of 9.30 on the WorldModelBench robot subset, outperforming all baselines. It obtains the best Instruction Following score (2.36), matching the 16B Cosmos3-Nano, which indicates its strong language grounding capability. In Physics Adherence, Kairos-4B reaches perfect scores in Newtonian mechanics, fluid dynamics, and gravity (1.00 each), achieving a high overall Physics Adherence score of 4.96. Additionally, it demonstrates robust Common Sense reasoning, highlighted by perfect temporal quality (1.00). Overall, these results show that Kairos-4B delivers leading physical modeling and reasoning capabilities while remaining highly parameter-efficient.  Figure~\ref{fig:qualitativ_wmbench} presents qualitative results of selected samples from the WorldModelBench robot subset.

\begin{table}[t]
\centering
\caption{Evaluation on WorldModelBench Robot Set. For each column, the highest score is \textbf{bolded}. Models marked with * indicate results reproduced by our team.}
\label{tab:worldmodelbench}
\resizebox{\textwidth}{!}{%
\begin{tabular}{l c c c c c c c c c c c}
\toprule
\multirow{2}{*}{Model} & \multirow{2}{*}{Param} & \textbf{\begin{tabular}[c]{@{}c@{}}Instruction\\ Following\end{tabular}} & \multicolumn{6}{c}{\textbf{Physics Adherence}} & \multicolumn{2}{c}{\textbf{Common Sense}} & \multirow{2}{*}{\textbf{\begin{tabular}[c]{@{}c@{}}Total\\ Score\end{tabular}}} \\
\cmidrule(lr){3-3} \cmidrule(lr){4-9} \cmidrule(lr){10-11}
& & Overall & Newton & Deform. & Fluid & Penetr. & Grav. & Overall & Frame & Temp. & \\
\midrule
Lingbot*~\cite{lingbot-world} & 28B & 2.14 & \textbf{1.00} & 0.96 & \textbf{1.00} & 0.96 & \textbf{1.00} & 4.92 & \textbf{1.00} & 0.98 & 9.04 \\
Cosmos3-Nano*~\cite{aditi2026cosmos3omnimodalworld} & 16B & \textbf{2.36} & \textbf{1.00} & \textbf{0.98} & \textbf{1.00} & 0.98 & \textbf{1.00} & 4.96 & 0.98 & 0.96 & 9.26 \\
Abot-Physworld*~\cite{chen2026abot} & 14B & 2.10 & \textbf{1.00} & 0.92 & \textbf{1.00} & 0.96 & \textbf{1.00} & 4.88 & \textbf{1.00} & 0.98 & 8.96 \\
Cosmos-Predict2.5*~\cite{nvidia2025worldsimulationvideofoundation} & 14B & 2.14 & \textbf{1.00} & 0.92 & \textbf{1.00} & 0.94 & \textbf{1.00} & 4.86 & \textbf{1.00} & 0.94 & 8.94 \\
Wan2.2*~\cite{wan2025} & 5B & 2.04 & \textbf{1.00} & 0.78 & \textbf{1.00} & 0.86 & 0.98 & 4.62 & 0.96 & 0.90 & 8.52 \\
Cosmos-Predict2.5*~\cite{nvidia2025worldsimulationvideofoundation} & 2B & 2.14 & \textbf{1.00} & \textbf{0.98} & \textbf{1.00} & 0.96 & \textbf{1.00} & 4.94 & 0.98 & 0.98 & 9.04 \\
GigaWorld-0*~\cite{gigaworldteam2025gigaworld0worldmodelsdata} & 2B & 1.50 & \textbf{1.00} & \textbf{0.98} & \textbf{1.00} & \textbf{1.00} & \textbf{1.00} & \textbf{4.98} & 0.98 & \textbf{1.00} & 8.46 \\
\rowcolor{gray!20}
\textbf{Kairos} & 4B & \textbf{2.36} & \textbf{1.00} & \textbf{0.98} & \textbf{1.00} & 0.98 & \textbf{1.00} & 4.96 & 0.98 & \textbf{1.00} & \textbf{9.30} \\
\bottomrule
\end{tabular}%
}
\end{table}

\begin{figure}[!t]
    \centering
    \includegraphics[width=1\linewidth]{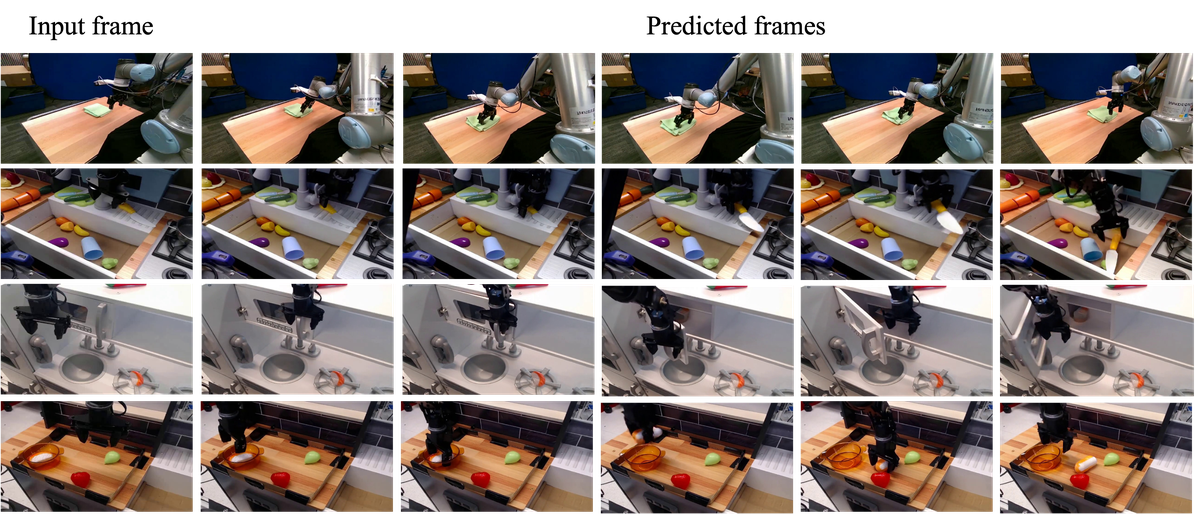}
    \caption{Kairos samples on the WordelModelBench robot subset.}
    \label{fig:qualitativ_wmbench}
\end{figure}

\subsubsection{DreamGen Bench}\label{sec:eval_dreamgen}

DreamGen Bench~\cite{jang2025dreamgen} is a video generation benchmark specifically designed for robotics, aiming to systematically measure the generalization capability of Video World Models on specific robotic embodiments. The benchmark primarily evaluates two core metrics: Instruction Following, which assesses whether generated videos strictly adhere to given task instructions; and Physics Adherence, which evaluates whether generated videos conform to real-world physical laws.
DreamGen Bench quantifies the generalization ability of video generation models by testing their performance across three key dimensions: novel object manipulation, novel behavior execution, and novel environment adaptation. Research results indicate that models achieving higher scores on this benchmark also demonstrate better performance in downstream robot policy training when using their generated synthetic data, showing a significant positive correlation between the two.

Table~\ref{tab:dreamgen} presents the evaluation results of Kairos-4B on the DreamGen Bench. Our model ranks first in both Average Physical Adherence (AVG\_PA, 0.538) and the overall Average Score (AVG\_Score, 0.618), while achieving a highly competitive Average Instruction Following (AVG\_IF, 0.698), second only to the 14B Wan2.2 (0.703). The leading AVG\_PA score reflects strong physical-plausibility modeling under this benchmark, and its strong AVG\_IF reflects reliable instruction-following ability. Despite having only 4B parameters, Kairos-4B attains the best overall performance, outperforming substantially larger competitors, which highlights its strong performance and generalization capability. Figure~\ref{fig:qualitativ_dreamgen} displays selected visualization samples from our DreamGen test set.

\begin{table}[t]
\centering
\caption{Evaluation on DreamGen Bench. For each column, the highest score is \textbf{bolded}. Models marked with * indicate results reproduced by our team.}
\label{tab:dreamgen}
\resizebox{\textwidth}{!}{%
\begin{tabular}{l c c c c c c c c c | c} 
\toprule
\multirow{2}{*}{Method} & \multirow{2}{*}{Param} & \multicolumn{2}{c}{GR1-Object} & \multicolumn{2}{c}{GR1-Behavior} & \multicolumn{2}{c}{GR1-Env} & \multirow{2}{*}{\begin{tabular}[c]{@{}c@{}}AVG\_PA\end{tabular}} & \multirow{2}{*}{\begin{tabular}[c]{@{}c@{}}AVG\_IF\end{tabular}} & \multirow{2}{*}{\begin{tabular}[c]{@{}c@{}}AVG\_Score\end{tabular}} \\
\cmidrule(lr){3-4} \cmidrule(lr){5-6} \cmidrule(lr){7-8}
& & Qwen-IF & PA & Qwen-IF & PA & Qwen-IF & PA & & & \\
\midrule
Cosmos-Predict2.5*~\cite{nvidia2025worldsimulationvideofoundation} & 14B & 0.260 & 0.515 & 0.553 & 0.418 & 0.621 & 0.553 & 0.495 & 0.478 & 0.487 \\
Cosmos-Predict2.5*~\cite{nvidia2025worldsimulationvideofoundation} & 2B & \textbf{0.840} & 0.374 & 0.277 & 0.375 & 0.586 & 0.507 & 0.419 & 0.568 & 0.494 \\
GigaWorld-0~\cite{gigaworldteam2025gigaworld0worldmodelsdata} & 2B & 0.540 & 0.481 & 0.638 & 0.446 & 0.586 & 0.529 & 0.485 & 0.588 & 0.537 \\
Wan2.2*~\cite{wan2025} & 5B & 0.420 & 0.458 & 0.553 & 0.180 & 0.690 & 0.303 & 0.314 & 0.554 & 0.434 \\
Wan2.2~\cite{wan2025} & 14B & 0.780 & 0.531 & 0.570 & 0.477 & 0.760 & 0.549 & 0.519 & \textbf{0.703} & 0.611 \\
Cosmos3-Nano*~\cite{aditi2026cosmos3omnimodalworld} & 16B & 0.460 & 0.509 & 0.468 & 0.455 & \textbf{0.793} & 0.552 & 0.505 & 0.574 & 0.540 \\

ABot-PhysWorld*~\cite{chen2026abot} & 14B & 0.400 & 0.493  & 0.404 & 0.431 & 0.414 & 0.528 & 0.484 & 0.434 & 0.459 \\
Lingbot*~\cite{lingbot-world} & 28B & 0.400 & \textbf{0.545} & 0.617 & 0.340 & 0.690 & 0.513 & 0.466 & 0.569 & 0.518 \\
\rowcolor{gray!20}
\textbf{Kairos} & 4B & 0.660 & 0.544 & \textbf{0.745} & \textbf{0.489} & 0.690 & \textbf{0.581} & \textbf{0.538} & 0.698 & \textbf{0.618} \\
\bottomrule
\end{tabular}%
}
\end{table}

\begin{figure}[!t]
    \centering
    \includegraphics[width=1\linewidth]{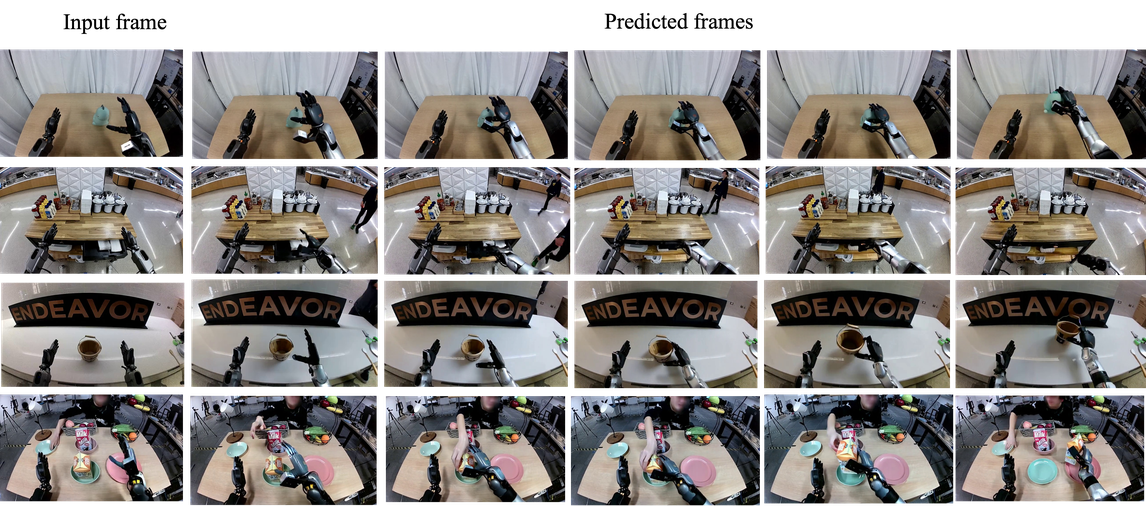}
    \caption{Kairos samples on the DreamGen dataset.}
    \label{fig:qualitativ_dreamgen}
\end{figure}

\subsubsection{PAI-Bench-Robot}\label{sec:eval_paibench_robot}

PAI-Bench \cite{zhou2025pai} is a benchmark designed to evaluate the visual quality and physical plausibility of generated videos in real-world physical AI scenarios. The benchmark reports two primary metrics: Domain Score and Quality Score. The Domain Score evaluates the model’s capability on domain-specific physical AI tasks, while the Quality Score measures the perceptual quality of generated videos. The Quality Score is computed using eight evaluation metrics adapted from VBench  \cite{huang2024vbench}. Meanwhile, the Domain Score is obtained through a VQA-based evaluation protocol spanning seven domains, including av, common, human, industry, misc, physics, and robotics. The final Overall Score is defined as the average of the Domain Score and the Quality Score. To specifically assess the model’s capability in embodied scenarios, we conduct quantitative evaluation on the Robotics subset of PAI-Bench. 

Table~\ref{tab:paibench_robot} presents the benchmark results between Kairos-4B and baseline models under the PAI-Bench TI2V evaluation mode. Among small-scale models ($<$10B), our model achieves the best overall performance, ranking first in both Domain Score (88.59) and Overall Score (82.57). Notably, using only 4B parameters, Kairos-4B matches or surpasses several large-scale ($\geq$10B) baselines---it remains essentially on par with the 16B Cosmos3-Nano (82.62 vs.\ 82.57) and outperforms 14B models such as Wan2.1 and the non-DPO Cosmos-Predict2.5. These results indicate that our model not only generates high-quality videos but also delivers strong physical modeling and instruction-following capability with far fewer parameters.
Beyond quantitative analysis, qualitative evaluation offers more intuitive insights into the model's generation behavior. Figure~\ref{fig:qualitativ_paibench} showcases visualization samples from the PAI-Bench-Robot test set, illustrating the model's generation quality across diverse scenarios.

\begin{table}[t]
\centering
\caption{Evaluation on the PAI-Bench-Robot set. For each model group, the highest score in each column is \textbf{bolded}. Models marked with * indicate results reproduced by our team.}
\label{tab:paibench_robot}
\resizebox{\textwidth}{!}{%
\begin{tabular}{l c c c c c c c c c c c}
\toprule
\multirow{2}{*}{Model} & \multirow{2}{*}{Param} & \multicolumn{8}{c}{Quality Score} & \multirow{2}{*}{\begin{tabular}[c]{@{}c@{}}Domain\\ Score\end{tabular}} & \multirow{2}{*}{\begin{tabular}[c]{@{}c@{}}Overall\\ Score\end{tabular}} \\
\cmidrule(lr){3-10}
& & i2v-bg & i2v-s & aes & img & bg-con & mot & sub-con & o-con & & \\
\midrule
\multicolumn{12}{c}{\textit{Large-scale Models ($\geq$10B)}} \\
\midrule
Lingbot*~\cite{lingbot-world} & 28B & 97.7 & \textbf{97.6} & 50.2 & 67.1 & 93.1 & 98.9 & 91.2 & 19.8 & 82.98 & 79.97 \\
Cosmos3-Nano*~\cite{aditi2026cosmos3omnimodalworld} & 16B & \textbf{98.2} & 95.3 & 48.4 & \textbf{73.0} & 92.1 & \textbf{99.3} & 91.7 & 19.7 & 88.04 & 82.62 \\
Abot-Physworld~\cite{chen2026abot} & 14B & 97.8 & 95.0 & 46.2 & 69.1 & \textbf{93.7} & 99.2 & \textbf{94.1} & 19.3 & 87.85 & 82.32 \\
Abot-Physworld + DPO~\cite{chen2026abot} & 14B & 97.7 & 94.8 & 46.7 & 69.2 & \textbf{93.7} & 99.1 & 93.6 & 19.4 & \textbf{93.06} & \textbf{84.91} \\
Cosmos-Predict2.5*\cite{nvidia2025worldsimulationvideofoundation} & 14B & 94.3 & 92.2 & 48.0 & 72.0 & 93.1 & 99.1 & 91.3 & 19.2 & 82.60 & 79.40 \\
Wan2.5~\cite{wan25} & / & 94.4 & 94.3 & \textbf{54.8} & 64.6 & 89.9 & 96.2 & 87.8 & 21.9 & 86.44 & 80.96 \\
Veo 3.1~\cite{veo31} & / & 93.2 & 96.1 & 54.6 & 72.4 & 92.2 & 97.1 & 91.5 & \textbf{22.1} & 83.50 & 80.45 \\
Wan2.1~\cite{wan2025} & 14B & 97.5 & 94.5 & 47.2 & 71.2 & 93.2 & 99.2 & 91.9 & 19.2 & 83.91 & 80.32 \\
WoW-wan~\cite{chi2025wowworldomniscientworld} & 14B & 96.1 & 92.9 & 46.6 & 70.3 & 93.0 & 98.6 & 91.5 & 19.4 & 83.01 & 79.53 \\
Sora v2 Pro~\cite{sora2pro} & / & 95.1 & 92.9 & 53.2 & 69.6 & 92.9 & 97.0 & 91.6 & 22.0 & 76.26 & 76.52 \\
\midrule
\multicolumn{12}{c}{\textit{Small-scale Models ($<$10B)}} \\
\midrule
Wan2.2*~\cite{wan2025}  & 5B & \textbf{97.8} & \textbf{97.4} & \textbf{49.3} & 70.6 & 92.3 & \textbf{99.1} & 90.8 & \textbf{19.4} & 80.17 & 78.63 \\
UnifoLM-WMA-0~\cite{unifolm_wma0} & 3B & 96.4 & 94.0 & 45.5 & 65.6 & \textbf{94.2} & 98.8 & \textbf{94.1} & 18.8 & 66.93 & 71.43 \\
Cosmos-Predict2.5*~\cite{nvidia2025worldsimulationvideofoundation} & 2B & 93.7 & 91.2 & \textbf{49.3} & \textbf{74.1} & 92.0 & \textbf{99.1} & 90.2 & 19.1 & 80.44 & 78.26 \\
GigaWorld-0~\cite{gigaworldteam2025gigaworld0worldmodelsdata} & 2B & 96.7 & 96.1 & 47.6 & 65.1 & 92.2 & \textbf{99.1} & 91.1 & \textbf{19.4} & 85.83 & 80.87 \\
\rowcolor{gray!20}
\textbf{Kairos} & 4B & \textbf{97.8} & 94.8 & 46.8 & 69.0 & 93.2 & \textbf{99.1} & 93.6 & 18.1 & \textbf{88.59} & \textbf{82.57} \\
\bottomrule
\end{tabular}%
}
\end{table}

\begin{figure}[!t]
    \centering
    \includegraphics[width=1\linewidth]{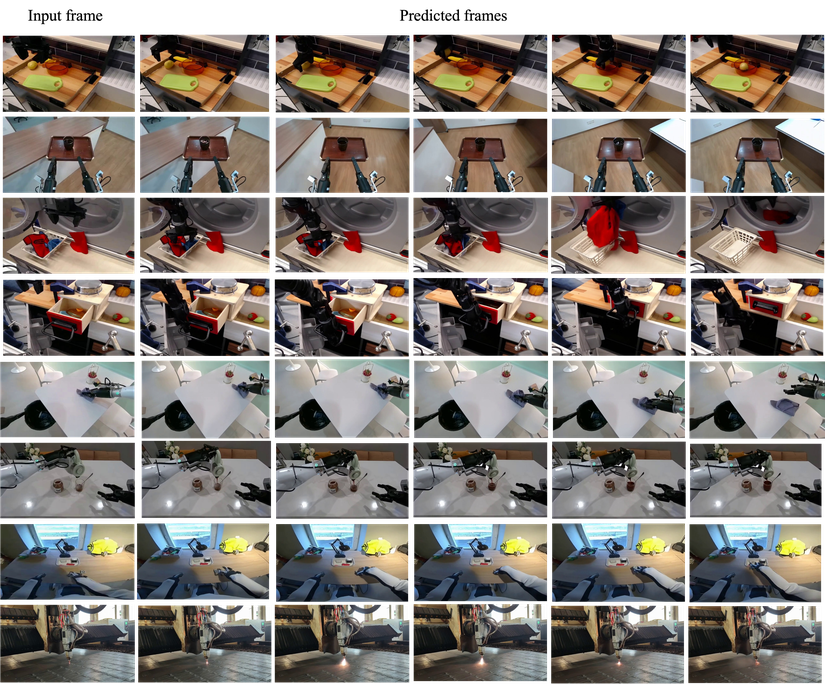}
    \caption{Kairos samples on the PAI-Bench-Robot dataset.}
    \label{fig:qualitativ_paibench}
\end{figure}

\subsubsection{Human Evaluation}\label{sec:eval_human}

Video generation involves complex physical causal logic and semantic consistency, and existing objective metrics cannot fully and accurately capture human perceptions of visual quality and semantic coherence. Human evaluation directly assesses a model’s performance in instruction following, adherence to physical laws, and content coherence, effectively compensating for the limitations of automated evaluation. Therefore, introducing human evaluation is a critical step in ensuring that generated outputs align with human preferences and possess practical application value.
To this end, we recruited 10 volunteers to conduct subjective evaluations on the complete test sets of PAI-Bench, WorldModelBench, and DreamGen for five models: Kairos-4B, Cosmos-Predict2.5 (2B/14B), Wan2.2-5B, and Lingbot-28B. To ensure fairness, all models were anonymized with random identifiers, and volunteers ranked the outputs without knowing the model identities. The final subjective evaluation results were obtained by averaging the rankings across different volunteers.

\begin{figure}[t]
    \centering
    \includegraphics[width=1\linewidth]{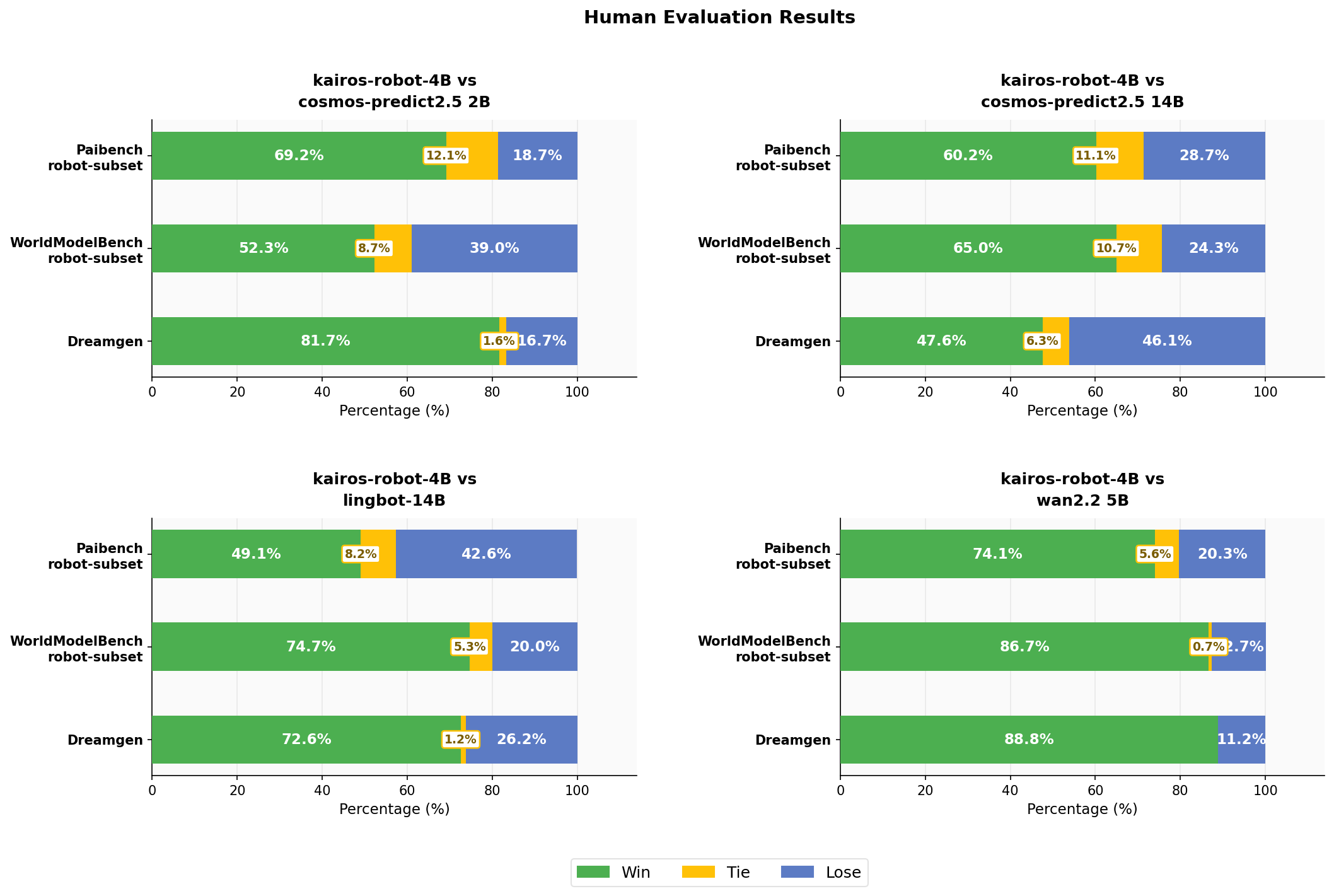}
    \caption{Human evaluation results.}
    \label{fig:qualitativ_human_eval}
\end{figure}

Figure~\ref{fig:qualitativ_human_eval} presents the pairwise win rates against the strongest baselines on each benchmark (selected comparisons; full pairwise tables for the remaining baselines are deferred for space):
\begin{itemize}
    \item \textbf{PAI-Bench robot subset.} Despite a $\sim$4--7$\times$ parameter gap, Kairos-4B reaches a 60.2\% win rate against the 14B Cosmos-Predict2.5 and 49.1\% against the 28B Lingbot; it attains a dominant 74.1\% win rate against Wan~2.2-5B.
    \item \textbf{WorldModelBench robot subset.} Kairos-4B achieves 86.7\% vs.\ Wan~2.2-5B and 74.7\% vs.\ Lingbot-28B; against Cosmos-Predict2.5-14B it leads with a 65.0\% win rate, suggesting stronger world-model consistency.
    \item \textbf{DreamGen.} Kairos-4B wins 88.8\% of comparisons against Wan~2.2-5B and 47.6\% against Cosmos-Predict2.5-14B---competitive given the 3.5$\times$ parameter gap.
\end{itemize}

\subsubsection{Ablation Studies: Human-Centric Scaling and VLM Selection}\label{sec:eval_emb_ablation}

\textbf{Effect of Human-Centric Data Scaling.} To validate the effectiveness of human-centric data scaling on our embodied world model, we perform ablations on WorldModelBench-Robot, with results summarized in Table~\ref{tab:wm_gen_ablation}. Starting from a baseline trained without scaled-up human-centric data and adopting Qwen2.5-VL-7B-Instruct~\cite{bai2025qwen25vltechnicalreport} as the VLM encoder, we inject large-scale human-centric data into pre-training. This substantially improves instruction following (2.10 $\rightarrow$ 2.33), leading to a clear gain in the overall benchmark score (9.08 $\rightarrow$ 9.25). This improvement reflects a stronger world-modeling capability: it follows instructions more accurately and generates dynamics that are better aligned with real-world interaction. We attribute this to the richer behavioral priors carried by human-centric data, which enhance the world model's generalization across diverse behaviors, resulting in more correct instruction following.

\textbf{Effect of Understanding.} Building on the large-scale human-centric pre-training data, we further upgrade the VLM encoder from Qwen2.5-VL-7B-Instruct to Qwen3.5-2B~\cite{qwen3.5}. Notably, although Qwen3.5-2B has fewer parameters, it exhibits stronger multimodal understanding capability (as shown in Table~\ref{tab:vlm_encoder_compare}). This upgrade yields further gains in both instruction following (2.33 $\rightarrow$ 2.36) and total score (9.25 $\rightarrow$ 9.30). A stronger VLM enables the world model to better interpret and align the language instruction and the visual context of the initial frame, leading to more accurate instruction grounding and higher-quality future predictions.

\begin{table}[t]
\centering
\begin{minipage}[t]{0.48\textwidth}
\centering
\captionof{table}{Ablation study on human-centric data scaling and VLM selection}
\label{tab:wm_gen_ablation}
\resizebox{\textwidth}{!}{%
\begin{tabular}{cc cc}
\toprule
\textbf{\begin{tabular}[c]{@{}c@{}}Human-Centric\\ Scaling\end{tabular}} & \textbf{\begin{tabular}[c]{@{}c@{}}Stronger VLM\\ Encoder\end{tabular}} & \textbf{\begin{tabular}[c]{@{}c@{}}Instruction\\ Following $\uparrow$\end{tabular}} & \textbf{\begin{tabular}[c]{@{}c@{}}Total\\ Score $\uparrow$\end{tabular}} \\
\midrule
            &            & 2.10 & 9.08 \\
\checkmark  &            & 2.33 & 9.25 \\
\checkmark  & \checkmark & 2.36 & 9.30 \\
\bottomrule
\end{tabular}%
}
\end{minipage}
\hfill
\begin{minipage}[t]{0.48\textwidth}
\centering
\captionof{table}{Vision-language capability comparison referred from~\cite{bai2025qwen25vltechnicalreport}~\cite{qwen35hf2025}.}
\label{tab:vlm_encoder_compare}
\resizebox{\textwidth}{!}{%
\begin{tabular}{l c c}
\toprule
\multirow{1}{*}{\textbf{Benchmark}} & \textbf{Qwen2.5-VL-7B} & \textbf{Qwen3.5-2B} \\
\midrule
MMMU~\cite{mmmu}            & 58.6 & 64.2\,/\,64.2 \\
MMMU-Pro~\cite{mmupro}        & 38.3 & 50.3\,/\,47.7 \\
MathVista\textsubscript{mini}~\cite{mathvista} & 68.2 & 76.7\,/\,73.9 \\
RealWorldQA~\cite{realworldqa}     & 68.5 & 74.5\,/\,71.2 \\
MMStar~\cite{mmstar}          & 63.9 & 71.7\,/\,68.0 \\
\bottomrule
\end{tabular}%
}
\end{minipage}
\end{table}

\subsection{World Action Model Benchmarks}
We finetuned Kairos and evaluated it on two main-stream benchmarks: LIBERO-Plus and RoboTwin 2.0. Results show that Kairos achieved highly competitive performance with only minimal downstream finetuning.

\subsubsection{RoboTwin 2.0}\label{sec:eval_robotwin}

RoboTwin 2.0 \citep{chen2025robotwin} is a challenging benchmark for bimanual robotic manipulation, comprising over 50 tasks that require precise coordination between two robotic arms. To comprehensively evaluate our method, we compare against two representative paradigms for embodied control.
The first category consists of Vision-Language-Action (VLA) models, including $\pi_0$, X-VLA, $\pi_{0.5}$, StarVLA, Abot-M0, LingBot-VLA and G0.5. These methods typically learn a direct mapping from visual observations and language instructions to robot actions. The second category comprises World-Action-Model (WAM) approaches, including JEPA-VLA, GigaWorld-Policy, Motus, LingBot-VA, Fast-WAM, Being-H0.7, AIM, SANTS, and MotuBrain. Unlike VLA-based methods, WAM approaches explicitly model both environment dynamics and action evolution, enabling joint prediction of future states and executable actions for long-horizon reasoning and planning.
As shown in Table~\ref{RoboTwin_Eval_results}, Kairos achieves strong and consistent performance across the RoboTwin 2.0 benchmark. It leads on the Clean setting (96.9) and achieves the highest average success rate (96.1), while MotuBrain attains the best Randomized score (96.1). These results demonstrate that jointly modeling world dynamics and action evolution can improve planning and execution for complex bimanual manipulation.

\begin{table}[htbp]
  \centering
  \small
  \caption{Results on RoboTwin 2.0 benchmark.}
 \label{RoboTwin_Eval_results}
  \begin{tabular}{l c c | c} 
    \toprule
    \textbf{Model} & \textbf{Clean} & \textbf{Randomized} & \textbf{Average} \\ 
\midrule
\midrule
\textbf{\# VLA}& & \\
$\pi$0 \citep{black2026pi0visionlanguageactionflowmodel} & 65.9 & 58.4 &62.2 \\
X-VLA \citep{zheng2025xvlasoftpromptedtransformerscalable} &72.9 &72.8 &72.9 \\
$\pi$0.5 \citep{intelligence2025pi05visionlanguageactionmodelopenworld} & 82.7 & 76.8 &79.8 \\
starVLA  \citep{community2026starvla} & 88.2 & 88.3 & 88.3 \\
ABot-M0 \citep{yang2026abot}  & 81.2 &80.4 & 80.8 \\
LingBot-VLA  \citep{wu2026pragmatic} &86.5& 85.3 & 85.9\\
G0.5 \citep{galaxea2026g05} & 93.7 & 92.8 & 93.2 \\
\midrule
\midrule
\textbf{\# WAM}& & \\
JEPA-VLA \citep{miao2026jepavlavideopredictiveembedding} & 73.5 & - & -  \\
GigaWorld-Policy \citep{ye2026gigaworld}& 86.0 & 85.0 & 85.5 \\
Motus \citep{bi2025motusunifiedlatentaction} & 88.7 & 87.0 & 87.8 \\
LingBot-VA \cite{lingbot-va2026} & 92.9 & 91.6 & 92.2 \\
Fast-WAM \citep{yuan2026fastwam} & 91.9 & 91.8 & 91.8 \\
Being-H0.7 \citep{beingbeyond2026beingh07} & 90.2 & 89.6 & 89.8 \\
AIM \citep{fan2026aimintentawareunifiedworld}& 94.0 & 92.1 & 93.1\\
SANTS \citep{sun2026santsstateadaptiveschedulerworld}& 94.6 & 94.2 & 94.4 \\
MotuBrain \citep{motubrainteam2026motubrainadvancedworldaction} & 95.8 & \textbf{96.1} & 96.0  \\
\midrule
\midrule
\rowcolor{gray!20} Kairos & \textbf{96.9}  & 95.2 & \textbf{96.1} \\
\bottomrule
  \end{tabular}
\end{table}

\subsubsection{LIBERO-Plus}\label{sec:eval_libero}

LIBERO-Plus \citep{fei25libero-plus} is an extended version of LIBERO. Compared with the original benchmark, LIBERO-Plus places substantially stronger emphasis on scene-level generalization, robustness to visual distribution shift, compositional manipulation reasoning, and long-horizon policy stability. As shown in Table~\ref{tab:libero_plus_eval_results}, Kairos achieves state-of-the-art performance after fine-tuning on LIBERO-Plus. The results indicate that the model generalizes effectively to perturbed evaluation settings and exhibits strong robustness across diverse environmental variations. 
\begin{table}[t]
\centering
    \caption{Results on LIBERO-Plus benchmark.}
\label{tab:libero_plus_eval_results}
\resizebox{\textwidth}{!}{%
\begin{tabular}{llllllll|l}
\toprule

\textbf{Method}&       \textbf{Camera}&\textbf{Robot}&\textbf{Language}&\textbf{Light}&\textbf{Background}&\textbf{Noise}&\textbf{Layout}&\textbf{Average}\\
\midrule
\midrule
\textbf{\# VLA}&       &&&&&&&\\
 ACoT-VLA \citep{zhong2026acotvlaactionchainofthoughtvisionlanguageaction}&       \textbf{96.6}&70.4&79.7&95.1&\textbf{97.1}&95.9&85.0&88.0\\
 $\pi_0$ \citep{black2026pi0visionlanguageactionflowmodel}&       61.0&40.8&63.5&89.3&84.1&80.1&76.4&69.4\\
 $\pi_{0.5}$ \citep{intelligence2025pi05visionlanguageactionmodelopenworld}&       75.8&\textbf{79.4}&83.3&95.5&95.0&89.6&\textbf{87.0}&85.7\\
 Being-H0.5 \citep{beingbeyond2026beingh05}&       -&-&-&-&-&-&-&83.1\\
 MINT-4B \citep{huang2026mimicintentjusttrajectories}&       -&-&-&-&-&-&-&84.1\\
 VLANeXt \citep{wu2026vlanextrecipesbuildingstrong}& 90.4& 65.7& 81.8& 95.9& 82.5& 94.1& 80.8 &83.9\\
  ProGAL-VLA \citep{darabi2026progalvlagroundedalignmentprospective}& 93.2& 71.5& 93.6& 86.8& 92.3&  74.8& 86.7&85.5\\
 RoVLA \citep{luo2026rovlamulticonsistencyconstraintsrobust}& \textbf{96.6}& 32.0& 91.5& 95.9& 96.1& 95.1& 74.1&82.0\\
 OpenVLA-OFT \citep{kim2025finetuningvisionlanguageactionmodelsoptimizing}&  92.8&  30.3& 85.8 & 94.9 & 93.9 & 89.3 & 77.6&79.6\\
 Gr00t-N1.6 \citep{nvidia2025gr00tn1openfoundation}& 92.6& 33.5& 80.1& 93.6& 95.4& 93.6& 75.0 &79.4\\
 ABot-M0 \citep{yang2026abotm0vlafoundationmodel}& 60.4& 67.9& 86.4 & 96.2 & 91.6& 86.4& 82.6&80.5\\
 \midrule
 \midrule
 \textbf{\# WAM}&       &&&&&&&\\
 Being-H0.7 \citep{beingbeyond2026beingh07}&       -&-&-&-&-&-&-&84.8\\
 \midrule
 \midrule
\rowcolor{gray!20} Kairos&       95.5&72.6&86.8&\textbf{97.7}&95.8&\textbf{96.8}&81.5&89.0\\
\rowcolor{gray!20} Kairos-joint&       95.9&74.6&\textbf{95.3}&97.1&\textbf{97.1}&95.4&83.8&\textbf{90.8}\\
\bottomrule
\end{tabular}
}
\end{table}

\subsubsection{Ablation Studies: Human-Centric Pretraining and Joint World-Action Training}\label{sec:eval_wam_ablation}

\textbf{Effect of Embodied Human-centric Pretraining.} To investigate the impact of embodied human-centric pretraining, we compare Kairos with VideoDiT pretrained either with or without large-scale human-centric data, while keeping all other settings unchanged. As results shown in Table~\ref{tab:human_pretraining}, incorporating human-centric data leads to a significant gain in LIBERO-Plus benchmark, which suggests the effectiveness of native human-centric pretraining. Kairos can effectively leverage transferable action-relevant knowledge learned from human-centric data to improve performance on unseen tasks. More broadly, these findings highlight the potential of large-scale human-centric data as a scalable supervision source complementary to robot trajectories, offering a promising path toward reducing real-robot data requirements and enabling more general-purpose world action models.

\textbf{Effect of Joint Training of Generation and Prediction.} Building upon the human-centric pretrained WAM, we further investigate the impact of Generation-Prediction joint training on action prediction. Specifically, an ablated WAM variant optimizing only the ActionDiT is constructed, with results shown in Table~\ref{tab:joint_training}. Training only the ActionDiT leads to a consistent performance degradation across LIBERO-Plus benchmarks. We attribute this degradation to the loss of world-modeling supervision provided by the generation objective. By jointly optimizing generation and prediction, VideoDiT learns control-relevant interaction dynamics and produces more informative visual representations, resulting in stronger conditioning signals and improved action prediction performance.

\textbf{Effect of Joint Denoising of Generation and Prediction.} 
Rather than omitting video generation during inference, we explore a variant of Kairos, denoted as Kairos-joint (see Table \ref{tab:libero_plus_eval_results}). In this configuration, future video tokens and action tokens are jointly denoised, allowing action prediction to actively attend to video generation at inference time. Experimental results demonstrate that this explicit future imagination further elevates performance from 89.0 to 90.8. This improvement highlights the distinct advantages of the joint attention mechanism in coupling generation and prediction.

\begin{table*}[t]
\centering

\begin{minipage}[t]{0.48\textwidth}
\vspace{0pt}
\centering
\caption{Effect of embodied human-centric pretraining}
\label{tab:human_pretraining}

\footnotesize
\setlength{\tabcolsep}{4pt}

\begin{tabular}{lcc}
\toprule
\textbf{Model} &
\textbf{Avg.} &
\textbf{Gain} \\
\midrule
w/o human-centric data & 83.0 & -- \\
w/ human-centric data & \textbf{89.0} & +6.0 \\
\bottomrule
\end{tabular}

\end{minipage}
\hfill
\begin{minipage}[t]{0.48\textwidth}
\vspace{0pt}
\centering
\caption{Effect of Joint Training of Generation and Prediction}
\label{tab:joint_training}

\footnotesize
\setlength{\tabcolsep}{4pt}

\begin{tabular}{lcc}
\toprule
\textbf{Model} & \textbf{Avg.} & \textbf{Gain} \\
\midrule
Action Prediction Only & 65.8 & -- \\
Video Generation \& Action Prediction & \textbf{89.0} & +23.2 \\
\bottomrule
\end{tabular}

\end{minipage}

\end{table*}

\subsection{General World Model Benchmarks}
Beyond the strong embodied capabilities of our model, we further evaluate its world modeling and physical reasoning abilities on general video generation benchmarks. In addition to the previously introduced PAI-Bench and WorldModelBench, we also include VideoPhy~\cite{bansal2024videophy}, a benchmark designed to evaluate the physical validity and semantic consistency of generated videos, with a particular focus on realistic entity interactions and motion dynamics. To examine temporal reasoning and consistency, we further conduct evaluations under a long-term (15-second) generation setting, enabling assessment across different temporal scales. For the evaluation, we select Cosmos-Predict2.5-2B/14B~\cite{nvidia2025worldsimulationvideofoundation}, and Wan2.2-5B~\cite{wan2025} as representative baseline models.

\begin{figure}[!t]
    \centering
    \includegraphics[width=1\linewidth]{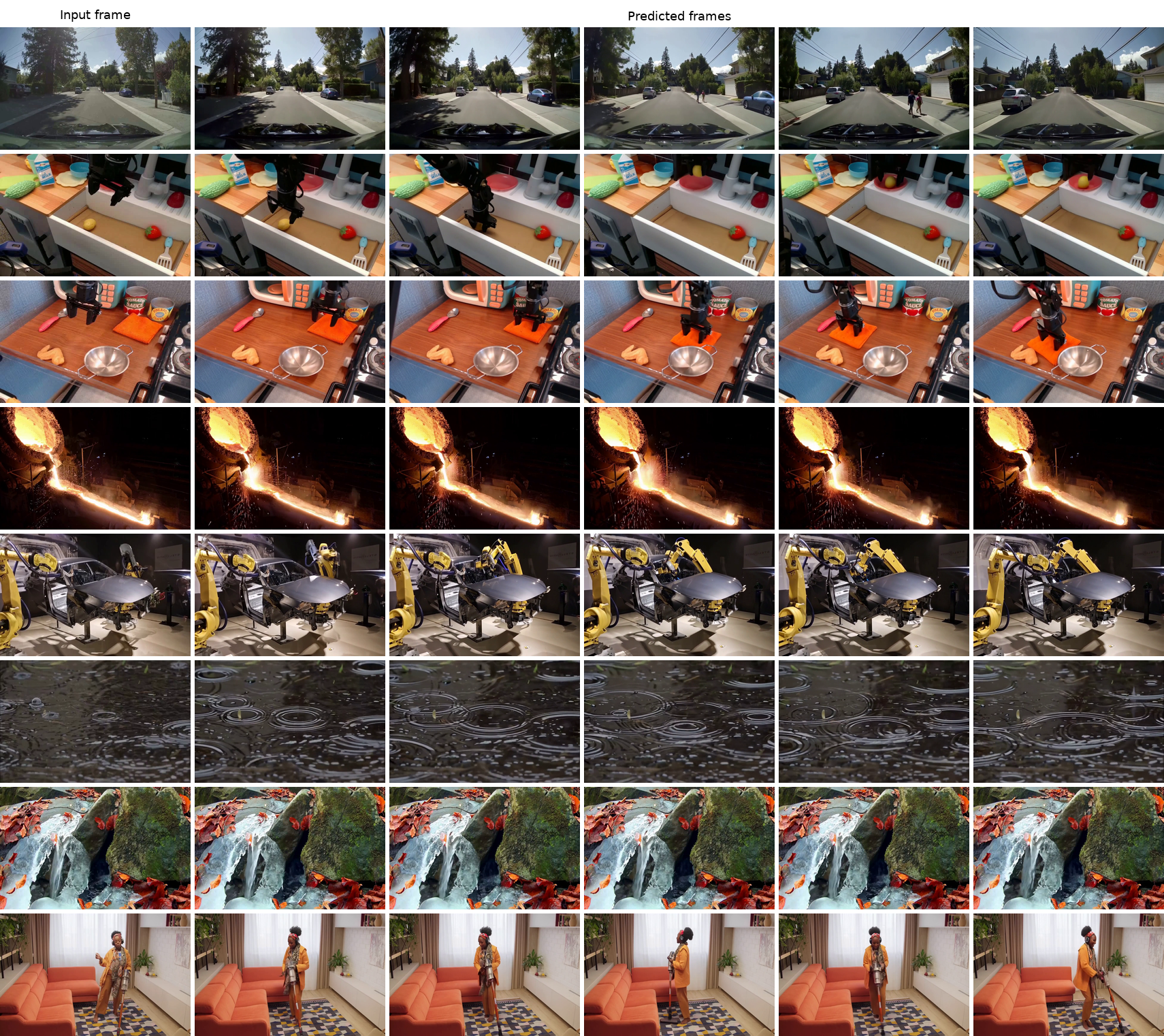}
    \caption{Kairos samples on the PAI-Bench dataset.}
    \label{fig:qualitativ}
\end{figure}

\subsubsection{PAI-Bench}\label{sec:eval_paibench_full}

Following the same evaluation protocol as in the embodied setting, we extend the evaluation to the full set of PAI-Bench domains, including av, physics, and other real-world scenarios. In the TI2V setting, we compare our method with aforementioned open-source models.
Quantitative results are summarized in Table~\ref{tab:paibench}. With only 4B parameters, Kairos achieves an Overall Score of 80.8---competitive with Cosmos-Predict2.5-2B/14B (both 81.0) while substantially smaller, and ahead of other open-source baselines at comparable scale. In particular, Kairos achieves the best i2v-background score and remains competitive on background consistency, indicating strong stability in background scenes during video generation. The consistently strong performance across multiple domains, including robotics, further demonstrates the robustness of Kairos in diverse real-world scenarios.

\begin{table}[h]
\centering
\caption{Evaluation on PAI-Bench. For each column, the highest score is \textbf{bolded}. }
\label{tab:paibench}
\resizebox{\textwidth}{!}{%
\begin{tabular}{l c cccccccc cccccc c}
\toprule
\multirow{2}{*}{Model} & \multirow{2}{*}{Param} 
& \multicolumn{8}{c}{Quality Score} 
& \multicolumn{6}{c}{Domain Score} 
& \multirow{2}{*}{\begin{tabular}[c]{@{}c@{}}Overall\\Score\end{tabular}} \\

\cmidrule(lr){3-10} \cmidrule(lr){11-16}

& & i2v-bg & i2v-s & aes & img & bg-con & mot & sub-con & o-con 
& av & cs & ro & in & hu & ph & \\

\midrule

Cosmos-Predict2.5 & 2B 
& 97.4 & 96.6 & 52.4& \textbf{70.8} & 94.2 & \textbf{99.1} & 92.5 & 20.1 
& 66.1 & 94.1 & 80.8 & 87.8 & 81.4 & \textbf{93.9} & \textbf{81.0} \\

Cosmos-Predict2.5 & 14B 
& \textbf{97.9} & \textbf{97.2} & \textbf{52.5} & 70.0 & \textbf{94.8} & \textbf{99.1} & \textbf{93.4} & 20.1 
& \textbf{67.8} & 94.2 & 79.9 & 87.7 & 80.0 & 93.5 & \textbf{81.0} \\

Wan2.2 & 5B 
& 96.7 & 95.9 & 51.9 & 69.9 & 93.7 & 98.8 & 91.8 & 20.3
& 65.2 & 93.1 & 79.3 & \textbf{88.4} & 83.0 & 91.5 & 80.4  \\
\rowcolor{gray!20}
Kairos & 4B 
& \textbf{97.9} & 96.5 & 51.9 & 68.8 & 94.5 & 98.7 & 92.0 & \textbf{21.3} 
& 64.4 & \textbf{94.3} & \textbf{84.0} & 84.5 & \textbf{84.1} & 92.8 & 80.8 \\

\bottomrule
\end{tabular}
}
\end{table}

While quantitative metrics provide an objective measure of model performance, qualitative evaluation is also important for a comprehensive assessment. Objective scores may not always fully capture the perceptual quality of generated videos, and qualitative analysis can help complement these metrics by providing a more direct examination of visual fidelity, prompt alignment, and temporal consistency. For qualitative evaluation, we select representative high-quality video samples generated by Kairos from each sub-domain of PAI-Bench, as shown in Fig.~\ref{fig:qualitativ}, covering diverse scenarios such as autonomous driving, industrial manufacturing, indoor human activities, and robotic environments. The results show that Kairos is capable of generating realistic and high-quality videos across different domains. The generated videos demonstrate strong prompt adherence and accurate first-frame conditioning, while maintaining good physical consistency throughout the video sequence.

\subsubsection{WorldModelBench}\label{sec:eval_wmb_full}

Following the same benchmark introduced in the embodied evaluation, we further evaluate Kairos on the full WorldModelBench dataset. In the embodied evaluation we focus on the TI2V setting, as embodied scenarios are more sensitive to the initial state of the robot or manipulator. Here we follow the official evaluation protocol and conduct experiments under both TI2V and T2V settings.

\begin{table}[h]
\centering
\caption{Evaluation on WorldModelBench. For each column, the highest score is \textbf{bolded}.}
\label{tab:worldmodelbench-Ti2V}
\resizebox{\textwidth}{!}{%
\begin{tabular}{l c c c c c c c c c c c}
\toprule
\multirow{2}{*}{\textbf{Model}} & \multirow{2}{*}{\textbf{Mode}} & \multirow{2}{*}{\textbf{Param}} & \multirow{2}{*}{\textbf{Instruction}} & \multicolumn{5}{c}{\textbf{Physics Adherence}} & \multicolumn{2}{c}{\textbf{Common Sense}} & \multirow{2}{*}{\textbf{\begin{tabular}[c]{@{}c@{}}Total\\ Score\end{tabular}}} \\
 \cmidrule(lr){5-9} \cmidrule(lr){10-11}
& & & & Newton & Deform. & Fluid & Penetr. & Grav. & Frame & Temporal & \\
\midrule
Cosmos-Predict2.5 & TI2V & 2B & 2.37 & \textbf{1.00} & 0.8 & \textbf{0.99} & 0.85 & \textbf{1.00}  & 0.87 & 0.82 & 8.71 \\
Cosmos-Predict2.5 & TI2V & 14B & \textbf{2.45} & \textbf{1.00} & 0.84 & \textbf{0.99} & 0.88 & \textbf{1.00}  & \textbf{0.93} & 0.87 & \textbf{8.95} \\
Wan2.2 & TI2V & 5B & 2.30 & \textbf{1.00} & 0.78 & \textbf{0.99} & 0.82 & 0.99 & 0.87 & 0.79 & 8.53 \\
\rowcolor{gray!20}
Kairos & TI2V &  4B & 2.36 & \textbf{1.00} & \textbf{0.85} & \textbf{0.99} & \textbf{0.89} & 0.99 & 0.92 & \textbf{0.89} & 8.89 \\
\hline
Cosmos-Predict2.5 & T2V & 2B & 2.30 & \textbf{1.00} & 0.91 & 0.99 & 0.90 & \textbf{1.00}  & \textbf{1.00} & 0.95 & 9.01 \\
Cosmos-Predict2.5 & T2V & 14B & 2.30 & \textbf{1.00} & \textbf{0.91} & 0.99 & \textbf{0.92} & \textbf{1.00}  & \textbf{1.00} & \textbf{0.97} & \textbf{9.09} \\
Wan2.2 & T2V &  5B & 2.18 & \textbf{1.00} & 0.87 & \textbf{1.00} & 0.88 & 0.99 & 0.99 & 0.95 & 8.87 \\
\rowcolor{gray!20}
Kairos & T2V & 4B & \textbf{2.33} & \textbf{1.00} & 0.83 & \textbf{1.00} & 0.90 & \textbf{1.00} & 0.99 & 0.93 & 8.99 \\
\bottomrule
\end{tabular}%
}
\end{table}

Table~\ref{tab:worldmodelbench-Ti2V} reports detailed quantitative comparisons between Kairos and other baseline models on this benchmark. Compared with mainstream open-source models of similar scale, Kairos achieves competitive scores under both settings, and remains close to the 14B Cosmos-Predict2.5 baseline despite its compact 4B size. In the TI2V setting, where first-frame conditioning introduces stricter constraints on video generation, Kairos obtains the best scores on several physics and common-sense metrics, suggesting strong capability in maintaining object structural consistency and avoiding physically implausible interactions.

\begin{figure}[t]
    \centering
    \includegraphics[width=1\linewidth]{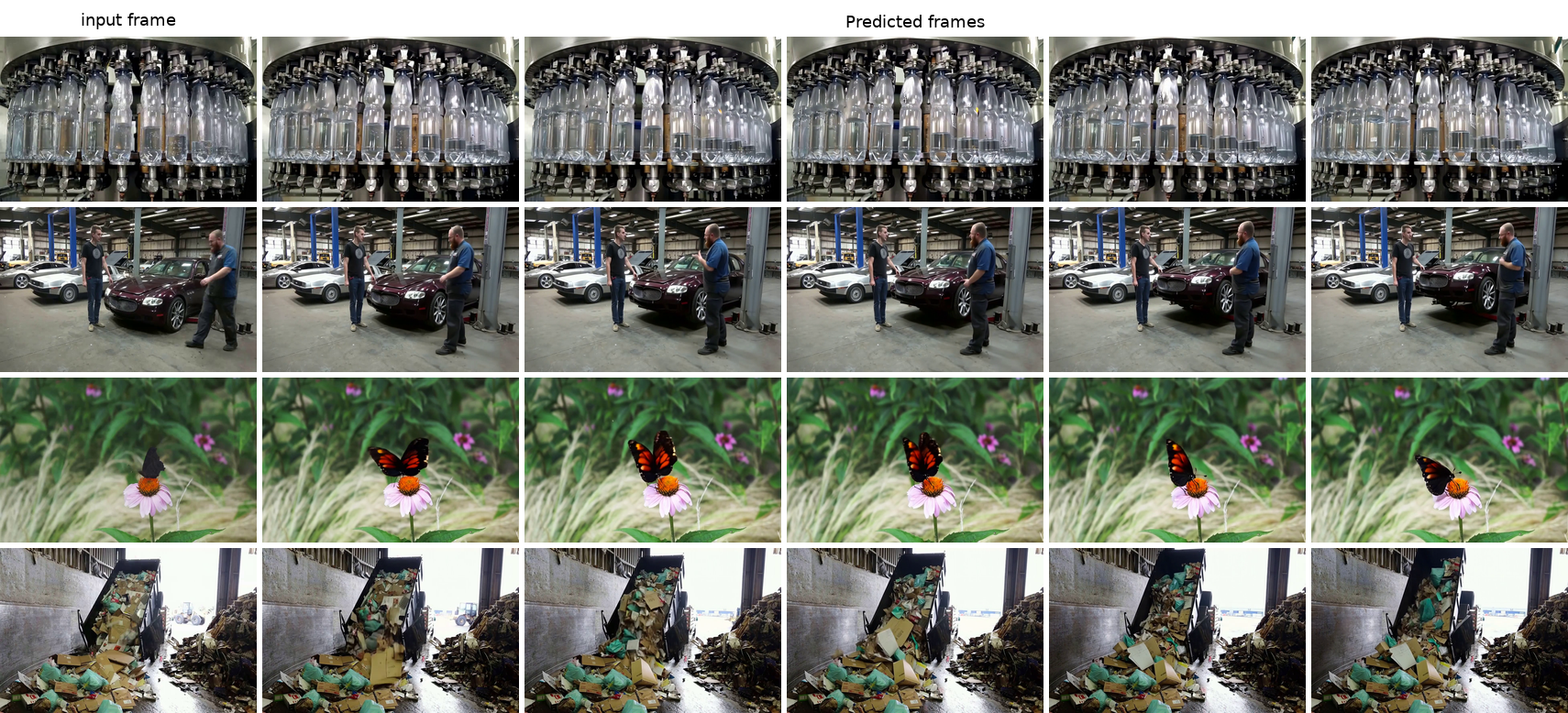}
    \caption{Kairos samples (TI2V) on the WorldModelBench dataset.}
    \label{fig:wmb_ex}
\end{figure}
\begin{figure}[!h]
    \centering
    \includegraphics[width=1\linewidth]{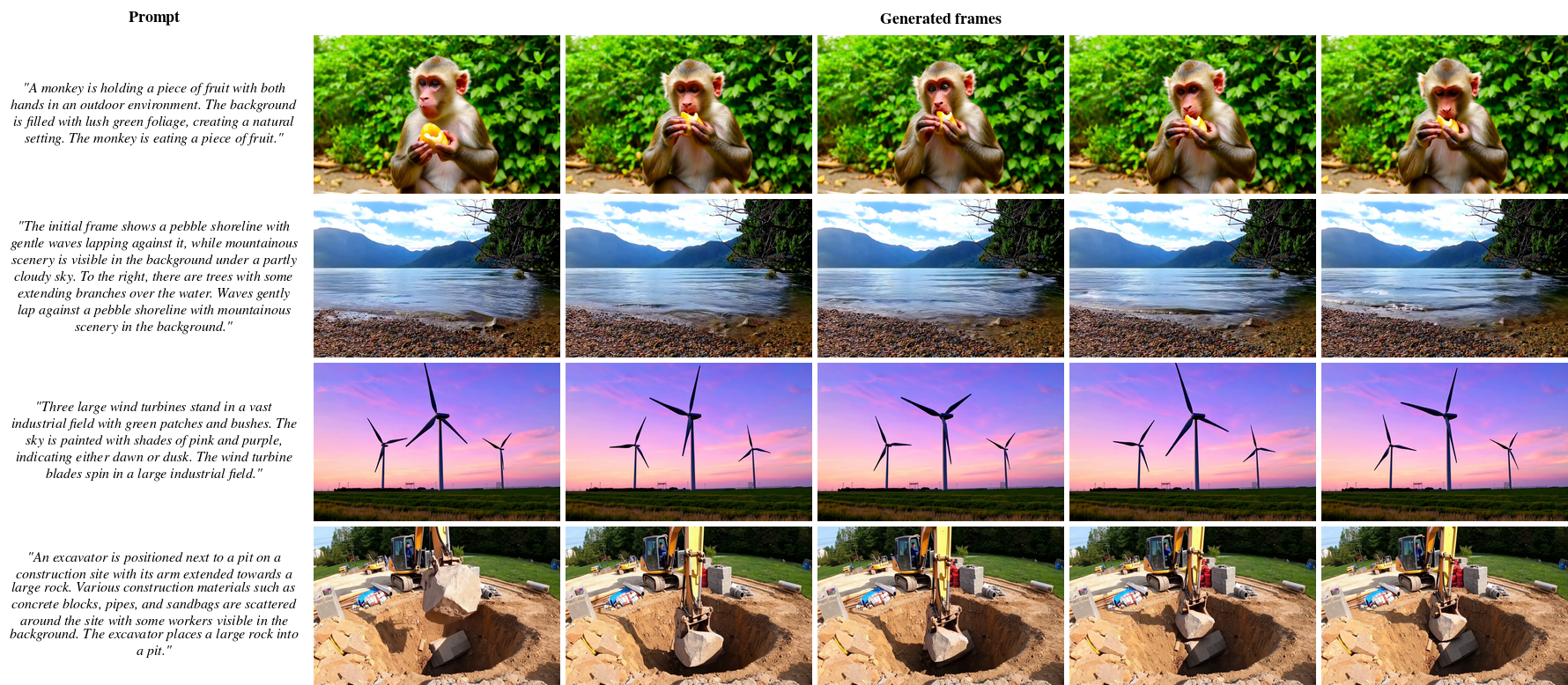}
    \caption{Kairos samples (T2V) on the WorldModelBench dataset.}
    \label{fig:wmb_ex_prompt}
\end{figure}

We further conduct qualitative analysis of Kairos on WorldModelBench, focusing on its instruction-following capability and physics adherence. As illustrated in Fig.~\ref{fig:wmb_ex}, Kairos successfully accomplishes the specified tasks across a variety of complex scenarios. For example, in scenes such as the rotation of bottled water in a spiral filling machine and a vehicle being lifted by a hydraulic elevator, the model accurately generates dynamic processes that align with the given instructions, demonstrating strong instruction-following ability. Meanwhile, the generated videos also exhibit plausible physical behaviors. For instance, flower petals tremble in response to the motion of a butterfly, and the motion trajectories and accumulation patterns of falling garbage appear physically reasonable. Similar behaviors can also be observed in the T2V setting, as shown in Fig.~\ref{fig:wmb_ex_prompt}.

\begin{figure}[t]
\centering
\includegraphics[width=1.0\textwidth]{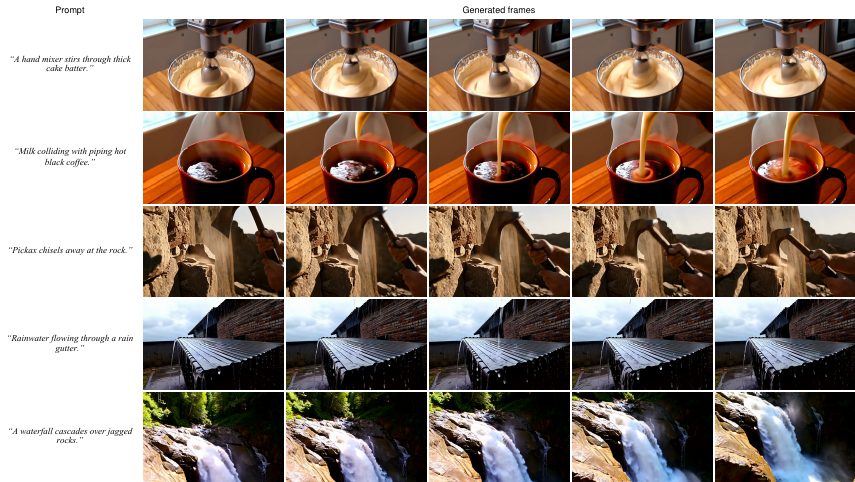}
\caption{
    \centering
Kairos samples on the VideoPhy dataset.}
\label{fig:kairos_videophy_benchmark}
\end{figure}

\subsubsection{VideoPhy}\label{sec:eval_videophy}

To evaluate the physical reasoning capability of our model, we benchmark it on VideoPhy, which contains 688 human-verified prompts. These prompts describe interactions between entities with different physical properties, including solid–solid, solid–fluid and fluid–fluid. Following the benchmark’s evaluation protocol, we generate videos using 344 test prompts. Each generated video is evaluated using two metrics: Semantic Adherence (SA) and Physical Commonsense (PC). SA measures whether the generated video correctly reflects the entities and actions described in the carefully designed prompts that simulate diverse physical interactions, while PC evaluates whether the generated scene is consistent with real-world physical laws. Both metrics are defined as binary judgments (0/1) and are computed using an auto evaluator VIDEOCON-PHYSICS provided by the benchmark. We report the \textbf{Average Score (SA=1, PC=1)} as the final results. Prompt enhancement can effectively enrich visual details in generated videos, thereby improving overall generation quality. Since the prompts in VideoPhy are relatively short, we apply model-specific prompt enhancement strategies to ensure a fair comparison across different models. Specifically, for Wan2.2-5B, we follow the official recommendation and employ Qwen/Qwen2.5-7B-Instruct for prompt enhancement. For Cosmos-Predict2.5-2B/14B, since the prompt enhancement module has been removed in its NVIDIA official implementation, we adopt the same prompt enhancement strategy used for Kairos, namely leveraging Qwen3-8B. To ensure a fair comparison, we report the best performance for each model with or without prompt enhancement.

\begin{table}[h]
\centering
\caption{Evaluation on VideoPhy. The highest score is \textbf{bolded}.}
\label{tab:videophy}
\resizebox{0.8\textwidth}{!}{
\begin{tabular}{lcccc}
\toprule
\textbf{Model}
& Cosmos-Predict2.5-2B &Cosmos-Predict2.5-14B
& Wan2.2-5B
& Kairos \\
\midrule
\textbf{Average Score}
& 44.64 & 45.16 & 38.85 & \textbf{45.55} \\
\bottomrule
\end{tabular}
}
\end{table}

As shown in Table~\ref{tab:videophy}, our model achieves the highest average score on VideoPhy with 45.55, outperforming Wan2.2-5B and Cosmos-Predict2.5-2B/14B. Remarkably, with only 4B parameters, Kairos outperforms Cosmos-Predict2.5-14B, demonstrating both high parameter efficiency and the ability to generate videos that adhere to real-world physical laws. Figure~\ref{fig:kairos_videophy_benchmark} presents qualitative results of Kairos on VideoPhy, covering a range of scenarios that involve physical interactions. For the prompt “A hand mixer stirs through thick cake batter”, the generated video shows a mixer that continuously rotates within the dense cake batter, creating clearly visible swirling and folding patterns that reflect the behavior of viscous fluids under mechanical agitation. For “Milk colliding with piping hot black coffee”, the poured milk gradually disperses into the coffee, forming a natural mixing process with smooth and coherent fluid dynamics. Our model also captures diverse physical interactions in other scenarios, producing realistic scenes such as rocks breaking under a pickaxe, rainwater flowing through a gutter, and a waterfall cascading over jagged rocks. These examples indicate that Kairos correctly reflects the entities and actions described in the prompts, demonstrating strong semantic adherence while maintaining physically consistent motion patterns and interactions.

\begin{figure}[!tbp]
\centering
\includegraphics[width=\textwidth]{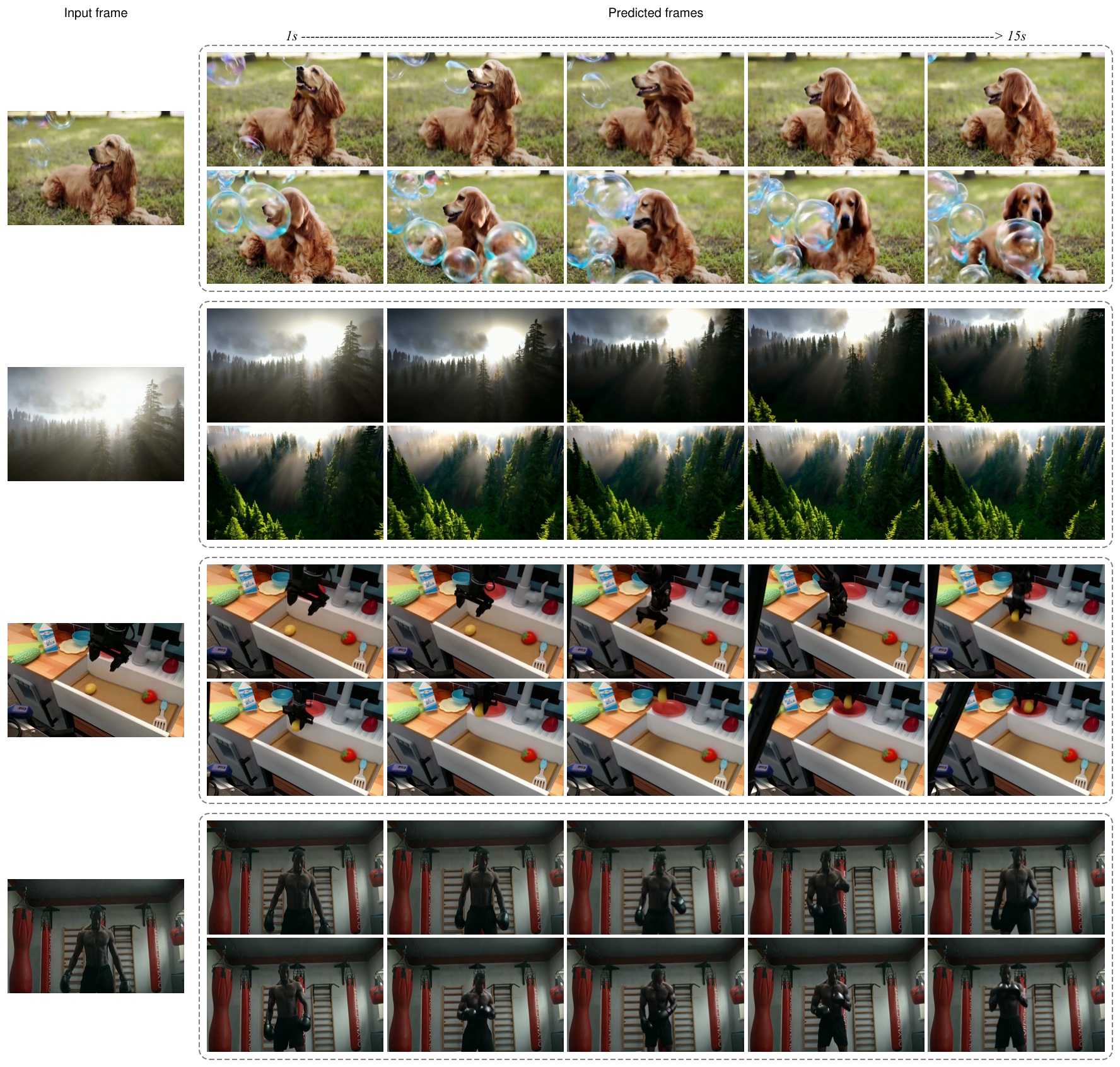}
\caption{Kairos samples on the PAI-Bench-\textbf{15s} dataset.}
\label{fig:kairos_paibench_long_horizoin_results}
\end{figure}

\subsection{Long-Horizon Generation}\label{sec:eval_long_horizon}

To further evaluate long-horizon video generation, we conduct experiments on PAI-Bench using a 15-second generation setting. PAI-Bench offers a diverse set of real-world scenarios and measures models’ ability to capture realistic physical dynamics while maintaining physically plausible behaviors. The extended 15-second generation setting enables a more comprehensive examination. Long-horizon consistency is regret-relevant because delayed failures, forgotten object states, and accumulated rollout drift can affect future physical cost.

\begin{table}[ht]
\centering
\caption{Evaluation on PAI-Bench-\textbf{15s}. For each column, the highest score is \textbf{bolded}. }
\label{tab:paibench-15s}
\resizebox{\textwidth}{!}{%
\begin{tabular}{l c cccccccc cccccc c}
\toprule
\multirow{2}{*}{Model} & \multirow{2}{*}{Param} 
& \multicolumn{8}{c}{Quality Score} 
& \multicolumn{6}{c}{Domain Score} 
& \multirow{2}{*}{\begin{tabular}[c]{@{}c@{}}Overall\\Score\end{tabular}} \\

\cmidrule(lr){3-10} \cmidrule(lr){11-16}

& & i2v-bg & i2v-s & aes & img & bg-con & mot & sub-con & o-con 
& av & cs & ro & in & hu & ph & \\

\midrule

Cosmos-Predict2.5 & 2B 
& 93.6 & 92.0 & \textbf{55.0} & \textbf{71.9} & 91.0 & \textbf{99.3} & 87.6 & 20.6 
& 59.7 & 82.6 & 72.8 & 84.9 & 77.3 & 90.6 & 77.2 \\

Cosmos-Predict2.5 & 14B 
& 93.7 & 91.8 & 52.9 & 68.4 & 92.4 & 99.2 & 89.0 & 20.4
& 61.9 & 82.6 & 68.1 & 81.2 & 75.4 & 89.2 & 76.2 \\

Wan2.2 & 5B 
& \textbf{97.9} & \textbf{97.1} & 51.9 & 67.6 & 91.5 & 99.2 & \textbf{90.7} & 20.7
& 54.3 & \textbf{90.9} & 70.1 & 84.1 & 79.6 & \textbf{91.4} & 77.8  \\
\rowcolor{gray!20}
Kairos & 4B 
& 97.1 & 95.6 & 51.9 & 68.8 & \textbf{93.4} & 98.5 & 89.8 & \textbf{21.5} 
& \textbf{66.7} & 89.2 & \textbf{80.4} & \textbf{86.8} & \textbf{83.2} & 90.0 & \textbf{79.9} \\

\bottomrule
\end{tabular}
}
\end{table}

Tables~\ref{tab:paibench} and~\ref{tab:paibench-15s} present the quantitative results of PAI-Bench under 5 and 15 second settings. When the generation horizon is extended to 15 seconds, the baseline models exhibit noticeable degradations in both quality metrics and domain-specific metrics. For example, Cosmos-Predict2.5-2B/14B shows clear drops in image-to-video consistency metrics, including i2v-bg (97.4/97.9→93.6/93.7) and i2v-s (96.6/97.2→92.0/91.8), indicating the increasing difficulty in preserving subject and background consistency in longer sequences. In addition, several domain-level scores decrease substantially, such as autonomous driving (66.1/67.8→59.7/61.9), common sense (94.1/94.2→82.6/82.6) and robot (80.8/79.9→72.8/68.1). Wan2.2-5B exhibits similar trends, with noticeable decreases in autonomous driving (65.2→54.3) and robot (79.3→70.1).

Under the 15-second setting, our model achieves the best overall score of 79.9, outperforming Cosmos-Predict2.5-2B/14B (77.2/76.2) and Wan2.2-5B (77.8). In particular, Kairos maintains strong performance in several domain-level metrics, including autonomous driving (66.7), robot (80.4), industry (86.8) and human (83.2), while also preserving high image-to-video consistency with i2v-bg (97.1) and i2v-s (95.6). Relative to the compared baselines, Kairos shows the smallest degradation when moving from 5-second to 15-second generation on Overall Score, suggesting better relative preservation of scene consistency and physical interaction; absolute long-horizon preservation beyond the 15-second setting is not evaluated here.

Figure~\ref{fig:kairos_paibench_long_horizoin_results} presents the qualitative results of our model in the long-horizon generation. The visual results show that Kairos maintains consistent object appearance and scene structure over long temporal durations while producing natural and coherent motion. In the example of a dog interacting with floating bubbles, the dog's head pose and attention gradually change following the bubble's trajectory with smooth transitions. In the forest scene, the morning mist and sunlight evolve smoothly over time. As the sun gets stronger, the mist fades and the trees in the background become more visible. These results indicate that Kairos produces temporally coherent and visually consistent long-horizon videos.

\subsection{Efficiency--Capability Trade-off as Deployment Proxy}\label{sec:eval_efficiency}

Although inference efficiency is discussed in detail in Section~\ref{sec:inference}, it is also part of evaluation because deployment readiness is a prerequisite for Physical AI. A world model that performs well offline but cannot run under practical latency and memory constraints cannot participate in observation--action--feedback loops.

The current efficiency results (cf.\ Tables~\ref{tab:different hardware} and~\ref{tab:cmp table}) show that Kairos offers a favorable efficiency--capability trade-off. Under a 720P, 5-second TI2V setting, Kairos-4B uses 23.5~GB memory, requires 2.3 PFlops, and achieves 43 seconds on 1 GPU and 9 seconds on 4 GPUs---substantially more efficient than larger baselines such as Lingbot-28B and Cosmos-Predict2.5-14B, and also more efficient than similarly scaled models such as Wan2.2-5B. Additional experiments show that Kairos maintains lower latency across 480P and 720P settings and scales approximately linearly as video duration increases.

\subsection{Regret-Relevant Cases}\label{sec:eval_regret_cases}

To further examine whether Kairos preserves information that matters for regret-aware
embodied prediction, we present two representative manipulation cases. In our formulation,
the physical cost $J_H(\cdot \mid H_t, g)$ is conditioned on the task goal $g$. Therefore,
an imagined rollout should keep the goal, manipulated object, and target state coherent
over time. A rollout may look plausible as a video, but if it drifts away from the requested
manipulation, it becomes less useful for evaluating low-cost embodied execution.

We compare Kairos with Cosmos-Predict2.5-2B using the same input frame and task goal. The
comparison focuses on whether the generated future remains organized around the intended
manipulation and shows progress toward the target state.

Figure~\ref{fig:regret_relevant_cases} shows two representative cases. In the ``put the
clothes in the washing machine'' case, Kairos keeps the rollout centered on moving the
clothes toward the washer opening, while Cosmos-Predict2.5-2B shows weaker task-completion
progress. In the ``place the yellow container into the plastic bag'' case, Kairos more
clearly preserves the intended container-to-bag interaction, whereas the baseline shows
weaker object--goal coherence and less clear insertion progress.

These cases suggest that Kairos better preserves goal-conditioned control information in
embodied future prediction. For regret-aware Physical AI, this matters because the value of
an imagined rollout depends on whether it preserves the task conditions that determine
physical cost.

\begin{figure}[t]
    \centering
    \includegraphics[width=\linewidth]{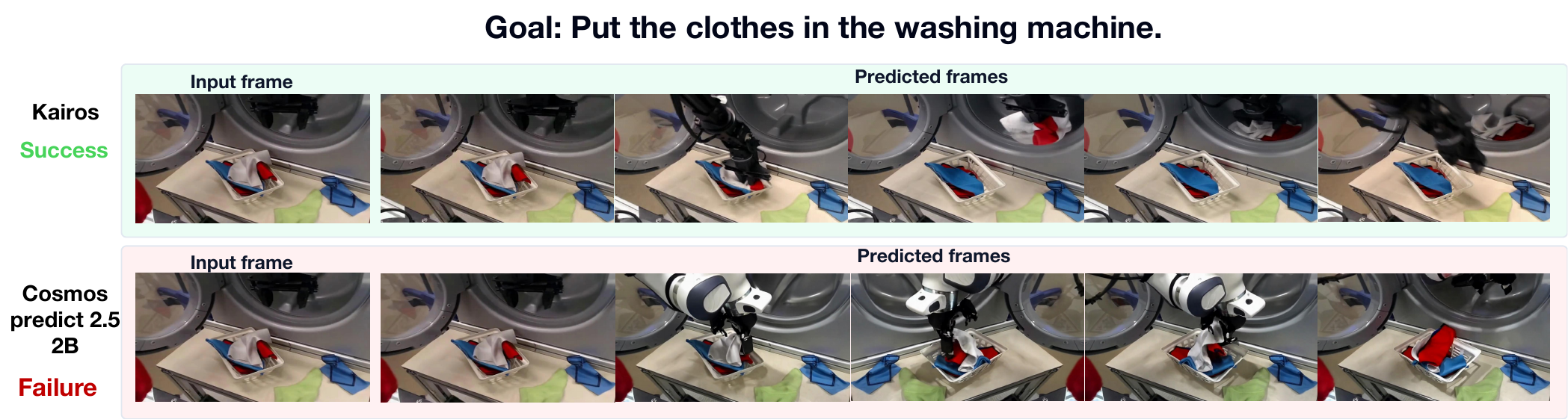}
    \vspace{0.5em}
    \includegraphics[width=\linewidth]{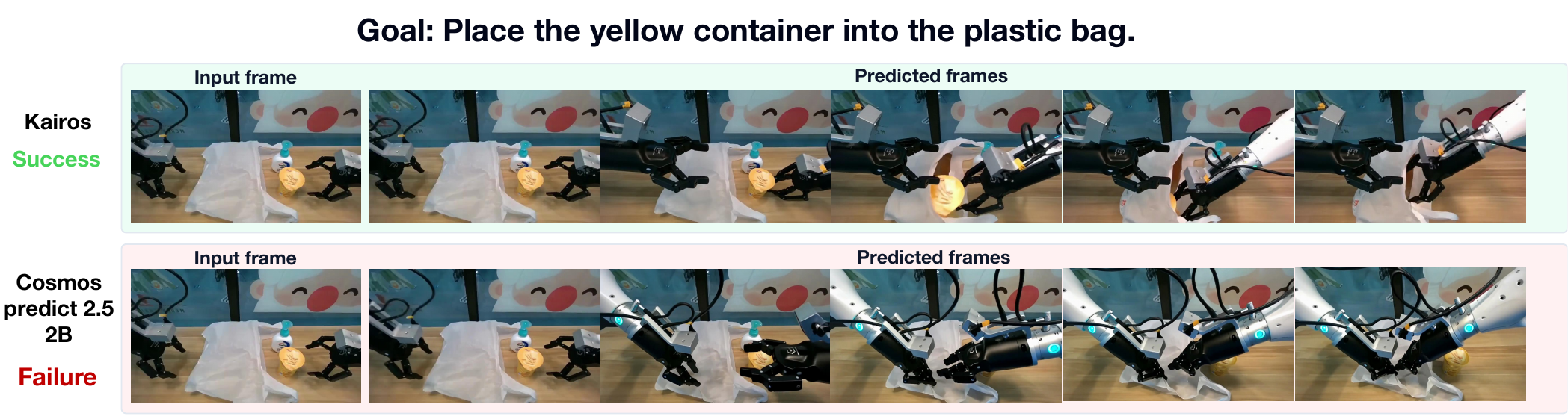}
    \caption{
    Regret-relevant cases in embodied future prediction. Given the same input frame
    and task goal, Kairos generates more goal-conditioned future rollouts than
    Cosmos-Predict2.5-2B. In the clothes-to-washing-machine case, Kairos better preserves
    the intended manipulation goal, while the baseline shows weaker task-completion
    progress. In the yellow-container-to-plastic-bag case, Kairos better preserves the
    insertion intent, whereas the baseline shows weaker object--goal coherence. These cases
    suggest that Kairos maintains goal-conditioned control information in embodied future
    prediction.
    }
    \label{fig:regret_relevant_cases}
\end{figure}

\subsection{Current Evaluation Limitations and Future Closed-Loop Validation}\label{sec:eval_limitations}

The current evaluations demonstrate strong performance across embodied generation, world-action benchmarks, general physical reasoning, long-horizon generation, and inference efficiency. However, several important capabilities required for a complete Physical AI world model remain unmeasured.

These missing evaluations are precisely the evaluations needed to move from proxy evidence about components of $J_H$ to direct evidence of regret-aware decision support. A complete validation would compare Kairos against baseline representations or policies under matched tasks, and measure whether decisions based on $Z_t$ achieve lower realized physical cost, fewer unsafe events, lower recovery effort, smaller imagined--real rollout gaps, or better task success.

\begin{itemize}
    \item \textbf{Imagined--real rollout correlation.} Current video benchmarks evaluate generated plausibility but do not directly measure whether imagined rollouts match real robot rollouts. Future work should compare predicted future states with actual robot executions under the same initial conditions.
    \item \textbf{Counterfactual action validation.} Current WAM benchmarks evaluate action performance but do not systematically test multiple alternative actions from the same initial state. Future evaluation should measure whether Kairos predicts different outcomes for different interventions, such as grasping versus pushing, lifting versus sliding, or stopping versus continuing.
    \item \textbf{Failure prediction and safety filtering.} Current benchmarks do not directly measure whether Kairos can anticipate failure before execution or reduce unsafe events. Future evaluation should test whether the model can predict slips, collisions, unstable grasps, excessive force, human-proximity risk, or irreversible states before they occur.
    \item \textbf{Recovery learning.} Current benchmarks primarily evaluate success or generated quality. Future evaluation should measure whether Kairos can represent recovery strategies after failure and whether recovery data improves policy robustness.
    \item \textbf{Policy improvement from imagined experience.} Current results show strong WAM performance and proxy self-alignment capability, but do not directly demonstrate that imagined rollouts improve a real robot policy. Future evaluation should measure whether Kairos-generated or Kairos-evaluated rollouts lead to measurable success-rate gains in real deployment.
    \item \textbf{Uncertainty and calibration.} A world model should know when its prediction is uncertain. Future evaluation should include uncertainty calibration, risk estimation, and discrepancy between imagined and real outcomes.
\end{itemize}

Table~\ref{tab:future_eval} summarizes the future evaluation protocol, which should include the following real-robot tests:

\begin{table}[t]
\centering
\caption{Future closed-loop evaluation protocol for moving from proxy capabilities to validated regret-reducing Physical AI.}
\label{tab:future_eval}
\footnotesize
\renewcommand{\arraystretch}{1.4}
\setlength{\tabcolsep}{5pt}
\begin{tabular}{p{3cm} p{5.6cm} p{5.6cm}}
\toprule
\textbf{Future capability} & \textbf{Suggested evaluation protocol} & \textbf{Target evidence} \\
\midrule
Rollout fidelity & Compare imagined rollouts with real rollouts under matched initial states & High correlation between predicted and observed outcomes \\
\rowcolor{gray!10}
Counterfactual closure & Execute different actions from the same initial state & Predicted outcome branches match real action effects \\
Failure prediction & Predict failure before execution & High precision and recall for slips, collisions, drops, task failures \\
\rowcolor{gray!10}
Safety filtering & Use Kairos to reject unsafe candidate actions & Reduction in unsafe events without large task-success loss \\
Recovery learning & Evaluate recovery after induced failures & Higher recovery success and lower human intervention \\
\rowcolor{gray!10}
Policy improvement & Train or rank policies with imagined rollouts & Measurable real policy improvement over baseline \\
Calibration & Compare predicted uncertainty with rollout error & Well-calibrated risk and uncertainty estimates \\
\bottomrule
\end{tabular}
\end{table}

These future evaluations will determine whether Kairos moves from proxy world-action capability to validated closed-loop regret reduction.

\subsection{Summary}\label{sec:eval_summary}

The evaluation results show that Kairos establishes several important prerequisites for control-sufficient world-action modeling. On embodied world-model benchmarks, Kairos-robot-4B demonstrates strong physical plausibility, instruction grounding, and parameter efficiency. On World Action Model benchmarks (RoboTwin~2.0, LIBERO-Plus), Kairos achieves strong manipulation performance and robustness, supporting the value of jointly modeling world dynamics and action evolution. Ablation studies show that human-centric pretraining and joint generation--prediction training contribute meaningfully to action-relevant performance. General world-model benchmarks show that Kairos retains broad physical and semantic priors. Long-horizon generation results suggest stronger multi-timescale state maintenance. Efficiency evaluation indicates that Kairos offers a favorable deployment-oriented trade-off.

Taken together, these results support the central claim of the report: Kairos is a regret-aware world-action model stack that learns, maintains, predicts, and runs control-relevant information for Physical AI. At the same time, the claim boundary remains clear. The next stage of Kairos evaluation should directly measure real-robot rollout correlation, counterfactual action prediction, failure anticipation, safety filtering, recovery learning, and policy improvement from imagined experience. Only then can Kairos be fully validated as a closed-loop regret-reducing Physical AI system. Overall, the evaluation results should be read as evidence that Kairos improves several components needed for a low-regret world-action state, rather than as a direct measurement of representation-induced regret itself.

\section{Related Work}

\subsection{Video Generation Models}
\textbf{Diffusion-based Video Generation.}
The success of Diffusion Models (DMs) \cite{ho2020denoising} in image synthesis has catalyzed their extension to the video domain. Early pioneers like Video Diffusion Models (VDM) {\cite{ho2022video}} first extended the standard 2D U-Net to a 3D structure \cite{wu2026geometry} by replacing 2D convolutions with space-time factorized convolutions. To alleviate the heavy computational burden of 3D operators, many subsequent works {\cite{singer2022make, ho2022imagen, wu2023tune, zhou2022magicvideo, wang2025lavie}} adopted a "spatial-then-temporal" paradigm, inserting 1D temporal attention layers after 2D spatial blocks to capture dynamic dependencies. 
A significant architectural shift occurred with the introduction of Diffusion Transformers (DiT) {\cite{peebles2023scalable}}, which demonstrated superior scalability over U-Net. This has led to the emergence of high-performance open-source video models such as LTX-Video {\cite{hacohen2024ltx}}, which refines the VAE decoder for high-frequency detail reconstruction, and HunyuanVideo {\cite{kong2024hunyuanvideo}}, which integrates Multimodal Large Language Models (MLLMs) as text encoders to enhance text-video alignment. Building upon these advancements, Wan {\cite{wan2025}} meticulously optimizes each critical module—from the autoencoder to the text-video alignment—and provides comprehensive ablation studies to facilitate future video generation research. Furthermore, recent advancements in Flow Matching {\cite{lipman2023flowmatching}} have further optimized the training efficiency and generation quality of these diffusion-based frameworks.

\textbf{Autoregressive Video Generation.} 
Another prominent paradigm treats video generation as a sequence modeling task, analogous to Large Language Models (LLMs). Early works like VideoGPT {\cite{yan2021videogpt}} combined VQ-VAE \cite{van2017neural} with GPT-like architectures \cite{radford2018improving} to autoregressively model discrete latent tokens in a spatio-temporal grid. To extend the duration of generated content, TATS {\cite{ge2022long}} introduced a time-agnostic VQGAN and a time-sensitive transformer to synthesize thousands of frames. While diffusion models have recently dominated the field, autoregressive frameworks remain highly competitive for long-term consistency and streaming generation. For instance, VideoPoet {\cite{kondratyuk2023videopoet}} utilizes a large-scale transformer to unify multiple video-related tasks within a single LLM-style framework. More recently, MAGI-1 {\cite{teng2025magi}} incorporates causal constraints and KV caching to achieve real-time, high-fidelity video synthesis, demonstrating the enduring potential of autoregressive modeling in simulating physical dynamics and causal sequences.

\textbf{Large Video Foundation Models.} 
The landscape of video generation has been revolutionized by the emergence of large-scale foundation models that act as general-purpose world simulators. A landmark moment was the introduction of Sora \href{https://openai.com/index/video-generation-models-as-world-simulators/}{\cite{brooks2024video}}, which demonstrated that Scaling Laws \cite{kaplan2020scaling} previously observed in LLMs also apply to video: increasing parameters and training data significantly enhances the model's understanding of 3D geometry and world dynamics. 
Following this, several powerful industrial models have emerged to push the boundaries of cinematic synthesis. Kling \href{https://klingai.com/}{\cite{kuaishou2024kling}} (now updated to the Kling O3 architecture in 2026) utilizes a 3D-VAE and a computationally efficient full-attention mechanism to support ultra-long, complex human motion synthesis. Similarly, Luma Dream Machine \href{https://lumalabs.ai/dream-machine}{\cite{luma2024dream}} and Runway Gen-3 Alpha \href{https://runwayml.com/blog/introducing-gen-3-alpha/}{\cite{runway2024gen3}} focus on high-fidelity motion and temporal smoothness through massive multi-modal pre-training. 
 Furthermore, ByteDance's Seedance series \href{https://lf3-static.bytednsdoc.com/obj/eden-cn/bdeh7uhpsuht/Seedance%201.0%20Paper.pdf}{\cite{bytedance2025seedance}} has introduced a unified multimodal joint generation architecture. The latest Seedance 2.0 natively supports text, image, audio, and video inputs, leveraging a Mixture-of-Transformer-Experts (MoT) design to balance spatiotemporal consistency with high-fidelity cinematic aesthetics.

\subsection{World Models}
 
World models aim to learn compact representations of environments that enable agents to predict future states and simulate interactions. Early work such as World Models~\cite{ha2018world} demonstrated that agents could be trained entirely within a latent ``dream'' environment produced by a VAE and recurrent dynamics model. This paradigm was further extended by latent dynamics approaches including PlaNet~\cite{hafner2019planet} and the Dreamer series~\cite{hafner2019dream, hafner2020mastering, hafner2023mastering}, which perform reinforcement learning directly within imagined trajectories generated by learned world models.

With the rapid progress of large-scale generative models, the concept of world models has expanded beyond reinforcement learning toward video-based simulation of real-world dynamics. Recent approaches leverage transformer and diffusion architectures to learn rich spatiotemporal representations from large video corpora. For example, UniSim~\cite{yang2023unisim} proposes a neural closed-loop simulator for autonomous driving that generates consistent sensor observations under different actions. Similarly, large-scale video generation models such as Sora~\cite{brooks2024video} have demonstrated the capability to simulate complex physical interactions and long-term temporal consistency, suggesting that generative video models may serve as general-purpose world simulators.

More recent work explores interactive world modeling by conditioning generation on actions \cite{wu2026multiworld} or control signals. GAIA-1~\cite{hu2023gaia} learns a generative world model for autonomous driving that predicts future video conditioned on vehicle controls. Genie~\cite{bruce2024genie} further introduces generative interactive environments capable of synthesizing playable worlds from video data. Recently, Cosmos~\cite{nvidia2025worldsimulationvideofoundation} proposes a large-scale world foundation model designed for robotics and physical AI, enabling realistic simulation of environments conditioned on agent actions. These advances indicate a promising direction toward unified world simulators that combine perception, dynamics modeling, and controllable environment generation.

\subsection{World Action Models}

\textbf{From VLA to WAMs.}
Generalist embodied agents have long relied on Vision-Language-Action (VLA) models that learn reactive, direct observation-to-action mappings without modeling how the physical world evolves under intervention. To overcome the short-sightedness and weak physical grounding of such purely reactive policies, the World Action Model (WAM) paradigm unifies predictive environment dynamics with motor control by jointly targeting the distribution over future states and actions. By internalizing physical laws through future simulation, WAMs enable stronger long-horizon reasoning and spatial awareness than standard VLAs—though this unified objective introduces sharp architectural trade-offs between generation fidelity and inference latency.

\textbf{Cascaded WAMs.}
Cascaded WAMs synthesize future visual representations conditioned on task goals before extracting control actions. Explicit methods first forecast raw 2D RGB frames using inverse dynamics or multimodal conditioning \cite{du2023learning,du2024video,intelligence2026pi}. To mitigate spatial hallucinations inherent in pixel-level generation, subsequent explicit architectures extract intermediate 2D geometric representations like optical flow and point tracks \cite{ko2024learning,bharadhwaj2024track2act,xu2024flow}, or extend into 3D and 4D spaces for rigorous spatial reasoning \cite{zhi20253dflowaction,wang2026mvista}. Despite providing highly interpretable plans, explicit decoding suffers from severe computational latency, rendering high-frequency closed-loop control intractable \cite{yan2026s}. To circumvent this bottleneck, implicit cascaded WAMs bypass high-dimensional image rendering by encoding anticipated futures strictly within continuous or discrete latent spaces. This latent planning involves learning quantized semantic codebooks \cite{ye2025latent}, imposing geometric mask bottlenecks \cite{lou2026mask}, or operating entirely within diffusion representations \cite{huang2024ardup}. To further accelerate reactive control, frameworks extract intermediate network features \cite{hu2024video} or self-distill multi-step generation into a single feed-forward pass  \cite{yan2026s}. Ultimately, while explicit WAMs excel in modularity with off-the-shelf simulators, implicit approaches prioritize the execution scalability and efficiency required for real-time deployment.

\textbf{Joint World-Action Models.}
Departing from the two-stage pipeline, joint WAMs co-model decision-making and physical dynamics within a shared space, splitting into autoregressive and diffusion-based families. The autoregressive family casts states and actions as a unified sequence over a shared vocabulary: causal transformers first predict pixel-space trajectories as implicit planners \cite{wu2024unleashing, cheang2024gr}, then structure a visual chain-of-thought via intermediate subgoal images \cite{zhao2025cot}, and ultimately fuse text, images, and discrete actions into one vocabulary \cite{cen2025worldvla}. Since strict autoregressive action decoding causes error propagation and trajectory drift, remedies include action-attention masking \cite{cen2025worldvla} and a hybrid design with a parallel continuous action head to avoid quantization errors \cite{cen2025rynnvla}; still, sequentially decoding high-dimensional visual states incurs heavy latency that hampers reactive deployment. Conversely, diffusion-based WAMs frame prediction as denoising or flow-matching, jointly optimizing states and actions to expose diverse marginal and conditional distributions for deeper causal understanding and stronger sample efficiency. Early designs fuse all modalities in a single diffusion transformer \cite{guo2024prediction,shen2026videovla,kim2026cosmos}, while later works resolve action-image modality conflicts through separate or bridged token streams \cite{won2025dual,yang2025covar,chen2025unified}. To avoid costly pixel-level reconstruction, several adopt implicit latent world modeling \cite{zheng2025flare,zhao2026frappe,li2026world,yuan2026adaworldpolicy}, and others decouple diffusion timesteps or scale via mixture-of-transformers \cite{zhu2025unified,bi2025motus,team2026motubrain,lyu2026lda}. For inference efficiency, recent frameworks adopt lightweight heads, asynchronous sampling, or speculative decoding \cite{li2025unified,ye2026gigaworld,yuan2026fast,guo2026unified,deng2026dexworldmodel}. Across both families, the central challenge remains balancing high-fidelity joint simulation in training against lightweight, asymmetric modality decoupling at inference.


\subsection{Efficient Attention Mechanisms}

The quadratic complexity of the standard self-attention mechanism in Transformers poses a major challenge for modeling long sequences such as high-resolution videos. Given an input sequence of length $N$, vanilla attention requires $\mathcal{O}(N^2)$ time and memory, which quickly becomes prohibitive as the spatial and temporal dimensions increase. To address this limitation, a large body of work has explored efficient attention mechanisms that reduce computational complexity while preserving the modeling capacity of Transformers.

Early approaches focused on approximating the softmax attention through kernelization or low-rank decomposition. Linear Transformers~\cite{katharopoulos2020transformers} reformulated softmax attention using kernel feature maps, enabling attention computation in linear time with respect to sequence length. Similarly, Performer~\cite{choromanski2020rethinking} introduced FAVOR+ random feature approximations to achieve scalable attention with theoretical guarantees. Other works explored sparse attention patterns, restricting interactions to local or predefined structures in order to reduce computational cost \cite{child2019generating,beltagy2020longformer,zaheer2020big}.

More recent research has shifted toward designing sequence models with recurrent or state-space style updates that scale linearly with sequence length. Retentive Networks (RetNet)~\cite{sun2023retentive} replace the softmax attention with a retention mechanism that supports both parallel training and recurrent inference. The Mamba architecture~\cite{gu2024mamba,dao2024transformers,lahoti2026mamba,huang2024localmamba,pei2025efficientvmamba} further demonstrates that selective state space models can achieve strong sequence modeling performance while maintaining linear complexity. 
Building on these developments, recent works further explore gated update mechanisms to improve the stability of long-context modeling. Delta-based architectures such as Gated Delta Networks~\cite{yanggated} introduce gated update rules that mitigate key-value interference during sequence updates. In a similar spirit, Gated Linear Attention (GLA) incorporates gating mechanisms into linear attention to regulate information flow\cite{yang2023gated,li2025minimax}. By introducing learnable gates into the attention update, GLA selectively integrates new key-value information while suppressing outdated context, enabling stable long-range dependency modeling with $\mathcal{O}(N)$ complexity \cite{qin2023hierarchically,qin2024hgrn2,katsch2023gateloop,beck2024xlstm}. Such efficiency is particularly beneficial for tasks involving extremely long token sequences.

\section{Conclusion and Future Works}

This report introduced \textbf{Kairos}, a regret-aware native world-action model stack for Physical AI. We made this notion explicit through a horizon-level regret formulation: a useful compressed state should support decisions whose expected physical cost approaches the cost achievable from the full task-relevant history. The central motivation is that a world model for embodied intelligence should not be understood as a full simulator of all future pixels. In real physical systems, the decisive requirement is to learn and maintain a compact internal state that preserves the information needed for control: object state, spatial relations, contact condition, task progress, action consequences, failure boundaries, safety risks, and deployment uncertainty. We refer to this representation as a \emph{control-sufficient state}.

Kairos is designed around this principle. The current report establishes several model-side, data-side, memory-side, and inference-side prerequisites for future regret-aware Physical AI; real-world closed-loop regret minimization itself is positioned as future work. These prerequisites include acquiring control-relevant information from heterogeneous experience, compressing this information into a shared world-action state, maintaining the state across multiple temporal scales, connecting prediction to future actions, and running the model under realistic deployment constraints.

The first component of Kairos is the \textbf{Cross-Embodiment Data Curriculum}. Instead of treating open-world videos, human-centric data, and robot interaction data as a flat mixture, Kairos organizes them by intervention strength. Open-world videos provide passive physical observation; human-centric data provides intentional task structure and behavior; robot data provides embodied action grounding. This curriculum is designed to move the model along a developmental pathway from passive physical priors toward action grounding; direct empirical validation of action--outcome causation is deferred to Section~\ref{sec:future_works}. Its significance is not only data scale, but the organization of heterogeneous experience into a path that can support control-sufficient state learning.

The second component is the \textbf{Native Understanding--Generation--Prediction Architecture}. In Kairos, understanding, generation, and prediction are not separate modules connected after the fact. They are three interfaces to a shared world-action state $Z_t$. World Understanding constructs $Z_t$ from multimodal history, instruction, robot state, and physical context. World Generation regularizes and probes $Z_t$ through physically plausible future imagination. World Prediction turns $Z_t$ into a world-action interface through joint modeling of future visual states and future action tokens. This architecture moves Kairos beyond passive visual forecasting toward a world-action model that can, in future extensions, support counterfactual action evaluation.

The third component is \textbf{Hybrid Linear Temporal Attention}. Long-horizon Physical AI requires more than longer context: it requires maintaining control-relevant state variables across different timescales. Sliding-window attention supports local dynamics such as motion continuity and contact changes; dilated sliding-window attention supports mid-range dependencies such as object interaction and subtask transitions; gated linear attention supports persistent global memory for object permanence, task progress, delayed effects, and long-range causal context. The theoretical analysis in this report supports the need for persistent memory and provides bounds for hybrid multi-scale temporal memory under stated assumptions. It should be interpreted as theoretical support for the design, not as a universal guarantee of real-world long-horizon correctness.

The fourth component is the \textbf{Deployment-Aware System Co-Design}. For Physical AI, inference efficiency is not merely an implementation detail: a world model that cannot run under practical latency, memory, communication, and hardware constraints cannot participate in observation--action--feedback loops. Kairos therefore incorporates timestep distillation, hardware-aware inference optimization, mixed parallelism, quantization, caching, and action-only prediction pathways. These mechanisms improve deployment readiness and make imagined rollout or action prediction more practical. They do not, by themselves, prove real-time closed-loop robot control; instead, they provide necessary infrastructure for future deployment-side validation.

The fifth component is the \textbf{Data Engine for Control Information Density}. Kairos builds a large-scale data pipeline for collection, curation, tagging, captioning, enhanced text annotation, and high-throughput processing. Under the control-sufficient perspective, the value of this pipeline is not only that it processes large volumes of data, but that it can be extended to identify and structure data that most reduces uncertainty about action consequences, contact dynamics, failure boundaries, recovery strategies, and safety risks. This reframes data construction for Physical AI from raw scale toward control information density.

The evaluation results provide encouraging evidence for these design choices. Kairos achieves strong performance on embodied world-model benchmarks, world-action benchmarks, general world-model benchmarks, long-horizon generation, and inference-efficiency evaluations. These results suggest that Kairos learns several capabilities relevant to Physical AI, including physical plausibility, instruction grounding, joint world-action prediction, long-horizon consistency, and deployment-oriented efficiency. Direct validation will require future experiments that measure imagined--real rollout correlation, counterfactual action accuracy, failure prediction, safety filtering, recovery learning, and measurable policy improvement in real robot settings.

Taken together, Kairos is a step toward control-sufficient world modeling for Physical AI. It moves world modeling beyond static video generation by integrating cross-embodiment pretraining, unified world-action state modeling, hybrid temporal memory, deployment-aware inference, and scalable data engineering. Closed-loop validation in real robots is left to future work.

The future of Kairos is therefore not simply to make generated videos more realistic or benchmarks higher. The more important goal is to determine whether the internal world-action state can help physical agents make fewer costly mistakes. The next stage of research should directly test whether Kairos can predict what matters for control, distinguish the consequences of different actions, generalize across intervention distributions, maintain state over much longer horizons, and learn from experience with high control information density.

\subsection{Future Works}\label{sec:future_works}

Future Kairos development will focus on five directions that move it from proxy world-action capability toward directly validated regret-aware Physical AI.

The regret formulation also clarifies the target of future work. The next stage is not only to improve benchmark scores, but to test whether Kairos-based decision support achieves lower realized physical cost than baseline representations or policies under real or high-fidelity simulated closed-loop execution. This requires matched-task comparisons that measure task success, unsafe events, recovery cost, human intervention, latency, and imagined--real rollout error.

\subsubsection{Direct Evaluation of Control-Sufficient States}
The first future direction is to directly evaluate whether the internal state $Z_t$ is control-sufficient. Current benchmarks mainly evaluate generated videos, action prediction scores, long-horizon consistency, or inference efficiency. These are useful proxies, but they do not directly answer the most important question: \emph{does the internal state preserve the information needed for control?}

Future work should introduce explicit probes and evaluation protocols for the following variables:
\begin{equation}
    Z_t \;\longrightarrow\; \big\{\,\hat{p}_{\text{task}},\ \hat{p}_{\text{fail}},\ \hat{p}_{\text{act}},\ \hat{p}_{\text{risk}},\ \hat{p}_{\text{gap}}\,\big\},
\end{equation}
where $\hat{p}_{\text{task}}$ denotes predicted task progress, $\hat{p}_{\text{fail}}$ denotes predicted failure events, $\hat{p}_{\text{act}}$ denotes the predicted future state under a given action, $\hat{p}_{\text{risk}}$ denotes safety or instability risk, and $\hat{p}_{\text{gap}}$ denotes the expected discrepancy between imagined and real outcomes. These probes correspond to measurable components of the physical cost $c$ and the predicted cost $\widehat{J}_H$. Their purpose is to test whether $Z_t$ contains the information needed for low-cost decision-making, rather than merely whether it supports visually plausible generation.

A useful control-sufficient state should support accurate prediction of at least five quantities.

\textbf{Task progress.} For long-horizon tasks, the model should know what has already been completed, what remains to be done, which subgoal is active, and which future state would count as successful progress.

\textbf{Failure risk.} The model should identify states that are likely to lead to grasp failure, slip, collision, object drop, unstable contact, wrong-object selection, or task interruption.

\textbf{Action consequences.} Given candidate actions, the model should estimate how the object, robot, and task state will change.

\textbf{Safety risk.} The state should encode risk related to human proximity, excessive force, collision, unstable objects, irreversible state changes, and uncertain physical interaction.

\textbf{Reality gap.} The model should estimate when its imagined future is likely to diverge from real execution. This is essential for deciding when to trust imagined rollouts and when to request real feedback, additional sensing, or conservative action.

A practical evaluation protocol can combine representation probing, rollout prediction, and downstream policy evaluation. For example, one can freeze $Z_t$ and train lightweight probes for task progress, failure, contact state, and rollout discrepancy. If these variables can be predicted from $Z_t$, this provides evidence that the representation contains control-relevant information. More stringent evaluation should test whether policies or evaluators using $Z_t$ outperform policies using visual embeddings, language embeddings, or generic video-generation latents.

The key metric should not be reconstruction quality alone. Future evaluation should report progress-prediction accuracy, failure-prediction precision and recall, risk calibration, action prediction error, and imagined--real discrepancy calibration. These metrics would directly test whether Kairos learns a state that is sufficient for control rather than merely sufficient for plausible video generation.

\subsubsection{Counterfactual Action Validation}
The second future direction is to validate counterfactual action prediction. A world model for Physical AI must answer not only ``what will happen next?'' but ``what will happen if the agent takes this action, and what would happen under a different action?'' This is the requirement of \emph{counterfactual closure}.

The current Kairos architecture is designed to support this through World Prediction, Video DiT, Action DiT, mixed attention, and action-only inference (Section~\ref{sec:world_prediction}). Future work should directly test this capability under controlled conditions.

A rigorous evaluation protocol should begin from the same initial state and execute multiple alternative actions:
\begin{equation}
    \{a^{(1)},\,a^{(2)},\,\dots,\,a^{(K)}\}.
\end{equation}
Kairos should predict the corresponding future outcomes:
\begin{equation}
    \big\{\hat{o}^{(k)}_{t+1:t+H},\,\hat{a}^{(k)}_{t+1:t+H},\,\hat{s}^{(k)}_{t+H}\big\}_{k=1}^{K}.
\end{equation}
The robot or simulator should then execute the same actions from matched or carefully reset initial states to obtain real outcomes $\{o^{(k)},s^{(k)}\}$. The evaluation should measure whether the predicted outcome branches match the real outcome branches.

For \emph{manipulation}, this could include grasp versus push, lift versus slide, fast versus slow movement, left-side grasp versus right-side grasp, direct placement versus intermediate repositioning, or stop versus continue under unstable contact. For \emph{navigation}, it could include turning left versus right, taking a narrow passage versus detouring, or proceeding versus waiting in a dynamic scene.

Useful metrics include branch classification accuracy, action prediction error, contact outcome prediction, success/failure prediction under each action, and rank correlation between predicted and real action quality. The model should not only predict a plausible future; it should predict \emph{different} futures for different interventions and preserve the correct ordering of action choices.

This evaluation is essential because observation--action correlation is not enough for deployment. A model may imitate common actions without understanding their consequences. Counterfactual validation tests whether Kairos has moved toward action--outcome causation. It will also clarify when future video generation is necessary and when action-only inference is sufficient. A successful result would not require perfect pixel-level prediction; what matters is whether Kairos correctly predicts the control-relevant consequences of each action: whether the object moves as intended, whether contact is stable, whether the task progresses, whether a failure occurs, and whether the predicted risk matches real execution.

\subsubsection{Interventional Generalization Across Observation, Human Intervention, and Robot Intervention}
The third future direction is to evaluate interventional generalization. Kairos is trained through a Cross-Embodiment Data Curriculum that progresses from passive observation, to intentional human behavior, to embodied robot interaction. This design is motivated by the fact that Physical AI deployment is not i.i.d.\ prediction: once a robot acts, it changes the future data distribution.

Future work should directly test whether CEDC improves generalization across environments, tasks, and embodiments. The evaluation should be organized around the three intervention levels:
\begin{equation}
    \mathcal{D}_{\text{obs}} \;\longrightarrow\; \mathcal{D}_{\text{human}} \;\longrightarrow\; \mathcal{D}_{\text{robot}}.
\end{equation}

\textbf{Passive observation.} At the first level, large-scale passive observation data should provide broad physical and semantic priors. Evaluation should test whether models trained with such data generalize better to unseen objects, unseen scenes, and unseen physical phenomena.

\textbf{Human intervention.} At the second level, human-centric data should provide intentional task structure. Evaluation should test whether human-centric pretraining improves instruction following, task decomposition, long-horizon planning, recovery behavior, and object manipulation priors.

\textbf{Robot intervention.} At the third level, robot data should provide embodiment-specific action grounding. Evaluation should test whether adding robot data improves action prediction, contact stability, success rate, and adaptation to robot-specific control constraints.

A strong evaluation design should compare several training regimes:
\begin{itemize}
    \item $\mathcal{D}_{\text{obs}}$ only;
    \item $\mathcal{D}_{\text{obs}} + \mathcal{D}_{\text{human}}$;
    \item $\mathcal{D}_{\text{robot}}$ only;
    \item flat mixture $\mathcal{D}_{\text{obs}} \cup \mathcal{D}_{\text{human}} \cup \mathcal{D}_{\text{robot}}$ (no curriculum);
    \item full staged CEDC.
\end{itemize}
The goal is to determine whether the full curriculum improves generalization beyond any single data source or flat mixture.

Future benchmarks should test \emph{cross-environment generalization} (new backgrounds, lighting, layouts, scenes), \emph{cross-task generalization} (new task combinations, object categories, instruction compositions), and \emph{cross-embodiment generalization} (across grippers, arms, dexterous hands, humanoid platforms, camera viewpoints, action spaces). Key metrics should include action success rate, action prediction accuracy, task-progress prediction, failure prediction, data efficiency, and robustness under perturbation. The strongest evidence would show that CEDC improves not only benchmark scores but also action prediction and policy robustness under intervention-induced distribution shift.

\subsubsection{Multi-Timescale Memory Beyond 15-Second Generation}
The fourth future direction is to extend multi-timescale memory evaluation beyond 15-second video consistency. The current long-horizon evaluation (Section~\ref{sec:eval_long_horizon}) shows that Kairos can preserve scene consistency and physical interaction over extended video generation. However, real Physical AI tasks often last much longer than 15 seconds. Household manipulation, warehouse operation, mobile navigation, inspection, care assistance, and human--robot collaboration may require state maintenance over minutes, hours, or even longer periods.

Future work should therefore evaluate memory at multiple temporal scales:
\begin{equation}
    \text{seconds} \;\to\; \text{minutes} \;\to\; \text{hours} \;\to\; \text{day-scale}.
\end{equation}

\textbf{Seconds scale.} The model should maintain local contact dynamics, motion continuity, slip, collision, and immediate action consequences.

\textbf{Minutes scale.} The model should maintain subtask progress, object locations, tool-use history, and intermediate task dependencies.

\textbf{Hours scale.} The model should maintain environment changes, repeated interaction patterns, user preferences, task schedules, and accumulated uncertainty.

\textbf{Day-scale or longer.} The model should maintain persistent world knowledge, repeated failure patterns, scene regularities, and long-term adaptation signals.

Future benchmarks should move beyond long video generation and test real task-state memory. For example, a robot may be asked to complete a multi-step household task where some objects are moved, hidden, or used earlier and become relevant later. The model should remember which objects were moved, which subtasks were completed, which failures occurred, and which recovery strategies were attempted. Another benchmark could test navigation memory, where a robot must integrate observations across many rooms and preserve spatial--semantic state over extended exploration.

Useful metrics include object permanence accuracy, subtask-state tracking, delayed-effect prediction, task-history recall, recovery-history use, and long-horizon action success. For real robot deployment, the model should be tested on whether memory improves action selection, reduces repeated mistakes, and supports recovery after delayed failures.

This direction should also evaluate the roles of the three temporal pathways in Kairos. Sliding-Window Attention should be tested on fast local dynamics; Dilated Sliding-Window Attention on mid-range event dependencies; Gated Linear Attention on persistent global context. Ablations should measure which memory branch is necessary for which temporal scale and task type. As future work, this defines a staged evaluation path from 15-second generation consistency toward real long-horizon state maintenance in Physical AI.

\subsubsection{Control-Information-Density Data Engine}
The fifth future direction is to develop a more explicit control-information-density data engine. The current Kairos data pipeline (Section~\ref{sec:data}) already supports large-scale collection, curation, tagging, captioning, enhanced text annotation, and high-throughput processing. The next step is to make data selection more directly aligned with control value.

The key principle is that not all data contributes equally to Physical AI. A short near-boundary failure or recovery clip may teach more about control than a long ordinary success clip, but not every failure should be treated as maximally valuable. Future data construction should therefore follow the priority order:
\begin{equation}
    \begin{aligned}
    &\text{near-boundary failure and recovery data}
    >
    \text{near-boundary successful data} \\
    &>
    \text{contact-rich data}
    >
    \text{ordinary successful trajectories} \\
    &>
    \text{ordinary observation videos},
    \end{aligned}
\end{equation}

\textbf{Near-boundary failure and recovery data} should include marginal grasp failures, object slips close to recovery, near collisions, unstable contacts, regrasping, repositioning, replanning, human correction, retry behavior, and successful return from error states. These data reveal where execution breaks within the normal task regime and how the agent can recover.

\textbf{Near-boundary successful data} should include marginal grasps that remain stable, near slips that are corrected, near collisions that are avoided, partial successes completed through correction, and other cases close to the decision boundary between success and failure. These data reveal the conditions under which the system remains controllable near failure or safety margins.

\textbf{Contact-rich data} should include tactile events, force interaction, friction, deformation, sliding, tool use, support changes, and grasp stability. These data reveal the physical mechanisms through which actions affect the world.

Ordinary successful trajectories and ordinary observation videos remain useful---they provide task execution examples and broad physical priors. Extreme, non-diagnostic, or far out-of-distribution failures may be useful for safety filtering and anomaly detection, but they should not dominate the data pipeline if the goal is learning action consequences, failure boundaries, and recovery strategies.

A future Kairos data engine should explicitly estimate or approximate control information density. This can be done through event detectors, robot logs, simulation labels, tactile--force signals, human annotation, rollout discrepancy, and model uncertainty. Each data segment can be scored by how much it reduces uncertainty about action outcomes, failure boundaries, contact dynamics, recovery strategies, or safety risks.

Future work should also align ego-centric human data, robot trajectories, and simulation data around shared event labels. For example, a slip event in human hand--object interaction, a robot gripper slip, and a simulator-labeled friction failure should be mapped into a common contact-failure taxonomy. This would allow Kairos to learn from human, robot, and simulated experience in a more unified way.

The data engine should support three downstream capabilities: (1)~it should improve failure prediction by providing enough examples of how and why actions fail; (2)~it should improve safety filtering by exposing unsafe, near-unsafe, and high-uncertainty states; (3)~it should improve recovery behavior by teaching the model what to do after an error occurs.

The long-term goal is to \emph{close the data loop}. Kairos should not only train on static datasets. It should eventually use real deployment feedback to identify high-value data segments, prioritize annotation, update the world-action model, and improve future policy evaluation. This would turn the data engine into a continuous source of control-relevant learning signal.

\subsection{Outlook}\label{sec:outlook}

The next stage of Kairos should be evaluated by a stricter standard than visual generation quality alone. A Physical AI world model should be judged by whether its internal state predicts task progress, whether it distinguishes the consequences of different actions, whether it generalizes across intervention distributions, whether it maintains state over real task horizons, and whether its data engine improves failure prediction, safety filtering, and recovery. Kairos has established a foundation for this direction through its cross-embodiment curriculum, native world-action architecture, hybrid temporal memory, deployment-aware inference, and scalable data pipeline. The central future question is no longer whether Kairos can generate plausible futures, but whether its imagined futures and internal states can be reliably connected to real physical outcomes.

By pursuing the five directions above, Kairos can move from proxy world-action modeling toward directly validated regret-aware Physical AI. This progression requires careful experiments, calibrated evaluation, and real robot validation. If successful, future Kairos systems will provide a practical path toward physical agents that learn from heterogeneous experience, predict action consequences, anticipate failures, recover from mistakes, and improve through grounded feedback.

\newpage
\bibliographystyle{unsrt}
\bibliography{main}

\newpage
\beginappendix
\setcounter{section}{0} 

\section{Contributors}

Kairos is contributed to by the following people.

\textbf{Advisor}: \\
Dacheng Tao, Xiaogang Wang

\textbf{Project Lead}: \\
Fei Wang, Shan You, Qiming Zhang

\textbf{Core Contributor}: \\
Tao Huang, Zuoyi Fu

\textbf{Contributor}: \\
Zhisheng Zheng, Yunlong Xi, Feng Lv, Xiaoming Wu, Zeyu Liu, Cong Wan, Pu Li, Ruiqing Yang, Xiaoou Li, Wei Wang, Kangkang Zhu, Yuwei Zhang, Shi Fu, Zheng Zhang, Xiaoning Wu, Xuzeng Fan

\textbf{Acknowledgement}: \\
We would like to thank Anke Tang, Changhui Du, Huiwen Xue, Jiakai Huang, Junxi Jia, Lichen Man, Menglin Geng, Ruixuan Zhang, Shuaiqi Cheng, Shuo Huang, Weijie Sun, Yu Li, Yunpeng Wang, Zhongbo Wu, Zihao Gao for their valuable support and contributions to this project, including data preparation, model evaluation, infrastructure support, architecture analysis, and helpful discussions.

\newpage
\section{Theoretical Analysis}
\label{sec:theory}
The proposed world model is grounded in a unified understanding-generation-prediction substrate and a hybrid temporal backbone. To formally analyze its long-horizon consistency, this section investigates future targets requiring extended world-state information, such as object permanence, delayed physical effects, and multi-stage task variables. We address two central questions: when is a bounded recent window fundamentally insufficient, and under what conditions can a hybrid multi-scale memory recover near-Bayes-optimal prediction?

Our theoretical contribution is twofold. First, we establish an information-theoretic necessity result: whenever the Bayes-optimal predictor of a long-horizon target relies on history outside a finite recent window, any window-restricted predictor incurs a strictly positive, irreducible excess risk. This proves the fundamental necessity of a persistent internal state. Second, we establish the approximate sufficiency of a hybrid multi-scale temporal memory. We prove that if the Bayes predictor factorizes into a shared predictive state alongside short-range, mid-range, and contractive global-memory branches, the resulting predictor yields an explicit excess-risk bound controlled by branch-wise approximation errors and a geometrically discounted global-memory perturbation term.

This formal analysis perfectly mirrors the model's architectural design. The necessity result explains the inevitable degradation of purely local temporal mechanisms in tasks involving delayed effects or multi-stage structures. Conversely, the sufficiency result mathematically validates the proposed hybrid design: short- and mid-range pathways efficiently capture localized and intermediate motion, while the persistent global memory propagates supra-window context with controlled drift. Together, these theorems provide a rigorous mathematical justification for the architectural logic underlying the model's long-horizon capabilities.

\subsection{Problem Setup and Theoretical Scope}
\label{subsec:theory_setup}

\paragraph{\textbf{Standing probabilistic setup.}}
Let $(\Omega,\cF,\Prob)$ be a probability space.
We model the available interaction stream as a discrete-time partially observed controlled process
\[
\{(O_t,A_t)\}_{t\ge 1},
\]
where
\[
O_t \in \mathcal O,
\qquad
A_t \in \mathcal A.
\]
Here, $O_t$ denotes the observable input available to the model, and $A_t$ denotes the control or action signal that can influence future evolution. Fix a time index $t\ge 1$, a prediction horizon $\tau \ge 1$, and a window length $1\le w<t$. Let $Y_t^{(\tau)} \in L^2(\Omega,\cF,\Prob)$
be any square-integrable future target, and write
\[
Y := Y_t^{(\tau)}.
\]
Throughout, $Y$ may represent a future latent-frame coordinate, an object-permanence indicator, a delayed physical-effect event, a task-progress variable, or any other long-horizon functional of the future world state. All scalar statements below extend coordinatewise to vector-valued targets. For any square-integrable scalar-, vector-, or matrix-valued random variable $Z$, define
\[
\norm{Z}_{L^2} := \bigl(\E[\norm{Z}^2]\bigr)^{1/2},
\]
where $\norm{\cdot}$ denotes the absolute value, the Euclidean norm, or the Frobenius norm according to context.

\begin{definition}[History and recent window]
\label{def:history_window}
The complete history up to time $t$ is defined as
\begin{equation}
H_t := (O_1,\dots,O_t,\; A_1,\dots,A_{t-1}),
\label{eq:history_def}
\end{equation}
generating the associated $\sigma$-field $\cH_t := \sigma(H_t)$. For a window length $1\le w<t$, the recent $w$-step window is given by
\begin{equation}
W_t^{(w)}
:=
\bigl(
O_{t-w+1},\dots,O_t,\;
A_{t-w},\dots,A_{t-1}
\bigr),
\label{eq:window_def}
\end{equation}
with its corresponding $\sigma$-field denoted by $\cW_t^{(w)} := \sigma(W_t^{(w)})$. Here, $\sigma(\cdot)$ denotes the sigma-field generated by the enclosed random variables, i.e., the collection of all events or information measurable from the corresponding observation--action history. By construction, it naturally follows that
\begin{equation}
\cW_t^{(w)} \subseteq \cH_t.
\label{eq:window_subset_history}
\end{equation}
\end{definition}

\begin{definition}[Predictors and optimal risks]
\label{def:predictors_and_risks}
For any sub-$\sigma$-field $\mathcal G \subseteq \cF$, define
\[
L^2(\mathcal G)
:=
\bigl\{
Z \in L^2(\Omega,\cF,\Prob) : Z \text{ is } \mathcal G\text{-measurable}
\bigr\}.
\]
For any $Z \in L^2(\Omega,\cF,\Prob)$, define the squared prediction risk
\begin{equation}
\mathcal R_t(Z) := \E[(Y-Z)^2].
\label{eq:general_risk_def}
\end{equation}
The optimal full-history risk is
\begin{equation}
R_{\mathrm{full}}^\star
:=
\inf_{Z \in L^2(\cH_t)} \mathcal R_t(Z),
\label{eq:Rfull_def}
\end{equation}
and the optimal recent-window risk is
\begin{equation}
R_{w}^\star
:=
\inf_{Z \in L^2(\cW_t^{(w)})} \mathcal R_t(Z).
\label{eq:Rw_def}
\end{equation}
\end{definition}

\begin{definition}[Persistent internal state]
\label{def:persistent_internal_state}
A recursively updated internal state $M_t \in \R^d$ is called a persistent internal state if there exist measurable update maps $\Phi_t$ such that
\begin{equation}
M_t = \Phi_t(M_{t-1}, \zeta_t),
\qquad
\zeta_t := (O_t,A_{t-1}),
\qquad t\ge 2,
\label{eq:persistent_state_update}
\end{equation}
with $M_1$ given.
Thus the model compresses historical information into a state that is propagated through time, rather than recomputing prediction from scratch from a bounded local context at every step.
\end{definition}

\begin{definition}[Exact sufficient state]
\label{def:exact_sufficient_state}
A recursively updated state $S_t^\star \in \R^d$ is called an exact sufficient state for the target $Y$ if:
\begin{enumerate}
    \item $S_t^\star$ is $\cH_t$-measurable;
    \item there exist measurable maps $T_t$ such that
    \begin{equation}
    S_t^\star = T_t(S_{t-1}^\star, \zeta_t),
    \qquad
    \zeta_t := (O_t,A_{t-1}),
    \qquad t\ge 2;
    \label{eq:exact_state_update}
    \end{equation}
    \item there exists a measurable decoder $g_t:\R^d\to\R$ such that
    \begin{equation}
    \E[Y\mid \cH_t] = g_t(S_t^\star)
    \qquad \text{a.s.}
    \label{eq:exact_sufficient_state}
    \end{equation}
\end{enumerate}
In other words, $S_t^\star$ retains exactly the historical information that is relevant for predicting $Y$.
\end{definition}

\begin{remark}[Scope of the two results]
\label{rem:scope_of_two_results}
The subsequent necessity theorem relies solely on the filtration generated by the process history. It imposes no architectural assumptions and does not require the existence of an exact sufficient state. Furthermore, the ensuing sufficiency theorem introduces an architecture-motivated factorization of this state into four components: a shared predictive representation, a short-range state, a mid-range state, and a global recurrent memory. This conceptualization serves as the theoretical counterpart to our unified understanding--generation--prediction substrate, alongside its local, dilated, and global temporal pathways.
\end{remark}

\begin{lemma}[Conditional expectation is the $L^2$-optimal predictor]
\label{lem:l2_projection}
Let $\mathcal G \subseteq \cF$ be any sub-$\sigma$-field and let $Y\in L^2(\Omega,\cF,\Prob)$.
Define $\eta_{\mathcal G} := \E[Y\mid \mathcal G].$ Then $\eta_{\mathcal G}$ is the unique minimizer of squared risk over $L^2(\mathcal G)$, i.e.,
\begin{equation}
\inf_{Z \in L^2(\mathcal G)} \E[(Y-Z)^2]
=
\E[(Y-\eta_{\mathcal G})^2].
\label{eq:l2_projection_statement}
\end{equation}
Moreover, for every $Z\in L^2(\mathcal G)$,
\begin{equation}
\E[(Y-Z)^2]
=
\E[(Y-\eta_{\mathcal G})^2]
+
\E[(\eta_{\mathcal G}-Z)^2].
\label{eq:pythagorean_identity}
\end{equation}
\end{lemma}

\begin{proof}[Proof of Lemma \ref{lem:l2_projection}]
Fix any $Z\in L^2(\mathcal G)$ and write
\[
\eta := \eta_{\mathcal G} = \E[Y\mid \mathcal G].
\]
Then
\[
Y-Z = (Y-\eta) + (\eta-Z).
\]
Squaring both sides and taking expectations gives
\begin{align}
\E[(Y-Z)^2]
&=
\E[(Y-\eta)^2]
+
\E[(\eta-Z)^2]
+
2\E[(Y-\eta)(\eta-Z)].
\label{eq:l2_projection_expand}
\end{align}
Since $\eta-Z$ is $\mathcal G$-measurable,
\begin{align}
\E[(Y-\eta)(\eta-Z)]
&=
\E\Bigl[
\E[(Y-\eta)(\eta-Z)\mid \mathcal G]
\Bigr]
\nonumber\\
&=
\E\Bigl[
(\eta-Z)\,\E[Y-\eta\mid \mathcal G]
\Bigr]
=
0,
\end{align}
because $\E[Y-\eta\mid \mathcal G]=\E[Y\mid \mathcal G]-\eta=0$.
Substituting back into Eq. (\ref{eq:l2_projection_expand}) proves Eq. (\ref{eq:pythagorean_identity}), and Eq. (\ref{eq:l2_projection_statement}) follows immediately.
Uniqueness holds because equality requires $\E[(\eta-Z)^2]=0$, i.e., $Z=\eta$ almost surely.
\end{proof}

\subsection{Necessity of Persistent Internal States}
\label{subsec:necessity_persistent_state}

This subsection formalizes the obstruction faced by purely local temporal models.  In long-horizon prediction, two historical trajectories might share identical recent observations and actions yet differ in earlier unobserved or supra-window events that dictate the future. Such events include instances where an object remains present despite temporary occlusion, the prior triggering of a delayed physical effect, or the completion of a specific stage in a task involving multiple steps. Consequently, predictors relying solely on recent windows are not merely more difficult to train but are fundamentally insufficient from a statistical perspective.

\begin{theorem}[Supra-window dependence implies the necessity of persistent state]
\label{thm:necessity_persistent_state}
Fix $t\ge 1$, $\tau\ge 1$, and $1\le w<t$, and define
\begin{equation}
m_t := \E[Y\mid \cH_t],
\qquad
m_t^{(w)} := \E[Y\mid \cW_t^{(w)}].
\label{eq:m_and_mw_def}
\end{equation}
Then the following statements hold.
\begin{enumerate}
    \item[\textup{(i)}] The optimal full-history and recent-window risks are
    \begin{equation}
    R_{\mathrm{full}}^\star
    =
    \E[(Y-m_t)^2]
    =
    \E[\Var(Y\mid \cH_t)],
    \label{eq:Rfull_formula}
    \end{equation}
    \begin{equation}
    R_{w}^\star
    =
    \E[(Y-m_t^{(w)})^2]
    =
    \E[\Var(Y\mid \cW_t^{(w)})].
    \label{eq:Rw_formula}
    \end{equation}

    \item[\textup{(ii)}] The excess risk incurred by restricting prediction to the recent window satisfies the exact identity
    \begin{equation}
    R_{w}^\star - R_{\mathrm{full}}^\star
    =
    \E[(m_t-m_t^{(w)})^2]
    =
    \E\!\bigl[\Var(m_t\mid \cW_t^{(w)})\bigr].
    \label{eq:risk_gap_exact_identity}
    \end{equation}

    \item[\textup{(iii)}] Consequently,
    \begin{equation}
    R_{w}^\star > R_{\mathrm{full}}^\star
    \quad\Longleftrightarrow\quad
    m_t \text{ is not } \cW_t^{(w)}\text{-measurable}.
    \label{eq:iff_nonmeasurable_app}
    \end{equation}
\end{enumerate}
Specifically, if the Bayes predictor cannot be fully recovered from the most recent $w$ steps, every window-restricted predictor inherently incurs a strictly positive, irreducible excess risk. Consequently, any Bayes-optimal recursive architecture must retain supra-window information, and implementing such retention via recursive state propagation necessitates a persistent internal state.
\end{theorem}

\begin{proof}[Proof of Theorem \ref{thm:necessity_persistent_state}]
Part \textup{(i)} follows directly from Lemma~\ref{lem:l2_projection} with $\mathcal G=\cH_t$ and $\mathcal G=\cW_t^{(w)}$, respectively. Indeed,
\[
R_{\mathrm{full}}^\star = \E[(Y-m_t)^2],
\qquad
R_w^\star = \E[(Y-m_t^{(w)})^2],
\]
and these equalities are equivalent to Eq. (\ref{eq:Rfull_formula}) and Eq. (\ref{eq:Rw_formula}) by the definition of conditional variance.

For part \textup{(ii)}, start from the decomposition
\begin{equation}
Y - m_t^{(w)}
=
(Y-m_t) + (m_t-m_t^{(w)}).
\label{eq:risk_gap_decomposition}
\end{equation}
After squaring and taking expectations,
\begin{align}
\E[(Y-m_t^{(w)})^2]
&=
\E[(Y-m_t)^2]
+
\E[(m_t-m_t^{(w)})^2]
+
2\E[(Y-m_t)(m_t-m_t^{(w)})].
\label{eq:risk_gap_expand}
\end{align}
Because both $m_t$ and $m_t^{(w)}$ are $\cH_t$-measurable, so is $m_t-m_t^{(w)}$. Hence
\begin{align}
\E[(Y-m_t)(m_t-m_t^{(w)})]
&=
\E\Bigl[
\E[(Y-m_t)(m_t-m_t^{(w)})\mid \cH_t]
\Bigr]
\nonumber\\
&=
\E\Bigl[
(m_t-m_t^{(w)})\,\E[Y-m_t\mid \cH_t]
\Bigr]
=
0.
\end{align}
Therefore,
\begin{equation}
R_w^\star - R_{\mathrm{full}}^\star
=
\E[(m_t-m_t^{(w)})^2].
\label{eq:risk_gap_first_expression}
\end{equation}

To derive the second expression in Eq. (\ref{eq:risk_gap_exact_identity}), note that by the tower property,
\begin{equation}
m_t^{(w)} = \E[m_t\mid \cW_t^{(w)}].
\label{eq:mw_as_condexp_of_m}
\end{equation}
Thus
\[
m_t - m_t^{(w)}
=
m_t - \E[m_t\mid \cW_t^{(w)}],
\]
and applying the conditional-variance identity to the random variable $m_t$ conditioned on $\cW_t^{(w)}$ gives
\begin{equation}
\E[(m_t-m_t^{(w)})^2]
=
\E\!\bigl[\Var(m_t\mid \cW_t^{(w)})\bigr].
\label{eq:conditional_variance_of_mt}
\end{equation}
Combining Eq. (\ref{eq:risk_gap_first_expression}) and Eq. (\ref{eq:conditional_variance_of_mt}) proves part \textup{(ii)}.

For part \textup{(iii)}, Eq. (\ref{eq:risk_gap_first_expression}) implies
\[
R_w^\star = R_{\mathrm{full}}^\star
\quad\Longleftrightarrow\quad
\E[(m_t-m_t^{(w)})^2]=0
\quad\Longleftrightarrow\quad
m_t = m_t^{(w)} \ \text{a.s.}
\]
Using Eq. (\ref{eq:mw_as_condexp_of_m}), this is equivalent to
\[
m_t = \E[m_t\mid \cW_t^{(w)}] \ \text{a.s.}
\]
A square-integrable random variable equals its conditional expectation with respect to a $\sigma$-field if and only if it is measurable with respect to that $\sigma$-field. Hence
\[
R_w^\star = R_{\mathrm{full}}^\star
\quad\Longleftrightarrow\quad
m_t \text{ is } \cW_t^{(w)}\text{-measurable},
\]
and Eq. (\ref{eq:iff_nonmeasurable_app}) follows by negation.
\end{proof}

\begin{corollary}[Explicit lower bound under an atomic recent-window mismatch]
\label{cor:atomic_window_lower_bound}
Suppose there is a window value $s$ for which the event $E := \{W_t^{(w)} = s\}$ occurs with positive probability, $\Prob(E)>0$. Furthermore, suppose there exist disjoint events $E_1, E_2 \in \cH_t$ that partition $E$ into $E_1 \cup E_2 = E$, along with distinct constants $\mu_1, \mu_2 \in \R$, such that $m_t = \mu_1$ almost surely on $E_1$ and $m_t = \mu_2$ almost surely on $E_2$. If we define the conditional probability $\alpha := \Prob(E_1 \mid E) \in (0,1)$, then the excess risk is bounded below by
\begin{equation}
R_w^\star - R_{\mathrm{full}}^\star \ge \Prob(E)\alpha(1-\alpha)(\mu_1-\mu_2)^2.
\label{eq:explicit_lower_bound}
\end{equation}
\end{corollary}

\begin{proof}[Proof of Corollary \ref{cor:atomic_window_lower_bound}]
By Theorem~\ref{thm:necessity_persistent_state},
\begin{equation}
R_w^\star - R_{\mathrm{full}}^\star
=
\E\!\bigl[\Var(m_t\mid \cW_t^{(w)})\bigr].
\label{eq:corollary_start}
\end{equation}
Since conditional variance is nonnegative almost surely, we may restrict the expectation to the event $E$ and obtain
\begin{equation}
R_w^\star - R_{\mathrm{full}}^\star
\ge
\E\!\bigl[\Var(m_t\mid \cW_t^{(w)})\,\mathbbm{1}_E\bigr].
\label{eq:corollary_restrict_E}
\end{equation}

Now $E=\{W_t^{(w)}=s\}\in \cW_t^{(w)}$ and, by assumption, it is an atom of the $\sigma$-field $\cW_t^{(w)}$.
Hence every $\cW_t^{(w)}$-measurable random variable is almost surely constant on $E$.
In particular, there exists a constant $c\in\R$ such that
\[
\Var(m_t\mid \cW_t^{(w)}) = c
\qquad \text{a.s. on } E.
\]
Moreover, because $W_t^{(w)}=s$ almost surely on $E$, conditioning on $\cW_t^{(w)}$ and restricting to $E$ is equivalent to conditioning on the event $E$ itself. Thus
\[
c = \Var(m_t\mid E),
\]
and therefore
\begin{equation}
\E\!\bigl[\Var(m_t\mid \cW_t^{(w)})\,\mathbbm{1}_E\bigr]
=
\Prob(E)\,\Var(m_t\mid E).
\label{eq:variance_on_atom}
\end{equation}

Next, since $E_1$ and $E_2$ form a partition of $E$, and
\[
m_t=\mu_1 \quad \text{a.s. on } E_1,
\qquad
m_t=\mu_2 \quad \text{a.s. on } E_2,
\]
it follows that under the conditional law given $E$, the random variable $m_t$ takes the value $\mu_1$ with probability
\[
\Prob(E_1\mid E)=\alpha,
\]
and the value $\mu_2$ with probability
\[
\Prob(E_2\mid E)=1-\alpha.
\]
Hence
\[
\E[m_t\mid E] = \alpha \mu_1 + (1-\alpha)\mu_2,
\]
and
\[
\E[m_t^2\mid E] = \alpha \mu_1^2 + (1-\alpha)\mu_2^2.
\]
Therefore
\begin{align}
\Var(m_t\mid E)
&=
\E[m_t^2\mid E] - \bigl(\E[m_t\mid E]\bigr)^2
\nonumber\\
&=
\alpha\mu_1^2 + (1-\alpha)\mu_2^2
-
\bigl(\alpha \mu_1 + (1-\alpha)\mu_2\bigr)^2
\nonumber\\
&=
\alpha(1-\alpha)(\mu_1-\mu_2)^2.
\label{eq:binary_variance_formula}
\end{align}
Combining Eq. (\ref{eq:corollary_restrict_E}), Eq. (\ref{eq:variance_on_atom}), and Eq. (\ref{eq:binary_variance_formula}) yields
\[
R_w^\star - R_{\mathrm{full}}^\star
\ge
\Prob(E)\,\alpha(1-\alpha)(\mu_1-\mu_2)^2,
\]
which is exactly Eq. (\ref{eq:explicit_lower_bound}).
\end{proof}

\begin{remark}[Information-theoretic interpretation]
\label{rem:information_theoretic_meaning}
Theorem~\ref{thm:necessity_persistent_state} yields the exact identity
\begin{equation}
R_w^\star - R_{\mathrm{full}}^\star
=
\norm{
\E[Y\mid \cH_t] - \E[Y\mid \cW_t^{(w)}]
}_{L^2}^2.
\label{eq:information_theoretic_identity}
\end{equation}
Thus, the necessity claim is quantitative rather than merely qualitative: whenever the full-history Bayes predictor depends on information outside the recent window, every bounded-window predictor incurs an exactly measurable irreducible excess risk. In long-horizon world modeling, this formalizes the intuition that purely local temporal context is insufficient whenever future consistency depends on persistent world state beyond the visible short-range past.
\end{remark}

\begin{remark}[Architectural interpretation]
\label{rem:architectural_interpretation}
Theorem~\ref{thm:necessity_persistent_state} is a statement about information rather than about a particular neural implementation. What it proves is that supra-window information must be preserved if one seeks Bayes-optimal long-horizon prediction. In a recursive predictor, the natural implementation of such preservation is a propagated internal memory in the sense of Definition~\ref{def:persistent_internal_state}. This is the theoretical reason why a long-horizon world model requires a dedicated persistent memory pathway rather than only local attention.
\end{remark}

\begin{remark}[Connection to long-horizon evaluation]
\label{rem:connection_to_experiments}
To connect the theorem with practical scenarios, let $Y$ denote a future target directly relevant to long-horizon world-model evaluation. This target could be a future latent-frame coordinate, an object-identity consistency variable, a collision outcome, or a task-progress indicator situated several seconds ahead. Consequently, the gap presented in Eq. (\ref{eq:risk_gap_exact_identity}) provides a formal expression for the identical difficulty probed by our long-horizon evaluations. Specifically, whenever the historical cause governing the future target remains unrecoverable from the recent local window alone, any purely local temporal model inevitably incurs a nonzero prediction gap regardless of its optimization quality.
\end{remark}

\subsection{Approximate Sufficiency of a Hybrid Multi-Scale Temporal Memory}
\label{subsec:kairos_sufficiency}

The necessity result explains why some persistent state is unavoidable, but it does not yet explain why the particular temporal factorization used in the architecture is adequate. We now show that a hybrid multi-scale temporal memory is approximately sufficient whenever the Bayes predictor admits the corresponding decomposition and the global recurrent memory evolves stably. The result formalizes the complementary roles of a shared predictive representation, a short-range local branch, a mid-range branch, and a global memory branch.

\begin{definition}[Exact hybrid multi-scale predictive decomposition]
\label{def:exact_hybrid_decomp}
We say that the Bayes predictor admits an exact hybrid multi-scale predictive decomposition if there exist square-integrable, $\cH_t$-measurable random variables
\[
U_t^\star \in \R^{d_U},
\qquad
C_t^\star \in \R^{d_C},
\qquad
D_t^\star \in \R^{d_D},
\qquad
G_t^\star \in \R^{d_v \times d_k},
\]
together with a measurable decoder
\[
h_t:\R^{d_U}\times\R^{d_C}\times\R^{d_D}\times\R^{d_v\times d_k}\to\R,
\]
such that
\begin{equation}
\mu_t^\star := \E[Y\mid \cH_t] = h_t(U_t^\star,C_t^\star,D_t^\star,G_t^\star)
\qquad \text{a.s.}
\label{eq:exact_decomposition}
\end{equation}
The four components are interpreted as follows:
\begin{itemize}
    \item $U_t^\star$: a shared physical predictive state;
    \item $C_t^\star$: a short-range local state;
    \item $D_t^\star$: a mid-range dilated state;
    \item $G_t^\star$: a global recurrent causal memory.
\end{itemize}
\end{definition}

\begin{remark}[Architecture correspondence]
\label{rem:architecture_correspondence}
Although stated abstractly, Definition~\ref{def:exact_hybrid_decomp} is specifically designed to parallel our concrete temporal factorization. Within the implemented architecture, $C_t^\star$ and $D_t^\star$ map to the short- and mid-range temporal pathways instantiated by SWA and DSWA. Furthermore, $G_t^\star$ represents the global gated memory pathway instantiated by GLA. Finally, the variable $U_t^\star$ embodies the shared predictive substrate that is continuously reused across understanding, generation, and prediction tasks.
\end{remark}

\begin{definition}[Hybrid multi-scale predictor and gated global-memory update]
\label{def:hybrid_predictor}
A hybrid multi-scale predictor is of the form
\begin{equation}
\hat\mu_t = h_t(\hat U_t,\hat C_t,\hat D_t,\hat G_t),
\label{eq:learned_predictor}
\end{equation}
where $\hat U_t,\hat C_t,\hat D_t,\hat G_t$ are square-integrable, $\cH_t$-measurable estimators of $U_t^\star,C_t^\star,D_t^\star,G_t^\star$, respectively. The global-memory branch is modeled by a gated delta update. Given a decay gate $\alpha \in (0, 1)$, a writing strength $\beta \in (0, 1)$, a current value $v \in \mathbb{R}^{d_v}$, and a key $k \in \mathbb{R}^{d_k}$, we define the update map for a state $S \in \mathbb{R}^{d_v \times d_k}$ as follows:
\begin{equation}
F(S;\alpha,\beta,v,k)
:=
\alpha S + \beta (v-Sk)k^\top,
\qquad
S\in\R^{d_v\times d_k}.
\label{eq:F_map}
\end{equation}
Then the exact and learned global memories satisfy
\begin{equation}
G_t^\star = F(G_{t-1}^\star;\alpha_t^\star,\beta_t^\star,v_t^\star,k_t^\star),
\label{eq:exact_G_update}
\end{equation}
\begin{equation}
\hat G_t = F(\hat G_{t-1};\hat\alpha_t,\hat\beta_t,\hat v_t,\hat k_t).
\label{eq:learned_G_update}
\end{equation}
\end{definition}

\begin{definition}[Non-global branch approximation quality]
\label{def:branch_errors}
For the hybrid predictor in Definition~\ref{def:hybrid_predictor}, define the non-global branch
approximation errors at time $t$ by
\begin{equation}
\varepsilon_t^U := \|\hat U_t-U_t^\star\|_{L^2},\qquad
\varepsilon_t^{\mathrm{SWA}} := \|\hat C_t-C_t^\star\|_{L^2},\qquad
\varepsilon_t^{\mathrm{DSWA}} := \|\hat D_t-D_t^\star\|_{L^2}.
\label{eq:branch_error_defs}
\end{equation}
Here, $\varepsilon_t^U$ measures the approximation quality of the shared predictive substrate,
$\varepsilon_t^{\mathrm{SWA}}$ measures the short-range branch error, and
$\varepsilon_t^{\mathrm{DSWA}}$ measures the mid-range branch error.
\end{definition}

\begin{definition}[One-step gate approximation errors]
\label{ass:gla_errors}
Define
\begin{equation}
\varepsilon_t^\alpha := \norm{\hat\alpha_t-\alpha_t^\star}_{L^2},
\qquad
\varepsilon_t^\beta := \norm{\hat\beta_t-\beta_t^\star}_{L^2},
\label{eq:gate_errors}
\end{equation}
\begin{equation}
\varepsilon_t^v := \norm{\hat v_t-v_t^\star}_{L^2},
\qquad
\varepsilon_t^k := \norm{\hat k_t-k_t^\star}_{L^2}.
\label{eq:feature_errors}
\end{equation}
Also define the initial global-memory discrepancy
\begin{equation}
e_0 := \norm{\hat G_0-G_0^\star}_{L^2}.
\label{eq:initial_memory_error}
\end{equation}
\end{definition}

For convenience, define the Bayes risk
\begin{equation}
\mathcal R_t^\star := \mathcal R_t(\mu_t^\star) = \E[(Y-\mu_t^\star)^2].
\label{eq:bayes_risk}
\end{equation}

\begin{lemma}[Bayes excess identity]
\label{lem:bayes_excess_identity}
For any square-integrable, $\cH_t$-measurable predictor $\hat\mu_t$,
\begin{equation}
\mathcal R_t(\hat\mu_t)-\mathcal R_t^\star
=
\|\hat\mu_t-\mu_t^\star\|_{L^2}^2.
\label{eq:bayes_excess_identity}
\end{equation}
\end{lemma}

\begin{proof}[Proof of Lemma \ref{lem:bayes_excess_identity}]
Recall that
\[
\mathcal R_t(\hat\mu_t)=\E\bigl[(Y-\hat\mu_t)^2\bigr],
\qquad
\mathcal R_t^\star=\E\bigl[(Y-\mu_t^\star)^2\bigr],
\]
where $\mu_t^\star=\E[Y\mid \cH_t]$ is the Bayes predictor. Hence,
\begin{align}
\mathcal R_t(\hat\mu_t)-\mathcal R_t^\star
&=
\E\bigl[(Y-\hat\mu_t)^2-(Y-\mu_t^\star)^2\bigr].
\label{eq:bayes_excess_start}
\end{align}
Expanding the difference of squares gives
\begin{align}
(Y-\hat\mu_t)^2-(Y-\mu_t^\star)^2
&=
\bigl((Y-\mu_t^\star)+(\mu_t^\star-\hat\mu_t)\bigr)^2-(Y-\mu_t^\star)^2
\nonumber\\
&=
(\mu_t^\star-\hat\mu_t)^2
+
2(Y-\mu_t^\star)(\mu_t^\star-\hat\mu_t).
\label{eq:bayes_excess_expand}
\end{align}
Substituting Eq. (\ref{eq:bayes_excess_expand}) into Eq. (\ref{eq:bayes_excess_start}), we obtain
\begin{align}
\mathcal R_t(\hat\mu_t)-\mathcal R_t^\star
&=
\E\bigl[(\mu_t^\star-\hat\mu_t)^2\bigr]
+
2\,\E\bigl[(Y-\mu_t^\star)(\mu_t^\star-\hat\mu_t)\bigr].
\label{eq:bayes_excess_cross}
\end{align}
It therefore remains to show that the cross term vanishes.

Since both $\hat\mu_t$ and $\mu_t^\star$ are $\cH_t$-measurable, the difference
$\mu_t^\star-\hat\mu_t$ is also $\cH_t$-measurable and square-integrable. By the tower property,
\begin{align}
\E\bigl[(Y-\mu_t^\star)(\mu_t^\star-\hat\mu_t)\bigr]
&=
\E\!\left[
\E\bigl[(Y-\mu_t^\star)(\mu_t^\star-\hat\mu_t)\mid \cH_t\bigr]
\right]
\nonumber\\
&=
\E\!\left[
(\mu_t^\star-\hat\mu_t)\,\E[Y-\mu_t^\star\mid \cH_t]
\right].
\label{eq:cross_term_conditional}
\end{align}
Because $\mu_t^\star=\E[Y\mid \cH_t]$, we have
\[
\E[Y-\mu_t^\star\mid \cH_t]
=
\E[Y\mid \cH_t]-\mu_t^\star
=
0.
\]
Thus the right-hand side of Eq. (\ref{eq:cross_term_conditional}) is zero, and therefore
\[
\E\bigl[(Y-\mu_t^\star)(\mu_t^\star-\hat\mu_t)\bigr]=0.
\]
Returning to Eq. (\ref{eq:bayes_excess_cross}), we conclude that
\begin{align}
\mathcal R_t(\hat\mu_t)-\mathcal R_t^\star
&=
\E\bigl[(\mu_t^\star-\hat\mu_t)^2\bigr]
=
\|\hat\mu_t-\mu_t^\star\|_{L^2}^2,
\end{align}
which proves Eq. (\ref{eq:bayes_excess_identity}).
\end{proof}

\begin{lemma}[Contraction of the exact gated delta update]
\label{lem:gla_contraction}
Fix $(\alpha,\beta,v,k)$ and define $F$ by Eq. (\ref{eq:F_map}). Let $\rho:=\alpha + \beta \norm{k}_2^2$.  Suppose $ \norm{k}_2^2 < 1-\alpha$. Therefore, given $\alpha, \beta \in (0, 1)$, the update is strictly contractive with a factor $\rho < 1$, satisfying for all $S, T \in \mathbb{R}^{d_v \times d_k}$:
\begin{equation}
\norm{F(S;\alpha,\beta,v,k)-F(T;\alpha,\beta,v,k)}_F
\le
\rho \norm{S-T}_F.
\label{eq:gla_contraction_result}
\end{equation}
\end{lemma}

\begin{proof}[Proof of Lemma \ref{lem:gla_contraction}]
We prove that the gated delta update is contractive with respect to its memory-state argument.
Fix two memory states $S,T \in \mathbb R^{d_v\times d_k}$ and let
\[
\Delta := S-T.
\]
Since the value term $v$ appears identically in both updates, it cancels when taking the difference.
Indeed, by the definition of $F$ in Eq. (\ref{eq:F_map}),
\begin{align}
F(S;\alpha,\beta,v,k)-F(T;\alpha,\beta,v,k)
&=
\alpha S + \beta (v-Sk)k^\top
-\alpha T - \beta (v-Tk)k^\top
\nonumber\\
&=
\alpha(S-T)+\beta\bigl[(v-Sk)-(v-Tk)\bigr]k^\top
\nonumber\\
&=
\alpha\Delta-\beta\Delta kk^\top.
\label{eq:gla_diff_expand}
\end{align}
Thus, the difference between the two updated memory states depends only on the previous discrepancy
$\Delta$ and on the rank-one correction induced by the key $k$.

We now bound the Frobenius norm of the right-hand side. By the triangle inequality,
\begin{equation}
\|F(S;\alpha,\beta,v,k)-F(T;\alpha,\beta,v,k)\|_F
\le
\alpha\|\Delta\|_F+\beta\|\Delta kk^\top\|_F.
\label{eq:gla_triangle}
\end{equation}
It remains to control the second term. Observe first that
\[
\Delta kk^\top = (\Delta k)k^\top.
\]
This is a rank-one matrix. For any vectors $a$ and $b$, the Frobenius norm of the outer product
factorizes as
\[
\|ab^\top\|_F = \|a\|_2\,\|b\|_2.
\]
Applying this identity with $a=\Delta k$ and $b=k$ yields
\begin{equation}
\|\Delta kk^\top\|_F
=
\|(\Delta k)k^\top\|_F
=
\|\Delta k\|_2\,\|k\|_2.
\label{eq:rankone_identity}
\end{equation}
Next, using the standard operator-norm inequality together with
$\|\Delta\|_2 \le \|\Delta\|_F$, we obtain
\begin{equation}
\|\Delta k\|_2
\le
\|\Delta\|_2\,\|k\|_2
\le
\|\Delta\|_F\,\|k\|_2.
\label{eq:operator_bound}
\end{equation}
Combining Eq. (\ref{eq:rankone_identity}) and Eq. (\ref{eq:operator_bound}), we conclude that
\begin{equation}
\|\Delta kk^\top\|_F
\le
\|\Delta\|_F\,\|k\|_2^2.
\label{eq:delta_kk_bound}
\end{equation}
Substituting Eq. (\ref{eq:delta_kk_bound}) into Eq. (\ref{eq:gla_triangle}) gives
\begin{align}
\|F(S;\alpha,\beta,v,k)-F(T;\alpha,\beta,v,k)\|_F
&\le
\alpha\|\Delta\|_F+\beta\|\Delta\|_F\|k\|_2^2
\nonumber\\
&=
\bigl(\alpha+\beta\|k\|_2^2\bigr)\|\Delta\|_F.
\end{align}
Then we obtain
\begin{align}
\|F(S;\alpha,\beta,v,k)-F(T;\alpha,\beta,v,k)\|_F
&\le
\rho\|\Delta\|_F
=
\rho\|S-T\|_F.
\end{align}
This proves Eq. (\ref{eq:gla_contraction_result}).
\end{proof}

\begin{lemma}[One-step perturbation bound for the global-memory update]
\label{lem:gla_perturbation}
Fix $t \ge 1$ and define the one-step update discrepancies by
\begin{equation}
\varepsilon_t^\alpha := \|\hat\alpha_t-\alpha_t^\star\|_{L^2},\qquad
\varepsilon_t^\beta := \|\hat\beta_t-\beta_t^\star\|_{L^2},\qquad
\varepsilon_t^v := \|\hat v_t-v_t^\star\|_{L^2},\qquad
\varepsilon_t^k := \|\hat k_t-k_t^\star\|_{L^2}.
\label{eq:one_step_error_defs}
\end{equation}
For notational convenience, let $B_G,B_k,B_v>0$ denote envelope constants satisfying
$\|\hat G_{t-1}\|_F \le B_G$,
$\|\hat k_t\|_2,\|k_t^\star\|_2 \le B_k$, and
$\|\hat v_t\|_2,\|v_t^\star\|_2 \le B_v$
almost surely. Define
\begin{equation}
\xi_t
:=
B_G\,\varepsilon_t^\alpha
+
B_k(B_v+B_GB_k)\,\varepsilon_t^\beta
+
B_k\,\varepsilon_t^v
+
(B_v+2B_GB_k)\,\varepsilon_t^k.
\label{eq:xi_def}
\end{equation}
Then
\begin{equation}
\Bigl\|
F(\hat G_{t-1};\hat\alpha_t,\hat\beta_t,\hat v_t,\hat k_t)
-
F(\hat G_{t-1};\alpha_t^\star,\beta_t^\star,v_t^\star,k_t^\star)
\Bigr\|_{L^2}
\le
\xi_t.
\label{eq:one_step_perturbation}
\end{equation}
\end{lemma}

\begin{proof}[Proof of Lemma \ref{lem:gla_perturbation}]
For notational brevity, write
\[
\hat G := \hat G_{t-1},
\quad
\hat\alpha := \hat\alpha_t,\quad \alpha^\star := \alpha_t^\star,
\quad
\hat\beta := \hat\beta_t,\quad \beta^\star := \beta_t^\star,
\quad
\hat v := \hat v_t,\quad v^\star := v_t^\star,
\quad
\hat k := \hat k_t,\quad k^\star := k_t^\star.
\]
By the definition of $F$,
\begin{align}
F(\hat G;\hat\alpha,\hat\beta,\hat v,\hat k)
-
F(\hat G;\alpha^\star,\beta^\star,v^\star,k^\star)
&=
(\hat\alpha-\alpha^\star)\hat G
+
(\hat\beta-\beta^\star)(\hat v-\hat G\hat k)\hat k^\top
\nonumber\\
&\qquad
+
\beta^\star\Bigl[(\hat v-\hat G\hat k)\hat k^\top - (v^\star-\hat Gk^\star)(k^\star)^\top\Bigr].
\label{eq:lem3_decomposition}
\end{align}
Hence
\begin{equation}
\norm{
F(\hat G;\hat\alpha,\hat\beta,\hat v,\hat k)
-
F(\hat G;\alpha^\star,\beta^\star,v^\star,k^\star)
}_F
\le
T_1 + T_2 + T_3,
\label{eq:T123}
\end{equation}
where
\[
T_1 := \abs{\hat\alpha-\alpha^\star}\,\norm{\hat G}_F,
\qquad
T_2 := \abs{\hat\beta-\beta^\star}\,\norm{(\hat v-\hat G\hat k)\hat k^\top}_F,
\]
and
\[
T_3 := \beta^\star \norm{(\hat v-\hat G\hat k)\hat k^\top-(v^\star-\hat Gk^\star)(k^\star)^\top}_F.
\]
Then we obtain:
\begin{equation}
T_1 \le B_G \abs{\hat\alpha-\alpha^\star}.
\label{eq:T1_bound}
\end{equation}
Using the fact that $\norm{ab^\top}_F = \norm{a}_2\norm{b}_2$, we obtain
\begin{align}
T_2
&=
\abs{\hat\beta-\beta^\star}\,
\norm{\hat v-\hat G\hat k}_2 \,
\norm{\hat k}_2
\nonumber\\
&\le
\abs{\hat\beta-\beta^\star}\,
(B_v+B_GB_k)\,B_k.
\label{eq:T2_bound}
\end{align}
For $T_3$, since $\beta^\star\in(0,1)$, it suffices to bound the norm inside. First,
\begin{align}
(\hat v-\hat G\hat k)\hat k^\top-(v^\star-\hat Gk^\star)(k^\star)^\top
&=
(\hat v-v^\star)\hat k^\top
+
v^\star(\hat k-k^\star)^\top
\nonumber\\
&\qquad
-
\hat G(\hat k\hat k^\top-k^\star (k^\star)^\top).
\label{eq:T3_expand}
\end{align}
Therefore,
\begin{align}
T_3
&\le
\norm{(\hat v-v^\star)\hat k^\top}_F
+
\norm{v^\star(\hat k-k^\star)^\top}_F
+
\norm{\hat G(\hat k\hat k^\top-k^\star (k^\star)^\top)}_F
\nonumber\\
&\le
B_k\norm{\hat v-v^\star}_2
+
B_v\norm{\hat k-k^\star}_2
+
\norm{\hat G}_F \norm{\hat k\hat k^\top-k^\star (k^\star)^\top}_F.
\label{eq:T3_three_terms}
\end{align}
Now write
\[
\hat k\hat k^\top-k^\star (k^\star)^\top
=
\hat k(\hat k-k^\star)^\top
+
(\hat k-k^\star)(k^\star)^\top.
\]
Hence
\begin{equation}
\norm{\hat k\hat k^\top-k^\star (k^\star)^\top}_F
\le
2B_k \norm{\hat k-k^\star}_2,
\label{eq:kk_bound}
\end{equation}
and therefore, 
\begin{equation}
T_3
\le
B_k \norm{\hat v-v^\star}_2
+
(B_v+2B_GB_k)\norm{\hat k-k^\star}_2.
\label{eq:T3_bound}
\end{equation}
Combining Eq. (\ref{eq:T1_bound}), Eq. (\ref{eq:T2_bound}), and Eq. (\ref{eq:T3_bound}), then taking the $L^2$ norm and applying Minkowski's inequality, yields
\begin{align}
&\norm{
F(\hat G_{t-1};\hat\alpha_t,\hat\beta_t,\hat v_t,\hat k_t)
-
F(\hat G_{t-1};\alpha_t^\star,\beta_t^\star,v_t^\star,k_t^\star)
}_{L^2}
\nonumber\\
&\qquad\le
B_G\,\varepsilon_t^\alpha
+
B_k(B_v+B_GB_k)\,\varepsilon_t^\beta
+
B_k\,\varepsilon_t^v
+
(B_v+2B_GB_k)\,\varepsilon_t^k
=
\xi_t,
\end{align}
which proves Eq. (\ref{eq:one_step_perturbation}).
\end{proof}

\begin{theorem}[Approximate sufficiency of a hybrid multi-scale temporal memory]
\label{thm:kairos_sufficiency}
Suppose that the Bayes predictor admits the decomposition in Definition~\ref{def:exact_hybrid_decomp}, the decoder $h_t$ is coordinate-wise Lipschitz with constants $L_U$, $L_C$, $L_D$, and $L_G$, and the conditions in Lemma~\ref{lem:gla_contraction} hold. Define
\begin{equation}
e_t := \|\hat G_t-G_t^\star\|_{L^2}, \qquad
\varepsilon_t := \max\!\bigl\{\varepsilon_t^U,\varepsilon_t^{\mathrm{SWA}},\varepsilon_t^{\mathrm{DSWA}}\bigr\},
\qquad
L := L_U+L_C+L_D .
\label{eq:et_eps_def}
\end{equation}
Then, for every $t\ge 0$, the following hold:
\begin{enumerate}
    \item[\textup{(i)}] \textbf{Global-memory error bound.}
   The global-memory branch satisfies
\begin{equation}
    e_t \le \rho^t e_0 + \frac{1-\rho^t}{1-\rho}\, \sup_{1\le i\le t}\xi_i,
    \label{eq:main_memory_bound}
\end{equation}
with $\xi_i$ defined in Eq. (\ref{eq:xi_def}), and $\rho<1$. 
In particular, define $\bar\xi :=  \sup_{i\ge 1}\xi_i$,  we have
\begin{equation}
    e_t \le \frac{\bar\xi}{1-\rho} \quad \text{as } t \to \infty.
    \label{eq:main_memory_limsup}
\end{equation}

   \item[\textup{(ii)}] \textbf{Long-horizon excess-risk bound.}
    Define $\varepsilon := \limsup_{t\to\infty} \varepsilon_t$. Then, the hybrid predictor asymptotically satisfies
    \begin{equation}
    \mathcal R_t(\hat\mu_t)-\mathcal R_t^\star 
    \le
    \left(
    L \varepsilon
    +
    \frac{L_G \bar\xi}{1-\rho}
    \right)^2 \quad \text{as } t \to \infty.
    \label{eq:main_risk_bound_infty}
    \end{equation}
\end{enumerate}
\end{theorem}

\begin{proof}[Proof of Theorem \ref{thm:kairos_sufficiency}]
We prove parts \textup{(i)} and \textup{(ii)} in sequence.

\paragraph{Part \textup{(i)}: global-memory error bound.}
Recall from Definition~\ref{def:hybrid_predictor} that the exact and learned global-memory states satisfy
\[
G_t^\star
=
F(G_{t-1}^\star;\alpha_t^\star,\beta_t^\star,v_t^\star,k_t^\star),
\qquad
\hat G_t
=
F(\hat G_{t-1};\hat\alpha_t,\hat\beta_t,\hat v_t,\hat k_t).
\]
Therefore, by the definition of $e_t$,
\begin{align}
e_t
&=
\|\hat G_t-G_t^\star\|_{L^2}
\nonumber\\
&=
\Bigl\|
F(\hat G_{t-1};\hat\alpha_t,\hat\beta_t,\hat v_t,\hat k_t)
-
F(G_{t-1}^\star;\alpha_t^\star,\beta_t^\star,v_t^\star,k_t^\star)
\Bigr\|_{L^2}.
\label{eq:proof_part1_start_refined}
\end{align}

To separate the perturbation of the current update from the propagation of the previous memory error, we add and subtract the intermediate term
\[
F(\hat G_{t-1};\alpha_t^\star,\beta_t^\star,v_t^\star,k_t^\star)
\]
inside the norm. By the triangle inequality,
\begin{align}
e_t
&\le
\Bigl\|
F(\hat G_{t-1};\hat\alpha_t,\hat\beta_t,\hat v_t,\hat k_t)
-
F(\hat G_{t-1};\alpha_t^\star,\beta_t^\star,v_t^\star,k_t^\star)
\Bigr\|_{L^2}
\nonumber\\
&\qquad
+
\Bigl\|
F(\hat G_{t-1};\alpha_t^\star,\beta_t^\star,v_t^\star,k_t^\star)
-
F(G_{t-1}^\star;\alpha_t^\star,\beta_t^\star,v_t^\star,k_t^\star)
\Bigr\|_{L^2}.
\label{eq:add_subtract_global_refined}
\end{align}

We now bound the two terms on the right-hand side separately.

For the first term, Lemma~\ref{lem:gla_perturbation} gives the one-step perturbation bound
\begin{equation}
\Bigl\|
F(\hat G_{t-1};\hat\alpha_t,\hat\beta_t,\hat v_t,\hat k_t)
-
F(\hat G_{t-1};\alpha_t^\star,\beta_t^\star,v_t^\star,k_t^\star)
\Bigr\|_{L^2}
\le
\xi_t.
\label{eq:first_term_bound_refined}
\end{equation}

For the second term, applying the contraction property from Lemma~\ref{lem:gla_contraction} guarantees that
\begin{equation}
\Bigl\|
F(\hat G_{t-1};\alpha_t^\star,\beta_t^\star,v_t^\star,k_t^\star)
-
F(G_{t-1}^\star;\alpha_t^\star,\beta_t^\star,v_t^\star,k_t^\star)
\Bigr\|_{L^2}
\le
\rho_{t} \|\hat G_{t-1}-G_{t-1}^\star\|_{L^2}
=
\rho_{t} e_{t-1},
\label{eq:second_term_bound_refined}
\end{equation}
where the contraction factor satisfies $\rho_t < 1$ for all $t \ge 1$. We then suppose that $\rho := \sup_{t \ge 1} \rho_t < 1$.

Substituting Eq. (\ref{eq:first_term_bound_refined}) and Eq. (\ref{eq:second_term_bound_refined}) into Eq. (\ref{eq:add_subtract_global_refined}), we obtain the one-step recursion
\begin{equation}
e_t \le \xi_t + \rho e_{t-1}.
\label{eq:one_step_et_recursion_refined}
\end{equation}

We next unroll this recursion. Repeated application of Eq. (\ref{eq:one_step_et_recursion_refined}) yields
\begin{align}
e_t
&\le
\xi_t + \rho e_{t-1}
\nonumber\\
&\le
\xi_t + \rho \xi_{t-1} + \rho^2 e_{t-2}
\nonumber\\
&\le \cdots \le
\rho^t e_0 + \sum_{i=1}^{t}\rho^{\,t-i}\xi_i.
\label{eq:unrolled_global_recursion}
\end{align}
Since $\xi_i \le \sup_{1\le j\le t}\xi_j$ for every $1\le i\le t$, we further obtain
\begin{align}
e_t
&\le
\rho^t e_0 + \Bigl(\sup_{1\le j\le t}\xi_j\Bigr)\sum_{i=1}^{t}\rho^{\,t-i}
\nonumber\\
&=
\rho^t e_0 + \Bigl(\sup_{1\le j\le t}\xi_j\Bigr)\sum_{m=0}^{t-1}\rho^m
\nonumber\\
&=
\rho^t e_0 + \frac{1-\rho^t}{1-\rho}\,\sup_{1\le i\le t}\xi_i,
\label{eq:closed_form_et_bound}
\end{align}
which proves Eq. (\ref{eq:main_memory_bound}). Now let
\[
\bar\xi := \sup_{i\ge 1}\xi_i.
\]
Then Eq. (\ref{eq:closed_form_et_bound}) implies
\begin{equation}
e_t
\le
\rho^t e_0 + \frac{1-\rho^t}{1-\rho}\,\bar\xi.
\label{eq:closed_form_et_bound_uniform}
\end{equation}
Since $\rho<1$, we have $\rho^t\to 0$ and $(1-\rho^t)/(1-\rho)\to 1/(1-\rho)$ as $t\to\infty$. Hence
\begin{equation}
e_t \le \frac{\bar\xi}{1-\rho}
\qquad \text{as } t\to\infty,
\label{eq:et_asymptotic_refined}
\end{equation}
which proves the asymptotic statement in part \textup{(i)}.

\paragraph{Part \textup{(ii)}: long-horizon excess-risk bound.}
By Definition~\ref{def:exact_hybrid_decomp}, the Bayes predictor has the form
\[
\mu_t^\star = h_t(U_t^\star,C_t^\star,D_t^\star,G_t^\star),
\]
whereas the learned predictor is given by
\[
\hat\mu_t = h_t(\hat U_t,\hat C_t,\hat D_t,\hat G_t).
\]
Since $h_t$ is coordinate-wise Lipschitz, we have the pointwise estimate
\begin{align}
|\hat\mu_t-\mu_t^\star|
&\le
L_U\|\hat U_t-U_t^\star\|
+
L_C\|\hat C_t-C_t^\star\|
+
L_D\|\hat D_t-D_t^\star\|
+
L_G\|\hat G_t-G_t^\star\|_F.
\label{eq:pointwise_decoder_bound_refined}
\end{align}
Taking $L^2$ norms on both sides and applying Minkowski's inequality yields
\begin{align}
\|\hat\mu_t-\mu_t^\star\|_{L^2}
&\le
L_U\|\hat U_t-U_t^\star\|_{L^2}
+
L_C\|\hat C_t-C_t^\star\|_{L^2}
+
L_D\|\hat D_t-D_t^\star\|_{L^2}
+
L_G\|\hat G_t-G_t^\star\|_{L^2}
\nonumber\\
&=
L_U\varepsilon_t^U
+
L_C\varepsilon_t^{\mathrm{SWA}}
+
L_D\varepsilon_t^{\mathrm{DSWA}}
+
L_G e_t.
\label{eq:predictor_error_before_grouping_refined}
\end{align}

By the definitions
\[
\varepsilon_t
=
\max\!\bigl\{\varepsilon_t^U,\varepsilon_t^{\mathrm{SWA}},\varepsilon_t^{\mathrm{DSWA}}\bigr\},
\qquad
L:=L_U+L_C+L_D,
\]
the first three terms can be grouped as
\begin{align}
L_U\varepsilon_t^U
+
L_C\varepsilon_t^{\mathrm{SWA}}
+
L_D\varepsilon_t^{\mathrm{DSWA}}
&\le
(L_U+L_C+L_D)\varepsilon_t
=
L\,\varepsilon_t.
\label{eq:group_non_global_terms}
\end{align}
Substituting Eq. (\ref{eq:group_non_global_terms}) into Eq. (\ref{eq:predictor_error_before_grouping_refined}), we obtain
\begin{equation}
\|\hat\mu_t-\mu_t^\star\|_{L^2}
\le
L\,\varepsilon_t + L_G e_t.
\label{eq:predictor_error_grouped_refined}
\end{equation}

Using the asymptotic estimate from part \textup{(i)}, we further obtain
\begin{equation}
\|\hat\mu_t-\mu_t^\star\|_{L^2}
\le
L\,\varepsilon_t + \frac{L_G\bar\xi}{1-\rho}
\qquad \text{as } t\to\infty.
\label{eq:predictor_error_asymptotic_refined}
\end{equation}

Finally, Lemma~\ref{lem:bayes_excess_identity} gives
\[
\mathcal R_t(\hat\mu_t)-\mathcal R_t^\star
=
\|\hat\mu_t-\mu_t^\star\|_{L^2}^2.
\]
Substituting Eq. (\ref{eq:predictor_error_asymptotic_refined}) into this identity yields
\begin{equation}
\mathcal R_t(\hat\mu_t)-\mathcal R_t^\star
\le
\left(
L\varepsilon+\frac{L_G\bar\xi}{1-\rho}
\right)^2
\qquad \text{as } t\to\infty,
\label{eq:final_risk_bound_refined}
\end{equation}
which proves Eq. (\ref{eq:main_risk_bound_infty}).
\end{proof}

\begin{corollary}[Exact sufficiency in the realizable case]
\label{cor:exact_sufficiency_app}
Suppose that the hypotheses of Theorem~\ref{thm:kairos_sufficiency} hold. Assume further that the learned hybrid state exactly recovers the Bayes decomposition at every time step, in the sense that $\varepsilon_t^U
=
\varepsilon_t^{\mathrm{SWA}}
=
\varepsilon_t^{\mathrm{DSWA}}
=
0$, $\varepsilon_t^\alpha
=
\varepsilon_t^\beta
=
\varepsilon_t^v
=
\varepsilon_t^k
=
0$
for all  $t$, and that the initial global-memory state is exactly aligned, $e_0 = 0$.
Then, for every $t$,
\begin{equation}
\hat G_t = G_t^\star
\qquad\text{and}\qquad
\hat\mu_t = \mu_t^\star
\qquad \text{a.s.}
\label{eq:exact_sufficiency_predictor}
\end{equation}
Consequently,
\begin{equation}
\mathcal R_t(\hat\mu_t)=\mathcal R_t^\star.
\label{eq:exact_sufficiency_risk}
\end{equation}
\end{corollary}

\begin{proof}[Proof of Corollary \ref{cor:exact_sufficiency_app}]
Under conditions in  Corollary \ref{cor:exact_sufficiency_app}, the quantity $\xi_t$ in Eq. (\ref{eq:xi_def}) satisfies $\xi_t = 0$ for all $t$. Since $e_0=0$, the memory recursion Eq. (\ref{eq:main_memory_bound}) gives $e_t=0$ for all $t$. Substituting these equalities into Eq. (\ref{eq:closed_form_et_bound}) and Eq. (\ref{eq:predictor_error_grouped_refined}) yields
\[
\norm{\hat\mu_t-\mu_t^\star}_{L^2} = 0,
\]
By Lemma~\ref{lem:bayes_excess_identity}, this is equivalent to
\[
\mathcal R_t(\hat\mu_t)-\mathcal R_t^\star = 0.
\]
which implies $\hat\mu_t=\mu_t^\star$ almost surely. This proves both Eq. (\ref{eq:exact_sufficiency_predictor}) and Eq. (\ref{eq:exact_sufficiency_risk}).
\end{proof}

\begin{remark}[Interpretation of the bound]
\label{rem:interpretation}
Theorem~\ref{thm:kairos_sufficiency} shows that the long-horizon prediction error is controlled by two quantities. The first is the aggregated non-global approximation term $L\varepsilon$, which summarizes the quality of the shared, short-range, and mid-range branches. The second is the global-memory term, whose asymptotic contribution is bounded by $L_G\bar\xi/(1-\rho)$, where $\bar\xi$ measures the worst-case one-step perturbation and $\rho<1$ ensures geometric damping of memory drift. In this sense, a hybrid multi-scale temporal memory is approximately sufficient: once the Bayes predictor factorizes across the four architectural roles, the remaining long-horizon degradation is fully accounted for by branchwise approximation quality together with a contractively controlled accumulation of global-memory perturbations.
\end{remark}

\begin{remark}[Why the global-memory branch is nontrivial]
\label{rem:gla_nontrivial}
Only the global branch can accumulate error across time. The contraction factor $\rho<1$ turns this accumulation into a geometrically discounted sum, preventing uncontrolled long-horizon drift. Without such a property, even small one-step errors in the global memory could grow super-linearly and destroy long-horizon consistency.
\end{remark}

\begin{remark}[Architectural and experimental relevance]
\label{rem:scope}
Theorem~\ref{thm:kairos_sufficiency} establishes architectural sufficiency rather than providing a performance guarantee for any specific model checkpoint. It serves to elucidate why the multi-scale temporal factorization is a principled design for long-horizon world modeling. In the concrete architecture, the short- and mid-range approximation terms reflect the representation quality of the SWA and DSWA pathways, whereas the discounted memory term characterizes the stability of the GLA pathway. This theoretical mechanism strongly complements the empirical long-horizon results presented in this work.
\end{remark}

\begin{remark}[Necessity--sufficiency template]
\label{rem:necessity_sufficiency_pair}
Taken together, Theorem~\ref{thm:necessity_persistent_state} and Theorem~\ref{thm:kairos_sufficiency} provide a compact necessity--sufficiency template for long-horizon world modeling: persistent memory is necessary whenever the Bayes predictor is not recent-window measurable, and a hybrid multi-scale temporal memory is approximately sufficient whenever the Bayes predictor admits the corresponding factorization and the global memory evolves contractively.
\end{remark}

\end{document}